\documentclass{article}
\usepackage{a4wide}
\usepackage{graphicx}
\usepackage{mathrsfs}
\usepackage{bbm}
\usepackage{bm}
\usepackage{xcolor}
\usepackage{amsfonts}
\usepackage{amsthm}
\usepackage{amsmath}
\usepackage{amssymb}
\usepackage[round]{natbib}
\usepackage{enumitem}
\usepackage{hyperref}
\usepackage{algorithm}
\usepackage[utf8]{inputenc}
\usepackage[english]{babel}
\usepackage[noend]{algpseudocode}
\usepackage{bigints}
\usepackage{stackengine}
\usepackage{crossreftools}
\usepackage{caption}
\usepackage{subcaption}
\usepackage{booktabs}
\newcommand{\ra}[1]{\renewcommand{\arraystretch}{#1}}

\newcommand{\optionaldesc}[2]{%
  \phantomsection
  #1\protected@edef\@currentlabel{#1}\label{#2}%
}

\makeatletter
\newcommand*{\itemequationsmall}[3][]{%
  \item
  \begingroup
    \refstepcounter{equation}%
    \ifx\\#1\\%
    \else  
      \label{#1}%
    \fi
    \sbox0{#2}%
    
    \sbox2{\small{$\displaystyle#3\m@th$}}%
    \sbox4{\@eqnnum}%
    \dimen@=.5\dimexpr\linewidth-\wd2\relax
    \ifcase
        \ifdim\wd0>\dimen@
          \z@
        \else
          \ifdim\wd4>\dimen@
            \z@
          \else 
            \@ne
          \fi 
        \fi
      \@latex@warning{Equation is too large}%
    \fi
    \noindent   
    \rlap{\copy0}%
    \rlap{\hbox to \linewidth{\hfill\copy2\hfill}}%
    \hbox to \linewidth{\hfill\copy4}%
    \hspace{0pt}
  \endgroup
  \ignorespaces 
}
\makeatother

\makeatletter
\newcommand*{\itemequationscript}[3][]{%
  \item
  \begingroup
    \refstepcounter{equation}%
    \ifx\\#1\\%
    \else  
      \label{#1}%
    \fi
    \sbox0{#2}%
    
    \sbox2{\scriptsize{$\displaystyle#3\m@th$}}%
    \sbox4{\@eqnnum}%
    \dimen@=.5\dimexpr\linewidth-\wd2\relax
    \ifcase
        \ifdim\wd0>\dimen@
          \z@
        \else
          \ifdim\wd4>\dimen@
            \z@
          \else 
            \@ne
          \fi 
        \fi
      \@latex@warning{Equation is too large}%
    \fi
    \noindent   
    \rlap{\copy0}%
    \rlap{\hbox to \linewidth{\hfill\copy2\hfill}}%
    \hbox to \linewidth{\hfill\copy4}%
    \hspace{0pt}
  \endgroup
  \ignorespaces 
}
\makeatother

\makeatletter
\newcommand*{\itemequationnormal}[3][]{%
  \item
  \begingroup
    \refstepcounter{equation}%
    \ifx\\#1\\%
    \else  
      \label{#1}%
    \fi
    \sbox0{#2}%
    
    \sbox2{\normalsize{$\displaystyle#3\m@th$}}%
    \sbox4{\@eqnnum}%
    \dimen@=.5\dimexpr\linewidth-\wd2\relax
    \ifcase
        \ifdim\wd0>\dimen@
          \z@
        \else
          \ifdim\wd4>\dimen@
            \z@
          \else 
            \@ne
          \fi 
        \fi
      \@latex@warning{Equation is too large}%
    \fi
    \noindent   
    \rlap{\copy0}%
    \rlap{\hbox to \linewidth{\hfill\copy2\hfill}}%
    \hbox to \linewidth{\hfill\copy4}%
    \hspace{0pt}
  \endgroup
  \ignorespaces 
}
\makeatother

\makeatletter
\newcommand*{\itemequationnormals}[3][]{%
  \item
  \begingroup
    \refstepcounter{equation}%
    \ifx\\#1\\%
    \else  
      \label{#1}%
    \fi
    \sbox0{#2}%
    
    \sbox2{\normalsize{$\displaystyle#3\m@th$}}%
    \dimen@=.5\dimexpr\linewidth-\wd2\relax
    \ifcase
        \ifdim\wd0>\dimen@
          \z@
        \else
          \ifdim\wd4>\dimen@
            \z@
          \else 
            \@ne
          \fi 
        \fi
      \@latex@warning{Equation is too large}%
    \fi
    \noindent   
    \rlap{\copy0}%
    \rlap{\hbox to \linewidth{\hfill\copy2\hfill}}%
    \hbox to \linewidth{\hfill\copy4}%
    \hspace{0pt}
  \endgroup
  \ignorespaces 
}
\makeatother

\newcommand\Item[1][]{%
  \ifx\relax#1\relax  \item \else \item[#1] \fi
  \abovedisplayskip=0pt\abovedisplayshortskip=0pt~\vspace*{-\baselineskip}}

\newtheorem{theorem}{Theorem}
\newtheorem{lemma}{Lemma}
\newtheorem{corollary}{Corollary}

\newtheorem{proposition}{Proposition}

\newtheorem{definition}{Definition}

\def\bdtheta{\bm{\theta}}
\def\bbdtheta{\bm{\theta}^*}

\def\bdlambda{\bm{\lambda}}
\def\bdxi{\bm{\xi}}
\def\bdx{\mathbf x}
\def\bdX{\mathbf X}
\def\bdy{\mathbf y}

\def\bdt{\mathbf t}
\def\bdu{\mathbf u}
\def\bde{\mathbf e}
\def\bdr{\mathbf r}
\def\bdZ{\mathbf Z}
\def\bdz{\mathbf z}
\def\bdv{\mathbf v}
\def \fff{\mathscr{F}}
\def\bddelta{\mathbf \Delta}
\def\bdrho{\bm{\rho}}
\def\bdnu{\bm{\nu}}
\def\EE{\mathbb{E}}
\def\fmin{F_{\min}}
\def\limk{\lim\limits_{k \rightarrow \infty}}
\def\limN{\lim\limits_{N \rightarrow \infty}}
\DeclareMathOperator*\uplim{\overline{lim}}

\def\uplimN{\uplim\limits_{N \rightarrow \infty}}

\def\daone{\mathbbm{1}}
\def\PP{\mathbb{P}}
\def\ZZ{\mathbb{Z}}
\def\RR{\mathbb{R}}
\def\NN{\mathbb{N}}
\def\hcal{\mathcal{H}}
\def\tiln{\tilde{n}}
\def\barn{\bar{n}}
\def\opt{\bdtheta^{opt}}
\def\lp{\left(}
\def\rp{\right)}
\def\lpp{\Big(}
\def\rpp{\Big)}
\def\lppp{\big(}
\def\rppp{\big)}
\def\lv{\left\|}
\def\rv{\right\|}
\def\lvv{\Big\|}
\def\rvv{\Big\|}
\def\lvvv{\big\|}
\def\rvvv{\big\|}
\def\lw{\left\{}
\def\rw{\right\}}
\def\lww{\Big\{}
\def\rww{\Big\}}
\def\lwww{\big\{}
\def\rwww{\big\}}
\def\lss{\Big[}
\def\rss{\Big]}
\def\lsss{\big[}
\def\rsss{\big]}

\def\tsigma{\tilde{\sigma}}
\def\barf{\bar{F}}
\def\csg{C_{sg}^{(2)}}
\def\cng{C_{ng}^{(2)}}
\def\cjin{c_{Jin}}
\def\cnn{c_{2,\tiln}}
\def\cnnn{c_{3,\tiln}}
\def\ones{\daone_{S_{opt}}}
\def \rgo{R_{good}}
\def \rgor{R_{good,r}}

\def \rgorl{R_{good,r}^L}
\def \rgb{R_{bad}}

\def\sb{S_{bad}}
\def \sopt{S_{opt}}
\def \ssopt{S_{opt}^*}
\def \snopt{S_{N,opt}}
\def \ssnopt{S_{N,opt}^*}
\newcommand\wtheta[1]{\widetilde{\bdtheta_{#1}}}
\newcommand\wwtheta[1]{\widetilde{\bbdtheta_{#1}}}
\def \tlambda{\tilde{\lambda}_{min}}
\def \lambdamin{\tilde{\lambda}_{min}}
\def \halflambda{\frac{\tilde{\lambda}_{min}}{2}}

\newcommand{\cdelta}[1]{c_{#1,\delta}}
\newcommand{\cdn}[1]{c_{#1,\delta,t}}
\newcommand{\nablatilde}[1]{\nabla \tilde{F}_n(\tilde{\bdtheta}_{#1})}
\newcommand{\nablaf}[1]{\nabla F(\bdtheta_{#1})}
\newcommand{\tildetheta}[1]{\tilde{\bdtheta}_{#1}}
\newcommand{\ssum}[2]{\sum\limits_{#1}\limits^{#2}}
\newcommand{\pprod}[2]{\prod\limits_{#1}\limits^{#2}}
\newcommand{\deltah}[2]{\Delta \hcal_{#1,#2}}
\newcommand{\deltad}[1]{\Delta D_{#1}}
\newcommand{\yis}[1]{\daone_{S_{n,#1}}}
\newcommand{\yisc}[1]{\daone_{S_{n,#1}^c}}
\newcommand{\twoline}[1]{\underline{\underline{#1}}}
\newcommand{\cnlemma}[1]{c_{n,#1}}

\begin{document}
\title{Online Bootstrap Inference with Nonconvex Stochastic Gradient Descent Estimator}
\date{\today}
\author{Yanjie Zhong\\
\and
Todd Kuffner\\
\and
Soumendra Lahiri}
\maketitle

\begin{abstract}
In this paper, we investigate the theoretical properties of stochastic gradient descent (SGD) for statistical inference in the context of nonconvex optimization problems, which have been relatively unexplored compared to convex settings. Our study is the first to establish provable inferential procedures using the SGD estimator for general nonconvex objective functions, which may contain multiple local minima.

We propose two novel online inferential procedures that combine SGD and the multiplier bootstrap technique. The first procedure employs a consistent covariance matrix estimator, and we establish its error convergence rate. The second procedure approximates the limit distribution using bootstrap SGD estimators, yielding asymptotically valid bootstrap confidence intervals. We validate the effectiveness of both approaches through numerical experiments.

Furthermore, our analysis yields an intermediate result: the in-expectation error convergence rate for the original SGD estimator in nonconvex settings, which is comparable to existing results for convex problems. We believe this novel finding holds independent interest and enriches the literature on optimization and statistical inference.
\end{abstract}

\section{Introduction}
\label{sec:intro}

Theoretical analysis of nonconvex optimization problems has long been a focus of the optimization, statistics, and machine learning community. Nonconvex optimization problems are prevalent in a wide variety of real-world applications, including
recommender system (\cite{koren2009matrix}), image processing (\cite{chan1998total}), portfolio selection (\cite{lobo2007portfolio}) and genetics (\cite{breheny2011coordinate}), making them an important topic of study. In general, with a potentially nonconvex objective function $F(\cdot)$, people want to discover
$$
\opt \in \mathop{\arg\min}\limits_{\bdtheta\in \Theta} F(\bdtheta),
$$
where $\Theta$ is the feasible parameter space. Given that nonconvex optimization is generally an NP-hard problem (\cite{murty1985some}), people are also interested in identifying the local minima. In our work, we focus on the stochastic gradient descent procedure (SGD), which is one of the most successful optimization methods. Its update rule is
\begin{equation}
\bdtheta_{n+1} = \bdtheta_n - \gamma_{n+1} \nabla f(\bdtheta_n;\bdy_{n+1}),
\label{mainupdate}
\end{equation}
where $\gamma_{n+1}$ is the stepsize, $\nabla f(\bdtheta_n;\bdy_{n+1})$ is a noisy estimator of the gradient of the objective function at $\bdtheta_n$. We consider the Robbins-Monro stepsize schedule, where $\gamma_n$ is proportional to $n^{-\alpha}$, for some positive $\alpha$. Though in our theoretical analysis, we simply assume that $\bdy_n$ is from a single observation, our results can be trivially generalized to the mini-batch case.

Introduced by \cite{robbins1951stochastic}, SGD is one of the most popular methods in solving the nonconvex optimization problem and it has been continually studied by researchers for decades (\cite{kiefer1952stochastic}; \cite{ljung1977analysis}; \cite{kushner2003stochastic}; \cite{nemirovski2009robust}; \cite{benveniste2012adaptive}). Practitioners also favor SGD for several reasons. Firstly, it is an extremely simple algorithm and can be easily implemented. Secondly, it has low computational complexity since only a single observation or a small portion of observations need to be processed in each iteration. Thirdly, it perfectly satisfies the need for online estimation and learning, as SGD can avoid revisiting data. In fact, many data are naturally generated in a streaming way, like the social media data (\cite{bifet2010sentiment}; \cite{lin2012large}), high frequency trading data (\cite{cartea2015algorithmic}), network traffic data (\cite{ma2009identifying}) and so on. The challenge of storing massive amounts of data and the need for timely responses and decisions have increased the popularity of SGD and its variants (\cite{bottou2004large}).

In the convex setting, where the objective function $F(\cdot)$ is convex, extensive research has been conducted on the properties of Stochastic Gradient Descent (SGD) estimators, particularly the convergence of SGD trajectories and error bounds, see, for example, \cite{borkar2009stochastic}, \cite{moulines2011non},\cite{rakhlin2012making} \cite{bach2013non}, \cite{godichon2019lp} and \cite{harvey2019tight}. However, the quantification of uncertainty for SGD estimators and the process of making statistical inferences with SGD remain less understood. Recently, the significance of this research direction has been recognized, leading to the development of inferential results under the convex assumption. \cite{chen2020statistical} estimated the asymptotic covariance matrix of the average SGD estimators by using non-overlapping batch-means. Borrowing ideas from \cite{lahiri2013resampling}, \cite{zhu2020fully} proposed to use overlapping batches, resulting in a fully online covariance matrix estimator. \cite{fang2018online} showed that random-weighted multiplier bootstrap can provide valid asymptotic inference. Furthermore, \cite{su2018uncertainty} modified the original one-way SGD into a novel hierarchical algorithm to construct t-based confidence interval. 

On the contrary, when dealing with potentially nonconvex objective functions, the understanding of statistical inference with SGD estimators remains limited. Though \cite{pelletier1998weak} and \cite{fort2015central} showed the weak convergence to Gaussian distribution, there's no practical inferential procedure to the best of our knowledge. Nonconvex optimization is typically viewed as more challenging than convex optimization due to the potential presence of multiple disjoint local minima and the difficulties posed by saddle points. The nature of nonconvex functions renders it impossible to naively adapt the previously mentioned inferential results to the nonconvex setting. In this work, we strive to bridge this gap by employing fundamentally different proofing techniques. Incorporating the use of multiplier bootstrap, we introduce two inferential procedures with theoretical guarantees under nonrestrictive conditions. In particular, we do not impose convex-like assumptions (e.g. the Polyak-Łojasiewicz condition, see \cite{karimi2016linear} and \cite{khaled2020better}) and we allow the objective function $F(\cdot)$ to have multiple disconnected local minima. Our contributions are 3-fold and can be summarized as follows:

\begin{itemize}
    \item By the combination of SGD and multiplier bootstrap, we construct a consistent covariance matrix estimator. Under some global conditions on the smoothness of $F$ and the strength of stochasticity on $\nabla f - \nabla F$, we show that the expected operator norm of the error decay in an $O\lppp N^{\alpha-1} + N^{1-2\alpha} + N^{-\frac{\alpha}{2}+\epsilon}\rppp$ rate for any $\epsilon>0$ (Theorem \ref{thm:covmat}), where $\alpha$ is the stepsize parameter. Building on this result, we introduce an online procedure to construct confidence interval for the explored local minimum. It is worth noting that, although SGD is not generally guaranteed to converge to a specific local minimum, our confidence interval is valid uniformly for any local minimum to which the SGD trajectories converge. We will elaborate on this point in the subsequent section.

    \item Under weaker local conditions on the objective function $F(\cdot)$ and the gradient noise $\nabla f - \nabla F$, we demonstrate the conditional weak convergence of the bootstrap SGD estimator (Theorem \ref{bootthm}). Armed with this finding, we can construct consistent bootstrap confidence intervals for the identified local minimum. This result extends the finding presented in \cite{fang2018online} to the nonconvex regime. The analysis requires only local conditions, as it does not involve the in-expectation behavior of the SGD trajectories.

    \item As an intermediate outcome of our first-part analysis, we provably show an in-expectation bound on the norm of the error of the SGD estimator (Theorem \ref{thm:secbound}). Specifically, for any local minimum $\opt$, we show that
    $$
    \EE \lppp \|\bdtheta_N - \opt\|^2\daone\lwww \limk \bdtheta_k=\opt\rwww \rppp= O(N^{-\alpha}).
    $$
    This finding represents a novel result in the literature of SGD for nonconvex targets. If we can additionally ensure the uniqueness of local minimum and convergence, we can recover the bound $\EE \| \bdtheta_N -\opt\|^2 = O(N^{-\alpha})$, which aligns with the known results in the convex setting (\cite{moulines2011non}; \cite{godichon2019lp}; \cite{jentzen2021strong}). This bound servers as a solid foundation for proofing the bound related to covariance matrix. Just like its counterpart in the convex regime, we believe that it has a promising potential for widespread application.
\end{itemize}
Statistical inference with SGD in the nonconvex regime has only been studied by \cite{yu2020analysis}. They considered the SGD procedure with a fixed-constant stepsize schedule. However, they only established nontrivial confidence intervals with width shrinking in correlation with the stepsize under strong growth conditions (\cite{lunde2021bootstrapping}). In contrast, our assumptions can accommodate objective functions with complex landscapes and multiple local minima, making our results applicable to a much broader range of applications. We demonstrate the effectiveness of our proposed inferential procedures in a later section with examples that include Gaussian mixture model, logistic regression with nonconvex regularization and principle component analysis.  

We know that the major difficulty of solving nonconvex optimization problems is the non-uniqueness of local minima, indicating that the gradient direction may not point toward the global minimum. Without imposing strong structural assumptions, the best outcome we can hope for from any efficient algorithm is a local minimum, and SGD is no exception. Compared to the convex case, it might be less intuitive to comprehend why conducting statistical inference with a local minimum finder is sensible. But actually, in many scenarios, it is not difficult to appreciate the significance of making inferences with SGD estimators. Firstly, under certain contraction-type conditions (\cite{yu2020analysis}), we can guarantee the uniqueness of global minimum. Consequently, SGD can help us to identify the true parameter of the underlying model. Then, we can make inferences on the global minimum, similar to what we do in the convex setting. Secondly, a wide range of applications exhibit perfect global symmetry, where every local minimum is a global minimum (\cite{jin2017escape}). Such applications include deep neural network (\cite{kawaguchi2016deep}), tensor decomposition (\cite{ge2015escaping}; \cite{ge2017optimization}) and a series of low-rank matrix problems (\cite{bhojanapalli2016global}; \cite{ge2016matrix}; \cite{ge2017no}; \cite{park2017non}). In these cases, SGD can also help us to approximate a global minimum and facilitate meaningful inference. Thirdly, even if we only have partial symmetry, we can still make inferences with SGD, regardless of to which the SGD trajectories converge. For illustration, we present a toy example under a Gaussian-2 mixture model. We consider an unequal mixture of $N((\theta_1,\theta_2)^T,I_2)$ and $N((-\theta_1,\theta_2)^T,I_2)$ with $\theta_1=3$ and $\theta_2=0$. The landscape of log likelihood function (the negative objective function) is displayed in the Figure \ref{fig:preliminary}, from which we can observe an optimal point $(3,0)$ and a sub-optimal point $(-3,0)$. No matter which local minimum the SGD trajectories approach to, we can always statistically test whether $\theta_2=0$.    

\begin{figure}[h]
    \centering
    \begin{subfigure}[t]{0.24\textwidth}
        \centering
        \includegraphics[width=1\linewidth,height=0.8\linewidth]{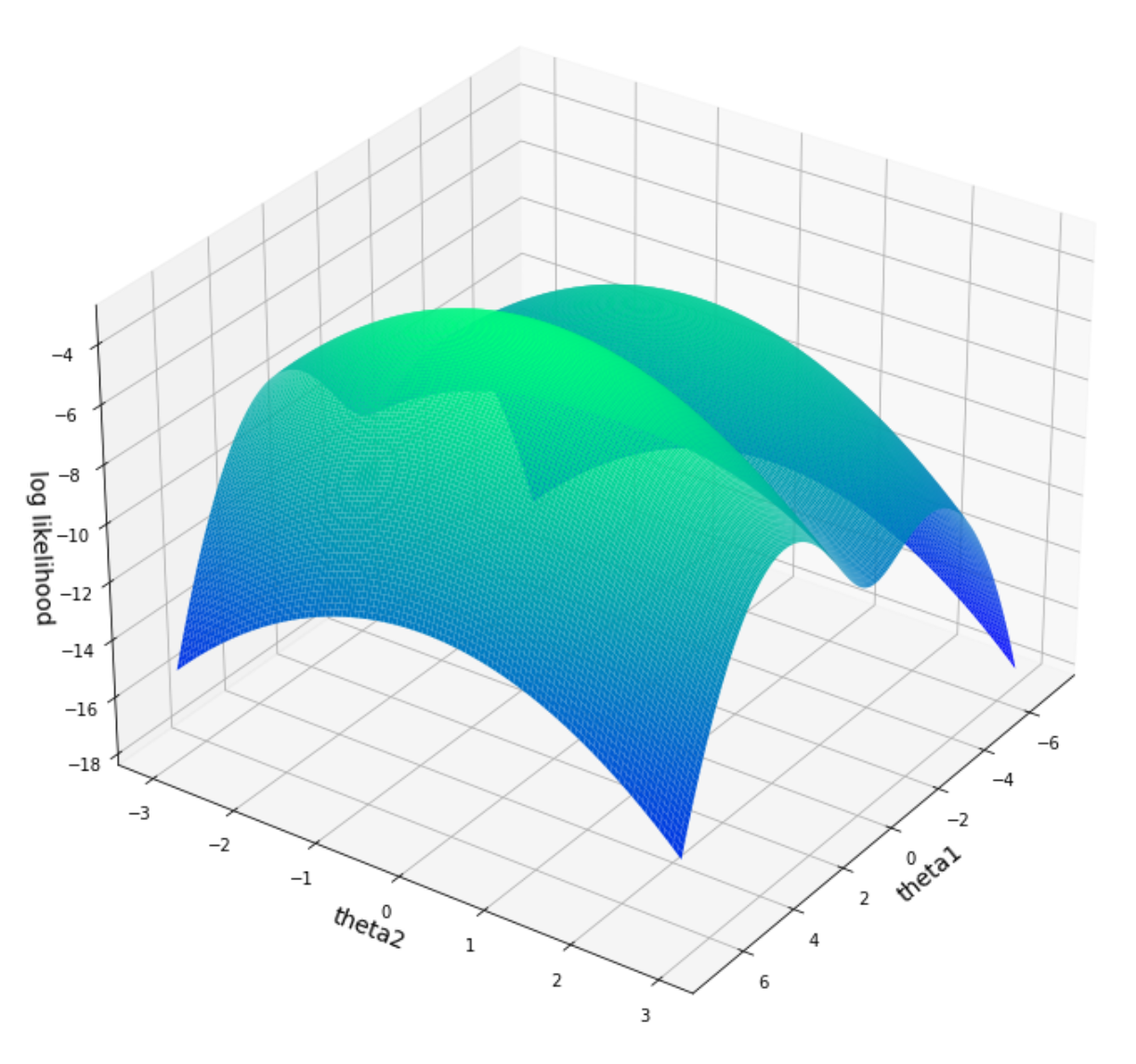} 

    \end{subfigure}
    \hfill
    \begin{subfigure}[t]{0.24\textwidth}
        \centering
        \includegraphics[width=1\linewidth,height=0.8\linewidth]{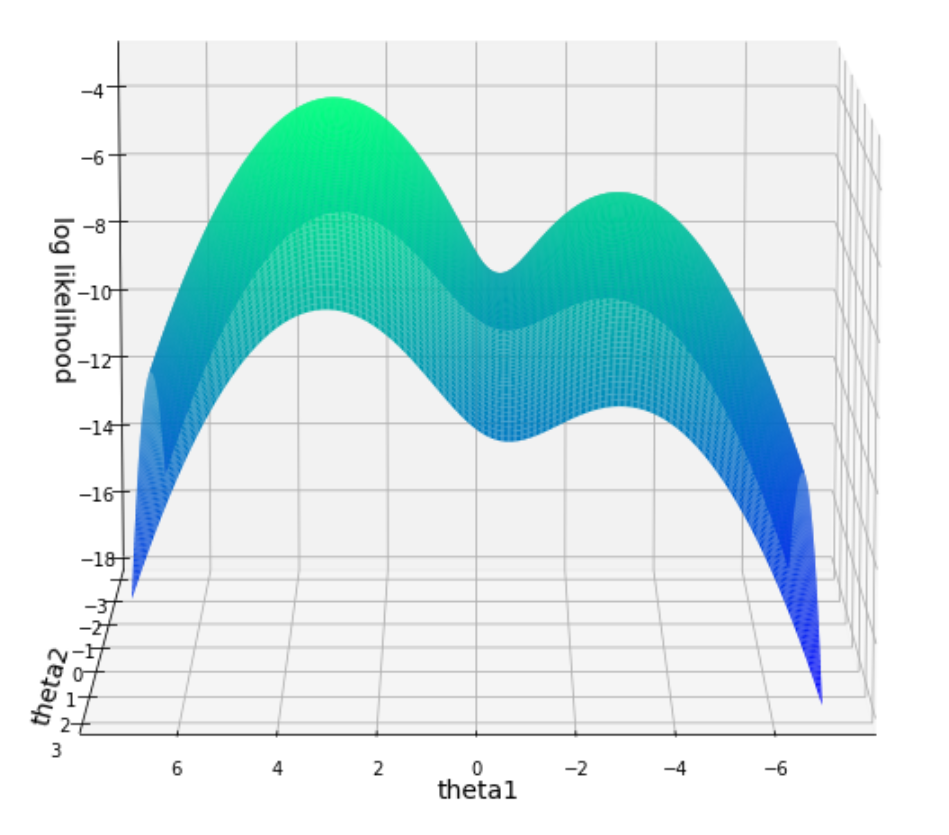} 
        
    \end{subfigure}
    \hfill
    \begin{subfigure}[t]{0.24\textwidth}
        \centering
        \includegraphics[width=1\linewidth,height=0.8\linewidth]{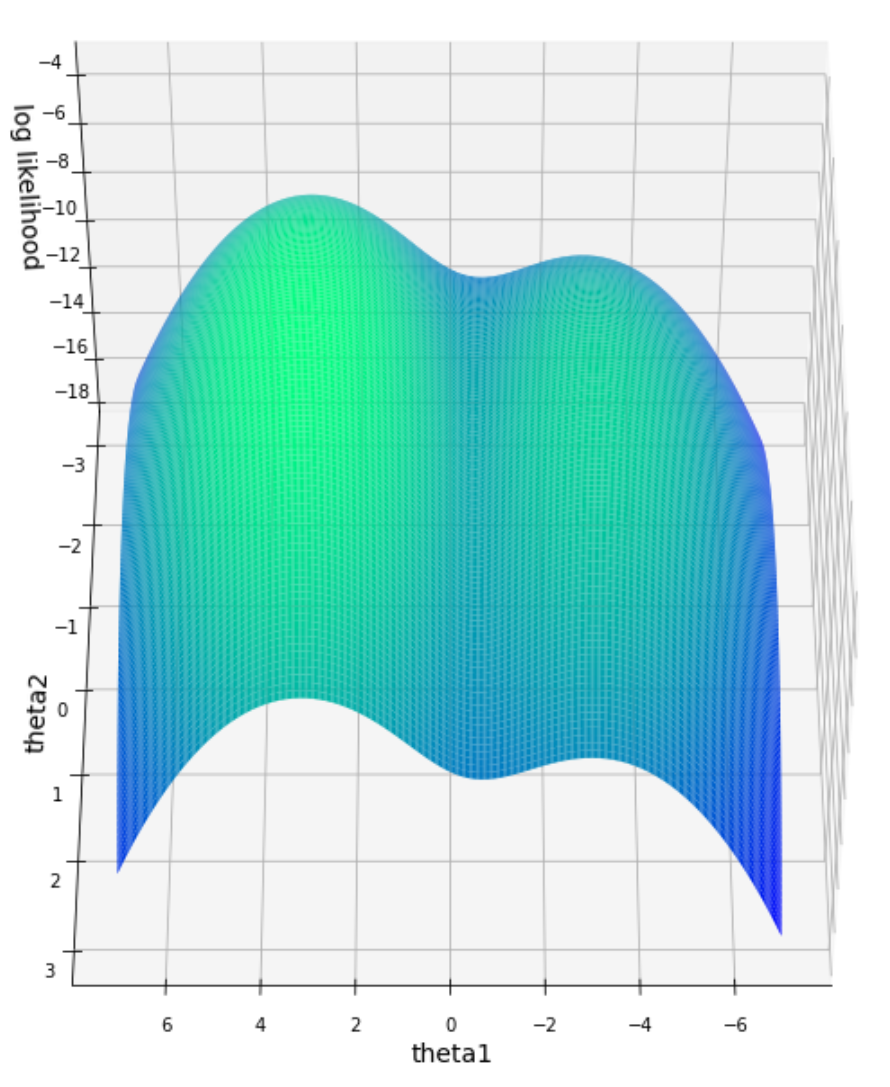} 
        
    \end{subfigure}
    \hfill
    \begin{subfigure}[t]{0.24\textwidth}
        \centering
        \includegraphics[width=1\linewidth,height=0.8\linewidth]{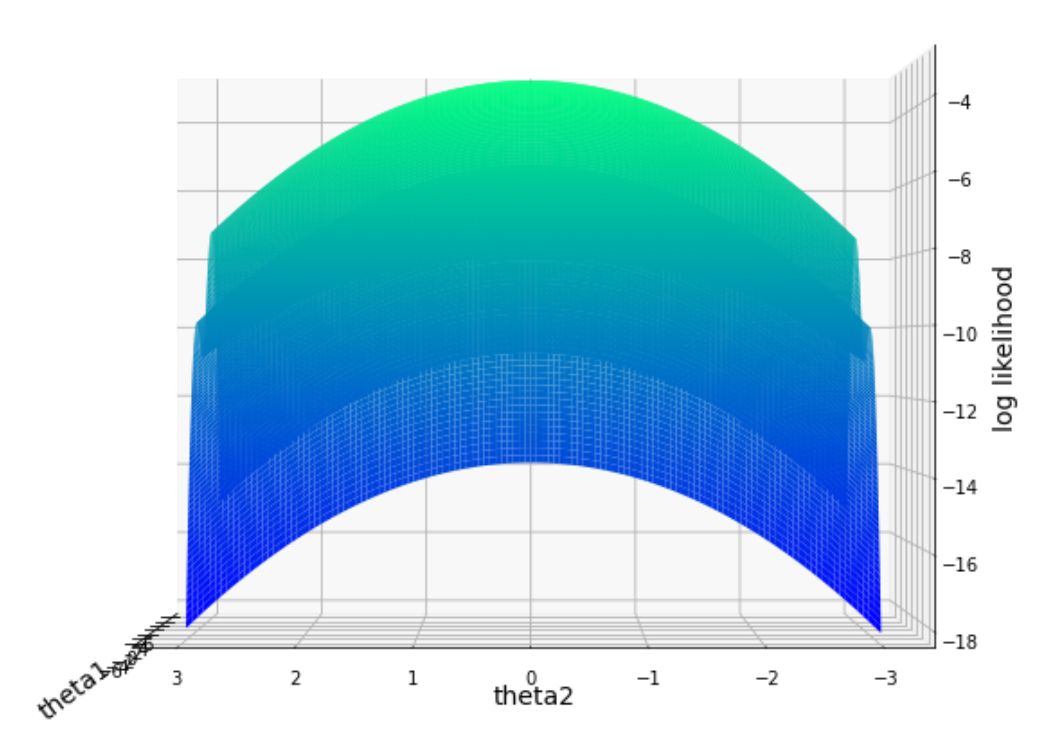} 
        
    \end{subfigure}
\caption{Expected negative loss function under a non-symmetric Gaussian 2-mixture model, shown in 4 different angles.}
\label{fig:preliminary}
\end{figure}


\subsection{Notations and Global Assumptions}
\label{subsec:nota}
In our work, $\EE_n$ and $\PP_n$ respectively denote the expectation and probability conditional on information up to step $n$. $\EE_*$ and $\PP_*$ respectively denote the expectation and probability conditional on $\bdtheta_0$ and $\{\bdy_n\}_{n\in \ZZ_+}$.

We introduce the big O and small o notations used in our work. For 2 sequences of positive numbers $\{a_n\}_{n\in NN}$ and $\{b_n\}_{n\in NN}$, we say $a_n = O(b_n)$ if ${a_n}/{b_n}$ is uniformly bounded. For a sequence of random variables (vectors) $\{\bdx_n\}_{n\in \NN}$ and some univariate positive-valued function $g$. We say $\bdx_n = O_p(g(n))$ if for any $\varepsilon > 0$, there exists a $M_{\varepsilon} > 0$ and $n_{\varepsilon}$ such that $\PP\lppp g^{-1}(n) \|\bdx_n\| > M_{\varepsilon}\rppp \leq \varepsilon$ when $n \ge n_{\varepsilon}$. We say $\bdx_n = o_p(g(n))$ if $\{g^{-1}(n)\bdx_n\}_{n\in \NN}$ converges to $\bm{0}$ in probability. We say $\bdx_n = O_L(g(n))$ if $\{\EE \|g^{-1}(n)\bdx_n\| \}_{n\in \NN}$ is bounded. For some given set $S$, we say $\bdx_n = O_L(g(n))$ on $S$ if $\{\EE \|g^{-1}(n)\bdx_n \daone_S\| \}_{n\in \NN}$ is bounded. 

The following setups and assumptions are applied to the whole paper:
\begin{itemize}
    \item We consider a parameter space $\Theta\subset \RR^d$ and we assume that the whole SGD trajectories $\{\bdtheta_n\}_{n\ge 0}$ is contained in $\Theta$. We denote the set of local minimum by $\Theta^{opt}$. Any $\opt \in \Theta^{opt}$ is in the interior of $\Theta$. In addition, the objective function is lower-bounded, i.e. $\fmin \triangleq \inf\{F(\bdtheta):\bdtheta\in \Theta\} > -\infty$.
    \item We consider the Robbins-Monro stepsize schedule with $\gamma_n = C{n^{-\alpha}}$ for some $\frac{1}{2}<\alpha<1$ and $C>0$. We assume that $\{\bdy_n\}_{n\in \NN}$ are i.i.d. from some distribution $\mathcal{P}_Y$ and $\nabla f(\bdtheta;\bdy)$ is an unbiased estimator of $\nabla F(\bdtheta)$ for any $\bdtheta \in \Theta$.
    \item $\{w_{n}\}_{n\ge 1} \subseteq \RR$ used in our multiplier bootstrap procedure is a sequence of i.i.d. random coefficients drawn from distribution $\mathcal{P}_{W}$ with mean 1, variance 1 and finite $4$th moment.
\end{itemize}

\subsection{Outline of the Paper}
The rest of the paper will be organized as follows. In section \ref{sec:prelim}, we introduce known weak convergence results related to SGD. In section \ref{sec:covmax}, we present an online procedure for estimating the covariance matrix and constructing confidence intervals. We establish error bounds for both the SGD and the covariance matrix estimator, as well as demonstrate the asymptotic exactness of confidence intervals. In section \ref{sec:quantile}, we propose another online inferential procedure producing bootstrap confidence intervals, the validity of which are also shown. In section \ref{sec:example}, we use Gaussian mixture model and nonconvex regularized logistic regression as examples to showcase the reasonableness of our assumptions and applicability of our inferential methods. In section \ref{sec:empirical}, we conduct numerical experiments to empirically demonstrate the effectiveness of our proposed methods. Finally, in the last section, we conclude our work by discussing the limitations and potentials of our work.

\section{Preliminaries and Background}
\label{sec:prelim}

To enhance the convergence rate of the original gradient descent procedure (\ref{mainupdate}), \cite{ruppert1988efficient} and \cite{polyak1992acceleration} proposed to take an average step
\begin{equation}
\bar{\bdtheta}_N = \frac{1}{N+1}\sum\limits_{n=0}\limits^N \bdtheta_n,
\label{averagestep}
\end{equation}
and output $\bar{\bdtheta}_N$ as the final estimator. Note that $\bar{\bdtheta}_N = \frac{N}{N+1} \bar{\bdtheta}_{N-1} + \frac{1}{N+1}\bdtheta_{N}$, which implies that the Polyak-Ruppert average procedure can be implemented in an online fashion. \cite{pelletier1998weak} and \cite{fort2015central} proved the weak convergence for the average stochastic approximation estimator, which is more general than the SGD estimator. However, for simplicity, we have restricted our focus to the SGD estimator in this study. The findings in our work can be readily extended to the broader stochastic approximation context by using suitable conditions. For clarity, we present assumptions and weak convergence results tailored for the Polyak-Ruppert average SGD estimator.

It is intuitive that, in order to guarantee the SGD procedure to work well, conditions controlling the variability of $\{\nabla f(\bdtheta_n;\bdy_{n+1})\}_{n\ge 0}$ are needed. For example, for any $\opt \in \Theta^{opt}$, we can consider the following local Lyapunov type condition.
\begin{enumerate}[leftmargin=1.2cm]
\item [{\crtcrossreflabel{(LC1)}[LC1]}] There exists $\tau>0$ such that $\EE_{\bdy\sim \mathcal{P}_Y} \|\nabla F(\opt)-\nabla f(\opt;\bdy)\|^{2+\tau}<\infty$. There is a $\delta>0$ and an increasing and continuous function $\eta_1:[0,\infty)\rightarrow [0,\infty)$ with $\eta_1(0)=0$ such that when $\|\bdtheta-\opt\|\leq \delta$,
$$
\EE_{\bdy\sim \mathcal{P}_Y}\|\nabla f(\opt;\bdy)-\nabla f(\bdtheta;\bdy)\|^{2+\tau} \leq \eta_1(\|\opt-\bdtheta\|).
$$
\end{enumerate}
The following weaker Lindeberg type condition can also be used in place of condition \ref{LC1}.
\begin{enumerate}[leftmargin=1.2cm]
\item [{\crtcrossreflabel{(LC2)}[LC2]}] Suppose that $\EE_{\bdy\sim \mathcal{P}_Y} \|\nabla F(\opt)-\nabla f(\opt;\bdy)\|^{2}<\infty$. In addition, there is a $\delta>0$ and an increasing and continuous function $\eta_2:[0,\infty)\rightarrow [0,\infty)$ with $\eta_2(0)=0$ such that when $\|\bdtheta-\opt\|\leq \delta$,
$$
\EE_{\bdy\sim \mathcal{P}_Y}\|\nabla f(\opt;\bdy)-\nabla f(\bdtheta;\bdy)\|^{2} \leq \eta_2(\|\opt-\bdtheta\|).
$$
There also exists a decreasing function $\eta_{\delta}:[0,\infty)\rightarrow [0,\infty)$ such that $\lim\limits_{M\rightarrow \infty} \eta_{\delta}(M) = 0$ and
$$
\EE_{\bdy\sim \mathcal{P}_Y}\lsss\|\nabla f(\opt;\bdy)-\nabla f(\bdtheta;\bdy)\|^{2}\daone\{\|\nabla f(\opt;\bdy)-\nabla f(\bdtheta;\bdy)\| \ge M\}\rsss \leq \eta_{\delta}(M),
$$
when $\|\bdtheta-\opt\|\leq \delta$.

\end{enumerate}

In addition, we also need the following weak conditions on local behaviors of $F$ around $\opt$ and richness of noises.

\begin{enumerate}[leftmargin=1.2cm]
\item [{\crtcrossreflabel{(LC3)}[LC3]}] $F(\bdtheta)$ is thrice continuously differentiable in the neighborhood of  $\opt$ and $A \triangleq \nabla^2 F(\opt)$ is positive definite.
\end{enumerate}

\begin{enumerate}[leftmargin=1.2cm]
\item [{\crtcrossreflabel{(LC4)}[LC4]}] $U\triangleq \EE_{\bdy\sim \mathcal{P}_Y}\lsss \nabla f(\opt;\bdy) \nabla f(\opt;\bdy)^T\rsss$ is positive definite.
\end{enumerate}

Equipped with the above conditions, we have the following weak convergence results.

\begin{lemma}
Under conditions \ref{LC1} (or \ref{LC2}), \ref{LC3} and \ref{LC4}, we have 
$$
\arraycolsep=1.1pt\def\arraystretch{1.5}
\begin{array}{rl}
& \sqrt{N} \lppp \bar{\bdtheta}_N-\opt \rppp \daone \lww \limk \bdtheta_k = \opt \rww\\

=&   \frac{A^{-1}}{\sqrt{N+1}} \sum\limits_{n=0}\limits^N (\nabla F(\bdtheta_n) - \nabla f(\bdtheta_n;\bdy_{n+1}))\daone \lww \limk \bdtheta_k = \opt \rww + o_p(1),
\end{array}
$$
and for any $\bdt\in \RR^d$,
$$
\resizebox{.98\hsize}{!}{$
\limN \EE \lss \daone \lww \limk \bdtheta_k=\opt \rww \exp\lppp i\sqrt{N}\bdt^T(\bar{\bdtheta}_N-\opt )\rppp \rss = \PP\lpp \limk \bdtheta_k=\opt\rpp \exp\lppp -\frac{1}{2}\bdt^T A^{-1} U A^{-1}\bdt\rppp .
$}
$$
\label{basiclemma1}
\end{lemma}

The rough message given by Lemma \ref{basiclemma1} is that, on $\lwww\limk \bdtheta_k=\opt\rwww$, we have 
$$
\sqrt{N} \lppp \bar{\bdtheta}_N - \opt\rppp \overset{w}{\longrightarrow} N\lppp 0, A^{-1}UA^{-1}\rppp.
$$

However, as pointed out by \cite{chen2020statistical}, although we have the limiting distribution result, it can not be immediately converted into a practical confidence interval for $\opt$. To conduct inference with the SGD estimator, we need to either directly estimate the covariance matrix $A^{-1}UA^{-1}$, or approximate the distribution function of $N\lppp 0, A^{-1}UA^{-1}\rppp$ without a direct estimation of the covariance matrix. In the convex optimization domain, \cite{chen2020statistical}, \cite{zhu2020fully} and \cite{su2018uncertainty} pursued the former method whereas \cite{fang2018online} followed the latter. Expanding upon the work of \cite{fang2018online}, we consider both direct and indirect inferential strategies using the multiplier bootstrap method, presenting a general framework in Algorithm \ref{algo1}. We primarily focus on the bootstrap-based framework owing to its versatility. On the one hand, we can estimate the covariance matrix by computing the sample covariance of bootstrap SGD estimators. On the other hand, the empirical distribution based on bootstrap SGD estimators provides a reasonable approximation to our desired Gaussian distribution. It is worth noting that our proofing techniques can be easily adapted to other inferential methods, allowing us to potentially extend the methods introduced in \cite{chen2020statistical}, \cite{zhu2020fully} and \cite{su2018uncertainty} to the nonconvex regime.

\begin{algorithm}
\caption{General Inference Framework Based on SGD and Multiplier Bootstrap}
\hspace*{\algorithmicindent} \textbf{Input:} Initializer $\bdtheta_0\in\RR^d$, stepsize $\{\gamma_n\}_{n\ge1}$, number of iterations $N$, number of bootstrap estimators $B$, distribution to sample random weights $\mathcal{P}_W$
\begin{algorithmic}[1]
\State $\bar{\bdtheta}_0 = \bdtheta_{0}$
\For{$b=1,2,\ldots,B$}
\State $\bdtheta_{0}^{(b)} = \bdtheta_{0}, \bar{\bdtheta}_{0}^{(b)}= \bdtheta_{0}$
\EndFor
\For{$n=0,1,\ldots,N-1$}
\State $\bdtheta_{n+1} = \bdtheta_n - \gamma_{n+1} \nabla f(\bdtheta_n;\bdy_{n+1})$
\State $\bar{\bdtheta}_{n+1} = \frac{n}{n+1}\bar{\bdtheta}_{n} + \frac{1}{n+1} \bdtheta_{n+1}$
\For{$b=1,2,\ldots,B$}
\State Randomly sample $w^{(b)}_{n+1}\sim \mathcal{P}_W$
\State $\bdtheta_{n+1}^{(b)} = \bdtheta_n^{(b)} - \gamma_{n+1} w^{(b)}_{n+1}\nabla f(\bdtheta_n^{(b)};\bdy_{n+1})$
\State $\bar{\bdtheta}_{n+1}^{(b)} = \frac{n}{n+1}\bar{\bdtheta}_{n}^{(b)} + \frac{1}{n+1} \bdtheta_{n+1}^{(b)}$
\EndFor
\EndFor

\State Construct confidence interval based on $\bar{\bdtheta}_N$ and $\{\bar{\bdtheta}_N^{(b)}\}_{1\leq b\leq B}$.
\end{algorithmic}
\label{algo1}
\end{algorithm}

Random-weighting multiplier bootstrap is not new to us. It has been employed to perform statistical inference in numerous conventional problems (\cite{zhong1987random}; \cite{rao1992approximation}; \cite{jin2003rank}). However, it has not been utilized for online inference until recently (\cite{fang2018online}; \cite{xu2020scalable}). Given its easy implementation, it is of great significance to explore this method's theoretical properties in the general nonconvex setting.

\section{Bootstrap-Based Covariance Matrix Estimator}
\label{sec:covmax}

This section focuses on the inferential results using a consistent covariance matrix estimator. In the first subsection, we provide conditions regarding the objective function and stochastic gradient noise. In the second subsection, we derive bounds on the parameter estimation error. As previously stated, these bounds can be applied beyond the context of the current work. In the third subsection, we provide a bound on the covariance matrix estimator and propose an online procedure to construct confidence intervals.

\subsection{Conditions and Setups}

\begin{enumerate}[leftmargin=1.2cm]
\item [{\crtcrossreflabel{(CM1)}[cm_1]}] There exists positive constants $C_s$ and $C_{hl}$ such that for any $\bdtheta_1,\bdtheta_2 \in \Theta$,
$$
\|\nabla F(\bdtheta_1) - \nabla F(\bdtheta_2) \| \leq C_s\|\bdtheta_1 - \bdtheta_2\|,
$$
$$
\|\nabla^2 F(\bdtheta_1) - \nabla^2 F(\bdtheta_2) \| \leq C_{hl}\|\bdtheta_1 - \bdtheta_2\|.
$$
There exists $\tlambda>0$ such that for any $\opt \in \Theta^{opt}$, $\lambda_{\min} (\nabla^2 F(\opt)) \ge \tlambda$.
\end{enumerate}

\noindent{\bf Remark:} \textit{Conditions on the smoothness of gradient and Hessian matrix are common in optimization literature. Similar conditions can be found in related works (\cite{ge2015escaping}; \cite{jin2021nonconvex}; \cite{vlaski2021second}). The local strong convexity condition can make sure that every local minimum is a strong attractor. Based on the latter part condition \ref{cm_1}, we can immediately know that there exists a positive constant $r_{good}^L$ such that for any $\opt \in \Theta^{opt}$, 
$$
\lambda_{\min} (\nabla^2 F(\bdtheta)) \ge \frac{\tlambda}{2}, \forall \| \bdtheta - \opt\| \leq r_{good}^L.
$$
In addition, we can require that $\| \bdtheta - \bdtheta'\| > r_{good}^L$, for any $\bdtheta,\bdtheta' \in \Theta^{opt}$.}

\begin{definition}
A random vector $\bm{X}\in \RR^d$ is $\sigma^2$-norm-subGaussian for $\sigma > 0$ if
$$
\PP \lppp \| \bm{X} - \EE \bm{X} \| 
 \ge t\rppp \leq 2 e^{-\frac{t^2}{2\sigma^2}},\ \forall t\ge 0.
$$
\end{definition}

\begin{enumerate}[leftmargin=1.2cm]
\item [{\crtcrossreflabel{(CM2)}[cm_2]}] There exists a positive constant $\csg$ such that for any $\bdtheta \in \Theta$, $\nabla f(\bdtheta; \bdy) - \nabla F(\bdtheta)$ is $\csg(1+\| \nabla F(\bdtheta)\|^{\beta_{sg}})$-norm-subGaussian for some $\beta_{sg}\in [0,2)$. There exists a positive constant $\sigma_{\min}$ such that for any $\bdtheta\in\Theta $, $\bdy \in \mathcal{Y}$ and $\bdnu \in \RR^d$ with $\| \bdnu\|=1$,
$$
\EE \langle \nabla f(\bdtheta;\bdy) - \nabla F(\bdtheta), \bdnu \rangle^2 \ge \sigma_{\min}^2.
$$
\end{enumerate}

\noindent{\bf Remark:} \textit{The first half of condition \ref{cm_2} is concerned with the tail of the stochastic gradient noise. A same type of condition is assumed in \cite{fang2019sharp} and \cite{jin2021nonconvex} to prove a high-probability saddle point avoidance result. However, our condition is more general because the norm-subGaussian parameter is allowed to depend on $\bdtheta$. It is worth noting that the subGaussian tail condition can be weakened if the weak growth condition \ref{cm_3} is strengthened. We will elaborate on this shortly. The second half of condition \ref{cm_2} makes sure that the stochastic gradient always possesses sufficiently strong randomness so that the SGD path will not become trapped at saddle points. Similar conditions can be found in all saddle point avoidance analyses (or additional noise can be added), see \cite{brandiere1996algorithmes}, \cite{ge2015escaping}, \cite{fang2019sharp}, \cite{mertikopoulos2020almost}, \cite{jin2021nonconvex}.}

\begin{enumerate}[leftmargin=1.2cm]
\item [{\crtcrossreflabel{(CM3)}[cm_3]}] $\Theta^{opt}$ is a countable set. There exists a positive constant $C_{tf}$ and a positive integer $\beta_{tf}$ such that for any $\bdtheta \in \Theta$, $\opt \in \Theta^{opt}$,
$$
\| \bdtheta - \opt\|^2 \leq C_{tf} \lppp1+\lppp F(\bdtheta) - \fmin\rppp^{\beta_{tf}}\rppp.
$$
\end{enumerate}

\noindent{\bf Remark:} \textit{This condition enables us to use the objective function value to bound the error. Intuitively, under such a condition, the landscape of the objective function resembles a basin, making it difficult for the path to deviate significantly. As we do not put an upper-bound on the value of $\beta_{tf}$, condition \ref{cm_3} places minimal restrictions on the growth rate of the objective function when $\|\bdtheta\|$ approaches infinity. To accommodate such a weak condition on the landscape of $F$, we have to assume a strong subGaussian tail condition \ref{cm_2} as compensation. In fact, if the objective function exhibits steep growth, such as a quadratic one, only certain low-order moment conditions are required for the stochastic gradient noise. To ensure clarity, we have opted to present the current condition sets. Any potential variants will be left for future research.}

\begin{enumerate}[leftmargin=1.2cm]
\item [{\crtcrossreflabel{(CM4)}[cm_4]}] We define $r_{good} \triangleq \frac{r_{good}^L}{9}$,
$$
\rgo(\opt) \triangleq \{ \bdtheta: \| \bdtheta - \opt \| \leq r_{good}\},\ \rgo^L(\opt) \triangleq \{ \bdtheta: \| \bdtheta - \opt \| \leq r_{good}^L\}.
$$
We let
$$
\rgo \triangleq \bigcup\limits_{\opt\in \Theta^{opt}} \rgo(\opt).
$$
There exist positive constant $b_0$ and $\tilde{\lambda}$ such that for any $\bdtheta \in \Theta$, if $\| \nabla F(\bdtheta)\| \leq b_0$ and $\lambda_{\min}(\nabla^2 F(\bdtheta)) > - \tilde{\lambda}$, then $\bdtheta \in \rgo$.
\end{enumerate}
\noindent{\bf Remark:} \textit{Under condition \ref{cm_4}, we can avoid the existence of non-strict saddle points so that we can guarantee that all saddle points are escapable. This condition can be satisfied if all saddle points is strict and the number of saddle points is finite, which is true in many applications \citep{ge2015escaping,sun2015nonconvex}.}

\subsection{Moment Bounds on the Parameter Estimation Error}

In this subsection, we present the error bounds for the SGD estimator. These serve as a crucial stepping stone for the subsequent analysis.

\begin{theorem}
\label{thm:secbound}
Suppose that conditions \ref{cm_1}-\ref{cm_4} hold. The step size parameter $\alpha$ satisfies that $\frac{1}{2} < \alpha < \frac{1}{2-\frac{\tilde{\lambda}}{10C_s+\tilde{\lambda}}}$. Another step size parameter $C$ satisfies that $CC_s(1+2\csg \beta_{sg}) \leq \frac{1}{2}$ and $C\leq 1$. Then for any $\opt \in \Theta^{opt}$, we have 
$$
\EE \lppp \|\bdtheta_N - \opt\|^2\daone\lwww \limk \bdtheta_k=\opt\rwww \rppp= O(N^{-\alpha}).
$$
\end{theorem}

In the convex optimization setting, $\EE\| \bdtheta_N - \opt\|^2 = O(N^{-\alpha})$ is a well-known result \citep{moulines2011non,chen2020statistical,godichon2019lp,jentzen2021strong}. As multiple local minima could exist in the nonconvex setting, it is generally inevitable to carry the indicator function $\daone\lwww \limk \bdtheta_k=\opt\rwww$. When the number of local minima is finite, the second-order error bound can hold uniformly. That is, regardless of which local minimum the SGD trajectories converge to, a uniform $O(N^{-\alpha})$ bound can be guaranteed.

Our proof strategy can be roughly summarized as follows. Firstly, we discover that with high probability, the SGD path can enter the good region $\rgo$ in a specific time window. To guarantee this, the SGD trajectories must not linger around saddle points for an extended period. Fortunately, we confirm this by extending the saddle point avoidance techniques employed in \cite{ge2015escaping}. Secondly, it is not hard to notice that, due to local strong convexity, once the SGD trajectories enter the good region, they will remain inside with high probability. Combining these 2 main findings, we can obtain the desired error bound. We also want to point out that the $\alpha \in \lpp \frac{1}{2}, \frac{1}{2-\frac{\tilde{\lambda}}{10C_s+\tilde{\lambda}}}\rpp$ restriction imposed in Theorem \ref{thm:secbound} could be manufactured. We believe that this condition can be weakened with more refined proof techniques. Nevertheless, our work contributes a consistent covariance matrix estimator, and future refinements can build upon our work and utilize our proof strategy.

The fourth-order error bound is also needed to build the proof for consistent covariance matrix estimator. Similar conclusions in the convex domain have been mentioned by \cite{chen2020statistical} and \cite{godichon2019lp}.

\begin{theorem}
\label{thm:fourbound}
Under conditions given in Theorem \ref{thm:secbound}, for any 
$\opt \in \Theta^{opt}$, we have 
$$
\EE \lppp \|\bdtheta_N - \opt\|^4\daone\lwww \limk \bdtheta_k=\opt\rwww \rppp= O(N^{-2\alpha}).
$$
\end{theorem}

In fact, higher-order results can be obtained in a similar fashion under the subGaussian tail assumption or other alternative moment conditions on the stochastic gradient noise.


\subsection{Bound on the Covariance Matrix Estimator}

In this subsection, we introduce the procedure of constructing a covariance matrix estimator and generating confidence intervals. Details are provided in Algorithm \ref{algo2}. Before diving into the algorithm, let us firstly present the consistency result for the bootstrap-based covariance matrix estimator. To demonstrate the consistency, higher-order differentiability condition is needed.

\begin{itemize}[leftmargin=1.2cm]
\item [{\crtcrossreflabel{(CM5)}[cm_5]}] The objective function $F(\cdot)$ is 4-times continuously differentiable. For $p=2,3,4$, $f(\bdtheta;\bdy)$ is $p$-times differentiable in $\bdtheta$ and there exist positive constants $C_{sh}^{(p)}$ such that
$$
\EE \lvvv \nabla^p f(\bdtheta;\bdy) \rvvv^2 \leq C_{sh}^{(p)},\ \forall \bdtheta \in \Theta.
$$
In addition, for any $\bdtheta \in \Theta$, $\EE \nabla^p f(\bdtheta;\bdy) = \nabla^p F(\bdtheta)$.
\end{itemize}

We have the following result:
\begin{theorem}
Suppose that conditions \ref{cm_1}-\ref{cm_5} hold. The step size parameter $\alpha$ satisfies that $\frac{1}{2} < \alpha < \frac{1}{2-\frac{\tilde{\lambda}}{10C_s+\tilde{\lambda}}}$. Another step size parameter $C$ satisfies that $CC_s(1+2\csg \beta_{sg}) \leq \frac{1}{2}$ and $C\leq 1$. For any $\opt \in \Theta^{opt}$, any $\epsilon>0$, we have
$$
\resizebox{.98\hsize}{!}{$
\arraycolsep=1.1pt\def\arraystretch{1.5}
\begin{array}{rl}
& \EE \lvv \EE_* \lpp \lpp\lww \sqrt{N}  (\bar{\bdtheta}_N^* - \bar{\bdtheta}_N) \rww\lww\sqrt{N}  (\bar{\bdtheta}_N^* - \bar{\bdtheta}_N) \rww^T  - A^{-1}UA^{-1}\rpp \daone \lwww \limk \bdtheta_k=\opt \rwww \daone\lwww \limk \bbdtheta_k=\opt\rwww \rpp \rvv\\
=& O\lppp N^{\alpha-1} + N^{1-2\alpha} + N^{-\frac{\alpha}{2}+\epsilon}\rppp,\\
\end{array}
$}
$$
where $\EE_*$ refers to the expectation conditional on $\bdtheta_0$ and $\{\bdy_n\}_{n\in \ZZ_+}$.
\label{thm:covmat}
\end{theorem}
\noindent \textbf{Remark:} \textit{From the above results, we can see that if the objective function has no saddle points, $\tilde{\lambda}$ can be $+\infty$, allowing us to flexibly choose $\alpha$ between $\frac{1}{2}$ and $1$. In such cases, the covariance matrix error bound can be arbitrarily close to $O\lppp N^{-\frac{1}{3}}\rppp$ when $\alpha$ approaches $\frac{2}{3}$, which is better than the existing results based on matrix-inversion-free methods in the convex literature. In comparison, in the convex setting, \cite{zhu2020fully} and \cite{chen2020statistical} provide a bound arbitrarily close to $O\lppp N^{-\frac{1}{8}}\rppp$ with $\alpha$ approaching to $\frac{1}{2}$. The plug-in estimator given in \cite{chen2020statistical}, which requires matrix inversion, can achieve a rate arbitrarily close to $ O\lppp N^{-\frac{1}{2}}\rppp$.} 

\textit{When saddle points exist, our error bound given in Theorem \ref{thm:covmat} depends on the value of $\tilde{\lambda}$, which quantifying the difficulty of escaping saddle points. In the best case, $\tilde{\lambda}$ can be up to $C_s$ and the feasible interval for $\alpha$ becomes $\lppp \frac{1}{2},\frac{11}{21}\rppp$. In this scenario, the error bound simplifies to $O\lppp N^{1-2\alpha}\rppp$ and can be arbitrarily close to $O\lppp N^{-\frac{1}{21}}\rppp$ when $\alpha$ approaches $\frac{11}{21}$. This upper bound is inferior to the one in the convex regime mentioned above. Therefore, it will of great significance if future works can show that $\alpha$ is allowed to increase to $\frac{2}{3}$ in the nonconvex regime. Though, the result presented in our work is sufficient for proving the consistency of confidence intervals.}

\textit{In addition, in some special cases, the error rate can be improved. For example, if the objective function is exactly quadratic in a local neighborhood of $\opt$, we can have an enhanced rate. Particularly, when the objective function is globally quadratic, the error rate can match the one shown in \cite{zhu2020fully} as the restriction on $\alpha$ can be eased to $\lppp \frac{1}{2},1\rppp$. Results can be summarized by the following lemma:
} 
\begin{lemma}
We assume same conditions given in Theorem \ref{thm:covmat}. In addition, there is a positive constant $r_q$ such that for any $\opt \in \Theta^{opt}$, we have
$$
\nabla F(\bdtheta) = \nabla^2 F(\opt) (\bdtheta - \opt),\ \forall \bdtheta:\| \bdtheta - \opt\| \leq r_q.
$$
Then, for any $\opt \in \Theta^{opt}$, any small positive constant $\epsilon$, we have
$$
\resizebox{.98\hsize}{!}{$
\EE \lvv \EE_* \lpp \lpp\lww \sqrt{N}  (\bar{\bdtheta}_N^* - \bar{\bdtheta}_N) \rww\lww\sqrt{N}  (\bar{\bdtheta}_N^* - \bar{\bdtheta}_N) \rww^T  - A^{-1}UA^{-1}\rpp \daone\lwww \limk \bdtheta_k=\opt\rwww \daone\lwww \limk \bbdtheta_k=\opt\rwww \rpp \rvv = O \lppp N^{\alpha-1} +  N^{-\frac{\alpha}{2}+\epsilon}  \rppp.
$}
$$
    \label{lemma:quadratic}
\end{lemma}



\begin{corollary}
We assume all conditions given in Theorem \ref{thm:covmat}. In addition, for a $\opt\in \Theta^{opt}$, we assume 2 mild conditions:
\begin{itemize}
\item $\PP_* \lppp \lvvv \bbdtheta_N - \bdtheta_N \rvvv \leq 2\sqrt{3}r_{good} \rppp >0$ almost surely on $\lww \limk \bdtheta_k=\opt \rww$ for sufficiently large $N$.
\item $\PP_* \lpp \limk \bbdtheta_k=\opt \rpp>0$ almost surely on $\lww \limk \bdtheta_k=\opt \rww$.
\end{itemize}
For $q\in (0,1)$, $z_{q}$ represents the $100(1-q)$\% quantile of the standard normal distribution. We let
$$
\hat{\Sigma}_N \triangleq \EE_*\lpp \lppp \sqrt{N}  (\bar{\bdtheta}_N^* - \bar{\bdtheta}_N) \rppp\lppp\sqrt{N}  (\bar{\bdtheta}_N^* - \bar{\bdtheta}_N) \rppp^T \big|\lvvv \bbdtheta_N - \bdtheta_N \rvvv \leq 2\sqrt{3}r_{good} \rpp.
$$
Then, for any prescribed $\bm{a}\in \RR^d$, we have
$$
\resizebox{.98\hsize}{!}{$
\lim\limits_{N\rightarrow \infty} \PP \lpp \bm{a}^T\opt \in \Big[ \bm{a}^T \bar{\bdtheta}_N - z_{\frac{q}{2}} \sqrt{ \frac{\bm{a}^T \hat{\Sigma}_N \bm{a} }{N}}, \bm{a}^T \bar{\bdtheta}_N + z_{\frac{q}{2}} \sqrt{ \frac{\bm{a}^T \hat{\Sigma}_N \bm{a} }{N}} \Big] \big| \limk \bdtheta_k=\opt \rpp = 1-q.
$}
$$
    \label{cor:ci1}
\end{corollary}

Given the above conclusions, with $B$ bootstrap SGD paths, we can consider to estimate the covariance matrix $A^{-1}UA^{-1}$ by
$$
\hat{\Sigma}_{N,B} \triangleq \frac{\ssum{b=1}{B} \frac{1}{N} \lpp \ssum{n=1}{N} \bdtheta_n^{(b)} - \bdtheta_n \rpp \lpp \ssum{n=1}{N} \bdtheta_n^{(b)} - \bdtheta_n \rpp^T \daone \lwww \lvvv \bdtheta_N^{(b)} - \bdtheta_N \rvvv \leq r_0\rwww}{ \ssum{b=1}{B} \daone \lwww \lvvv \bdtheta_N^{(b)} - \bdtheta_N \rvvv \leq r_0\rwww},
$$
where $r_0$ is some reasonably small constant. For convenience, we introduce the following notation:
$$
\hat{\Sigma}_N^{(b)} \triangleq \frac{1}{N} \lpp \ssum{n=1}{N} \bdtheta_n^{(b)} - \bdtheta_n \rpp \lpp \ssum{n=1}{N} \bdtheta_n^{(b)} - \bdtheta_n \rpp^T.
$$
To facilitate a fully online update, we have the following decomposition:
$$
\arraycolsep=1.1pt\def\arraystretch{1.5}
\begin{array}{rl}
\hat{\Sigma}_{N+1}^{(b)}  = &\frac{N}{N+1} \lppp \hat{\Sigma}_N^{(b)} + \lppp \bar{\bdtheta}_N^{(b)} - \bar{\bdtheta}_N\rppp \lppp \bdtheta_{N+1}^{(b)} - \bdtheta_{N+1} \rppp^T + \lppp \bdtheta_{N+1}^{(b)} - \bdtheta_{N+1} \rppp \lppp \bar{\bdtheta}_N^{(b)} - \bar{\bdtheta}_N\rppp^T \rppp\\
&+ \frac{1}{N+1} \lppp \bdtheta_{N+1}^{(b)} - \bdtheta_{N+1} \rppp \lppp \bdtheta_{N+1}^{(b)} - \bdtheta_{N+1} \rppp^T.
\end{array}
$$
Thanks to the aforementioned decomposition, our covariance matrix estimator and confidence interval can be calculated in an online manner, saving us from storing the history paths.

\begin{algorithm}
\caption{Inferential Procedure Based on Covariance Matrix Estimation}
\hspace*{\algorithmicindent} \textbf{Input:} Initializer $\bdtheta_0\in\RR^d$, stepsize $\{\gamma_n\}_{n\ge1}$, number of iterations $N$, number of bootstrap estimators $B$, distribution to sample random weights $\mathcal{P}_W$, confidence interval target $\bm{a} \in \RR^d$
\begin{algorithmic}[1]
\State $\bar{\bdtheta}_0 = \bdtheta_{0}$
\For{$b=1,2,\ldots,B$}
\State $\bdtheta_{0}^{(b)} = \bdtheta_{0}, \bar{\bdtheta}_{0}^{(b)}= \bdtheta_{0}, \hat{\Sigma}_0^{(b)} = I_d$
\EndFor
\For{$n=0,1,\ldots,N-1$}
\State $\bdtheta_{n+1} = \bdtheta_n - \gamma_{n+1} \nabla f(\bdtheta_n;\bdy_{n+1})$
\State $\bar{\bdtheta}_{n+1} = \frac{n}{n+1}\bar{\bdtheta}_{n} + \frac{1}{n+1} \bdtheta_{n+1}$

\For{$b=1,2,\ldots,B$}
\State Randomly sample $w^{(b)}_{n+1}\sim \mathcal{P}_W$
\State $\bdtheta_{n+1}^{(b)} = \bdtheta_n^{(b)} - \gamma_{n+1} w^{(b)}_{n+1}\nabla f(\bdtheta_n^{(b)};\bdy_{n+1})$
\State $\bar{\bdtheta}_{n+1}^{(b)} = \frac{n}{n+1}\bar{\bdtheta}_{n}^{(b)} + \frac{1}{n+1} \bdtheta_{n+1}^{(b)}$
\State $\hat{\Sigma}_{n+1}^{(b)} = \frac{n}{n+1} \lppp \hat{\Sigma}_n^{(b)} + \lppp \bar{\bdtheta}_n^{(b)} - \bar{\bdtheta}_n\rppp \lppp \bdtheta_{n+1}^{(b)} - \bdtheta_{n+1} \rppp^T + \lppp \bdtheta_{n+1}^{(b)} - \bdtheta_{n+1} \rppp \lppp \bar{\bdtheta}_n^{(b)} - \bar{\bdtheta}_n\rppp^T \rppp$
\State $\hat{\Sigma}_{n+1}^{(b)}= \hat{\Sigma}_{n+1}^{(b)} + \frac{1}{n+1} \lppp \bdtheta_{n+1}^{(b)} - \bdtheta_{n+1} \rppp \lppp \bdtheta_{n+1}^{(b)} - \bdtheta_{n+1} \rppp^T$
\EndFor
\EndFor

\State $\hat{\Sigma}_{N,B} = I_d, c=0$
\For{$b=1,2,\ldots,B$}
\If{$\lvvv \bdtheta_{n+1}^{(b)} - \bdtheta_{n+1} \rvvv \leq r_0$}
\State $c =c+1$
\State $\hat{\Sigma}_{N,B} = \frac{c-1}{c} \hat{\Sigma}_{N,B} + \frac{1}{c} \hat{\Sigma}_{N}^{(b)}$
\EndIf
\EndFor

\noindent\textbf{Output:} $\hat{\Sigma}_{N,B}$ as the covariance matrix estimator and $\Big[ \bm{a}^T \bar{\bdtheta}_N - 1.96\sqrt{\frac{\bm{a}^T \hat{\Sigma}_{N,B} \bm{a}}{N}}, \bm{a}^T \bar{\bdtheta}_N + 1.96\sqrt{\frac{\bm{a}^T \hat{\Sigma}_{N,B} \bm{a}}{N}} \Big]$ as the 95\% confidence interval for $\bm{a}^T \opt$
\end{algorithmic}
\label{algo2}
\end{algorithm}

\section{Bootstrap Confidence Interval}
\label{sec:quantile}

This section focuses on developing conditional weak convergence for the bootstrap estimators under local assumptions weaker than those provided in the previous section. Based on the theoretical results presented shortly, we can construct bootstrap confidence intervals without a direct estimation of the covariance matrix.

We can observe that, $w\nabla f(\bdtheta;\bdy)$ is an appropriate noisy estimator of $\nabla F(\bdtheta)$ as
\begin{equation*}
\EE_{\bdy\sim \mathcal{P}_Y,w\sim \mathcal{P}_W}\lppp w\nabla f(\bdtheta;\bdy)\rppp = \nabla F(\bdtheta).
\label{mainboot3}
\end{equation*}
In addition, we can also see that when condition \ref{LC1} holds, $w \nabla f(\bdtheta;\bdy)$ can enjoy similar good properties:
\begin{equation*}
\resizebox{.98\hsize}{!}{$
\arraycolsep=1.1pt\def\arraystretch{1.5} 
\begin{array}{rl}
& \EE_{\bdy\sim \mathcal{P}_Y,w\sim \mathcal{P}_W} \|\nabla F(\opt)-w\nabla f(\opt;\bdy)\|^{2+\tau}\\

\leq & C_{\tau}^1 \EE_{\bdy\sim \mathcal{P}_Y} \|\nabla F(\opt)-\nabla f(\opt;\bdy)\|^{2+\tau} + C_{\tau}^1 \EE_{\bdy\sim \mathcal{P}_Y,w\sim \mathcal{P}_W} \|\nabla f(\opt;\bdy)-w\nabla f(\opt;\bdy)\|^{2+\tau}\\

 = & C_{\tau}^1 \EE_{\bdy\sim \mathcal{P}_Y} \|\nabla F(\opt)-\nabla f(\opt;\bdy)\|^{2+\tau} + C_{\tau}^1 (\EE_{w\sim \mathcal{P}_W} |1-w|^{2+\tau})( \EE_{\bdy\sim \mathcal{P}_Y} \|\nabla f(\opt;\bdy)\|^{2+\tau})\\
 
 < & \infty,
\end{array}
\label{mainboot4}
$}
\end{equation*}
where $C_{\tau}^1$ is some constant only depending on $\tau$. We also have
\begin{equation*}
\arraycolsep=1.1pt\def\arraystretch{1.5} 
\begin{array}{rl}
&\EE_{\bdy\sim \mathcal{P}_Y,w\sim \mathcal{P}_W}\|w\nabla f(\opt;\bdy)-w\nabla f(\bdtheta;\bdy)\|^{2+\tau} \\

= &(\EE_{w\sim \mathcal{P}_W} |w|^{2+\tau})( \EE_{\bdy\sim \mathcal{P}_Y}\|\nabla f(\opt;\bdy)-\nabla f(\bdtheta;\bdy)\|^{2+\tau})\\

\leq & (\EE_{w\sim \mathcal{P}_W} |w|^{2+\tau}) \eta_1(\|\opt-\bdtheta\|),
\end{array}
\label{mainboot5}
\end{equation*}
when $\|\bdtheta-\opt\|\leq \delta$, where the right-hand side bound is obviously a continuous and increasing function of $\| \opt - \bdtheta\|$. Similar results can be obtained under condition \ref{LC2}:
$$
\EE_{\bdy\sim \mathcal{P}_Y,w\sim \mathcal{P}_W} \|\nabla F(\opt)-w\nabla f(\opt;\bdy)\|^{2} < \infty,
$$
and
\begin{equation*}
\resizebox{.98\hsize}{!}{$
\arraycolsep=1.1pt\def\arraystretch{1.5} 
\begin{array}{rl}
& \EE_{\bdy\sim \mathcal{P}_Y,w\sim \mathcal{P}_W}\lsss \|w\nabla f(\opt;\bdy)-w\nabla f(\bdtheta;\bdy)\|^{2}\daone\{\|w\nabla f(\opt;\bdy)-w\nabla f(\bdtheta;\bdy)\| \ge M\}\rsss \\

\leq & \EE_{\bdy\sim \mathcal{P}_Y,w\sim \mathcal{P}_W}\lsss \|w\nabla f(\opt;\bdy)-w\nabla f(\bdtheta;\bdy)\|^{2}\daone\{\|w\nabla f(\opt;\bdy)-w\nabla f(\bdtheta;\bdy)\| \ge M,|w| \ge \sqrt{M}\}\rsss\\

& + \EE_{\bdy\sim \mathcal{P}_Y,w\sim \mathcal{P}_W}\lsss \|w\nabla f(\opt;\bdy)-w\nabla f(\bdtheta;\bdy)\|^{2}\daone\{\|w\nabla f(\opt;\bdy)-w\nabla f(\bdtheta;\bdy)\| \ge M,|w| < \sqrt{M}\}\rsss\\

\leq & \EE_{w\sim \mathcal{P}_W}\lsss|w|^{2}\daone\{|w| \ge \sqrt{M}\}\rsss \EE_{\bdy\sim \mathcal{P}_Y}\lsss \|\nabla f(\opt;\bdy)-\nabla f(\bdtheta;\bdy)\|^{2}\rsss \\

& + \EE_{w\sim \mathcal{P}_W}|w|^2  \EE_{\bdy\sim \mathcal{P}_Y}\lsss \|\nabla f(\opt;\bdy)-\nabla f(\bdtheta;\bdy)\|^{2}\daone\{\|\nabla f(\opt;\bdy)-\nabla f(\bdtheta;\bdy)\| \ge  \sqrt{M}\}\rsss\\

\leq &  \EE_{w\sim \mathcal{P}_W}\lsss |w|^{2}\daone\{|w| \ge \sqrt{M}\}\rsss \eta_2(\delta) + 2\eta_{\delta}(\sqrt{M}),\\
\end{array}
\label{mainboot6}
$}
\end{equation*}
which is decreasing in $M$ and approaches 0 when $M\rightarrow +\infty$. With these observations, we have the following results:
\begin{lemma}
For any $\opt \in \Theta^{opt}$, under condition \ref{LC1} (or \ref{LC2}), \ref{LC3} and \ref{LC4}, we have
\begin{equation}
\arraycolsep=1.1pt\def\arraystretch{1.5} 
\begin{array}{rl}
& \sqrt{N}(\bar{\bdtheta}^*_N-\bar{\bdtheta}_N)\daone\{\limk \bdtheta_k=\opt,\limk \bdtheta^*_k=\opt\} \\

=&  \frac{A^{-1}}{\sqrt{N+1}} \sum\limits_{n=0}\limits^N (1-w_{n+1}) \nabla f(\opt;\bdy_{n+1})\daone\{\limk \bdtheta_k=\opt,\limk \bdtheta^*_k=\opt\} + o_p(1).
\end{array}
\label{exp3}
\end{equation}
\label{lemma:pre}
\end{lemma}

Based on the above expression (\ref{exp3}), we can expect that conditional on $(\bdtheta_0,\bdy_1,\bdy_2,\ldots)$, $\sqrt{N}(\bar{\bdtheta}^*_N-\bar{\bdtheta}_N)$ has an asymptotic Gaussian distribution on the set $\{\limk \bdtheta_k=\opt,\limk \bdtheta^*_k=\opt\}$. Also, this limit distribution must share a same covariance matrix as the limit distribution of $\sqrt{N}(\bar{\bdtheta}_N-\opt)$ on $\{\limk \bdtheta_k=\opt\}$. The conclusion can be formalized by the following theorem.

\begin{theorem}
We assume same conditions stated in Lemma \ref{basiclemma1}. For any $\bdt\in \RR^d$, the following holds in probability: 
\begin{equation}
\arraycolsep=1.1pt\def\arraystretch{1.5} 
\begin{array}{rl}
& \limN \EE_* \lss \daone \lww \limk \bdtheta_k=\opt,\limk \bbdtheta_k=\opt\rww \exp \lppp i\sqrt{N}\bdt^T (\bar{\bdtheta}^*_N - \bar{\bdtheta}_N)\rppp \rss\\

= & \daone \lww \limk \bdtheta_k=\opt \rww \PP_* \lpp \limk \bdtheta^*_k=\opt \rpp \exp\lpp-\frac{1}{2}\bdt^TA^{-1} U  A^{-T}\bdt\rpp.\\
\end{array}
\label{exp4}
\end{equation}
 
\label{bootthm}
\end{theorem}
Then, we can have the following probabilistic statement.
\begin{corollary}
We assume same conditions stated in Lemma \ref{basiclemma1}. In addition, suppose that $\PP\lppp \limk \bdtheta_k=\opt\rppp>0$ and $\PP_*\lppp \limk \bdtheta^*_k=\opt\rppp>0$ on $\{\limk \bdtheta_k=\opt\}$. Then, for any $\bdu\in \RR^d$, we have
\begin{equation}
\resizebox{.92\hsize}{!}{$
\sup\limits_{\bdu\in \RR^d} \Big| \PP_* \lpp \sqrt{N} (\bar{\bdtheta}^*_N - \bar{\bdtheta}_N)\leq \bdu \big| \limk \bdtheta^*_k=\opt\rpp - \PP\lpp \sqrt{N} (\bar{\bdtheta}_N - \opt)\leq \bdu \big| \limk \bdtheta_k=\opt \rpp \Big| = o_p(1)
$}
\label{supbound1}
\end{equation}
on $\{\limk \bdtheta_k=\opt\}$.
\label{bootcor}
\end{corollary}

Similar conclusions in the convex setting are present in \cite{fang2018online}. Due to the nonconvex nature, additional constraints like $\limk \bdtheta_k=\opt$ and $\limk \bbdtheta_k=\opt$ are inevitable. Consequently, compared to the proof in \cite{fang2018online}, our proof requires extra treatments as many desirable properties would no longer hold. For instance, $\{\nabla f(\bdtheta_n;\bdy_{n+1}) - \nablaf{n}\}_{n\in\NN}$ is a martingale adapted to the filtration $\{\fff_n\}_{n\in\NN}$. But conditional on $\lwww \limk \bdtheta_k=\opt\rwww$, this is not true. Unlike \cite{fang2018online}, which assumes globally valid conditions, we only necessitate more general and weaker local conditions.

Building upon the conclusions given in Corollary \ref{supbound1}, for any $\bm{a} \in \RR^d$, to construct a confidence interval for $\bm{a}^T \opt$, we can consider to calculate sample quantiles based on bootstrap estimators. To be more concrete, we let
$$
\tilde{q}_N(x) \triangleq \PP_* \lpp  \sqrt{N} \bm{a}^T(\bar{\bdtheta}^*_N - \bar{\bdtheta}_N)\leq x \big| \limk \bdtheta^*_k=\opt\rpp,\ x\ge0.
$$
It is not hard to know that $\tilde{q}_N(x)$ is monotonically increasing in $x$. Then, for any $q\in (0,1)$, we define the inverse random function
$$
\tilde{q}^{-1}_N(q) \triangleq \inf \lwww x:\tilde{q}_N(x) \ge q \rwww.
$$
We have the following results justifying the validity of confidence intervals:
\begin{corollary} Under conditions given in \ref{basiclemma1}, for any $\opt \in \Theta^{opt}$, $\bm{a} \in \RR^d$ and $q\in (0,1)$, we have the following conclusion:
$$
\lim\limits_{N\rightarrow \infty} \PP \lpp \bm{a}^T \opt \in \Big[ \bm{a}^T \bar{\bdtheta}_N + \frac{\tilde{q}_N^{-1}(q/2)}{\sqrt{N}}, \bm{a}^T \bar{\bdtheta}_N + \frac{\tilde{q}_N^{-1}(1-q/2)} {\sqrt{N}}\Big] \big| \limk \bdtheta_k=\opt \rpp = 1-q.
$$
\end{corollary}

A practical approximator for $\tilde{q}_N(x)$ is
$$
\hat{q}_{N,B}(x) \triangleq \frac{\ssum{b=1}{B} \daone \lwww  \sqrt{N} \bm{a}^T(\bar{\bdtheta}^{(b)}_N - \bar{\bdtheta}_N)\leq x,  \lvvv \bdtheta_N^{(b)} - \bdtheta_N \rvvv \leq r_0\rwww }{\ssum{b=1}{B} \daone \lwww \lvvv \bdtheta_N^{(b)} - \bdtheta_N \rvvv \leq r_0\rwww },
$$
where $r_0$ is a reasonably small positive constant. Then, the inverse function of $\hat{q}_{N,B}$ can be defined as:
$$
\hat{q}^{-1}_{N,B}(q) \triangleq \inf \lwww x:\hat{q}_{N,B}(x) \ge q \rwww, \ q\in(0,1).
$$
The online procedure for constructing bootstrap confidence intervals is described in Algorithm \ref{algo3}.

\begin{algorithm}
\caption{Inferential Procedure Built on Bootstrap-Based Quantile Estimation}
\hspace*{\algorithmicindent} \textbf{Input:} Initializer $\bdtheta_0\in\RR^d$, stepsize $\{\gamma_n\}_{n\ge1}$, number of iterations $N$, number of bootstrap estimators $B$, distribution to sample random weights $\mathcal{P}_W$, confidence interval target $\bm{a} \in \RR^d$
\begin{algorithmic}[1]
\State $\bar{\bdtheta}_0 = \bdtheta_{0}$
\For{$b=1,2,\ldots,B$}
\State $\bdtheta_{0}^{(b)} = \bdtheta_{0}, \bar{\bdtheta}_{0}^{(b)}= \bdtheta_{0}$
\EndFor
\For{$n=0,1,\ldots,N-1$}
\State $\bdtheta_{n+1} = \bdtheta_n - \gamma_{n+1} \nabla f(\bdtheta_n;\bdy_{n+1})$
\State $\bar{\bdtheta}_{n+1} = \frac{n}{n+1}\bar{\bdtheta}_{n} + \frac{1}{n+1} \bdtheta_{n+1}$
\For{$b=1,2,\ldots,B$}
\State Randomly sample $w^{(b)}_{n+1}\sim \mathcal{P}_W$
\State $\bdtheta_{n+1}^{(b)} = \bdtheta_n^{(b)} - \gamma_{n+1} w^{(b)}_{n+1}\nabla f(\bdtheta_n^{(b)};\bdy_{n+1})$
\State $\bar{\bdtheta}_{n+1}^{(b)} = \frac{n}{n+1}\bar{\bdtheta}_{n}^{(b)} + \frac{1}{n+1} \bdtheta_{n+1}^{(b)}$
\EndFor
\EndFor

\noindent\textbf{Output:} $ \Big[ \bm{a}^T \bar{\bdtheta}_N - \frac{\hat{q}_{N,B}^{-1}(0.025)}{\sqrt{N}}, \bm{a}^T \bar{\bdtheta}_N + \frac{\hat{q}_{N,B}^{-1}(0.975)} {\sqrt{N}}\Big]$ as the 95\% confidence interval for $\bm{a}^T \opt$
\end{algorithmic}
\label{algo3}
\end{algorithm}

The convergence of SGD trajectories to stationary point or local minimum has been studied by many researchers. There have been a lot of works showing the almost sure convergence under various sets of conditions, see \cite{ljung1977analysis}, \cite{pemantle1990nonconvergence}, \cite{brandiere1996algorithmes}, \cite{bertsekas2000gradient},  \cite{kushner2003stochastic}, \cite{borkar2009stochastic} and \cite{mertikopoulos2020almost}. If we have $\PP \lppp \{\bdtheta_k\}_{k\ge 0}\text{ converges to some }\opt \in \Theta^{opt} \rppp = 1$, we can have a more succinct result.

\begin{corollary}
Under conditions given in Lemma \ref{basiclemma1}, for any $\bdu\in \RR^d$, we have
\begin{equation*}
\sup\limits_{\bdu\in \RR^d} \Big| \PP_* \lpp \sqrt{N} (\bar{\bdtheta}^*_N - \bar{\bdtheta}_N)\leq \bdu \big| \limk \bdtheta^*_k=\limk \bdtheta_k\rpp - \PP\lpp \sqrt{N} (\bar{\bdtheta}_N - \limk \bdtheta_k)\leq \bdu  \rpp \Big| = o_p(1).
\end{equation*}
    \label{cor:boot}
\end{corollary}

\section{Examples}
\label{sec:example}
\subsection{Example 1: Gaussian Mixture Model}
\label{examplegmm}
Our theoretical analysis requires specific conditions to be satisfied, which can be met by a wide range of mixture models, making it possible to apply previously introduced bootstrap-based inferential procedures. As an example, we will consider the symmetric Gaussian 2-mixture model, which was previously analyzed by \cite{balakrishnan2017statistical} in their study of the local convergence rate of the EM algorithm. In this model, we assume that each observation $\bdy$ is independently collected from the distribution $N_d(\opt,\sigma^2I_d)$ with a probability of 0.5 and from the distribution $N_d(-\opt,\sigma^2I_d)$ with a probability of 0.5, where $I_d$ is the $d$-dimensional identity matrix. Same as \cite{balakrishnan2017statistical}, we simply assume that $\sigma^2$ is known.

For a single observation $\bdy\in \RR^d$, there's a latent variable $a$ associated with it such that $\bdy$ is from the distribution $N_d(\opt,\sigma^2I_d)$ if $a=1$ and from the distribution $N_d(-\opt,\sigma^2I_d)$ if $a=0$. Then, the complete data likelihood function associated with $(\bdy,a)$ is
$$
L(\bdtheta;\bdy,a) = \frac{1}{2}\lpp(2\pi)^{-d/2}(\sigma^2)^{-d/2}e^{-\frac{1}{2\sigma^2}\|\bdy-\bdtheta\|^2}\rpp^a \lpp (2\pi)^{-d/2}(\sigma^2)^{-d/2}e^{-\frac{1}{2\sigma^2}\|\bdy+\bdtheta\|^2}\rpp^{1-a},
$$
the logarithm of which is
$$
l(\bdtheta;\bdy,a) = \log \frac{1}{2} - \frac{d}{2}\log 2\pi -\frac{d}{2}\log \sigma^2 -\frac{a}{2\sigma^2}\|\bdy-\bdtheta\|^2 - \frac{1-a}{2\sigma^2}\|\bdy+\bdtheta\|^2.
$$
The $Q$ function, which acts as a surrogate of the observed data log likelihood function, can be characterized as
$$
\arraycolsep=1.1pt\def\arraystretch{1.5}
\begin{array}{rcl}
Q(\bdtheta,\tilde{\bdtheta};\bdy)&\triangleq&\EE\lsss  l(\bdtheta;\bdy,a) | \bdy,\tilde{\bdtheta}\rsss \\
& = & \log \frac{1}{2} - \frac{d}{2}\log 2\pi -\frac{d}{2}\log \sigma^2 -\frac{\phi(\bdy,\tilde{\bdtheta})}{2\sigma^2}\|\bdy-\bdtheta\|^2 - \frac{1-\phi(\bdy,\tilde{\bdtheta})}{2\sigma^2}\|\bdy+\bdtheta\|^2,
\end{array}
$$
where
$$
\phi(\bdy,\tilde{\bdtheta}) \triangleq \EE\lsss a | \bdy,\tilde{\bdtheta}\rsss =\frac{e^{-\frac{1}{2\sigma^2}\|\bdy-\tilde{\bdtheta}\|^2}}{e^{-\frac{1}{2\sigma^2}\|\bdy-\tilde{\bdtheta}\|^2}+e^{-\frac{1}{2\sigma^2}\|\bdy+\tilde{\bdtheta}\|^2}} = \frac{1}{1+e^{-\frac{2}{\sigma^2}\langle \bdy,\tilde{\bdtheta}\rangle}},
$$
and $\tilde{\bdtheta}$ is the estimator from the previous step. In the online learning setting, we can use the stochastic gradient EM approach, which is based on the gradient of the $Q$ function, to iteratively estimate the location parameter $\bdtheta$. That is, in the $(n+1)$-th iteration, the update rule is given by
\begin{equation}
\bdtheta_{n+1} = \bdtheta_n + \gamma_{n+1}\nabla_1 Q(\bdtheta_n,\bdtheta_n;\bdy_{n+1}),
\label{gmmupdate}
\end{equation}
where $\nabla_1 Q$ represents the gradient of $Q$ with respect to the first argument and $\gamma_{n+1}$ is the stepsize. To be more specific, under our current setting, we have
$$
-\nabla_1 Q(\bdtheta,\tilde{\bdtheta};\bdy) = \frac{1}{\sigma^2}\lsss \bdtheta + (1-2\phi(\bdy,\tilde{\bdtheta}))\bdy\rsss
$$
and $-\nabla_1 Q(\bdtheta,{\bdtheta};\bdy) = \frac{1}{\sigma^2}\lsss \bdtheta + (1-2\phi(\bdy,{\bdtheta}))\bdy\rsss $. In order to make sure that the algorithm will not stuck at $\bm{0}$, we can add some additional disturbance at each iteration:
\begin{equation}
\bdtheta_{n+1} = \bdtheta_n - \gamma_{n+1}(-\nabla_1 Q(\bdtheta_n,\bdtheta_n;\bdy_{n+1}) + \bdxi_{n+1})
\label{altsgd}
\end{equation}
with $\bdxi_{n+1}$ being i.i.d. sampled from $N(0,\sigma_{\xi}^2I_d)$ with some small $\sigma_{\xi}^2>0$. We let $\bdz = (\bdy^T,\bdxi^T)^T$, and $\bdz_{n} = (\bdy_n^T,\bdxi_n^T)^T$ for simplicity.  To fit the update rule (\ref{altsgd}) into the stochastic gradient descent framework, we would like to find differentiable functions $F:\RR^d\rightarrow \RR$ and $f:\RR^d \times \RR^{2d} \rightarrow \RR$ such that
$$
\left\{
\arraycolsep=1.1pt\def\arraystretch{1.5}
\begin{array}{l}
 \nabla f(\bdtheta;\bdz) = -\nabla_1 Q(\bdtheta,{\bdtheta};\bdy) + \bdxi,\\
 \nabla F(\bdtheta) = \EE \nabla f(\bdtheta;\bdz).
\end{array}
\right.
$$
To satisfy these conditions, we can let
$$
F(\bdtheta) \triangleq \frac{1}{\sigma^2}\lpp\frac{1}{2}\|\bdtheta\|^2 - 2\EE\lsss s(\langle \bdy,\bdtheta\rangle)\rsss \rpp,
$$
where $s(t)\triangleq \int^t_0 \frac{1}{1+e^{-\frac{2}{\sigma^2}u}}du$. Then, the update rule (\ref{altsgd}) can be reformulated as
$$
\bdtheta_{n+1} = \bdtheta_n - \gamma_{n+1}\nabla f(\bdtheta_n;\bdz_{n+1}),
$$
which coincides with the stochastic gradient descent procedure (\ref{mainupdate}).

Based on Lemma 2 of \cite{balakrishnan2017statistical}, we know that when $\|\opt\|$ is sufficiently large,
$$
A = \nabla^2F(\opt) = \frac{1}{\sigma^2}\Bigg( I_d - \frac{4}{\sigma^2}\EE\frac{\bdy\bdy^T}{\lppp e^{-\frac{1}{\sigma^2}\langle \bdy,\opt\rangle} + e^{\frac{1}{\sigma^2}\langle \bdy,\opt\rangle}\rppp^2}\Bigg)
$$
is positive definite. Likewise, due to symmetry, $\nabla^2F(-\opt)$ is also positive definite. In addition, $U = \EE\lsss \nabla f(\opt;\bdz)\nabla f(\opt;\bdz)^T\rsss $ is obviously positive definite. We can also verify condition \ref{LC1} for the current model, with details left in the Appendix \ref{suppexample2}. Hence, we can apply Algorithm \ref{algo3} to the current model to conduct inference on $\opt$. We provide simulations results in subsection \ref{subsec:empirical1} to demonstrate the effectiveness of our method. 

We also verify conditions \ref{cm_1}-\ref{cm_5} in the Appendix \ref{suppexample2}. Therefore, under the current model, the in-expectation error rate of the original SGD estimator can be guaranteed based on Theorem \ref{thm:secbound}. Furthermore, our results shown in Section \ref{sec:covmax} ensure that we can construct confidence intervals with Algorithm \ref{algo2}.

According to \cite{xu2016global}, there are 3 fixed points of the mapping $\Psi:\RR^d\rightarrow \RR^d$ with $\Psi(\bdtheta) = \EE\lsss (2\phi(Y,\bdtheta)-1)Y\rsss $, which are $\opt,-\opt$ and \textbf{0}. It is not hard to see that $\bdtheta$ is a fixed point of $\Psi$ if and only if it is a stationary point of $F$. For $F$, we can easily see that \textbf{0} is a saddle point while $\opt$ and $-\opt$ are two global minimizers. In fact, we can show that with probability 1, the SGD path generated by (\ref{altsgd}) converges to $\opt$ or $-\opt$.
\begin{proposition}
Assume that $\|\opt\|$ is sufficiently large such that $A$ is positive definite. Let the stepsize $\gamma_n=C\cdot n^{-\alpha}$ for some $\frac{1}{2}<\alpha <1$ and $C>0$. Then, $\{\bdtheta_n\}_{n\ge 0}$ converges to $\opt$ or $-\opt$ with probability 1.
\label{gaussianprop}
\end{proposition}

\subsection{Example 2: Logistic Regression with Concave Regularization}
\label{subsec:logistic}
Suppose that we have a sequence of data $\{(\bdx_i,y_i)\}_{1\leq i\leq M}$ at hand, where $\bdx_i\in \RR^d$ and $y_i\in \{-1,1\}$, $1\leq i\leq M$. When $d$ is large, it is common to consider the regularized logistic regression to model the relationship between $\bdx$ and $y$. That is, we need to minimize an objective function
$$
F_M(\bdtheta) = \frac{1}{M}\sum\limits_{i=1}\limits^{M}\log \lss 1+e^{-(\bdtheta^T\bdx_i)y_i}\rss + \lambda R(\bdtheta),
$$
where $\lambda$ is a tunable regularized parameter and $R(\cdot)$ is the regularized function. Specifically, if we consider a nonconvex regularized function $R(\bdtheta) = \sum\limits_{j=1}\limits^{d}\frac{\bdtheta^2(j)}{1+\bdtheta^2(j)}$, where $\bdtheta(j)$ is the $j$-th component of $\bdtheta$, the objective function is
\begin{equation}
F_M(\bdtheta) = \frac{1}{M}\sum\limits_{i=1}\limits^{M}\log \lss 1+e^{-(\bdtheta^T\bdx_i)y_i}\rss + \lambda\sum\limits_{j=1}\limits^{d}\frac{\bdtheta^2(j)}{1+\bdtheta^2(j)},
\label{example1eq1}
\end{equation}
which has also been studied in \cite{antoniadis2011penalized},  \cite{horvath2020adaptivity}, \cite{khaled2020better} and \cite{xin2020fast}. We suppose the minimizer of (\ref{example1eq1}) is $\opt$. To approximate the minimizer, we can adopt the stochastic gradient descent procedure. In the $(n+1)$-th iteration, we sample an index $i_{n+1}$ uniformly from $\{1,2,\ldots,M\}$ and use the update rule
$$
\resizebox{.98\hsize}{!}{$
\arraycolsep=1.1pt\def\arraystretch{1.5}
\begin{array}{rcl}
\bdtheta_{n+1}&=&\bdtheta_{n}-\gamma_{n+1}\nabla f_M(\bdtheta_n;i_{n+1})\\

& = & \bdtheta_{n}-\gamma_{n+1}\lss \lww 1+e^{(\bdtheta_n^T \bdx_{i_{n+1}})y_{i_{n+1}}}\rww^{-1}(-y_{i_{n+1}}\bdx_{i_{n+1}})+2\lambda \text{diag}\lww\frac{1}{(1+\bdtheta^2_n(1))^2},\ldots, \frac{1}{(1+\bdtheta^2_n(d))^2}\rww \bdtheta_n\rss,
\end{array}
$}
$$
where $\nabla f_M(\bdtheta_n;i_{n+1})$ can be viewed as a random surrogate of $\nabla F(\bdtheta_n)$:
\begin{equation}
\nabla F_M(\bdtheta) = \frac{1}{M}\sum\limits_{i=1}\limits^{M}\lss 1+e^{(\bdtheta^T\bdx_i)y_i}\rss^{-1}(-y_i\bdx_i)+2\lambda \text{diag}\lww\frac{1}{(1+\bdtheta^2(1))^2},\ldots, \frac{1}{(1+\bdtheta^2(d))^2}\rww \bdtheta.
\label{example1eq2}
\end{equation}
To avoid confusion, we want to point out that in the current example, $\{(\bdx_i,y_i)\}_{1\leq i\leq M}$ are considered as fixed, while randomness originates from randomly sampled index $\{i_n\}_{n\in \NN}$. Conditions required in our theoretical analysis can be satisfied under mild restrictions on the design matrix $X=[x_1,x_2,\ldots,x_M]^T$ and $\lambda$. We can observe that
$$
\resizebox{.98\hsize}{!}{$
\nabla^2F_M(\bdtheta) = \frac{1}{M}\sum\limits_{i=1}\limits^{M}\frac{e^{(\bdtheta^T\bdx_i)y_i}}{\lsss 1+e^{(\bdtheta^T\bdx_i)y_i}\rsss^2}\bdx_i\bdx_i^T + 2\lambda \text{diag}\lww\frac{1}{(1+\bdtheta^2(1))^2} - \frac{4\bdtheta^2(1)}{(1+\bdtheta^2(1))^3},\ldots, \frac{1}{(1+\bdtheta^2(d))^2} - \frac{4\bdtheta^2(d)}{(1+\bdtheta^2(d))^3} \rww.
$}
$$
From the given expression, it is straightforward to deduce the uniform boundedness of the Hessian matrix and higher-order derivatives, satisfying the first part of condition \ref{cm_1}. Similarly, for each $i=1,2,\ldots, M$, it is not hard to know that $\lvvv \nabla^p f_M(\bdtheta;i)\rvvv$ is uniformly bounded for $p\ge 2$. Considering that the index $i$ can be only chosen from a finite set, we can immediately see the validity of condition \ref{cm_5}. When $\lambda$ is exactly $0$, the objective function exhibits strict convexity provided that the design matrix is not column-dependent. Consequently, for a sufficiently large $\|\bdtheta\|$ and a reasonably small $\lambda$, the gradient norm $\| \nabla F_M(\bdtheta)\|$ remains bounded away from $0$, indicating that local minima can only exist within a bounded region. For any $\bdtheta$ within this bounded region, the first part on the right-hand side above is positive definite as long as the design matrix is column-independent, which is typical when $M$ is moderately large. Therefore, condition \ref{cm_4} and the latter part of condition \ref{cm_1} are valid when $\lambda$ is sufficiently small. The remaining conditions also hold under similar constraints. For further details, please refer to Appendix \ref{suppexample1}.

\section{Empirical Experiments}
\label{sec:empirical}


\subsection{Gaussian Mixture Model}
\label{subsec:empirical1}
In this part, we empirically explore the performance of our methods in Gaussian 2-mixture model. We outline the setup of the numerical experiment as follows: We let the stream of data $\{\bdy_k\}_{k\ge1}$ independently generated from the symmetric Gaussian 2-mixture model introduced in subsection \ref{examplegmm} with $\sigma^2=1$. We consider two cases with dimensions of $\bdtheta$ being $5$ and $25$, respectively. To stablize the SGD trajectory, we adopt a mini-batch of size $5$. In accordance with our theoretical analysis, we choose stepsize $\gamma_n = n^{-{47}/{92}}/2$ for $n\ge 1$ as $47/92$ is greater than $1/2$ and slightly smaller than $11/21$. The initializer $\bdtheta_0$ is uniformly sampled from the hypercube centered at $\bm{0}$ with side length of $10$. Since the model is rotation invariant, we simply let $\opt = s\cdot \bde_1$, where $s$ is the signal strength and $\bde_1$ is the elementary vector with its first component being $s$ and the others being $0$. Throughout our experiments, we let the number of multiplier bootstrap estimators be $B=500$ and let the length of the SGD path be $N=4000$. We repeat each experiment for $MC=6000$ times to ensure the robustness of our conclusions.

\begin{table}\centering
\ra{1.3}
\begin{tabular}{@{}llllclll@{}}\toprule
& \multicolumn{3}{c}{Uniform Bootstrap} & \phantom{abc}& \multicolumn{3}{c}{Exponential Bootstrap} \\
\cmidrule{2-4} \cmidrule{6-8} 
$s$& Bootstrap & Cov & Oracle && Bootstrap & Cov & Oracle\\ \midrule

$1.0$ & 0.936 (3.73) & 0.941 (4.00) & 0.887 (3.43)  && 0.930 (3.70) & 0.941 (3.98) & 0.885 (3.44) \\

$1.5$ & 0.935 (3.11) & 0.949 (3.21) & 0.930 (3.04) && 0.942 (3.10) & 0.951 (3.19) & 0.935 (3.04) \\

$2.0$ & 0.944 (2.95) & 0.955 (3.01) & 0.946 (2.93)  && 0.944 (2.94) & 0.951 (3.01) & 0.947 (2.95) \\

$2.5$ & 0.941 (2.90) & 0.951 (2.96) & 0.945 (2.90) && 0.940 (2.90) & 0.947 (2.95) & 0.945 (2.93) \\

$3.0$ & 0.943 (2.89) & 0.952 (2.95) & 0.952 (2.92) && 0.954 (2.88) & 0.960 (2.93) & 0.957 (2.90)  \\
\bottomrule
\end{tabular}
\caption{Gaussian Mixture Model ($d=5$): Average coverage rates and lengths (numbers in parentheses $\times\ 0.01$) of $95\%$ confidence intervals (CI). Results under "Uniform Bootstrap" are based on $[1-\sqrt{3},1+\sqrt{3}]$ uniformly distributed multipliers. Results under "Exponential Bootstrap" are based on multipliers drawn from the $\exp(1)$ distribution. The "Bootstrap" columns represent results for vanilla bootstrap CI. The "Cov" type CIs use estimated covariance matrix, while the "Oracle" type CIs leverage the true asymptotic covariance matrix.}
\label{table:gmm5}
\end{table}

We construct $95\%$ confidence intervals for the first component of $\opt$ and calculate the average coverage rate and average width. Results for $d=5$ and $d=25$ are presented in Table \ref{table:gmm5} and \ref{table:gmm50}, respectively. In both tables, the left part shows results with random bootstrap weights generated from a uniform distribution while the right part presents results with random bootstrap weights generated from an exponential distribution. Note that both distributions have a mean of $1$ and variance of $1$. We observe that, for $d=5$, confidence intervals built upon estimated covariance matrix can achieve nearly exact coverage when $s\ge 1.5$. Meanwhile, bootstrap confidence intervals and confidence intervals based on oracle covariance matrix exhibit similar performance. Though slightly under-coverage, they achieve coverage rates greater than $0.94$ when $s\ge2.0$. In the case of $d=25$, all three types of confidence intervals have increasing coverage rates as $s$ increases from $1.0$ to $3.0$. When $s$ equals $3.0$, confidence intervals built upon estimated covariance matrix can basically achieve the nominal $0.95$ coverage rate. Though under-coverage, bootstrap confidence intervals can consistently maintain coverage rates greater than $0.9$.

\begin{table}\centering
\ra{1.3}
\begin{tabular}{@{}llllclll@{}}\toprule
& \multicolumn{3}{c}{Uniform Bootstrap} & \phantom{abc}& \multicolumn{3}{c}{Exponential Bootstrap} \\
\cmidrule{2-4} \cmidrule{6-8} 
$s$& Bootstrap & Cov & Oracle && Bootstrap & Cov & Oracle\\ \midrule

$1.0$ & 0.909 (4.76) & 0.869 (6.73) & 0.474 (3.44) && 0.904 (4.72) & 0.863 (6.74) & 0.451 (3.42) \\

$1.5$ & 0.906 (3.47) & 0.901 (3.87) & 0.789 (3.08)  && 0.905 (3.44) & 0.897 (3.89) & 0.783 (3.05)  \\

$2.0$ & 0.916 (3.16) & 0.934 (3.30) & 0.891 (2.96)  && 0.914 (3.15) & 0.924 (3.30) & 0.878 (2.96)  \\

$2.5$ & 0.927 (3.08) & 0.942 (3.18) & 0.911 (2.91) && 0.921 (3.06) & 0.934 (3.16) & 0.909 (2.94)  \\

$3.0$ & 0.930 (3.04) & 0.950 (3.12) & 0.924 (2.92) && 0.929 (3.02) & 0.945 (3.11) & 0.925 (2.97)   \\
\bottomrule
\end{tabular}
\caption{Gaussian Mixture Model ($d=25$): Average coverage rates and lengths (numbers in parentheses $\times\ 0.01$) of $95\%$ confidence intervals (CI). Results under "Uniform Bootstrap" are based on $[1-\sqrt{3},1+\sqrt{3}]$ uniformly distributed multipliers. Results under "Exponential Bootstrap" are based on multipliers drawn from the $\exp(1)$ distribution. The "Bootstrap" columns represent results for vanilla bootstrap CI. The "Cov" type CIs use estimated covariance matrix, while the "Oracle" type CIs leverage the true asymptotic covariance matrix.}
\label{table:gmm50}
\end{table}


\subsection{Logistic Regression with Concave Regularization}
\label{subsec:empirical2}

In this part, we consider to optimize the concavely regularized logistic regression (\ref{example1eq1}). To generate the data, we let $\bdtheta_s \in \RR^{10}$ be $(1,1,0,\ldots,0)^T$. We let the sample size $M=1000$. $\{\bdx_i\}_{1\leq i\leq M}$ are i.i.d. generated from a $10$-dimensional Gaussian distribution centered at $\bm{0}$ and with covariance matrix $\Sigma_X$. We consider the cases of $\Sigma_X$ being an identity matrix and a Toeplitz matrix, respectively. For $1\leq i \leq M$, conditional on $\bdx_i$, $y_i$ is generated from $\mathcal{P}_{Y|X}$ with $\PP(y_i\pm 1|\bdx_i) = \frac{e^{\pm {\bdtheta_s^T \bdx_i}}}{1+e^{\pm \bdtheta_s^T \bdx_i}}$. In our analysis, we regard $\bdx_i,y_i,i=1,\ldots,M$, as given and fixed. Therefore, the optimal point $\opt$ could differ from $\bdtheta_s$. To obtain a near-precise estimation of $\opt$, we run the gradient descent algorithm with a fixed stepsize $0.5$ iteratively until we find a $\bdtheta$ such that $\lvvv \nabla F_M(\bdtheta) \rvvv \leq 0.1^{10}$.

We choose a mini-batch size of 5 in our simulation. We set the stepsize $\gamma_n = n^{-\frac{47}{92}},n\ge 1$. We sample the initializer uniformly from the hypercube centered at $\opt$ with a side length of 4. Throughout our experiments, we keep the number of multiplier bootstrap estimators to be $B=500$. Meanwhile, we let the total length of online SGD be $N=8000$ and each experiment is repeated for $MC=3000$ times. 

\begin{table}\centering
\ra{1.3}
\begin{tabular}{@{}llllclll@{}}\toprule
& \multicolumn{3}{c}{Uniform Bootstrap} & \phantom{abc}& \multicolumn{3}{c}{Exponential Bootstrap} \\
\cmidrule{2-4} \cmidrule{6-8} 
$\lambda$& Bootstrap & Cov & Oracle && Bootstrap & Cov & Oracle\\ \midrule

$0.000$ & 0.943 (5.80) & 0.935 (5.38) & 0.890 (4.98)  && 0.946 (5.98) & 0.938 (5.47) & 0.890 (4.98) \\

$0.025$ & 0.950 (5.32) & 0.951 (4.97) & 0.904 (4.62) &&0.953 (5.51) & 0.952 (5.06) & 0.904 (4.62) \\

$0.050$ & 0.943 (4.65) & 0.943 (4.35) & 0.902 (4.04) && 0.948 (4.85) & 0.945 (4.44) & 0.902 (4.04)  \\

$0.075$ & 0.938 (3.95) & 0.941 (3.69) & 0.891 (3.41) &&0.943 (4.15) & 0.942 (3.77) & 0.891 (3.41) \\

$0.100$ & 0.948 (3.33) & 0.954 (3.12) & 0.894 (2.87)  && 0.956 (3.55) & 0.957 (3.19) & 0.897 (2.87) \\
\bottomrule
\end{tabular}
\caption{Logistic Regression ($\Sigma_X = I_{10}$): Average coverage rates and lengths (numbers in parentheses $\times\ 0.01$) of $95\%$ confidence intervals (CI). Results under "Uniform Bootstrap" are based on $[1-\sqrt{3},1+\sqrt{3}]$ uniformly distributed multipliers. Results under "Exponential Bootstrap" are based on multipliers drawn from the $\exp(1)$ distribution. The "Bootstrap" columns represent results for vanilla bootstrap CI. The "Cov" type CIs use estimated covariance matrix, while the "Oracle" type CIs leverage the true asymptotic covariance matrix.}
\label{table:log1}
\end{table}

\begin{table}\centering
\ra{1.3}
\begin{tabular}{@{}llllclll@{}}\toprule
& \multicolumn{3}{c}{Uniform Bootstrap} & \phantom{abc}& \multicolumn{3}{c}{Exponential Bootstrap} \\
\cmidrule{2-4} \cmidrule{6-8} 
$\lambda$& Bootstrap & Cov & Oracle && Bootstrap & Cov & Oracle\\ \midrule

$0.000$ &0.903 (6.85) & 0.913 (6.43) & 0.822 (5.75)  && 0.912 (7.23) & 0.916 (6.59) & 0.822 (5.75)  \\

$0.025$ & 0.904 (6.64) & 0.905 (6.31) & 0.842 (5.69) && 0.916 (7.01) & 0.908 (6.47) & 0.842 (5.69)  \\

$0.050$ & 0.917 (6.14) & 0.928 (5.85) & 0.856 (5.25) && 0.923 (6.48) & 0.933 (6.00) & 0.856 (5.25)  \\

$0.075$ & 0.920 (5.31) & 0.934 (5.04) & 0.850 (4.48)  && 0.930 (5.73) & 0.933 (5.22) & 0.850 (4.48) \\

$0.100$ & 0.921 (4.43) & 0.932 (4.17) & 0.847 (3.65) && 0.937 (4.88) & 0.933 (4.35) & 0.847 (3.65) \\
\bottomrule
\end{tabular}
\caption{Logistic Regression ($\Sigma_X = T_{10,0.5}$): Average coverage rates and lengths (numbers in parentheses $\times\ 0.01$) of $95\%$ confidence intervals (CI). Results under "Uniform Bootstrap" are based on $[1-\sqrt{3},1+\sqrt{3}]$ uniformly distributed multipliers. Results under "Exponential Bootstrap" are based on multipliers drawn from the $\exp(1)$ distribution. The "Bootstrap" columns represent results for vanilla bootstrap CI. The "Cov" type CIs use estimated covariance matrix, while the "Oracle" type CIs leverage the true asymptotic covariance matrix. $T_{10,0.5}$ represents the Toeplitz matrix with $0.5^{|i-j|}$ being the element in the $i$-th row and $j$-th column of the matrix.}
\label{table:log2}
\end{table}


We construct $95\%$ confidence intervals for the first component of $\opt$. Results are summarized in Table \ref{table:log1} and \ref{table:log2}. We can see that when $\Sigma_X$ is an identity matrix, both the bootstrap confidence interval and confidence interval based on estimated covariance matrix can basically achieve the nominal coverage rate while the confidence interval based on the oracle covariance matrix shows under-coverage. When $\Sigma_X$ adopts a Toeplitz design, all three types of confidence intervals exhibits lower coverage rates. Though, both the bootstrap confidence interval and confidence interval based on estimated covariance matrix can maintain coverage rates higher than $0.9$ for all different values of $\lambda$.

\section{Discussion}
\label{sec:discussion}
Our study contributes to the existing literature by introducing two fully online procedures for performing statistical inference with SGD estimators in nonconvex settings. This fills a gap in understanding and harnessing the uncertainty of SGD estimators. Our theoretical analysis demonstrates the asymptotic validity of the confidence intervals and establishes the error convergence rate for a bootstrap-based covariance matrix estimator. We further validate the finite sample effectiveness of the proposed methods through numerical experiments on applications such as Gaussian mixture models and nonconvex logistic regression.

It is important to acknowledge the limitations of our study. Firstly, we suspect that the restriction on $\alpha$, as mentioned in Theorems \ref{thm:secbound} and \ref{thm:covmat}, is primarily due to our proof strategy and could be removable. If future research can alleviate this restriction, we may obtain error rates entirely comparable to those in the convex setting. Secondly, our study only considers the case where observations are i.i.d. To extend the applicability of our work to more general settings, it would be valuable to explore weaker assumptions. \cite{li2023online} and \cite{liu2023statistical} have considered Markovian data and $\phi$-mixing data for convex SGD, respectively. Adapting our analysis to these more general situations holds great potential.

In conclusion, our work represents an advancement in uncertainty quantification with SGD in broader contexts, offering valuable insights and laying the foundation for further exploration. Our proof techniques can be extended to validate other inferential frameworks, and our error bounds may yield additional conclusions. We anticipate that our findings will inspire new avenues of research in the field.

\newpage
\vspace{1cm}
\appendix
\noindent\textbf{\Large Appendix}
\section{Results Supporting Theorem \ref{bootthm}}
\subsection{Asymptotic Inference on Multi-targets Stochastic Approximation Estimator}

\cite{fort2015central} provided sufficient conditions for validating asymptotic inference on stochastic approximation estimator that can fit into a wide range of settings, including controlled Markov dynamics. A more general stochastic approximation procedure can be expressed as
\begin{equation}
\bdtheta_{n+1} = \bdtheta_n - \gamma_{n+1}\nabla F(\bdtheta_n) + \gamma_{n+1}\bde_{n+1} + \gamma_{n+1}\bdr_{n+1},
\label{generalsa}
\end{equation}
where $-\nabla F$ can be viewed as the mean field, $\{\bde_n\}$ is the error term and $\{\bdr_n\}$ is the remainder. We slightly weaken condition C2(b) given in \cite{fort2015central} and list them as follows.

\begin{enumerate}[label=\textbf{A\arabic*}]
\item \label{A1}
\begin{enumerate}[label=(\alph*)]
\item The parameter space $\Theta \subset \RR^d$. Suppose with probability 1, the iterate sequence $\{\bdtheta_n\}_{n\ge 0}\subset \Theta$.
\item $\opt$ is in the interior of $\Theta$ and satisfies that $\nabla F(\opt)=0$.
\item \label{A1c}$\nabla F(\bdtheta)$ is twice continuously differentiable in the neighborhood of $\opt$.
\item \label{A1d}$A$ is a positive definite matrix with the smallest eigenvalue $\mu>0$.
\end{enumerate}

\item \label{A2}
\begin{enumerate}[label=(\alph*)]
\item \label{A2a}$\{\bde_n\}_{n\ge 1}$ is a $\fff_n$-adapted martingale difference sequence and satisfies that $\EE[\bde_n|\fff_n]=0$, where $\fff_n = \sigma(\bdtheta_0,\bde_1,\bdr_1,\ldots,\bde_n,\bdr_n)$. (With a slight abuse of notation, we will denote $\EE [\cdot \mid \fff_n]$ as $\EE_n[\cdot]$.)
\item[{\crtcrossreflabel{(b1)}[A2b1]}] There exists $\tau>0,\delta>0$ such that the Lyapunov type condition holds,
$$
\sup_{n\ge 0} \EE\left[\|\bde_{n+1}\|^{2+\tau}\daone\{\|\bdtheta_n-\opt\|\leq \delta\}\right]<\infty.
$$
\item[{\crtcrossreflabel{(b2)}[A2b2]}]There exists $\delta>0$ such that the Lindeberg type condition holds,
$$
\sup_{n\ge 0} \EE\left[\|\bde_{n+1}\|^{2}\daone\{\|\bdtheta_n-\opt\|\leq \delta\}\right]<\infty,
$$
and
$$
\lim\limits_{M\rightarrow \infty}\sup_{n\ge 0} \EE\left[\|\bde_{n+1}\|^{2}\daone\{\|\bdtheta_n-\opt\|\leq \delta\}\daone\{\|\bde_{n+1}\|\ge M\}\right]\rightarrow 0.
$$
\item[{\crtcrossreflabel{(c)}[A2c]}]$\EE_n[\bde_{n+1}\bde_{n+1}^T]=U  + D_n$, where $U $ is a symmetric positive definite matrix and $D_n\rightarrow_{a.s.}0$ on $\{\limk \bdtheta_k=\opt\}$.
\end{enumerate}

\item  \label{A3} The remainder term $\bdr_n$ is $\fff_n$-adapted and can be decomposed as $\bdr_n = \bdr_n^{(1)} + \bdr_n^{(2)}$, such that for any fixed $m\ge1$,
\begin{enumerate}[label=(\alph*)]
\item \label{A3a}$\gamma_n^{-1/2}\bdr_n^{(1)}\daone\{\limk \bdtheta_k=\opt\} \daone\{\|\bdtheta_n-\opt\|\leq \delta,n\ge m\} = O(1)O_{L^2}(1)$.
\item \label{A3b}$\gamma_N^{\frac{1}{2}}\sum\limits_{n=1}\limits^{N} \bdr_n^{(2)}\daone\{\limk \bdtheta_k=\opt\}\daone\{\|\bdtheta_n-\opt\|\leq \delta,n\ge m\} = O(1)O_{L^2}(1)$.
\item $N^{-1/2} \sum\limits_{n=0}\limits^{N}\bdr_{n+1}\daone\{\limk \bdtheta_k=\opt\} \rightarrow 0$ in probability.
\end{enumerate}
\item \label{A4} The stepsize $\gamma_n = C\cdot n^{-\alpha}$, $\forall n\ge1$, for some $\alpha \in (1/2,1)$ and $C>0$.
\end{enumerate}

\noindent \textbf{Remarks:}
\begin{itemize}
\item Conditions \ref{A1}\ref{A1c} and \ref{A1}\ref{A1d} indicates that $F$ is smooth and strongly convex in a neighborhood around $\opt$. These requirements are commonly-seen in the literature.
\item We explain some notations mentioned in condition \ref{A3}. We say a sequence of random variable $\{X_n\}_{n\ge1}$ is $O(1)$ if there exists an everywhere finite random variable $\tilde{X}_b$ such that $X_n\leq \tilde{X}_b$ a.s. for any $n$. We say $\{X_n\}_{n\ge1}$ is $O_{L^2}(1)$ if there exists a positive number $M$ such that $\EE X_n^2 \leq M$ for any $n$.
\end{itemize}

Based on Lemma 5.5 and Theorem 3.2 in \cite{fort2015central}, we can directly have the following results without proof.

\begin{lemma}
Consider the stochastic approximation procedure given in (\ref{generalsa}). Assume conditions \ref{A1}, \ref{A2}\ref{A2a}, \ref{A2}\ref{A2b1}, \ref{A2}\ref{A2c}, \ref{A3} and \ref{A4} hold. Then for any $\bdt\in \RR^d$,
$$
\sqrt{N}(\bar{\bdtheta}_N-\opt)\daone\{\limk \bdtheta_k=\opt\} =  \frac{ A^{-1}}{\sqrt{N+1}} \sum\limits_{n=0}\limits^N \bde_{n+1}\daone\{\limk \bdtheta_k=\opt\} + o_p(1),
$$
and
$$
\resizebox{.95\hsize}{!}{$
\lim\limits_{N\rightarrow \infty}\EE\left[\daone\{\limk \bdtheta_k=\opt\} exp(i\sqrt{N}\bdt^T(\bar{\bdtheta}_N-\opt))\right] = \PP\left(\limk \bdtheta_k=\opt\right) exp\left(-\frac{1}{2}\bdt^T A^{-1} U   A^{-T}\bdt\right),
$}
$$
where $\bar{\bdtheta}_N = \frac{1}{N+1}\sum\limits_{n=0}\limits^{N}\bdtheta_n$.
\label{thm:fort15}
\end{lemma}

Actually, if we replace the Lyapunov type condition \ref{A2}\ref{A2b1} with the weaker one \ref{A2}\ref{A2b2}, we can still have the same conclusion valid.

\begin{lemma}
Consider the stochastic approximation procedure given in (\ref{generalsa}). Assume conditions \ref{A1}, \ref{A2}\ref{A2a}, \ref{A2}\ref{A2b2}, \ref{A2}\ref{A2c}, \ref{A3} and \ref{A4} hold. Then for any $\bdt\in \RR^d$,
$$
\sqrt{N}(\bar{\bdtheta}_N-\opt)\daone\{\limk \bdtheta_k=\opt\} = \frac{ A^{-1}}{\sqrt{N+1}} \sum\limits_{n=0}\limits^N \bde_{n+1}\daone\{\limk \bdtheta_k=\opt\} + o_p(1),
$$
and
$$
\resizebox{.95\hsize}{!}{$
\lim\limits_{N\rightarrow \infty}\EE\left[\daone\{\limk \bdtheta_k=\opt\} exp(i\sqrt{N}\bdt^T(\bar{\bdtheta}_N-\opt))\right] = \PP\left(\limk \bdtheta_k=\opt\right) exp\left(-\frac{1}{2}\bdt^T A^{-1} U   A^{-T}\bdt\right),
$}
$$
where $\bar{\bdtheta}_N = \frac{1}{N+1}\sum\limits_{n=0}\limits^{N}\bdtheta_n$.
\label{thm:fort15var}
\end{lemma}


\subsection{Proof of Lemma \ref{thm:fort15var}}
\begin{proof}
\cite{fort2015central} provided a series of sufficient conditions to guarantee the central limit theorem result stated in Lemma \ref{thm:fort15var}. Here, for convenience, we restate some of the key conditions in \cite{fort2015central}.\\

\noindent{\crtcrossreflabel{\textbf{AVER1}}[aver1]}
\begin{enumerate}
\item[{\crtcrossreflabel{(a)}[aver1a]}] $\{\bde_n,n\ge1\}$ is a $\fff_n$-adapted martingale difference sequence.

\item[{\crtcrossreflabel{(b)}[aver1b]}] There exists a sequence of sets $\{\mathcal{A}_m,m\ge1\}$ such that $\lim\limits_m\PP(\mathcal{A}_m|\limk \bdtheta_k=\opt) = 1$, and for any $m\ge 1$,
$$
\sup_{n}\EE[|\bde_n|^2\daone_{\mathcal{A}_{m,n-1}}]<\infty,
$$
where $\mathcal{A}_{m,n-1}\in \mathcal{F}_{n-1}$ and $\lim\limits_n\daone_{\mathcal{A}_{m,n}} = \daone_{\mathcal{A}_m}$ almost surely on $\{\limk \bdtheta_k=\opt\}$.

\item[{\crtcrossreflabel{(c)}[aver1c]}] We denote
$$
\mathcal{E}_{N+1} \triangleq \frac{1}{\sqrt{N+1}}\sum\limits_{n=0}\limits^{N}\bde_{n+1}.
$$
There exists a positive definite matrix $U $ such that for any $t\in \RR^d$,
$$
\begin{array}{rl}
& \lim\limits_{N\rightarrow \infty} \EE\left[ \daone\{\limk \bdtheta_k=\opt\} exp(i\bdt^T\mathcal{E}_{N+1})\right]\\

= & \EE\left[ \daone\{\limk \bdtheta_k=\opt\} exp\left(-\frac{1}{2}\bdt^TU \bdt\right)\right].
\end{array}
$$
\end{enumerate}

According to Theorem 3.2 in \cite{fort2015central}, \ref{aver1}, combined with \ref{A1}, \ref{A3} and \ref{A4} can lead to the desired central limit result. Therefore, it suffices to show that \ref{A2}\ref{A2a}, \ref{A2}\ref{A2b2} and \ref{A2}\ref{A2c} can jointly guarantee \ref{aver1}. 

Firstly, \ref{A2}\ref{A2a} is essentially same as \ref{aver1}\ref{aver1a}. To show \ref{aver1}\ref{aver1b}, we let 
$$
\mathcal{A}_{m,n} = \{\|\bdtheta_k-\opt\|\leq \delta,m\leq k\leq n\}
$$
if $m\leq n$ and $\mathcal{A}_{m,n} = \emptyset$ otherwise. For $m\ge 1$, we let $\mathcal{A}_m =  \{\|\bdtheta_k-\opt\|\leq \delta,m\leq k\}$. Then, we have
$$
\begin{array}{rl}
& \lim\limits_{m\rightarrow \infty}\PP(\mathcal{A}_m|\limk \bdtheta_k=\opt) = \PP(\bigcup\limits_{m=1}\limits^{\infty} \mathcal{A}_m|\limk \bdtheta_k=\opt)\\

 = & \PP(\bigcup\limits_{m=1}\limits^{\infty} \bigcap\limits_{k=m}\limits^{\infty}\{\|\bdtheta_k-\opt\|\leq \delta\} |\limk \bdtheta_k=\opt) = 1.
\end{array}
$$
We can easily know that for any $m,n\ge 1$, $\mathcal{A}_{m,n-1} \in \fff_{n-1}$ and $\lim\limits_n\daone_{\mathcal{A}_{m,n}} = \daone_{\mathcal{A}_m}$ everywhere. Since $\mathcal{A}_{m,n-1}\subset \{\|\bdtheta_{n-1}-\opt\|\leq \delta\}$, \ref{A2}\ref{A2b2} implies that $\sup_{n}\EE[|\bde_n|^2\daone_{\mathcal{A}_{m,n-1}}]<\infty$. Hence, \ref{aver1}\ref{aver1b} is also valid.

The last thing we need to show is that on $\{\limk \bdtheta_k=\opt\}$, $\frac{1}{\sqrt{N+1}}\sum\limits_{n=0}\limits^{N}\bde_{n+1} \rightarrow_d N(0,U )$, i.e. for any $\bdt\in \RR^d$,

$$
\resizebox{.92\hsize}{!}{$
\lim\limits_{N\rightarrow \infty}\EE\left[\daone\{\limk \bdtheta_k=\opt\} exp(i \bdt^T \frac{1}{\sqrt{N+1}}\sum\limits_{n=0}\limits^{N}\bde_{n+1})\right] = \PP\left(\limk \bdtheta_k=\opt\right) exp\left(-\frac{1}{2}\bdt^T U  \bdt\right),
$}
$$

which is equivalent to show

\begin{equation}
\resizebox{.92\hsize}{!}{$
\lim\limits_{N\rightarrow \infty}\EE\left[ exp(i \bdt^T \frac{1}{\sqrt{N+1}}\sum\limits_{n=0}\limits^{N}(\bde_{n+1}\daone\{\limk \bdtheta_k=\opt\}))\right] = \EE\left[ exp(-\frac{1}{2}\bdt^T(U  \daone\{\limk \bdtheta_k=\opt\})\bdt)\right].
$}
\label{f15eq1}
\end{equation}

Noticing that $\{\bde_{n} \daone\{\limk \bdtheta_k=\opt\}\}_{n\ge 1}$ is not a martingale difference sequence, we can't directly apply a martingale CLT to show (\ref{f15eq1}). To get around this problem, we find a real martingale difference sequence to approximate it. Based on Lemma \ref{lemma:setapprox1}, there exists a sequence of sets $\{B_n\}_{n\ge 0}$ such that $B_n\in \fff_n, \forall n$ and $\daone_{B_n}\rightarrow \daone\{ \limk \bdtheta_k=\opt\}$ a.s.. Based on Lemma \ref{lemma:setapprox2}, we know that as $N$ goes to infinity, 
$$
\kappa_N (\bdt)\triangleq \frac{1}{\sqrt{N+1}}\sum\limits_{n=0}\limits^N(\bdt^T \bde_{n+1}(\daone_{B_n} - \daone\{ \limk \bdtheta_k=\opt\})) \rightarrow 0
$$
a.s.. Therefore, based on Lemma \ref{lemma:asconverge1}, to show (\ref{f15eq1}), it suffices to show

\begin{equation}
\lim\limits_{N\rightarrow \infty}\EE\left[ exp(i \bdt^T \frac{1}{\sqrt{N+1}}\sum\limits_{n=0}\limits^{N}(\bde_{n+1}\daone_{B_n}))\right] = \EE\left[ exp(-\frac{1}{2}\bdt^T(U  \daone\{\limk \bdtheta_k=\opt\})\bdt)\right].
\label{f15eq2}
\end{equation}
Because $\EE_n[\bde_{n+1}\daone_{B_n}] = \daone_{B_n}\EE_n[\bde_{n+1}] = 0$, $\{\bde_{n+1}\daone_{B_n}\}_{n\ge 0}$ is a $\{\fff_n\}$-adapted martingale difference sequence. To apply the martingale CLT, we still need to verify the Lindeberg condition and variance condition.

\noindent\textbf{Lindeberg Condition}

Need to show for any $\varepsilon >0$,
\begin{equation}
\frac{1}{N+1}\sum\limits_{n=0}\limits^{N}\EE_n\left[ |\bdt^T\bde_{n+1}|^2 \daone_{B_n} \daone\left\{ \left| \frac{1}{\sqrt{N+1}}\bdt^T\bde_{n+1} \daone_{B_n} \right| \ge \varepsilon \right\} \right] \mathop{\rightarrow}\limits^{P} 0.
\label{f15eq3}
\end{equation}

It suffices to show
\begin{equation}
\frac{1}{N+1}\sum\limits_{n=0}\limits^{N}\daone_{B_n} \EE_n\left[ |\bdt^T\bde_{n+1}|^2  \daone\left\{ \|\bde_{n+1}\|\ge \frac{\sqrt{N+1}\varepsilon}{\|\bdt\|} \right\} \right] \mathop{\rightarrow}\limits^{P} 0.
\label{f15eq4}
\end{equation}

Based on Lemma \ref{lemma:setapprox2}, it is equivalent to show
\begin{equation}
\frac{1}{N+1} \daone\{\limk \bdtheta_k=\opt\} \sum\limits_{n=0}\limits^{N} \EE_n\left[ |\bdt^T\bde_{n+1}|^2  \daone\left\{ \|\bde_{n+1}\|\ge \frac{\sqrt{N+1}\varepsilon}{\|\bdt\|} \right\} \right] \mathop{\rightarrow}\limits^{P} 0.
\label{f15eq5}
\end{equation}

If $\omega \in \{\limk \bdtheta_k=\opt\}$, there exists $\zeta(\omega)\in \ZZ_+$ such that when $n> \zeta(\omega)$, $\|\bdtheta_n-\opt\|\leq \delta$. If $\omega \notin  \{\limk \bdtheta_k=\opt\}$, we simply let $\zeta(\omega)=1$. Therefore, $\zeta$ is everywhere finite. Then, we have
\begin{equation}
\begin{array}{rl}
& \frac{1}{N+1} \daone\{\limk \bdtheta_k=\opt\} \sum\limits_{n=0}\limits^{N} \EE_n\left[ |\bdt^T\bde_{n+1}|^2  \daone\left\{ \|\bde_{n+1}\|\ge \frac{\sqrt{N+1}\varepsilon}{\|\bdt\|} \right\} n\right]\\

= &\frac{1}{N+1} \daone\{\limk \bdtheta_k=\opt\} \sum\limits_{n=0}\limits^{N} \EE_n\left[ |\bdt^T\bde_{n+1}|^2  \daone\left\{ \|\bde_{n+1}\|\ge \frac{\sqrt{N+1}\varepsilon}{\|\bdt\|} \right\} \daone\{n\leq \zeta\}\right]\\

& + \frac{1}{N+1} \daone\{\limk \bdtheta_k=\opt\} \sum\limits_{n=0}\limits^{N} \EE_n\left[ |\bdt^T\bde_{n+1}|^2  \daone\left\{ \|\bde_{n+1}\|\ge \frac{\sqrt{N+1}\varepsilon}{\|\bdt\|} \right\} \daone\{n> \zeta\} \right],\\
\end{array}
\label{f15eq6}
\end{equation}
where the first term on the right-hand side of (\ref{f15eq6}) can be bounded by
$$
\frac{1}{N+1} \daone\{\limk \bdtheta_k=\opt\} \sum\limits_{n=0}\limits^{\zeta} \EE_n\left[ |\bdt^T\bde_{n+1}|^2  \right],
$$
which goes to 0 a.s., and the second term can be bounded by
\begin{equation}
\frac{\|\bdt\|^2}{N+1} \sum\limits_{n=0}\limits^{N} \EE_n\left[ \|\bde_{n+1}\|^2  \daone\left\{ \|\bde_{n+1}\|\ge \frac{\sqrt{N+1}\varepsilon}{\|\bdt\|} \right\} \daone\{\|\bdtheta_n-\opt\|\leq \delta\} \right].
\label{f15eq7}
\end{equation}
Notice that
$$
\begin{array}{rl}
& \EE\left[ \frac{\|\bdt\|^2}{N+1} \sum\limits_{n=0}\limits^{N} \EE_n\left[ \|\bde_{n+1}\|^2  \daone\left\{ \|\bde_{n+1}\|\ge \frac{\sqrt{N+1}\varepsilon}{\|\bdt\|} \right\} \daone\{\|\bdtheta_n-\opt\|\leq \delta\} \right] \right]\\

\leq & \|\bdt\|^2 \sup\limits_{0\leq n\leq N} \EE\left[ \|\bde_{n+1}\|^2  \daone\left\{ \|\bde_{n+1}\|\ge \frac{\sqrt{N+1}\varepsilon}{\|\bdt\|} \right\} \daone\{\|\bdtheta_n-\opt\|\leq \delta\} \right]\\

\leq & \|\bdt\|^2 \sup\limits_{n\ge 0} \EE\left[ \|\bde_{n+1}\|^2  \daone\left\{ \|\bde_{n+1}\|\ge \frac{\sqrt{N+1}\varepsilon}{\|\bdt\|} \right\} \daone\{\|\bdtheta_n-\opt\|\leq \delta\} \right]\\

\rightarrow & 0,
\end{array}
$$
which is based on \ref{A2}\ref{A2b2}. Therefore, we know the quantity given in (\ref{f15eq7}) converges to 0 in probability. Therefore, based on the decomposition (\ref{f15eq6}), we can conclude that the result given in (\ref{f15eq5}) is valid.

\noindent\textbf{Covariance Condition}

Need to show
\begin{equation}
\frac{1}{N+1} \sum\limits_{n=0}\limits^{N} \EE_n\left[|\bdt^T\bde_{n+1}|^2\daone_{B_n}\right] \mathop{\rightarrow}\limits^{P} \bdt^TU \bdt\daone\{\limk \bdtheta_k=\opt\}.
\label{f15eq8}
\end{equation}

Since $\daone_{B_n} \rightarrow \daone\{\limk \bdtheta_k=\opt\} $ a.s., based on Lemma \ref{lemma:setapprox2}, to show (\ref{f15eq8}), it suffices to show
\begin{equation}
\frac{1}{N+1}\daone\{\limk \bdtheta_k=\opt\} \sum\limits_{n=0}\limits^{N} \EE_n\left[|\bdt^T\bde_{n+1}|^2\right] \mathop{\rightarrow}\limits^{P} \bdt^TU \bdt\daone\{\limk \bdtheta_k=\opt\},
\label{f15eq9}
\end{equation}
which can be immediately obtained according to \ref{A2}\ref{A2c}.

With the above conditions being verified, we can use the martingale CLT result provided in Corollary 3.1 in \cite{hall1980martingale} to complete our proof.

\end{proof}
\subsection{Proof of Lemma \ref{basiclemma1}}
\begin{proof}
Based on our assumption, $\nabla f(\bdtheta_n;\bdy_{n+1})$ is a noisy estimator of the true gradient $\nabla F(\bdtheta_n)$ such that $\EE_n[\nabla f(\bdtheta_n;\bdy_{n+1})] = \nabla F(\bdtheta_n)$. Therefore, the update rule (\ref{mainupdate}) can be reformulated as
$$
\bdtheta_{n+1} = \bdtheta_n - \gamma_{n+1}\nabla F(\bdtheta_n) + \gamma_{n+1}(\nabla F(\bdtheta_n) - \nabla f(\bdtheta_n;\bdy_{n+1})) + \gamma_{n+1}\cdot 0,
$$
which can fit into the general update given in (\ref{generalsa}) with $\bde_{n+1} \triangleq \nabla F(\bdtheta_n) - \nabla f(\bdtheta_n;\bdy_{n+1})$ and $\bdr_{n+1}\equiv0$. Hence, we just need to verify condition \ref{A1}-\ref{A4} and use Lemma \ref{thm:fort15} or \ref{thm:fort15var} to justify Lemma \ref{basiclemma1}. We can easily see that conditions \ref{A1}, \ref{A2}\ref{A2a}, \ref{A3} and \ref{A4} are trivially hold under conditions given in Lemma \ref{basiclemma1}.

We \underline{firstly} show that under condition \ref{LC1}, condition \ref{A2}\ref{A2b1} is valid. We have
\begin{equation}
\begin{array}{rl}
& \EE\left[\|\bde_{n+1}\|^{2+\tau}\daone\{\|\bdtheta_n - \opt\|\leq \delta\}\right] \\

= &\EE\left[\|\nabla F(\bdtheta_n) - \nabla f(\bdtheta_n;\bdy_{n+1})\|^{2+\tau}\daone\{\|\bdtheta_n - \opt\|\leq \delta\}\right]\\

\leq & C_{\tau} \EE\left[ \left\{ \|\nabla F(\opt)-\nabla f(\opt;\bdy_{n+1})\|^{2+\tau} + \|\nabla F(\bdtheta_n) - \nabla F(\opt)\|^{2+\tau} \right. \right.\\

& \left.\left.+ \|\nabla f(\bdtheta_n;\bdy_{n+1}) - \nabla f(\opt;\bdy_{n+1})\|^{2+\tau} \right\} \daone\{\|\bdtheta_n - \opt\|\leq \delta\}\right],\\
\end{array}
\label{bleq1}
\end{equation}
where the last step is based on the Jensen's inequality and $C_{\tau}$ is a constant dependent solely on $\tau$. Based on condition \ref{LC1}, we have
\begin{equation}
\EE\left[ \|\nabla f(\bdtheta_n;\bdy_{n+1}) - \nabla f(\opt;\bdy_{n+1})\|^{2+\tau}\daone\{\|\bdtheta_n - \opt\|\leq \delta\}\right] \leq \eta_1(\delta),
\label{bleq2}
\end{equation}
which is a finite constant. Based on the Jensen's inequality, we have
\begin{equation}
\begin{array}{rl}
& \EE\left[ \|\nabla F(\bdtheta_n) - \nabla F(\opt)\|^{2+\tau}\daone\{\|\bdtheta_n - \opt\|\leq \delta\}\right]\\

\leq & \EE\left[ \|\nabla f(\bdtheta_n;\bdy_{n+1}) - \nabla f(\opt;\bdy_{n+1})\|^{2+\tau}\daone\{\|\bdtheta_n - \opt\|\leq \delta\}\right] \leq \eta_1(\delta).
\end{array}
\label{bleq3}
\end{equation}
In addition, under condition \ref{LC1}, we have
\begin{equation}
\begin{array}{rl}
& \EE\left[ \|\nabla F(\opt) - \nabla f(\opt;\bdy_{n+1})\|^{2+\tau}\daone\{\|\bdtheta_n - \opt\|\leq \delta\}\right] \\
\leq & \EE \|\nabla F(\opt) - \nabla f(\opt;\bdy_{n+1})\|^{2+\tau}\\
<&\infty.
\end{array}
\label{bleq4}
\end{equation}
Combining (\ref{bleq1})-(\ref{bleq4}), we know that $\EE\left[\|\bde_{n+1}\|^{2+\tau}\daone\{\|\bdtheta_n - \opt\|\leq \delta\}\right]$ is upper bounded by a finite constant independent of $n$ and hence condition \ref{A2}\ref{A2b1} is verified.

\underline{Secondly}, we show that under condition \ref{LC2}, condition \ref{A2}\ref{A2b2} is valid. In a same way shown in (\ref{bleq1})-(\ref{bleq4}), we have that $\EE\left[\|\bde_{n+1}\|^{2}\daone\{\|\bdtheta_n - \opt\|\leq \delta\}\right]$ is upper bounded by a finite constant independent of $n$, denoted by $C_u$. Next, we're to show
\begin{equation}
\lim\limits_{M\rightarrow \infty}\sup_{n\ge 0} \EE\left[\|\bde_{n+1}\|^{2}\daone\{\|\bdtheta_n-\opt\|\leq \delta\}\daone\{\|\bde_{n+1}\|\ge M\}\right]\rightarrow 0.
\label{bleq5}
\end{equation}
We have
\begin{equation}
\begin{array}{rl}
& \EE\left[\|\bde_{n+1}\|^{2}\daone\{\|\bdtheta_n-\opt\|\leq \delta\}\daone\{\|\bde_{n+1}\|\ge M\}\right]\\

= & \EE\left[\|\nabla F(\bdtheta_n) - \nabla f(\bdtheta_n;\bdy_{n+1})\|^{2}\daone\{\|\bdtheta_n-\opt\|\leq \delta\}\daone\{\|\bde_{n+1}\|\ge M\}\right]\\

\leq & 3 \EE\left[ \left\{ \|\nabla F(\opt)-\nabla f(\opt;\bdy_{n+1})\|^{2} + \|\nabla F(\bdtheta_n) - \nabla F(\opt)\|^{2} \right. \right.\\

& \left.\left.+ \|\nabla f(\bdtheta_n;\bdy_{n+1}) - \nabla f(\opt;\bdy_{n+1})\|^{2} \right\} \daone\{\|\bdtheta_n - \opt\|\leq \delta\}\daone\{\|\bde_{n+1}\|\ge M\}\right].\\
\end{array}
\label{bleq6}
\end{equation}
We separately bound 3 terms in the last part of (\ref{bleq6}). For the first term, we have
\begin{equation}
\resizebox{.92\hsize}{!}{$
\begin{array}{rl}
& \EE\left[\|\nabla F(\opt)-\nabla f(\opt;\bdy_{n+1})\|^{2} \daone\{\|\bdtheta_n - \opt\|\leq \delta\}\daone\{\|\bde_{n+1}\|\ge M\}\right]\\

= & \EE\left[\|\nabla F(\opt)-\nabla f(\opt;\bdy_{n+1})\|^{2} \daone\{\|\bdtheta_n - \opt\|\leq \delta\}\daone\{\|\bde_{n+1}\|\ge M, \|\nabla F(\opt)-\nabla f(\opt;\bdy_{n+1})\|\ge \log M\}\right]\\

& + \EE\left[\|\nabla F(\opt)-\nabla f(\opt;\bdy_{n+1})\|^{2} \daone\{\|\bdtheta_n - \opt\|\leq \delta\}\daone\{\|\bde_{n+1}\|\ge M, \|\nabla F(\opt)-\nabla f(\opt;\bdy_{n+1})\|< \log M\}\right]\\

\leq & \EE\left[\|\nabla F(\opt)-\nabla f(\opt;\bdy_{n+1})\|^{2} \daone\{ \|\nabla F(\opt)-\nabla f(\opt;\bdy_{n+1})\|\ge \log M\}\right]\\

& + (\log M)^2\EE\left[ \daone\{\|\bdtheta_n - \opt\|\leq \delta\}\daone\{\|\bde_{n+1}\|- \|\nabla F(\opt)-\nabla f(\opt;\bdy_{n+1})\|\ge M- \log M\}\right],\\
\end{array}
\label{bleq7}
$}
\end{equation}
where the first term is essentially 
$$
\EE_{\bdy\sim \mathcal{P}_Y}\left[\|\nabla F(\opt)-\nabla f(\opt;\bdy)\|^{2} \daone\{ \|\nabla F(\opt)-\nabla f(\opt;\bdy)\|\ge \log M\}\right],
$$ 
which is independent of $n$, and the second term in the last part of (\ref{bleq7}) can be handled as follows
$$
\begin{array}{rl}
& \EE\left[ \daone\{\|\bdtheta_n - \opt\|\leq \delta\}\daone\{\|\bde_{n+1}\|- \|\nabla F(\opt)-\nabla f(\opt;\bdy_{n+1})\|\ge M- \log M\}\right]\\

\leq & \frac{2}{(M-\log M)^2}  \EE\left[ \daone\{\|\bdtheta_n - \opt\|\leq \delta\} ( \|\bde_{n+1}\|^2+ \|\nabla F(\opt)-\nabla f(\opt;\bdy_{n+1})\|^2)\right]\\

\leq & \frac{2}{(M-\log M)^2}(C_u + \EE_{\bdy\sim \mathcal{P}_Y}\left[\|\nabla F(\opt)-\nabla f(\opt;\bdy)\|^{2}\right]),\\
\end{array}
$$
where $C_u$ is a constant mentioned right before equation (\ref{bleq5}). Therefore, the supremum of the first term in the last part of (\ref{bleq6}) can be upper bounded by
\begin{equation}
\begin{array}{rl}
& \sup\limits_{n\ge 0}\EE\left[\|\nabla F(\opt)-\nabla f(\opt;\bdy_{n+1})\|^{2} \daone\{\|\bdtheta_n - \opt\|\leq \delta\}\daone\{\|\bde_{n+1}\|\ge M\}\right]\\

\leq & \EE_{\bdy\sim \mathcal{P}_Y}\left[\|\nabla F(\opt)-\nabla f(\opt;\bdy)\|^{2} \daone\{ \|\nabla F(\opt)-\nabla f(\opt;\bdy)\|\ge \log M\}\right]\\

& + \frac{2(\log M)^2}{(M-\log M)^2}(C_u + \EE_{\bdy\sim \mathcal{P}_Y}\left[\|\nabla F(\opt)-\nabla f(\opt;\bdy)\|^{2}\right]).\\
\end{array}
\label{bleq8}
\end{equation}
Likewise, we have
\begin{equation}
\resizebox{.92\hsize}{!}{$
\begin{array}{rl}
& \EE\left[\|\nabla f(\bdtheta_n;\bdy_{n+1})-\nabla f(\opt;\bdy_{n+1})\|^{2} \daone\{\|\bdtheta_n - \opt\|\leq \delta\}\daone\{\|\bde_{n+1}\|\ge M\}\right]\\

\leq & \EE\left[\|\nabla f(\bdtheta_n;\bdy_{n+1})-\nabla f(\opt;\bdy_{n+1})\|^{2} \daone\{\|\bdtheta_n - \opt\|\leq \delta\}\daone\{\|\bde_{n+1}\|\ge M, \|\nabla f(\bdtheta_n;\bdy_{n+1})-\nabla f(\opt;\bdy_{n+1})\|\ge \log M\}\right]\\

& + \EE\left[\|\nabla f(\bdtheta_n;\bdy_{n+1})-\nabla f(\opt;\bdy_{n+1})\|^{2} \daone\{\|\bdtheta_n - \opt\|\leq \delta\}\daone\{\|\bde_{n+1}\|\ge M, \|\nabla f(\bdtheta_n;\bdy_{n+1})-\nabla f(\opt;\bdy_{n+1})\|< \log M\}\right]\\

\leq & \eta_{\delta}(\log M) + (\log M)^2 \EE\left[ \daone\{\|\bdtheta_n - \opt\|\leq \delta\} \daone\{ \|\bde_{n+1}\| - \|\nabla f(\bdtheta_n;\bdy_{n+1})-\nabla f(\opt;\bdy_{n+1})\| \ge M-\log M\}\right]\\

\leq & \eta_{\delta}(\log M) + \frac{2(\log M)^2}{(M-\log M)^2} (C_u + \eta_2(\delta)),\\
\end{array}
\label{bleq9}
$}
\end{equation}
which implies
\begin{equation}
\begin{array}{rl}
& \sup\limits_{n\ge 0} \EE\left[\|\nabla f(\bdtheta_n;\bdy_{n+1})-\nabla f(\opt;\bdy_{n+1})\|^{2} \daone\{\|\bdtheta_n - \opt\|\leq \delta\}\daone\{\|\bde_{n+1}\|\ge M\}\right]\\
 \leq &\eta_{\delta}(\log M) + \frac{2(\log M)^2}{(M-\log M)^2} (C_u + \eta_2(\delta)).\\
 \end{array}
\label{bleq10}
\end{equation}
To deal with the middle term in the last part of (\ref{bleq6}), we have
\begin{equation}
\begin{array}{rl}
& \EE\left[\|\nabla F(\bdtheta_n)-\nabla F(\opt)\|^{2} \daone\{\|\bdtheta_n - \opt\|\leq \delta\}\daone\{\|\bde_{n+1}\|\ge M\}\right]\\

\leq & \EE\left[ \EE_n\left[\|\nabla f(\bdtheta_n;\bdy_{n+1})-\nabla f(\opt;\bdy_{n+1})\|^{2}\right] \daone\{\|\bdtheta_n - \opt\|\leq \delta\}\daone\{\|\bde_{n+1}\|\ge M\}\right]\\

\leq & \EE\left[ \eta_2(\delta) \daone\{\|\bdtheta_n - \opt\|\leq \delta\}\daone\{\|\bde_{n+1}\|\ge M\}\right]\\

\leq & \eta_2(\delta) \frac{1}{M^2} \EE\left[ \|\bde_{n+1}\|^2  \daone\{\|\bdtheta_n - \opt\|\leq \delta\}\right] \leq \frac{1}{M^2}C_u \eta_2(\delta),\\
\end{array}
\label{bleq11}
\end{equation}
which implies
\begin{equation}
\sup\limits_{n\ge0}  \EE\left[\|\nabla F(\bdtheta_n)-\nabla F(\opt)\|^{2} \daone\{\|\bdtheta_n - \opt\|\leq \delta\}\daone\{\|\bde_{n+1}\|\ge M\}\right] \leq \frac{1}{M^2}C_u \eta_2(\delta).
\label{bleq12}
\end{equation}
Putting results (\ref{bleq6}), (\ref{bleq8}), (\ref{bleq10}) and (\ref{bleq12}) together, we know (\ref{bleq5}) is valid.

\underline{Thirdly}, we show that under condition \ref{LC2}, condition \ref{A2}\ref{A2c} is valid. We have decomposition
\begin{equation}
\resizebox{.92\hsize}{!}{$
\begin{array}{rl}
& \EE_n\left[ \bde_{n+1}\bde_{n+1}^T\right]\\

 =& \EE_n\left[ \nabla f(\opt;\bdy_{n+1}) \nabla f(\opt;\bdy_{n+1})^T\right] + \EE_n\left[\bde_{n+1}\bde_{n+1}^T - \nabla f(\opt;\bdy_{n+1}) \nabla f(\opt;\bdy_{n+1})^T\right] \\
 
 = & U  + \EE_n\left[\bde_{n+1}\bde_{n+1}^T - \nabla f(\opt;\bdy_{n+1}) \nabla f(\opt;\bdy_{n+1})^T\right]\\
 \triangleq &U  + D_n.\\
 \end{array}
 $}
\label{bleq13}
\end{equation}
Therefore, our job is to show 
\begin{equation}
D_n\daone\{\limk \bdtheta_k = \opt\} \rightarrow 0\  a.s.. 
 \label{bleq14}
\end{equation}
For simplicity, we denote $\nabla \triangleq \nabla f(\opt;\bdy_{n+1})$ here. We have
\begin{equation}
\resizebox{.92\hsize}{!}{$
\begin{array}{rcl}
\|D_n\| & = & \| \EE_n\left[\bde_{n+1}\bde_{n+1}^T - \nabla  \nabla^T\right]\|\\

& \leq & \EE_n\left[\|\bde_{n+1}\bde_{n+1}^T - \nabla\nabla^T\| \right]\\

& \leq & \EE_n\left[ \| \bde_{n+1}\bde_{n+1}^T - \bde_{n+1}\nabla^T\| n\right] + \EE_n\left[ \| \bde_{n+1}\nabla^T - \nabla\nabla^T\|\right]\\

& \leq & \left[ \EE_n\left[ \|\bde_{n+1}\|^2\right] \cdot \EE_n\left[ \| \bde_{n+1} - \nabla\|^2 \right]\right]^{\frac{1}{2}} + \left[ \EE_n\left[ \|\nabla\|^2\right] \cdot \EE_n\left[ \| \bde_{n+1} - \nabla\|^2 \right]\right]^{\frac{1}{2}},
 \end{array}
 \label{bleq15}
 $}
\end{equation}
where the last step is based on the Cauchy-Schwartz inequality. Considering that $\EE_n\left[ \|\nabla\|^2\right] = \EE_{\bdy\sim \mathcal{P}_Y}\left[ \|\nabla f(\opt;\bdy)\|^2 \right]$, which is a finite constant, based on (\ref{bleq15}), to show (\ref{bleq14}), it suffices to show
\begin{equation}
\EE_n\left[ \|\bde_{n+1}\|^2\right] \daone\{\limk \bdtheta_k = \opt\}\ is\ O(1),
\label{bleq16}
\end{equation}
(recall that $O(1)$ means finite almost everywhere) and
\begin{equation}
\EE_n\left[ \| \bde_{n+1} - \nabla\|^2 \right] \daone\{\limk \bdtheta_k = \opt\} \rightarrow 0\ a.s..
\label{bleq17}
\end{equation}
For any $\delta' <\delta$, let $A_n^{\delta'} = \{\|\bdtheta_m-\opt\|\leq \delta',m\ge n\}$. We have
\begin{equation}
\resizebox{.92\hsize}{!}{$
\begin{array}{rl}
&\EE_n\left[ \| \bde_{n+1} - \nabla\|^2 \right] \daone\{\limk \bdtheta_k = \opt\}\\

=& \EE_n\left[ \| \bde_{n+1} - \nabla\|^2 \right] \daone\{\limk \bdtheta_k = \opt\}(\daone_{A_n^{\delta'}} + \daone_{(A_n^{\delta'})^c})\\

\leq & 2\EE_n\left[\| \nabla F(\bdtheta_n)-\nabla F(\opt)\|^2 \right] \daone_{A_n^{\delta'}} + 2\EE_n\left[\| \nabla f(\bdtheta_n;\bdy_{n+1})-\nabla f(\opt;\bdy_{n+1})\|^2 \right] \daone_{A_n^{\delta'}}\\

& + \EE_n\left[ \| \bde_{n+1} - \nabla\|^2 \right] \daone\{\limk \bdtheta_k = \opt\}\daone_{(A_n^{\delta'})^c}\\

\leq & 4\EE_n\left[\| \nabla f(\bdtheta_n;\bdy_{n+1})-\nabla f(\opt;\bdy_{n+1})\|^2 \right] \daone_{A_n^{\delta'}} \\

& + \EE_n\left[ \| \bde_{n+1} - \nabla\|^2 \right] \daone\{\limk \bdtheta_k = \opt\}\daone_{(A_n^{\delta'})^c}\\

\leq & 4\EE_n\left[\| \nabla f(\bdtheta_n;\bdy_{n+1})-\nabla f(\opt;\bdy_{n+1})\|^2 \right] \daone_{A_n^{\delta'}} + o(1)\\

\leq & 4\eta_2(\delta') + o(1),
\end{array}
\label{bleq18}
$}
\end{equation}
where the 3rd step is based on the Jensen's inequality and the second to the last step is based on the fact that $\EE_n\left[ \| \bde_{n+1} - \nabla\|^2 \right]<\infty$ a.s. for any $n$, and there exists an a.s. finite $\zeta'$ such that $\daone\{\limk \bdtheta_k = \opt\}\daone_{(A_n^{\delta'})^c}=0$ when $n>\zeta'$. Since $\delta'$ can be arbitrarily small, (\ref{bleq18}) indicates that (\ref{bleq17}) is valid. At last, it is easy to see that (\ref{bleq16}) is valid based on (\ref{bleq17}) and the fact that $\EE_n\left[\|\nabla\|^2\right]$ is a finite constant.
\end{proof}

\subsection{Proof of Lemma \ref{lemma:pre}}
\begin{proof}
Let us firstly show that
\begin{equation}
\sqrt{N}(\bar{\bdtheta}_N-\opt)\daone\{\limk \bdtheta_k=\opt\} =  -\frac{A^{-1}}{\sqrt{N+1}} \sum\limits_{n=0}\limits^N  \nabla f(\opt;\bdy_{n+1})\daone\{\limk \bdtheta_k=\opt\} + o_p(1).
\label{exp1}
\end{equation}
To show (\ref{exp1}), based on Lemma \ref{basiclemma1}, it suffices to show
$$
\Delta_1 \triangleq  \frac{1}{\sqrt{N}} \sum\limits_{n=0}\limits^N \left\{ \nabla F(\bdtheta_n) - \nabla f(\bdtheta_n;\bdy_{n+1}) +\nabla f(\opt;\bdy_{n+1})\right\} \daone\{\limk \bdtheta_k=\opt\}
$$
is $o_p(1)$. Based on Lemma \ref{lemma:setapprox1}, there exists a sequence of sets $\{B_n\}_{n\ge0}$ such that $B_n\in \fff_n$ for any $n\ge0$ and $\daone_{B_n}\rightarrow \daone\{\limk \bdtheta_k = \opt\}$. Then, based on Lemma \ref{lemma:setapprox2}, to show $\Delta_1$ is $o_p(1)$, it suffices to show
$$
\Delta_2 \triangleq  \frac{1}{\sqrt{N}} \sum\limits_{n=0}\limits^N \left\{ \nabla F(\bdtheta_n) - \nabla f(\bdtheta_n;\bdy_{n+1}) +\nabla f(\opt;\bdy_{n+1})\right\} \daone_{B_n}
$$
is $o_p(1)$. Note that $\EE_n\left[ \nabla F(\bdtheta_n) - \nabla f(\bdtheta_n;\bdy_{n+1}) +\nabla f(\opt;\bdy_{n+1})\right]=0$. For any $\bdt\in \RR^d$, we will show that $\bdt^T\Delta_2 \rightarrow 0$ in distribution by the martingale CLT and hence obtain $\Delta_2 = o_p(1)$. We let $\nabla_n \triangleq  \nabla F(\bdtheta_n) - \nabla f(\bdtheta_n;\bdy_{n+1}) +\nabla f(\opt;\bdy_{n+1})$.\\
\noindent\textbf{Variance Condition}\\
We need to show
\begin{equation}
V_N^2 \triangleq \frac{1}{N}\sum\limits_{n=0}\limits^N \EE_n\left[\left\{ \bdt^T \left\{ \nabla F(\bdtheta_n) - \nabla f(\bdtheta_n;\bdy_{n+1}) +\nabla f(\opt;\bdy_{n+1})\right\} \right\}^2 \daone_{B_n}\right]  \mathop{\rightarrow}\limits^P 0.
\label{bbleq1}
\end{equation}
Based on Lemma \ref{lemma:setapprox2}, it suffices to show
\begin{equation}
\tilde{V}_N^2 \triangleq \frac{1}{N}\sum\limits_{n=0}\limits^N \EE_n\left[\| \nabla F(\bdtheta_n) - \nabla f(\bdtheta_n;\bdy_{n+1}) +\nabla f(\opt;\bdy_{n+1})\|^2  \right]  \daone\{\limk \bdtheta_k=\opt\}\mathop{\rightarrow}\limits^P 0.
\label{bbleq2}
\end{equation}
For any $0<\delta'<\delta$, and for $\omega \in \{\limk \bdtheta_k=\opt\}$, there exists finite $\tilde{\zeta}(\omega)$ such that when $n>\tilde{\zeta}(\omega)$, $\|\bdtheta_n-\opt\| \leq \delta'$. For $\omega \notin \{\limk \bdtheta_k=\opt\}$, we simply let $\tilde{\zeta}(\omega) = 1$. Then, we have
\begin{equation}
\resizebox{.92\hsize}{!}{$
\begin{array}{rcl}
\tilde{V}_N^2 &= &\frac{1}{N}\sum\limits_{n=0}\limits^N \EE_n\left[ \|\nabla_n\|^2  \right] \left\{ \daone\{n\leq \tilde{\zeta} \} + \daone\{n> \tilde{\zeta} \}\right\} \daone\{\limk \bdtheta_k=\opt\}\\

& \leq & \daone\{\limk \bdtheta_k=\opt\} \frac{1}{N}\sum\limits_{n=0}\limits^{N\wedge \tilde{\zeta}} \EE_n\left[ \|\nabla_n\|^2  \right] + \frac{1}{N}\sum\limits_{n=0}\limits^{N} \EE_n\left[ \|\nabla_n\|^2  \right]\daone\{\|\bdtheta_n-\opt\|\leq \delta'\}\\

& \leq &  \frac{1}{N}\sum\limits_{n=0}\limits^{ \tilde{\zeta}} \EE_n\left[ \|\nabla_n\|^2  \right] + \frac{1}{N}\sum\limits_{n=0}\limits^{N} \EE_n\left[ \|\nabla_n\|^2  \right]\daone\{\|\bdtheta_n-\opt\|\leq \delta'\},\\
\end{array}
\label{bbleq3}
$}
\end{equation}
where the first term in the last expression goes to 0 a.s.. As for the second term, we have
\begin{equation}
\begin{array}{rl}
& \frac{1}{N}\sum\limits_{n=0}\limits^{N} \EE_n\left[ \|\nabla_n\|^2 \right]\daone\{\|\bdtheta_n-\opt\|\leq \delta'\}\\

\leq & \frac{2}{N}\sum\limits_{n=0}\limits^{N} \EE_n\left[ \|\nabla F(\bdtheta_n) - \nabla F(\opt)\|^2  \right]\daone\{\|\bdtheta_n-\opt\|\leq \delta'\}\\

& + \frac{2}{N}\sum\limits_{n=0}\limits^{N} \EE_n\left[ \|\nabla f(\bdtheta_n;\bdy_{n+1}) - \nabla f(\opt;\bdy_{n+1})\|^2  \right]\daone\{\|\bdtheta_n-\opt\|\leq \delta'\}\\

\leq & \frac{4}{N}\sum\limits_{n=0}\limits^{N} \EE_n\left[ \|\nabla f(\bdtheta_n;\bdy_{n+1}) - \nabla f(\opt;\bdy_{n+1})\|^2  \right]\daone\{\|\bdtheta_n-\opt\|\leq \delta'\}\\

\leq & \frac{4}{N}\sum\limits_{n=0}\limits^{N}\eta_2(\delta') = 4 \eta_2(\delta'),\\
\end{array}
\label{bbleq4}
\end{equation}
where the 2nd step is based on the Cauchy-Schwartz inequality and the 3rd step is based on condition \ref{LC2}. Since $\delta'$ can be arbitrarily small, we can conclude that $\tilde{V}_N^2 = o(1)$ and hence (\ref{bbleq2}) is valid.

\noindent\textbf{Lindeberg Condition}\\
For any $\varepsilon>0$, we need to show
\begin{equation}
\frac{1}{N}\sum\limits_{n=0}\limits^N \EE_n\left[\left\{ \bdt^T \left\{ \nabla F(\bdtheta_n) - \nabla f(\bdtheta_n;\bdy_{n+1}) +\nabla f(\opt;\bdy_{n+1})\right\} \right\}^2 \daone_{B_n} \daone\left\{\left|\frac{t^T\nabla_n}{\sqrt{N}}\right| >\varepsilon \right\} \right]  \mathop{\rightarrow}\limits^P 0,
\end{equation}
which is obviously valid according to (\ref{bbleq1}).

Combining the results above, based on Corollary 3.1 in \cite{hall1980martingale}, we have $\Delta_2 \mathop{\rightarrow}\limits^{d} 0$ and consequently $\Delta_2 \mathop{\rightarrow}\limits^{P} 0$.

Likewise, we also have
\begin{equation}
\resizebox{.92\hsize}{!}{$
\sqrt{N}(\bar{\bdtheta}^*_N-\opt)\daone\{\limk \bdtheta^*_k=\opt\} = - \frac{A^{-1}}{\sqrt{N+1}} \sum\limits_{n=0}\limits^N w_{n+1} \nabla  f(\opt;\bdy_{n+1})\daone\{\limk \bdtheta^*_k=\opt\} + o_p(1).
$}
\label{exp2}
\end{equation}
Based on (\ref{exp1}) and (\ref{exp2}), we can immediately obtain our final conclusion.

\end{proof}



\subsection{Proof of Theorem \ref{bootthm}}
\begin{proof}
For simplicity, we denote
$$
I^* \triangleq \daone\{\limk \bdtheta_k=\opt,\limk \bdtheta^*_k=\opt\}.
$$
We will focus on showing that 

\begin{equation}
\begin{array}{rl}
& \lim\limits_{N\rightarrow \infty} \EE\left[ exp\left(i \frac{\bdt^T A^{-1}}{\sqrt{N+1}} \sum\limits_{n=0}\limits^N (w_{n+1}-1) \nabla f(\opt;\bdy_{n+1})I^*\right)|\bdtheta_0,\bdy_1,\bdy_2,\ldots \right]\\

= & \EE\left[exp\left(-\frac{1}{2}\bdt^T A^{-1} U   A^{-T}\bdt I^*\right)|\bdtheta_0,\bdy_1,\bdy_2,\ldots \right]\\

\end{array}
\label{thm1eq1}
\end{equation}
almost-surely, which combined with (\ref{exp3}) can immediately imply (\ref{exp4}).

Let $\fff = \sigma(\bdtheta_0,\bdy_1,w_1,\bdy_2,w_2,\ldots)$ and suppose that all random variables mentioned are defined on the probability space $(\Omega,\fff,\PP)$. We know that for any constant $c\in \RR$, the following results hold almost surely,
$$
\lim\limits_{N\rightarrow \infty} \sum\limits_{n=0}\limits^{N}\frac{1}{N+1}\|\nabla f(\opt;\bdy_{n+1})\|^2 \rightarrow \EE_{\bdy\sim\mathcal{P}_Y} \|\nabla f(\opt;\bdy)\|^2,
$$
$$
\lim\limits_{N\rightarrow \infty} \sum\limits_{n=0}\limits^{N}\frac{1}{N+1}\|\nabla f(\opt;\bdy_{n+1})\|^2 \daone\left\{ \| \nabla f(\opt;\bdy_{n+1})\| > c(N+1)^{1/4}\right\} \rightarrow 0,
$$
$$
\lim\limits_{N\rightarrow \infty} \sum\limits_{n=0}\limits^{N}\frac{1}{N+1} \nabla f(\opt;\bdy_{n+1}) \nabla f(\opt;\bdy_{n+1})^T \rightarrow U. 
$$
Therefore, to show (\ref{thm1eq1}), it suffices to show that for any $(\dot{\bdtheta}_0,\dot{\bdy}_1,\dot{\bdy}_2,\ldots)$ such that
\begin{equation}
\lim\limits_{N\rightarrow \infty} \sum\limits_{n=0}\limits^{N}\frac{1}{N+1}\|\nabla f(\opt;\dot{\bdy}_{n+1})\|^2 = \EE_{\bdy\sim\mathcal{P}_Y} \|\nabla f(\opt;\bdy)\|^2,
\label{condi1}
\end{equation}
\begin{equation}
\lim\limits_{N\rightarrow \infty} \sum\limits_{n=0}\limits^{N}\frac{1}{N+1}\|\nabla f(\opt;\dot{\bdy}_{n+1})\|^2 \daone\left\{ \| \nabla f(\opt;\dot{\bdy}_{n+1})\| > c(N+1)^{1/4}\right\} = 0,
\label{condi2}
\end{equation}
\begin{equation}
\lim\limits_{N\rightarrow \infty} \sum\limits_{n=0}\limits^{N}\frac{1}{N+1} \nabla f(\opt;\dot{\bdy}_{n+1}) \nabla f(\opt;\dot{\bdy}_{n+1})^T = U ,
\label{condi3}
\end{equation}
we can have
\begin{equation}
\resizebox{.92\hsize}{!}{$
\begin{array}{rl}
& \lim\limits_{N\rightarrow \infty} \EE\left[ exp\left(i \frac{\bdt^T A^{-1}}{\sqrt{N+1}} \sum\limits_{n=0}\limits^N (w_{n+1}-1) \nabla f(\opt;\bdy_{n+1})I^*\right)|\bdtheta_0 = \dot{\bdtheta}_0,\bdy_1 = \dot{\bdy}_1,\bdy_2 = \dot{\bdy}_2,\ldots \right]\\

= & \EE\left[exp\left(-\frac{1}{2}\bdt^T A^{-1} U   A^{-T}\bdt I^*\right)|\bdtheta_0 = \dot{\bdtheta}_0,\bdy_1 = \dot{\bdy}_1,\bdy_2 = \dot{\bdy}_2,\ldots \right].\\

\end{array}
$}
\label{thm1eq2}
\end{equation}

On the set $\dot{S}\triangleq \{\bdtheta_0 = \dot{\bdtheta}_0,\bdy_1 = \dot{\bdy}_1,\bdy_2 = \dot{\bdy}_2,\ldots\}$, consider the conditional probability space $(\Omega_*,\fff_*,\PP_*)$, where $\Omega_* = \{\omega\in \Omega:\bdtheta_0(\omega) = \dot{\bdtheta}_0,\bdy_1(\omega) = \dot{\bdy}_1,\bdy_2(\omega) = \dot{\bdy}_2,\ldots\}$, $\fff_* = \fff\cap \dot{S}$ and
$$
\PP_*(A) = \PP(A|\bdtheta_0 = \dot{\bdtheta}_0,\bdy_1 = \dot{\bdy}_1,\bdy_2 = \dot{\bdy}_2,\ldots),\ \forall A\in \fff_*.
$$
Let $\fff_0^W=\fff$, $\fff_n^W=\sigma(w_1,w_2,\ldots,w_n)$ for $n\ge1$ and ${\fff_n^W}_* = \fff_n^W \cap \dot{S}$ for $n\ge0$. We observe that, conditional on $\dot{S}$,
$$
I^*=\daone\{\limk \bdtheta_k = \opt,\limk \bdtheta^*_k = \opt\} \in {\fff_{\infty}^W}_* \triangleq \sigma\left(\bigcup\limits_{n=1}\limits^{\infty} {\fff_n^W}_*\right).
$$
Therefore, based on Lemma \ref{lemma:setapprox1}, there exists a sequence of sets $\{B_n(\dot{\bdtheta}_0,\dot{\bdy}_1,\dot{\bdy}_2,\ldots)\}_{n\ge 0}$ such that for each $n\ge 0$, $B_n(\dot{\bdtheta}_0,\dot{\bdy}_1,\dot{\bdy}_2,\ldots)\in {\fff_n^W}_*$ and
$$
\daone_{B_n(\dot{\bdtheta}_0,\dot{\bdy}_1,\dot{\bdy}_2,\ldots)}\rightarrow I^*=\daone\{\limk \bdtheta_k = \opt,\limk \bdtheta^*_k = \opt\},\ \PP_*{\text -}a.s.
$$
For simplicity, we denote $\EE_*\left[\cdot\right] \triangleq \EE\left[\cdot|\bdtheta_0 = \dot{\bdtheta}_0,\bdy_1 = \dot{\bdy}_1,\bdy_2 = \dot{\bdy}_2,\ldots\right]$. Based on Lemma \ref{lemma:setapprox2} and Lemma \ref{lemma:asconverge1}, to show (\ref{thm1eq2}), it is sufficient to show
\begin{equation}
\begin{array}{rl}
& \lim\limits_{N\rightarrow \infty} \EE_*\left[ exp\left(i \frac{\bdt^T A^{-1}}{\sqrt{N+1}} \sum\limits_{n=0}\limits^N (w_{n+1}-1) \nabla f(\opt;\bdy_{n+1})\daone_{B_n}\right) \right]\\

= & \EE_*\left[exp\left(-\frac{1}{2}\bdt^T A^{-1} U   A^{-T}\bdt I^*\right) \right].\\

\end{array}
\label{thm1eq3}
\end{equation}
We again use Corollary 3.1 in \cite{hall1980martingale} to show (\ref{thm1eq3}). Firstly, we verify the \textbf{martingale difference condition}: for any $0\leq n\leq N$,
\begin{equation}
\begin{array}{rl}
& \EE_*\left[\frac{\bdt^T A^{-1}}{\sqrt{N+1}}  (w_{n+1}-1) \nabla f(\opt;\bdy_{n+1})\daone_{B_n} |{\fff_n^W}_* \right] \\ 

= & \frac{\bdt^T A^{-1}}{\sqrt{N+1}}  \EE_*\left[ (w_{n+1}-1)|{\fff_n^W}_* \right] \nabla f(\opt;\dot{\bdy}_{n+1})\daone_{B_n}\\

= & 0,
\end{array}
\label{thm1eq4}
\end{equation}
which is due to the fact that $w_{n+1}$ is of mean 1 and independent of ${\fff_n^W}_*$. 

Secondly, we verify the \textbf{Lindeberg condition}: for any $\varepsilon >0$, if we denote 
$$
m_n\triangleq \frac{\bdt^T A^{-1}}{\sqrt{N+1}} (w_{n+1}-1) \nabla f(\opt;\bdy_{n+1})\daone_{B_n},
$$
we have
\begin{equation}
\resizebox{.92\hsize}{!}{$
\begin{array}{rl}
& \sum\limits_{n=0}\limits^N \EE_*\left[ m_n^2 \daone\left\{|m_n|>\varepsilon\right\} |{\fff_n^W}_* \right]\\

\leq & \sum\limits_{n=0}\limits^N \daone_{B_n} \EE_*\left[ \frac{1}{{N+1}} \|\bdt^T A^{-1}\|^2 (w_{n+1}-1)^2 \|\nabla f(\opt;\bdy_{n+1})\|^2 \right.\\

& \ \ \ \ \ \ \ \ \ \ \ \ \ \ \ \left.\daone\left\{ |w_{n+1}-1| \|\nabla f(\opt;\bdy_{n+1})\| > \frac{\sqrt{N+1}\varepsilon}{\| \bdt^T A^{-1}\|}\right\} |{\fff_n^W}_* \right] \\

\leq & \sum\limits_{n=0}\limits^N \daone_{B_n} \EE_*\left[ \frac{1}{{N+1}} \|\bdt^T A^{-1}\|^2 (w_{n+1}-1)^2 \|\nabla f(\opt;\bdy_{n+1})\|^2 \right.\\

& \ \ \ \ \ \ \ \ \ \ \ \ \ \ \ \left. \daone\left\{ |w_{n+1}-1| > \sqrt{\frac{\sqrt{N+1}\varepsilon}{\| \bdt^T A^{-1}\|}} \right\} |{\fff_n^W}_* \right] \\

& + \sum\limits_{n=0}\limits^N \daone_{B_n} \EE_*\left[ \frac{1}{{N+1}} \|\bdt^T A^{-1}\|^2 (w_{n+1}-1)^2 \|\nabla f(\opt;\bdy_{n+1})\|^2 \right.\\

& \ \ \ \ \ \ \ \ \ \ \ \ \ \ \ \ \ \  \left.\daone\left\{ \|\nabla f(\opt;\bdy_{n+1})\| > \sqrt{\frac{\sqrt{N+1}\varepsilon}{\| \bdt^T A^{-1}\|}} \right\} |{\fff_n^W}_* \right] \\

= & \sum\limits_{n=0}\limits^N  \left[\frac{1}{{N+1}} \|\bdt^T A^{-1}\|^2 \|\nabla f(\opt;\dot{\bdy}_{n+1})\|^2 \daone_{B_n} \right.\\

&\left.\ \ \ \ \ \  \EE_*\left[  (w_{n+1}-1)^2  \daone\left\{ |w_{n+1}-1| > \sqrt{\frac{\sqrt{N+1}\varepsilon}{\| \bdt^T A^{-1}\|}} \right\} |{\fff_n^W}_* \right] \right]\\

& + \sum\limits_{n=0}\limits^N \left[ \frac{1}{{N+1}}  \daone_{B_n} \|\bdt^T A^{-1}\|^2 \|\nabla f(\opt;\dot{\bdy}_{n+1})\|^2 \right.\\

& \left.\ \ \ \ \ \ \ \  \daone\left\{ \|\nabla f(\opt;\dot{\bdy}_{n+1})\| > \sqrt{\frac{\sqrt{N+1}\varepsilon}{\| \bdt^T A^{-1}\|}} \right\} \EE_*\left[  (w_{n+1}-1)^2  |{\fff_n^W}_* \right] \right]\\

= & \left\{ \sum\limits_{n=0}\limits^N \frac{1}{{N+1}} \|\bdt^T A^{-1}\|^2 \|\nabla f(\opt;\dot{\bdy}_{n+1})\|^2 \daone_{B_n} \right\} \EE \left[ (w-1)^2 \daone\left\{ |w-1| > \sqrt{\frac{\sqrt{N+1}\varepsilon}{\| \bdt^T A^{-1}\|}} \right\} \right]\\

& +  \sum\limits_{n=0}\limits^N \frac{1}{{N+1}} \|\bdt^T A^{-1}\|^2 \|\nabla f(\opt;\dot{\bdy}_{n+1})\|^2 \daone\left\{ \|\nabla f(\opt;\dot{\bdy}_{n+1})\| > \sqrt{\frac{\sqrt{N+1}\varepsilon}{\| \bdt^T A^{-1}\|}} \right\} \daone_{B_n},\\
\end{array}
$}
\label{thm1eq5}
\end{equation}
where the last step is based on the design $\EE_*\left[  (w_{n+1}-1)^2  |{\fff_n^W}_* \right]=1$. Based on (\ref{condi1}) and (\ref{condi2}), we know that
\begin{equation}
\resizebox{.92\hsize}{!}{$
\begin{array}{rl}
&\lim\limits_{N\rightarrow \infty}  \left\{ \sum\limits_{n=0}\limits^N \frac{1}{{N+1}} \|\bdt^T A^{-1}\|^2 \|\nabla f(\opt;\dot{\bdy}_{n+1})\|^2  I^* \right\} \EE \left[ (w-1)^2 \daone\left\{ |w-1| > \sqrt{\frac{\sqrt{N+1}\varepsilon}{\| \bdt^T A^{-1}\|}} \right\} \right]\\

= &  I^* \|\bdt^T A^{-1}\|^2 \EE_{\bdy\sim\mathcal{P}_Y} \|\nabla f(\opt;\bdy)\|^2 \cdot 0\\

= & 0,\\
\end{array}
$}
\label{thm1eq6}
\end{equation}
and
\begin{equation}
\resizebox{.92\hsize}{!}{$
\begin{array}{c}
\lim\limits_{N\rightarrow \infty} \sum\limits_{n=0}\limits^N \frac{1}{{N+1}} \|\bdt^T A^{-1}\|^2 \|\nabla f(\opt;\dot{\bdy}_{n+1})\|^2 \daone\left\{ \|\nabla f(\opt;\dot{\bdy}_{n+1})\| > \sqrt{\frac{\sqrt{N+1}\varepsilon}{\| \bdt^T A^{-1}\|}} \right\}  I^* = 0.
\end{array}
$}
\label{thm1eq7}
\end{equation}
Then, based on Lemma \ref{lemma:setapprox2}, (\ref{thm1eq4}), (\ref{thm1eq5}), (\ref{thm1eq6}) and (\ref{thm1eq7}), we have
$$
\resizebox{.92\hsize}{!}{$
\lim\limits_{N\rightarrow \infty} \sum\limits_{n=0}\limits^N \EE_*\left[ \frac{1}{{N+1}}\bdt^T A^{-1} (w_{n+1}-1)^2 \nabla f(\opt;\bdy_{n+1}) \nabla f(\opt;\bdy_{n+1})^T  A^{-T}\bdt \daone_{B_n} \daone\left\{|m_n|>\varepsilon\right\} |{\fff_n^W}_* \right] = 0,
$}
$$
$\PP_*$-almost-surely.

Thirdly, we verify the \textbf{conditional variance condition}: we have
\begin{equation}
\resizebox{.92\hsize}{!}{$
\begin{array}{rl}
& \sum\limits_{n=0}\limits^N \EE_*\left[ \frac{1}{{N+1}}\bdt^T A^{-1} (w_{n+1}-1)^2 \nabla f(\opt;\bdy_{n+1}) \nabla f(\opt;\bdy_{n+1})^T  A^{-T}\bdt \daone_{B_n}  |{\fff_n^W}_* \right]\\

= & \sum\limits_{n=0}\limits^N \frac{1}{{N+1}}\bdt^T A^{-1}  \nabla f(\opt;\dot{\bdy}_{n+1}) \nabla f(\opt;\dot{\bdy}_{n+1})^T  A^{-T}\bdt \daone_{B_n}.\\
\end{array}
$}
\label{thm1eq77}
\end{equation}
Based on (\ref{condi3}), we have
\begin{equation}
\begin{array}{rl}
& \lim\limits_{N\rightarrow \infty}  \sum\limits_{n=0}\limits^N \frac{1}{{N+1}}\bdt^T A^{-1}  \nabla f(\opt;\dot{\bdy}_{n+1}) \nabla f(\opt;\dot{\bdy}_{n+1})^T  A^{-T}\bdt I^* \\

= &  \bdt^T A^{-1}U   A^{-T}\bdt I^*. \\
\end{array}
\label{thm1eq8}
\end{equation}
Then, based on Lemma \ref{lemma:setapprox2}, (\ref{thm1eq77}) and (\ref{thm1eq8}), we have
$$
\resizebox{.98\hsize}{!}{$
\begin{array}{rl}
&  \lim\limits_{N\rightarrow \infty} \sum\limits_{n=0}\limits^N \EE_*\left[ \frac{1}{{N+1}}\bdt^T A^{-1} (w_{n+1}-1)^2 \nabla f(\opt;\bdy_{n+1}) \nabla f(\opt;\bdy_{n+1})^T  A^{-T}\bdt \daone_{B_n}  |{\fff_n^W}_* \right]\\

= & \bdt^T A^{-1}U   A^{-T}\bdt I^*
\end{array}
$}
$$
$\PP_*$-almost-surely.
\end{proof}

\subsection{Proof of Corollary \ref{bootcor}}
\begin{proof}
Based on Lemma \ref{basiclemma1}, we know that for any $\bdt\in \RR^d$,
$$
\sqrt{N}\bdt^T(\bar{\bdtheta}_N-\opt)\daone\{\limk \bdtheta_k = \opt\} \mathop{\rightarrow}\limits^{d} \bdt^T\bdZ\daone\{\limk \bdtheta_k = \opt\},
$$
where $\bdZ\sim N_d(0, A^{-1}U   A^{-T})$, independent of other already-mentioned random variables. Therefore, for any $\bdu\in \RR^d$, we have
\begin{equation}
\lim\limits_{N\rightarrow \infty}  \PP\left( \sqrt{N}(\bar{\bdtheta}_N-\opt)\daone\{\limk \bdtheta_k = \opt\} \leq \bdu\right) = \PP\left( \bdZ\daone\{\limk \bdtheta_k = \opt\} \leq \bdu \right).
\label{cor2eq1}
\end{equation}
Notice that
$$
\begin{array}{rl}
& \PP\left( \sqrt{N}(\bar{\bdtheta}_N-\opt)\daone\{\limk \bdtheta_k = \opt\} \leq \bdu\right)\\
 =& \PP\left( \sqrt{N}(\bar{\bdtheta}_N-\opt) \leq \bdu, \limk \bdtheta_k = \opt\right) + \PP\left(\limk \bdtheta_k \ne \opt\right) \daone\{ \bdu\ge 0\}\\
 \end{array}
$$
and
$$
\PP\left( \bdZ\daone\{\limk \bdtheta_k = \opt\} \leq \bdu \right) = \PP\left( \bdZ \leq \bdu \right) \PP\left(\limk \bdtheta_k = \opt\right) + \PP\left(\limk \bdtheta_k \ne \opt\right) \daone\{ \bdu\ge 0\}.
$$
Consequently, from (\ref{cor2eq1}), we have
$$
\lim\limits_{N\rightarrow \infty} \PP\left( \sqrt{N}(\bar{\bdtheta}_N-\opt) \leq \bdu, \limk \bdtheta_k = \opt\right) = \PP\left( \bdZ \leq \bdu \right) \PP\left(\limk \bdtheta_k = \opt\right),
$$
which further implies
$$
\lim\limits_{N\rightarrow \infty} \PP\left( \sqrt{N}(\bar{\bdtheta}_N-\opt) \leq \bdu \left|\right.\limk \bdtheta_k = \opt\right) = \PP\left( \bdZ \leq \bdu \right).
$$
Therefore, we know $ \sqrt{N}(\bar{\bdtheta}_N-\opt) \left|\right.\limk \bdtheta_k = \opt \mathop{\rightarrow}\limits^{d} \bdZ$ and have
\begin{equation}
\lim\limits_{N\rightarrow \infty} \sup\limits_{\bdu\in \RR^d}\left| \PP\left( \sqrt{N}(\bar{\bdtheta}_N-\opt) \leq \bdu \left|\right.\limk \bdtheta_k = \opt\right) - \PP\left( \bdZ \leq \bdu \right)\right| = 0.
\label{cor2eq2}
\end{equation}
Next, based on Theorem \ref{bootthm}, for any $\bdt\in \RR^d$, we have
\begin{equation}
\resizebox{.92\hsize}{!}{$
\begin{array}{rl}
& \lim\limits_{N\rightarrow \infty} \EE\left[ exp\left(i\sqrt{N}\bdt^T (\bar{\bdtheta}^*_N - \bar{\bdtheta}_N)\daone\{\limk \bdtheta_k=\opt,\limk \bdtheta^*_k=\opt\}\right)|\bdtheta_0,\bdy_1,\bdy_2,\ldots \right]\\

= & \EE\left[exp\left(-\frac{1}{2}\bdt^T A^{-1} U   A^{-T}\bdt \daone\{\limk \bdtheta_k=\opt,\limk \bdtheta^*_k=\opt\}\right)|\bdtheta_0,\bdy_1,\bdy_2,\ldots \right]\\
\end{array}
$}
\label{cor2eq3}
\end{equation}
in probability. Therefore, for any $\bdu\in \RR^d$, we have
\begin{equation}
\resizebox{.92\hsize}{!}{$
\begin{array}{rl}
& \lim\limits_{N\rightarrow \infty}\PP \left( \sqrt{N} (\bar{\bdtheta}^*_N - \bar{\bdtheta}_N)\daone\{\limk \bdtheta_k=\opt,\limk \bdtheta^*_k=\opt\} \leq \bdu |\bdtheta_0,\bdy_1,\bdy_2,\ldots\right)\\

\leq &  \PP\left( \tilde{\bdZ}\daone\{\limk \bdtheta_k=\opt,\limk \bdtheta^*_k=\opt\} \leq \bdu  |\bdtheta_0,\bdy_1,\bdy_2,\ldots\right)\\
\end{array}
$}
\label{cor2eq3_2}
\end{equation}
in probability, where $\tilde{\bdZ}\sim N_d(0, A^{-1}U   A^{-T})$, independent of other already-mentioned random variables. To be more specific, since (\ref{cor2eq3}) holds, for an arbitrary sequence of positive integers $N_1<N_2<N_3<\ldots$, there exists a subsequence $N_{m_1}<N_{m_2}<N_{m_3}<\ldots$, such that 

\begin{equation}
\resizebox{.92\hsize}{!}{$
\begin{array}{rl}
& \lim\limits_{q\rightarrow \infty} \EE\left[ exp\left(i\sqrt{N_{m_q}}\bdt^T (\bar{\bdtheta}^*_{N_{m_q}} - \bar{\bdtheta}_{N_{m_q}})\daone\{\limk \bdtheta_k=\opt,\limk \bdtheta^*_k=\opt\}\right)|\bdtheta_0,\bdy_1,\bdy_2,\ldots \right]\\

= & \EE\left[exp\left(-\frac{1}{2}\bdt^T A^{-1} U   A^{-T}\bdt \daone\{\limk \bdtheta_k=\opt,\limk \bdtheta^*_k=\opt\}\right)|\bdtheta_0,\bdy_1,\bdy_2,\ldots \right]\\
\end{array}
$}
\label{cor2eq3_3}
\end{equation}
almost-surely. As a result, according to the continuity theorem, for any $\bdu \in \RR^d$,

\begin{equation}
\resizebox{.92\hsize}{!}{$
\begin{array}{rl}
& \lim\limits_{q\rightarrow \infty}\PP \left( \sqrt{{N_{m_q}}} (\bar{\bdtheta}^*_{N_{m_q}} - \bar{\bdtheta}_{N_{m_q}})\daone\{\limk \bdtheta_k=\opt,\limk \bdtheta^*_k=\opt\} \leq \bdu |\bdtheta_0,\bdy_1,\bdy_2,\ldots\right)\\

\leq &  \PP\left( \tilde{\bdZ}\daone\{\limk \bdtheta_k=\opt,\limk \bdtheta^*_k=\opt\} \leq \bdu  |\bdtheta_0,\bdy_1,\bdy_2,\ldots\right)\\
\end{array}
$}
\label{cor2eq3_4}
\end{equation}
almost-surely. Due to the arbitrariness of $\{N_m\}_{m\in \ZZ_+}$, (\ref{cor2eq3_2}) holds.

Then, based on (\ref{cor2eq3_2}), we have
\begin{equation}
\begin{array}{rl}
& \lim\limits_{N\rightarrow \infty}\PP \left( \sqrt{N} (\bar{\bdtheta}^*_N - \bar{\bdtheta}_N)\daone\{\limk \bdtheta^*_k=\opt\} \leq \bdu |\bdtheta_0,\bdy_1,\bdy_2,\ldots\right)\\

\leq &  \PP\left( \tilde{\bdZ}\daone\{\limk \bdtheta^*_k=\opt\} \leq \bdu  |\bdtheta_0,\bdy_1,\bdy_2,\ldots\right)\\
\end{array}
\label{cor2eq4}
\end{equation}
in probability on $\{\limk \bdtheta_k=\opt\}$. Similar to the way we deal with (\ref{cor2eq1}), based on (\ref{cor2eq4}), we can have
\begin{equation}
\lim\limits_{N\rightarrow \infty}\PP \left( \sqrt{N} (\bar{\bdtheta}^*_N - \bar{\bdtheta}_N) \leq \bdu | \limk \bdtheta^*_k=\opt, \bdtheta_0,\bdy_1,\bdy_2,\ldots\right) = \PP\left(\tilde{Z}\leq \bdu\right)
\label{cor2eq5}
\end{equation}
in probability on $\{\limk \bdtheta_k=\opt\}$. Next, we use (\ref{cor2eq5}) to prove
\begin{equation}
\lim\limits_{N\rightarrow \infty} \sup\limits_{u\in \RR^d}\left| \PP\left( \sqrt{N}(\bar{\bdtheta}^*_N-\bar{\bdtheta}_N) \leq \bdu \left|\right.\limk \bdtheta^*_k=\opt, \bdtheta_0,\bdy_1,\bdy_2,\ldots\right) - \PP\left( \tilde{Z} \leq \bdu \right)\right| = 0
\label{cor2eq6}
\end{equation}
in probability on $\{\limk \bdtheta_k=\opt\}$. To show (\ref{cor2eq6}), it suffices to show that for any subsequence $\{N_m\}_{m\ge 1}$, there exits a further subsubsequence $\{N_{m_j}\}_{j\ge 1}$ such that
\begin{equation}
\sup\limits_{u\in \RR^d}\left| \PP\left( \sqrt{N_{m_j}}(\bar{\bdtheta}^*_{N_{m_j}}-\bar{\bdtheta}_{N_{m_j}}) \leq \bdu \left|\right.\limk \bdtheta^*_k=\opt, \bdtheta_0,\bdy_1,\bdy_2,\ldots\right) - \PP\left( \tilde{Z} \leq \bdu \right)\right| \rightarrow 0
\label{cor2eq7}
\end{equation}
almost-surely on $\{\limk \bdtheta_k=\opt\}$. In fact, based on (\ref{cor2eq5}), we know
\begin{equation}
\lim\limits_{N\rightarrow \infty}\PP \left( \sqrt{N_m} (\bar{\bdtheta}^*_{N_m} - \bar{\bdtheta}_{N_m}) \leq \bdu | \limk \bdtheta^*_k=\opt, \bdtheta_0,\bdy_1,\bdy_2,\ldots\right) = \PP\left(\tilde{Z}\leq \bdu\right)
\label{cor2eq8}
\end{equation}
in probability on $\{\limk \bdtheta_k=\opt\}$ and hence there must exists a subsubsequence $\{N_{m_j}\}_{j\ge 1}$ such that
\begin{equation}
\lim\limits_{N\rightarrow \infty}\PP \left( \sqrt{N_{m_j}} (\bar{\bdtheta}^*_{N_{m_j}} - \bar{\bdtheta}_{N_{m_j}}) \leq \bdu | \limk \bdtheta^*_k=\opt, \bdtheta_0,\bdy_1,\bdy_2,\ldots\right) = \PP\left(\tilde{Z}\leq \bdu\right)
\label{cor2eq9}
\end{equation}
almost-surely on $\{\limk \bdtheta_k=\opt\}$. The arbitrariness of $u$ implies (\ref{cor2eq7}). Combining (\ref{cor2eq2}) and (\ref{cor2eq6}), we complete the proof.
\end{proof}
\section{Results Supporting Theorem \ref{thm:secbound} and \ref{thm:fourbound}}

\subsection{Proof of Theorem \ref{thm:secbound}}

\begin{proof}
For simplicity, denote event $\lwww \limk \bdtheta_k=\opt\rwww $ by $S_{opt}$ and $\bddelta_n \triangleq \bdtheta_n - \opt $. We have the following decomposition,
\begin{equation}
\resizebox{.92\hsize}{!}{$
\arraycolsep=1.1pt\def\arraystretch{1.5} 
\begin{array}{rl}
     &  \EE \lppp \lv \bddelta_N \rv^2 \ones\rppp\\
     
=     & \EE \lppp \lv \bddelta_N\rv^2 \ones\daone\lwww \exists \frac{N}{4} \leq n \leq \frac{N}{2}, \bdtheta_n \in \rgo \rwww \rppp + \EE\lppp \lv \bddelta_N\rv^2 \ones \daone \lwww \forall \frac{N}{4} \leq n \leq \frac{N}{2}, \bdtheta_n \notin \rgo\rwww \rppp\\
\triangleq & A + B.\\
\end{array}
\label{eq:thm2_1}
$}
\end{equation}
$A$ can be further decomposed as follows,
\begin{equation}
\arraycolsep=1.1pt\def\arraystretch{1.5} 
\begin{array}{rl}
A  &\leq \EE \lppp \lv \bddelta_N\rv^2 \ones\daone\lwww \exists \frac{N}{4} \leq n \leq \frac{N}{2}, \bdtheta_n \in \rgo(\opt) \rwww \rppp \\

 & \ \ + \EE \lp \lv \bddelta_N\rv^2 \ones\daone\lw \exists \frac{N}{4} \leq n \leq \frac{N}{2}, \bdtheta_n \in \rgo \backslash \rgo(\opt) \rw \rp \\
&\triangleq A_1 + A_2.\\
\end{array}
\label{eq:thm2_2}
\end{equation}
Next, we have
\begin{equation}
\resizebox{.92\hsize}{!}{$
\arraycolsep=1.1pt\def\arraystretch{1.5} 
\begin{array}{rl}
A_1 \leq & \EE \lppp \lv \bddelta_N\rv^2 \daone \lwww \exists \frac{N}{4} \leq n \leq \frac{N}{2}, \bdtheta_t \in \rgo^L(\opt) \text{ for all } t\ge n\rwww  \rppp \\

 & + \EE \lppp \lv \bddelta_N\rv^2 \daone \lwww \exists \frac{N}{4} \leq n \leq \frac{N}{2}, \bdtheta_n \in \rgo(\opt) \text{ but } \bdtheta_t \notin \rgo^L(\opt)  \text{ for some } t\ge n\rwww \rppp \\
 \triangleq & A_{11} + A_{12}.\\
\end{array}
\label{eq:thm2_3}
$}
\end{equation}

Now, we are to show that \twoline{$A_{11} = O(N^{-\alpha})$}. For any $n\in \ZZ_+$, we have
\begin{equation}
\resizebox{.92\hsize}{!}{$
\arraycolsep=1.1pt\def\arraystretch{1.5} 
\begin{array}{rl}
\EE_n \| \bddelta_{n+1} \|^2 & = \| \bddelta_n\|^2 + \gamma_{n+1}^2 \EE_n \| \nabla f(\bdtheta_n;\bdy_{n+1})\|^2 - 2\gamma_{n+1} \langle \bddelta_n, \nablaf{n}\rangle\\

& \leq \| \bddelta_n\|^2 + 2\gamma_{n+1}^2  \| \nablaf{n} \|^2 + 2 \gamma_{n+1}^2 \EE_n \| \nabla f(\bdtheta_n;\bdy_{n+1}) - \nablaf{n}\|^2 - 2\gamma_{n+1} \langle \bddelta_n, \nablaf{n}\rangle\\

& \leq \| \bddelta_n\|^2 + 2\gamma_{n+1}^2 \lppp 1+2\beta_{sg} \csg \rppp \| \nablaf{n} \|^2  + 4 (4-\beta_{sg})\csg\gamma_{n+1}^2 - 2\gamma_{n+1} \langle \bddelta_n, \nablaf{n}\rangle,\\
\end{array}
\label{eq:thm2_4}
$}
\end{equation}
where the 3rd step is based on Lemma \ref{lemma:errorbound}. Based on condition \ref{cm_1}, we know that on $\lwww \bdtheta_n \in \rgo^L(\opt)\rwww$,
$$
\langle \bddelta_n, \nablaf{n}\rangle \ge \frac{1}{2}\tlambda \| \bddelta_n \|^2.
$$
Therefore, on $\lwww \bdtheta_n \in \rgo^L(\opt)\rwww$, we have
\begin{equation}
\resizebox{.92\hsize}{!}{$
\arraycolsep=1.1pt\def\arraystretch{1.5} 
\begin{array}{rl}
\EE_n \| \bddelta_{n+1} \|^2 & \leq \lppp 1-\tlambda \gamma_{n+1} \rppp \| \bddelta_n \|^2 + 2\gamma_{n+1}^2 \lppp 1+2\beta_{sg} \csg \rppp \| \nablaf{n} \|^2  + 4 (4-\beta_{sg})\csg\gamma_{n+1}^2\\

& \leq \lppp 1-\tlambda \gamma_{n+1} + 2C_s ( 1+2\beta_{sg} \csg ) \gamma_{n+1}^2 \rppp \| \bddelta_n \|^2 + 4 (4-\beta_{sg})\csg\gamma_{n+1}^2\\

& \leq \lppp 1-\frac{1}{2}\tlambda \gamma_{n+1}\rppp \| \bddelta_n \|^2 + 4 (4-\beta_{sg})\csg\gamma_{n+1}^2,\\
\end{array}
\label{eq:thm2_5}
$}
\end{equation}
where the 2nd step is based on condition \ref{cm_1} and the last step is true when $n \ge \lpp \frac{4CC_s (1+2\beta_{sg} \csg)}{\tlambda}\rpp^{\frac{1}{\alpha}}$. 
As a result, we have
\begin{equation}
\arraycolsep=1.1pt\def\arraystretch{1.5} 
\begin{array}{rl}
& \EE \lppp \| \bddelta_N \|^2 \daone \lwww \bdtheta_n \in \rgo^L(\opt), \frac{N}{2} \leq n \leq N-1 \rwww \rppp\\

= & \EE \lppp \lppp \EE_{N-1} \| \bddelta_N \|^2 \rppp \daone \lwww \bdtheta_n \in \rgo^L(\opt), \frac{N}{2} \leq n \leq N-1 \rwww \rppp\\

\leq & (1-\frac{1}{2}\tlambda \gamma_N) \EE \lppp \| \bddelta_{N-1} \|^2 \daone \lwww \bdtheta_n \in \rgo^L(\opt), \frac{N}{2} \leq n \leq N-1 \rwww \rppp + 4 (4-\beta_{sg})\csg\gamma_N^2\\

\leq & (1-\frac{1}{2}\tlambda \gamma_N) \EE \lppp \| \bddelta_{N-1} \|^2 \daone \lwww \bdtheta_n \in \rgo^L(\opt), \frac{N}{2} \leq n \leq N-2 \rwww \rppp + 4 (4-\beta_{sg})\csg\gamma_N^2\\

& \cdots\\

\leq & \lpp \pprod{n=\frac{N}{2}+1}{N} (1-\frac{1}{2}\tlambda \gamma_n)\rpp \EE \lvvv \bddelta_{\frac{N}{2}}\rvvv^2 + 4 (4-\beta_{sg})\csg \ssum{n=\frac{N}{2}+1}{N} \lpp \gamma_n^2 \pprod{j=n+1}{N} (1-\frac{1}{2}\tlambda \gamma_j)\rpp\\

\leq & \exp \lpp -\frac{1}{2}C\tlambda \ssum{n=\frac{N}{2}+1}{N} n^{-\alpha}\rpp \EE \lvvv \bddelta_{\frac{N}{2}}\rvvv^2 +  4 (4-\beta_{sg})\csg \ssum{n=\frac{N}{2}+1}{N}\lppp \gamma_n^2 (1-\frac{1}{2}\tlambda \gamma_N)^{N-n} \rppp\\

\leq & \exp \lppp -\frac{C\tlambda}{4} N^{1-\alpha} \rppp \EE \lvvv \bddelta_{\frac{N}{2}}\rvvv^2 +  4 (4-\beta_{sg})C^2 \csg \lpp\frac{N}{2}\rpp^{-2\alpha} \ssum{n=\frac{N}{2}+1}{N}(1-\frac{1}{2}\tlambda \gamma_N)^{N-n} \\

\leq & \exp \lppp -\frac{C\tlambda}{4} N^{1-\alpha} \rppp \EE \lvvv \bddelta_{\frac{N}{2}}\rvvv^2 +  4 (4-\beta_{sg})C^2 \csg \lpp\frac{N}{2}\rpp^{-2\alpha} (\frac{1}{2}\tlambda \gamma_N)^{-1}\\

= & O(N^{-\alpha}),\\
\end{array}
\label{eq:thm2_6}
\end{equation}
where the last step is based on Lemma \ref{lemma:momentbound} and condition \ref{cm_3}. Then, we can see that
\begin{equation}
A_{11} \leq \EE \lppp \| \bddelta_N \|^2 \daone \lwww \bdtheta_n \in \rgo^L(\opt), \frac{N}{2} \leq n \leq N-1 \rwww \rppp =O(N^{-\alpha}).
\label{eq:thm2_7}
\end{equation}

We can also show that \twoline{$A_{12} =O(N^{-\alpha})$}:
\begin{equation}
\resizebox{.92\hsize}{!}{$
\arraycolsep=1.1pt\def\arraystretch{1.5} 
\begin{array}{rl}
A_{12} &\leq \lppp\EE \lv \bddelta_N\rv^3\rppp^{\frac{2}{3}} \PP^{\frac{1}{3}}\lp \exists \frac{N}{4} \leq n \leq \frac{N}{2}, \bdtheta_n \in \rgo(\opt) \text{ but } \bdtheta_t \notin \rgo^L(\opt)  \text{ for some } t\ge n\rp\\

& \leq \lppp\EE \lv \bddelta_N\rv^3\rppp^{\frac{2}{3}} N^{-2\alpha}\\

& = O(N^{-\alpha}),\\
\end{array}
\label{eq:thm2_8}
$}
\end{equation}
where the 1st step is based on the H\"{o}lder's inequality, the 2nd step is true when $\frac{N}{4} \ge \cnlemma{7}(6\alpha, N)$ (Lemma \ref{lemma:trap}) and the last step is based on Lemma \ref{lemma:momentbound}. Based on (\ref{eq:thm2_7}) and (\ref{eq:thm2_8}), we have
$$
A_1 = O(N^{-\alpha}).
$$
To show \twoline{$A_2 = O(N^{-\alpha})$}, we have
\begin{equation}
\resizebox{.92\hsize}{!}{$
\arraycolsep=1.1pt\def\arraystretch{1.5} 
\begin{array}{rl}
 A_2& \leq \lppp\EE \lv \bddelta_N\rv^3\rppp^{\frac{2}{3}} \PP^{\frac{1}{3}} \lppp S_{opt} \cap \lw\exists \frac{N}{4} \leq n \leq \frac{N}{2}, \bdtheta_n \in \rgo \backslash \rgo(\opt)\rw \rppp\\

 & \leq \lppp\EE \lv \bddelta_N\rv^3\rppp^{\frac{2}{3}} \PP^{\frac{1}{3}} \lppp \exists \frac{N}{4} \leq n \leq \frac{N}{2},\bdtheta' \in \Theta^{opt}, \bdtheta_n \in \rgo(\bdtheta')\ \text{but}\ \bdtheta_t \notin \rgo(\bdtheta')\ \text{for some}\ t\ge n\rppp\\

 & = O (N^{-\alpha}),\\
\end{array}
\label{eq:thm2_9}
$}
\end{equation}
where the last step is similar to the 2nd step of (\ref{eq:thm2_8}). Therefore, we have
$$
A =A_1 + A_2 = O(N^{-\alpha}).
$$
To show \twoline{$B = O(N^{-\alpha})$}, we have
$$
\arraycolsep=1.1pt\def\arraystretch{1.5} 
\begin{array}{rl}
B & \leq  \EE\lppp \lv \bddelta_N\rv^2  \daone \lwww \forall \frac{N}{4} \leq n \leq \frac{N}{2}, \bdtheta_n \notin \rgo\rwww \rppp\\

 &\leq  \lppp\EE \lv \bddelta_N\rv^3\rppp^{\frac{2}{3}} \PP^{\frac{1}{3}}\lppp \forall \frac{N}{4} \leq n \leq \frac{N}{2}, \bdtheta_n \notin \rgo\rppp. \\

& = \lppp\EE \lv \bddelta_N\rv^3\rppp^{\frac{2}{3}} O (N^{-2\alpha})\\

& = O(N^{-\alpha}),\\
\end{array}
$$
where the 2nd step is based on the H\"{o}lder's inequality, the 3rd step is based on Lemma \ref{lemma:entergood} and the last step is based on Lemma \ref{lemma:momentbound} and condition \ref{cm_3}.

To sum up, we have
$$
\EE \lppp \lv \bddelta_N \rv^2 \ones\rppp  = O(N^{-\alpha}).
$$
\end{proof}

\subsection{Proof of Theorem \ref{thm:fourbound}}

\begin{proof}
The structure of the proof will be similar to the one of Theorem \ref{thm:secbound}. Therefore, we might skip some details in the current proof. We let
$$
\arraycolsep=1.1pt\def\arraystretch{1.5} 
\begin{array}{ccl}
S_{opt} & \triangleq & \lwww \limk \bdtheta_k=\opt\rwww,\\

\bddelta_n &\triangleq& \bdtheta_n - \opt.\\
\end{array}
$$
Then, we have the following decomposition,
\begin{equation}
\resizebox{.92\hsize}{!}{$
\arraycolsep=1.1pt\def\arraystretch{1.5} 
\begin{array}{rl}
     &  \EE \lppp \lv \bddelta_N \rv^4 \ones\rppp\\
     
=     & \EE \lppp \lv \bddelta_N\rv^4 \ones\daone\lwww \exists \frac{N}{4} \leq n \leq \frac{N}{2}, \bdtheta_n \in \rgo \rwww \rppp + \EE\lppp \lv \bddelta_N\rv^4 \ones \daone \lwww \forall \frac{N}{4} \leq n \leq \frac{N}{2}, \bdtheta_n \notin \rgo\rwww \rppp\\
\triangleq & A + B.\\
\end{array}
\label{eq:thm3_1}
$}
\end{equation}
Similar to the proof of Theorem \ref{thm:secbound}, we can have
\begin{equation}
B = O(N^{-2\alpha}),
\label{eq:thm3_2}
\end{equation}
because we have $\PP\lppp \forall \frac{N}{4} \leq n \leq \frac{N}{2}, \bdtheta_n \notin \rgo\rppp = O (N^{-\tau}),\forall \tau >0$, implied by Lemma \ref{lemma:entergood}. Then, it suffices to show $A = O(N^{-2\alpha})$. $A$ in (\ref{eq:thm3_1}) can be decomposed as follows,
\begin{equation}
\arraycolsep=1.1pt\def\arraystretch{1.5} 
\begin{array}{rl}
A  &\leq \EE \lppp \lv \bddelta_N\rv^4 \daone \lwww \exists \frac{N}{4} \leq n \leq \frac{N}{2}, \bdtheta_t \in \rgo^L(\opt) \text{ for all } t\ge n\rwww  \rppp \\

 &\ \  + \EE \lppp \lv \bddelta_N\rv^4 \daone \lwww \exists \frac{N}{4} \leq n \leq \frac{N}{2}, \bdtheta_n \in \rgo(\opt) \text{ but } \bdtheta_t \notin \rgo^L(\opt)  \text{ for some } t\ge n\rwww \rppp \\

 & \ \ + \EE \lp \lv \bddelta_N\rv^4 \ones\daone\lw \exists \frac{N}{4} \leq n \leq \frac{N}{2}, \bdtheta_n \in \rgo \backslash \rgo(\opt) \rw \rp, \\
\end{array}
\label{eq:thm3_3}
\end{equation}
where the last 2 terms can be shown to be $O(N^{-2\alpha})$ by using the same methods in Theorem \ref{thm:secbound}. 

By applying the H\"{o}lder's inequality and Young's inequality, we have
\begin{equation}
\resizebox{.92\hsize}{!}{$
\arraycolsep=1.1pt\def\arraystretch{1.5} 
\begin{array}{rl}
& \| \bddelta_{n+1}\|^4\\

= &\lppp \lvvv \bddelta_n - \gamma_{n+1} \nabla f(\bdtheta_n;\bdy_{n+1}\rvvv^2 \rppp^2\\

= & \lppp \| \bddelta_n\|^2 + \gamma_{n+1}^2 \lvvv \nabla f(\bdtheta_n;\bdy_{n+1})\rvvv^2 - 2\gamma_{n+1} \langle \bddelta_n, \nabla f(\bdtheta_n;\bdy_{n+1})\rangle \rppp^2\\

= & \| \bddelta_n\|^4 +\gamma_{n+1}^4 \lvvv \nabla f(\bdtheta_n;\bdy_{n+1})\rvvv^4 + 4 \gamma_{n+1}^2 \langle \bddelta_n, \nabla f(\bdtheta_n;\bdy_{n+1})\rangle^2 + 2 \gamma_{n+1}^2 \| \bddelta_n\|^2 \lvvv \nabla f(\bdtheta_n;\bdy_{n+1})\rvvv^2\\

&  - 4\gamma_{n+1} \langle \bddelta_n, \nabla f(\bdtheta_n;\bdy_{n+1})\rangle \| \bddelta_n\|^2 - 4\gamma_{n+1}^3 \lvvv \nabla f(\bdtheta_n;\bdy_{n+1})\rvvv^2 \langle \bddelta_n, \nabla f(\bdtheta_n;\bdy_{n+1})\rangle\\

\leq & \| \bddelta_n\|^4 +\gamma_{n+1}^4 \lvvv \nabla f(\bdtheta_n;\bdy_{n+1})\rvvv^4 +  6 \gamma_{n+1}^2 \| \bddelta_n\|^2 \lvvv \nabla f(\bdtheta_n;\bdy_{n+1})\rvvv^2\\

&  - 4\gamma_{n+1} \langle \bddelta_n, \nabla f(\bdtheta_n;\bdy_{n+1})\rangle \| \bddelta_n\|^2 - 4\gamma_{n+1}^3 \lvvv \nabla f(\bdtheta_n;\bdy_{n+1})\rvvv^3 \| \bddelta_n\| \\

\leq & \| \bddelta_n\|^4 +\gamma_{n+1}^4 \lvvv \nabla f(\bdtheta_n;\bdy_{n+1})\rvvv^4 + \frac{36}{\tlambda} \gamma_{n+1}^3 \lvvv \nabla f(\bdtheta_n;\bdy_{n+1})\rvvv^4 + \tlambda \gamma_{n+1} \| \bddelta_n\|^4\\

&  - 4\gamma_{n+1} \langle \bddelta_n, \nabla f(\bdtheta_n;\bdy_{n+1})\rangle \| \bddelta_n\|^2  + 4\gamma_{n+1}^3 \lppp \frac{3}{4} \lvvv \nabla f(\bdtheta_n;\bdy_{n+1})\rvvv^4 + \frac{1}{4}\| \bddelta_n\|^4\rppp\\

= & (1+ \tlambda \gamma_{n+1} + \gamma_{n+1}^3) \| \bddelta_n\|^4 + \lppp \gamma_{n+1}^4 + \frac{36}{\tlambda} \gamma_{n+1}^3 + 3 \gamma_{n+1}^3\rppp \lvvv \nabla f(\bdtheta_n;\bdy_{n+1})\rvvv^4\\

& - 4\gamma_{n+1} \langle \bddelta_n, \nabla f(\bdtheta_n;\bdy_{n+1})\rangle \| \bddelta_n\|^2.\\
\end{array}
\label{eq:thm3_4}
$}
\end{equation}
Then, on $\lwww \bdtheta_n \in \rgo^L(\opt)\rwww$, we have
\begin{equation}
\resizebox{.92\hsize}{!}{$
\arraycolsep=1.1pt\def\arraystretch{1.5} 
\begin{array}{rl}
& \EE_n \| \bddelta_{n+1}\|^4\\

\leq & (1+ \tlambda \gamma_{n+1} + \gamma_{n+1}^3) \| \bddelta_n\|^4 + 8\lppp \gamma_{n+1}^4 + \frac{36}{\tlambda} \gamma_{n+1}^3 + 3 \gamma_{n+1}^3\rppp \| \nablaf{n}\|^4\\

&+ \lppp \gamma_{n+1}^4 + \frac{36}{\tlambda} \gamma_{n+1}^3 + 3 \gamma_{n+1}^3\rppp 256(\csg)^2 \lppp 2 - \frac{\beta_{sg}}{2} +\frac{\beta_{sg}}{2} \| \nablaf{n}\|^4\rppp - 4\gamma_{n+1} \langle \bddelta_n, \nablaf{n}\rangle \| \bddelta_n\|^2\\

\leq & (1+ \tlambda \gamma_{n+1} + \gamma_{n+1}^3) \| \bddelta_n\|^4 + 256(\csg)^2 \lppp 2 - \frac{\beta_{sg}}{2}\rppp \lppp \gamma_{n+1}^4 + \frac{36}{\tlambda} \gamma_{n+1}^3 + 3 \gamma_{n+1}^3\rppp\\

& + (8 + 128\beta_{sg} (\csg)^2) \lppp \gamma_{n+1}^4 + \frac{36}{\tlambda} \gamma_{n+1}^3 + 3 \gamma_{n+1}^3\rppp \|\nablaf{n}\|^4 - 2\tlambda \gamma_{n+1} \| \bddelta_n\|^4\\

\leq & \lppp 1- \tlambda \gamma_{n+1} + \gamma_{n+1}^3 + C_s^4 (8 + 128\beta_{sg} (\csg)^2) \lppp \gamma_{n+1}^4 + \frac{36}{\tlambda} \gamma_{n+1}^3 + 3 \gamma_{n+1}^3\rppp\rppp \| \bddelta_n\|^4\\

& + 256(\csg)^2 \lppp 2 - \frac{\beta_{sg}}{2}\rppp \lppp \gamma_{n+1}^4 + \frac{36}{\tlambda} \gamma_{n+1}^3 + 3 \gamma_{n+1}^3\rppp\\

\leq & \lppp 1-\frac{1}{2}\tlambda \gamma_{n+1}\rppp \| \bddelta_n\|^4 + 1024 (\csg)^2 \lppp \frac{36}{\tlambda}+3\rppp \gamma_{n+1}^3,\\
\end{array}
\label{eq:thm3_5}
$}
\end{equation}
where the first step is based on Lemma \ref{lemma:errorbound}, the third step is based on condition \ref{cm_1} and the last step is true when $n$ is greater than some constant. For simplicity, we temporarily let
$$
c_1 \triangleq 1024 (\csg)^2 \lpp \frac{36}{\tlambda}+3\rpp.
$$
Then, we have
\begin{equation}
\arraycolsep=1.1pt\def\arraystretch{1.5} 
\begin{array}{rl}
& \EE \lppp \| \bddelta_N \|^4 \daone \lwww \bdtheta_n \in \rgo^L(\opt), \frac{N}{2} \leq n \leq N-1 \rwww \rppp\\

= & \EE \lppp \lppp \EE_{N-1} \| \bddelta_N \|^4 \rppp \daone \lwww \bdtheta_n \in \rgo^L(\opt), \frac{N}{2} \leq n \leq N-1 \rwww \rppp\\

\leq & (1-\frac{1}{2}\tlambda \gamma_N) \EE \lppp \| \bddelta_{N-1} \|^4 \daone \lwww \bdtheta_n \in \rgo^L(\opt), \frac{N}{2} \leq n \leq N-1 \rwww \rppp + c_1 \gamma_N^3\\

\leq & (1-\frac{1}{2}\tlambda \gamma_N) \EE \lppp \| \bddelta_{N-1} \|^4 \daone \lwww \bdtheta_n \in \rgo^L(\opt), \frac{N}{2} \leq n \leq N-2 \rwww \rppp + c_1 \gamma_N^3\\

& \cdots\\

\leq & \lpp \pprod{n=\frac{N}{2}+1}{N} (1-\frac{1}{2}\tlambda \gamma_n)\rpp \EE \lvvv \bddelta_{\frac{N}{2}}\rvvv^4 + c_1 \ssum{n=\frac{N}{2}+1}{N} \lpp \gamma_n^3 \pprod{j=n+1}{N} (1-\frac{1}{2}\tlambda \gamma_j)\rpp\\

\leq & \exp \lpp -\frac{1}{2}C\tlambda \ssum{n=\frac{N}{2}+1}{N} n^{-\alpha}\rpp \EE \lvvv \bddelta_{\frac{N}{2}}\rvvv^4 +  c_1 \ssum{n=\frac{N}{2}+1}{N}\lppp \gamma_n^3 (1-\frac{1}{2}\tlambda \gamma_N)^{N-n} \rppp\\

\leq & \exp \lppp -\frac{C\tlambda}{4} N^{1-\alpha} \rppp \EE \lvvv \bddelta_{\frac{N}{2}}\rvvv^4 +  c_1 C^3 \lpp\frac{N}{2}\rpp^{-3\alpha} \ssum{n=\frac{N}{2}+1}{N}(1-\frac{1}{2}\tlambda \gamma_N)^{N-n} \\

\leq & \exp \lppp -\frac{C\tlambda}{4} N^{1-\alpha} \rppp \EE \lvvv \bddelta_{\frac{N}{2}}\rvvv^4 +  c_1 C^3\lpp\frac{N}{2}\rpp^{-3\alpha} (\frac{1}{2}\tlambda \gamma_N)^{-1}\\

= & O(N^{-2\alpha}),\\
\end{array}
\label{eq:thm3_6}
\end{equation}
where the last step is based on Lemma \ref{lemma:momentbound} and condition \ref{cm_3}. As a result, we can see that the first term on the right-hand side of (\ref{eq:thm3_3}) is also $O(N^{-2\alpha})$. It means that $A=O(N^{-2\alpha})$, which concludes the proof.
\end{proof}

\subsection{Proof of Corollary \ref{cor:ci1}}
\begin{proof}
In this proof, for simplicity, we let
$$
\Sigma \triangleq A^{-1}UA^{-1},\ \ \Sigma^*_N \triangleq \lppp \sqrt{N}  (\bar{\bdtheta}_N^* - \bar{\bdtheta}_N) \rppp\lppp\sqrt{N}  (\bar{\bdtheta}_N^* - \bar{\bdtheta}_N) \rppp^T.
$$
We also denote
$$
\sopt \triangleq \lww \limk \bdtheta_k=\opt \rww,\ \ \ssopt \triangleq \lww \limk \bbdtheta_k=\opt \rww.
$$
For any $\varepsilon \in \lppp 0,\frac{q}{2} \rppp$, we have the following inequality:
\begin{equation}
\arraycolsep=1.1pt\def\arraystretch{1.5} 
\begin{array}{rl}
& \PP \lpp \bm{a}^T\opt \in \Big[ \bm{a}^T \bar{\bdtheta}_N - z_{\frac{q}{2}} \sqrt{ \frac{\bm{a}^T \hat{\Sigma}_N \bm{a} }{N}}, \bm{a}^T \bar{\bdtheta}_N + z_{\frac{q}{2}} \sqrt{ \frac{\bm{a}^T \hat{\Sigma}_N \bm{a} }{N}} \Big] \big| \sopt \rpp\\

\leq & \PP \lpp \bm{a}^T\opt \in \Big[ \bm{a}^T \bar{\bdtheta}_N - z_{\frac{q}{2}-\varepsilon} \sqrt{ \frac{\bm{a}^T \Sigma \bm{a} }{N}}, \bm{a}^T \bar{\bdtheta}_N + z_{\frac{q}{2}-\varepsilon} \sqrt{ \frac{\bm{a}^T \Sigma \bm{a} }{N}} \Big] \big| \sopt \rpp\\

& + \PP \lpp z_{\frac{q}{2}} \sqrt{\bm{a}^T \hat{\Sigma}_N \bm{a}} > z_{\frac{q}{2}-\varepsilon} \sqrt{\bm{a}^T \Sigma \bm{a} } \big| \sopt \rpp.\\
\end{array}
\label{eq:cor1_1}
\end{equation}
Based on the weak convergence result provided in Lemma \ref{basiclemma1}, we can conclude that
\begin{equation}
\resizebox{.92\hsize}{!}{$
\limN \PP \lpp \bm{a}^T\opt \in \Big[ \bm{a}^T \bar{\bdtheta}_N - z_{\frac{q}{2}-\varepsilon} \sqrt{ \frac{\bm{a}^T \Sigma \bm{a} }{N}}, \bm{a}^T \bar{\bdtheta}_N + z_{\frac{q}{2}-\varepsilon} \sqrt{ \frac{\bm{a}^T \Sigma \bm{a} }{N}} \Big] \big| \sopt\rpp = 1-q+2\varepsilon.
$}
\label{eq:cor1_2}
\end{equation}
Now, our goal is to show
$$
\limN \PP \lpp z_{\frac{q}{2}} \sqrt{\bm{a}^T \hat{\Sigma}_N \bm{a}} > z_{\frac{q}{2}-\varepsilon} \sqrt{\bm{a}^T \Sigma \bm{a} } \big| \sopt \rpp=0.
$$
If we temporarily let
$$
c_1 \triangleq  \lpp \frac{z_{\frac{q}{2}-\varepsilon}^2}{z_{\frac{q}{2}}^2} - 1\rpp \frac{\bm{a}^T \Sigma \bm{a}}{\|\bm{a}\|^2},\ \ V_n \triangleq  \lwww \lvvv \bbdtheta_N - \bdtheta_N \rvvv \leq 2\sqrt{3}r_{good} \rwww,
$$
we have
\begin{equation}
\arraycolsep=1.1pt\def\arraystretch{1.5} 
\begin{array}{rl}
& \PP \lpp z_{\frac{q}{2}} \sqrt{\bm{a}^T \hat{\Sigma}_N \bm{a}} > z_{\frac{q}{2}-\varepsilon} \sqrt{\bm{a}^T \Sigma \bm{a} } \big| \sopt \rpp\\

\leq & \PP \lpp z_{\frac{q}{2}}^2 \lppp \bm{a}^T (\hat{\Sigma}_N - \Sigma) \bm{a} \rppp > \lppp z_{\frac{q}{2}-\varepsilon}^2 - z_{\frac{q}{2}}^2\rppp \lppp \bm{a}^T \Sigma \bm{a} \rppp \big| \sopt \rpp\\

\leq & \PP \lppp \lvvv \hat{\Sigma}_N - \Sigma \rvvv > c_1\big| \sopt \rppp\\

= & \PP \lppp \lvvv \EE_* \lppp\lppp \Sigma_N^* -\Sigma \rppp \daone_{V_N} \rppp \rvvv > c_1 \PP_* (V_N) \big| \sopt \rppp\\

= & \PP \lppp \lvvv \EE_* \lppp\lppp \Sigma_N^* -\Sigma \rppp \daone_{V_N} \rppp \rvvv > c_1 \lppp \PP_* \lppp V_N \backslash \ssopt  \rppp + \PP_* \lppp \ssopt\rppp - \PP_* \lppp \ssopt \backslash V_N \rppp \rppp \big| \sopt \rppp\\

\leq & \PP \lppp \lvvv \EE_* \lppp\lppp \Sigma_N^* -\Sigma \rppp \daone_{V_N} \rppp \rvvv + c_1 \PP_* \lppp \ssopt \backslash V_N \rppp > c_1 \PP_* \lppp \ssopt\rppp \big| \sopt \rppp.\\
\end{array}
\label{eq:cor1_3}
\end{equation}
As $\PP_* \lppp \ssopt\rppp >0$ almost surely on $\sopt$, for any $\varepsilon_1>0$, there exists a positive $\tau_1$ such that
$$
\PP \lppp \PP_* \lppp \ssopt\rppp >\tau_1 \big| \sopt \rppp \ge 1-\varepsilon_1.
$$
As a result, from (\ref{eq:cor1_3}), we have
\begin{equation}
\arraycolsep=1.1pt\def\arraystretch{1.5} 
\begin{array}{rl}
& \PP \lpp z_{\frac{q}{2}} \sqrt{\bm{a}^T \hat{\Sigma}_N \bm{a}} > z_{\frac{q}{2}-\varepsilon} \sqrt{\bm{a}^T \Sigma \bm{a} } \big| \sopt \rpp\\

\leq &\varepsilon_1 + \PP \lppp \lvvv \EE_*\lppp \lppp \Sigma_N^* -\Sigma \rppp \daone_{V_N} \rppp \rvvv + c_1 \PP_* \lppp \ssopt \backslash V_N \rppp > c_1 \tau_1 \big| \sopt \rppp\\

= &\varepsilon_1 + \frac{1}{\PP\lppp\sopt\rppp} \PP \lppp \lppp \lvvv \EE_* \lppp\lppp \Sigma_N^* -\Sigma \rppp \daone_{V_N} \rppp \rvvv + c_1 \PP_* \lppp \ssopt \backslash V_N \rppp \rppp \daone_{\sopt} > c_1\tau_1\rppp\\

\leq &\varepsilon_1 + \frac{1}{\PP\lppp\sopt\rppp c_1\tau_1} \EE \lss   \lppp \lvvv \EE_* \lppp\lppp \Sigma_N^* -\Sigma \rppp \daone_{V_N} \rppp \rvvv + c_1 \PP_* \lppp \ssopt \backslash V_N \rppp \rppp \daone_{\sopt} \rss,\\
\end{array}
\label{eq:cor1_4}
\end{equation}
where the last step is based on the Markov inequality. For convenience, we denote
$$
A_N(\bdtheta, r) \triangleq \lwww \lvvv \bdtheta_n - \bdtheta\rvvv \leq r, \forall n \ge N \rwww, \bdtheta \in \Theta,N\in \ZZ_+ , r>0,
$$
and
$$
A_N^*(\bdtheta, r) \triangleq \lwww \lvvv \bbdtheta_n - \bdtheta\rvvv \leq r, \forall n \ge N \rwww, \bdtheta \in \Theta,N\in \ZZ_+ , r>0,
$$
To proceed with (\ref{eq:cor1_4}), we firstly see that for any $\tau_2>0$,
\begin{equation}
\arraycolsep=1.1pt\def\arraystretch{1.5} 
\begin{array}{rl}
& \EE \lsss \PP_* \lppp \ssopt \backslash V_N  \rppp \daone_{\sopt} \rsss = \PP \lppp \ssopt \cap \sopt \cap V_N^c \rppp\\

=& \PP\lppp A_N(\opt, \sqrt{3}r_{good}) \cap A_N^*(\opt, \sqrt{3}r_{good}) \cap V_N^c \rppp + O\lppp N^{-\tau_2}\rppp=O\lppp N^{-\tau_2}\rppp,\\
\end{array}
\label{eq:cor1_5}
\end{equation}
where the second step is based on Lemma \ref{lemma:mustbound}. In addition, we have the following bound for the other part in (\ref{eq:cor1_4}):
\begin{equation}
\arraycolsep=1.1pt\def\arraystretch{1.5} 
\begin{array}{rl}
& \EE \lss \lvvv \EE_* \lppp\lppp \Sigma_N^* -\Sigma \rppp \daone_{V_N} \rppp \rvvv \daone_{\sopt} \rss\\

\leq & \EE \lvvv \EE_* \lppp\lppp \Sigma_N^* -\Sigma \rppp \daone_{\ssopt} \daone_{\sopt}\rppp \rvvv + \EE \lvvv \lppp \Sigma_N^* -\Sigma \rppp \daone_{V_N\backslash \ssopt} \daone_{\sopt} \rvvv \\

= & O\lppp N^{1-2\alpha} \rppp +\EE \lvvv \lppp \Sigma_N^* -\Sigma \rppp \daone_{V_N\backslash \ssopt} \daone_{\sopt} \rvvv \\

\leq & O\lppp N^{1-2\alpha} \rppp + \lppp \EE \lvvv \Sigma_N^* -\Sigma \rvvv \rppp^{\frac{1}{2}} \PP^{\frac{1}{2}} \lppp V_N \cap \sopt \cap (\ssopt)^c\rppp,\\
\end{array}
\label{eq:cor1_6}
\end{equation}
where the 2nd step is based on Theorem \ref{thm:covmat} and the last step is based on the Cauchy's inequality. For any $\tau_3>0$, Wwe can derive a bound on $\PP \lppp V_N \cap \sopt \cap (\ssopt)^c\rppp$ as follows:
\begin{equation}
\arraycolsep=1.1pt\def\arraystretch{1.5} 
\begin{array}{rl}
& \PP \lppp V_N \cap \sopt \cap (\ssopt)^c\rppp\\

\leq & \PP \lppp V_N \cap A_N(\opt,\sqrt{3} r_{good}) \cap (\ssopt)^c\rppp + \PP \lppp \sopt \cap A_N(\opt,\sqrt{3} r_{good})^c \rppp\\

= & \PP \lppp V_N \cap A_N(\opt,\sqrt{3} r_{good}) \cap (\ssopt)^c\rppp + O\lppp N^{-\tau_3}\rppp\\

\leq & \PP \lpp \lwww \lvvv \bbdtheta_N - \opt \rvvv \leq 3\sqrt{3} r_{good}\rwww \cap (\ssopt)^c \rpp + O\lppp N^{-\tau_3}\rppp\\

\leq & \PP \lpp A_N^*(\opt, 9 r_{good}) \cap (\ssopt)^c \rpp + O\lppp N^{-\tau_3}\rppp\\

= & O\lppp N^{-\tau_3}\rppp,\\
\end{array}
\label{eq:cor1_7}
\end{equation}
where the 2nd step is based on Lemma \ref{lemma:mustbound}, the 4th step is based on Lemma \ref{lemma:trap} and the last step is based on Lemma \ref{lemma:mustconverge}. Based on (\ref{eq:cor1_6}), (\ref{eq:cor1_7}), Condition \ref{cm_3} and Lemma \ref{lemma:momentbound}, we can conclude that
\begin{equation}
\EE \lss \lvvv \EE_* \lppp\lppp \Sigma_N^* -\Sigma \rppp \daone_{V_N} \rppp \rvvv \daone_{\sopt} \rss = O\lppp N^{1-2\alpha} \rppp.
\label{eq:cor1_8}
\end{equation}
Based on (\ref{eq:cor1_4}), (\ref{eq:cor1_5}) and (\ref{eq:cor1_8}), we have
\begin{equation*}
\uplimN \PP \lpp z_{\frac{q}{2}} \sqrt{\bm{a}^T \hat{\Sigma}_N \bm{a}} > z_{\frac{q}{2}-\varepsilon} \sqrt{\bm{a}^T \Sigma \bm{a} } \big| \sopt \rpp \leq \varepsilon_1.
\label{eq:cor1_9}
\end{equation*}
As $\varepsilon_1$ can be arbitrarily small, we know that
\begin{equation}
\limN \PP \lpp z_{\frac{q}{2}} \sqrt{\bm{a}^T \hat{\Sigma}_N \bm{a}} > z_{\frac{q}{2}-\varepsilon} \sqrt{\bm{a}^T \Sigma \bm{a} } \big| \sopt \rpp =0.
\label{eq:cor1_10}
\end{equation}
Based on (\ref{eq:cor1_1}), (\ref{eq:cor1_2}) and (\ref{eq:cor1_10}), we have
$$
\uplimN \PP \lpp \bm{a}^T\opt \in \Big[ \bm{a}^T \bar{\bdtheta}_N - z_{\frac{q}{2}} \sqrt{ \frac{\bm{a}^T \hat{\Sigma}_N \bm{a} }{N}}, \bm{a}^T \bar{\bdtheta}_N + z_{\frac{q}{2}} \sqrt{ \frac{\bm{a}^T \hat{\Sigma}_N \bm{a} }{N}} \Big] \big| \sopt \rpp \leq 1-q+2\varepsilon.
$$
Given that $\varepsilon$ can be arbitrarily small, we can assert that
$$
\uplimN \PP \lpp \bm{a}^T\opt \in \Big[ \bm{a}^T \bar{\bdtheta}_N - z_{\frac{q}{2}} \sqrt{ \frac{\bm{a}^T \hat{\Sigma}_N \bm{a} }{N}}, \bm{a}^T \bar{\bdtheta}_N + z_{\frac{q}{2}} \sqrt{ \frac{\bm{a}^T \hat{\Sigma}_N \bm{a} }{N}} \Big] \big| \sopt \rpp \leq 1-q.
$$
The lower limit can be shown in a similar way. Therefore, we know that
$$
\limN \PP \lpp \bm{a}^T\opt \in \Big[ \bm{a}^T \bar{\bdtheta}_N - z_{\frac{q}{2}} \sqrt{ \frac{\bm{a}^T \hat{\Sigma}_N \bm{a} }{N}}, \bm{a}^T \bar{\bdtheta}_N + z_{\frac{q}{2}} \sqrt{ \frac{\bm{a}^T \hat{\Sigma}_N \bm{a} }{N}} \Big] \big| \sopt \rpp = 1-q.
$$

\end{proof}

\subsection{Related Lemmas}

\begin{lemma}
Suppose that conditions \ref{cm_1} and \ref{cm_2} are satisfied. The stepsize parameter $C$ is chosen such that $CC_s(1+2\csg\beta_{sg}) \leq \frac{1}{2}$ and $C\leq 1$. For $\tau \ge 1$, if $\tiln$ is sufficiently large, we have
$$
\ssum{n=0}{\tiln-1}\gamma_{n+1}\| \nablaf{n}\|^2 = O\lppp (\tau\log \tiln)^{\frac{2}{2-\beta_{sg}}}\rppp
$$
and
$$
F(\bdtheta_{\tiln})-\fmin = O\lppp (\tau\log \tiln)^{\frac{2}{2-\beta_{sg}}}\rppp
$$
with probability at least $1-\tiln^{-\tau}$.

To be more specific, we define the following constants,
$$
\arraycolsep=1.1pt\def\arraystretch{1.5}  
\begin{array}{rl}
c_1 \triangleq &F(\bdtheta_0) - F_{\min} + \frac{1}{2\alpha-1}C^2C_s C_{ng}^{(2)},\\

c_8 \triangleq & \frac{48(2-\beta_{sg})\alpha C^2}{2\alpha-1} (2\beta_{sg})^{\frac{\beta_{sg}}{2-\beta_{sg}}} (8C_s\csg)^{\frac{2}{2-\beta_{sg}}},\\

c_9 \triangleq & \max \lww \lpp \frac{c_8}{4\cjin (\csg)^{\frac{1}{2}}}\rpp^{\frac{4}{2+\beta_{sg}}} , (12\cjin (3\csg)^{\frac{1}{2}})^{\frac{4}{2-\beta_{sg}}}\rww.\\
\end{array}
$$
We define
$$
\arraycolsep=1.1pt\def\arraystretch{1.5} 
\begin{array}{rll}
\cnlemma{6} &\triangleq \max \Big\{ & (3c_1)^{-\frac{1}{\tau}},6,e^2,2\log \lpp 2(3c_1)^{\frac{2+\beta_{sg}}{2}}\rpp ,4(2+\beta_{sg})^2,\frac{1}{3c_1} \lpp \frac{e}{2} \rpp^{\frac{2}{2+\beta_{sg}}},\\

&&\exp \lpp \lpp \frac{2}{2-\beta_{sg}}\rpp^{\frac{2-\beta_{sg}}{\beta_{sg}}}  (16\beta_{sg} C_s\csg )^{-1}\rpp, \exp \lppp c_9^{-\frac{2-\beta_{sg}}{2}} \rppp,\\

&& \exp \lpp (16\beta_{sg} C_s\csg)^{-\frac{\beta_{sg}}{2}}(2-\beta_{sg})^{-\frac{2-\beta_{sg}}{2}} \lppp1-\frac{\beta_{sg}}{4}\rppp^{\frac{2-\beta_{sg}}{2}}\rpp,\\

&& \exp \lpp \lpp \frac{(2\alpha-1)(F(\bdtheta_0)-\fmin)}{4(2-\beta_{sg})\alpha C^2}\rpp^{\frac{2-\beta_{sg}}{2}} (2\beta_{sg})^{-\frac{\beta_{sg}}{2}} (8C_s\csg )^{-1}\rpp,\\

&& \exp \lpp \lpp \frac{(2\alpha-1)\cjin(3\csg)^{\frac{1}{2}}}{4(2-\beta_{sg})\alpha C^2}\rpp^{\frac{2(2-\beta_{sg})}{2+\beta_{sg}}} (2\beta_{sg})^{-\frac{2\beta_{sg}}{2+\beta_{sg}}} (8C_s\csg)^{-\frac{4}{2+\beta_{sg}}}\rpp\Big\}.
\end{array}
$$
If we suppose that $\tiln \ge \cnlemma{6}$, we have
$$
\ssum{n=0}{\tiln-1}\gamma_{n+1}\| \nablaf{n}\|^2 \leq c_9 (\tau \log \tiln )^{\frac{2}{2-\beta_{sg}}}
$$
and
$$
F(\bdtheta_{\tiln})-\fmin \leq \frac{1}{4} \lpp 8\cjin (3\csg)^{\frac{1}{2}}c_9^{\frac{2+\beta_{sg}}{4}} + c_8\rpp (\tau \log \tiln)^{\frac{2}{2-\beta_{sg}}}
$$
with probability at least $1-\tiln^{-\tau}$. For future reference, we define
$$
c_{fb} \triangleq \frac{1}{4} \lpp 8\cjin (3\csg)^{\frac{1}{2}}c_9^{\frac{2+\beta_{sg}}{4}} + c_8\rpp.
$$
    \label{lemma:fbound}
\end{lemma}
\begin{proof}
Let's firstly present the simplified implications of the sophisticated requirements on $\tiln$.
\begin{enumerate} 
    \itemequationsmall[eq:fbsimple1]{}{$$\tiln \ge \max \lww 2\log \lpp 2(3c_1)^{\frac{2+\beta_{sg}}{2}}\rpp ,4(2+\beta_{sg})^2 \rww \Longrightarrow \log \tiln^{\tau} \ge  \log \log \lpp 2(3c_1)^{\frac{2+\beta_{sg}}{2}} (\tiln^{\tau})^{\frac{2+\beta_{sg}}{2}} \rpp.$$}
    
    \itemequationnormal[eq:fbsimple2]{}{$$\tiln \ge \frac{1}{3c_1} \lpp \frac{e}{2} \rpp^{\frac{2}{2+\beta_{sg}}}
    \Longrightarrow \log \lpp 2(3c_1)^{\frac{2+\beta_{sg}}{2}} (\tiln^{\tau})^{\frac{2+\beta_{sg}}{2}} \rpp \ge 1.$$}

    \itemequationnormal[eq:fbsimple3]{}{$$\tiln \ge
    \exp \lpp \lpp \frac{2}{2-\beta_{sg}}\rpp^{\frac{2-\beta_{sg}}{\beta_{sg}}}  (16\beta_{sg} C_s\csg )^{-1}\rpp
    \Longrightarrow
    \frac{2-\beta_{sg}}{2} (16\beta_{sg}C_s\csg\tau \log \tiln)^{\frac{\beta_{sg}}{2-\beta_{sg}}} \ge 1.$$}

    \itemequationnormal[eq:fbsimple4]{}{$$ \arraycolsep=1.1pt\def\arraystretch{1.5}  
    \begin{array}{rl}
    & \tiln \ge
    \exp \lpp (16\beta_{sg} C_s\csg)^{-\frac{\beta_{sg}}{2}}(2-\beta_{sg})^{-\frac{2-\beta_{sg}}{2}} \lppp1-\frac{\beta_{sg}}{4}\rppp^{\frac{2-\beta_{sg}}{2}}\rpp\\
    \Longrightarrow&
    8C_s\csg\lppp 1-\frac{\beta_{sg}}{4}\rppp \leq  (2-\beta_{sg})(2\beta_{sg})^{\frac{\beta_{sg}}{2-\beta_{sg}}} (8C_s\csg \tau \log \tiln )^{\frac{2}{2-\beta_{sg}}}.\\
    \end{array}$$}

    \itemequationnormal[eq:fbsimple5]{}{$$\arraycolsep=1.1pt\def\arraystretch{1.5}  
    \begin{array}{rl}
    & \tiln \ge \exp \lpp \lpp \frac{(2\alpha-1)(F(\bdtheta_0)-\fmin)}{4(2-\beta_{sg})\alpha C^2}\rpp^{\frac{2-\beta_{sg}}{2}} (2\beta_{sg})^{-\frac{\beta_{sg}}{2}} (8C_s\csg )^{-1}\rpp\\
    \Longrightarrow&
    F(\bdtheta_0) - \fmin \leq \frac{4(2-\beta_{sg})\alpha C^2}{2\alpha-1} (2\beta_{sg})^{\frac{\beta_{sg}}{2-\beta_{sg}}} (8C_s\csg)^{\frac{2}{2-\beta_{sg}}} (\tau \log \tiln )^{\frac{2}{2-\beta_{sg}}}.\\
    \end{array}$$}

    \itemequationnormal[eq:fbsimple6]{}{$$\arraycolsep=1.1pt\def\arraystretch{1.5}  
    \begin{array}{rl}
    & \tiln \ge \exp \lpp \lpp \frac{(2\alpha-1)\cjin(3\csg)^{\frac{1}{2}}}{4(2-\beta_{sg})\alpha C^2}\rpp^{\frac{2(2-\beta_{sg})}{2+\beta_{sg}}} (2\beta_{sg})^{-\frac{2\beta_{sg}}{2+\beta_{sg}}} (8C_s\csg)^{-\frac{4}{2+\beta_{sg}}}\rpp\\
    \Longrightarrow&
    \cjin (3\csg)^{\frac{1}{2}} \leq \frac{4(2-\beta_{sg})\alpha C^2}{2\alpha-1} (2\beta_{sg})^{\frac{\beta_{sg}}{2-\beta_{sg}}} (8C_s\csg)^{\frac{2}{2-\beta_{sg}}} (\tau \log \tiln )^{\frac{2+\beta_{sg}}{2(2-\beta_{sg})}}.\\
    \end{array}$$}

    \itemequationnormal[eq:fbsimple7]{}{$$ 
    \tiln \ge \exp \lppp c_9^{-\frac{2-\beta_{sg}}{2}} \rppp 
    \Longrightarrow
    c_9(\tau \log \tiln )^{\frac{2}{2-\beta_{sg}}} \ge 1.$$}

    \itemequationnormal[eq:fbsimple8]{}{$$ 
    \tiln \ge \exp \lpp c_9^{-\frac{2-\beta_{sg}}{2}}\rpp
    \Longrightarrow
    (\tau \log \tiln)^{\frac{1}{2-\beta_{sg}}} \leq c_9^{\frac{\beta_{sg}}{4}} (\tau \log \tiln)^{\frac{2+\beta_{sg}}{2(2-\beta_{sg})}}.
    $$}

\end{enumerate}
In the current proof, we denote
$$
\bdxi_{n+1} \triangleq \nabla f(\bdtheta_{n};\bdy_{n+1}) - \nablaf{n},\ n\ge0.
$$
Then, we have
\begin{equation}
\arraycolsep=1.1pt\def\arraystretch{1.5}  
\begin{array}{rl}
& F(\bdtheta_{n+1}) - F(\bdtheta_n)\\

\leq &\langle \nablaf{n}, \bdtheta_{n+1} - \bdtheta_n\rangle +\frac{C_s}{2} \| \bdtheta_{n+1} - \bdtheta_n\|^2\\

=& -\gamma_{n+1} \| \nablaf{n} \|^2 - \gamma_{n+1} \langle \nablaf{n}, \bdxi_{n+1} \rangle + \frac{C_s}{2} \gamma_{n+1}^2 \| \nablaf{n} + \bdxi_{n+1}\|^2\\

\leq & -\gamma_{n+1} (1-C_s\gamma_{n+1})\| \nablaf{n}\|^2 + C_s\gamma_{n+1}^2 \| \bdxi_{n+1}\|^2 - \gamma_{n+1} \langle \nablaf{n}, \bdxi_{n+1} \rangle\\

\leq & -\gamma_{n+1} (1-C_s\gamma_{n+1} - 2C_s \csg \beta_{sg} \gamma_{n+1})\|\nablaf{n}\|^2 + 8C_s\csg \lppp1-\frac{\beta_{sg}}{4}\rppp \gamma_{n+1}^2\\

&+ C_s \gamma_{n+1}^2 (\|\bdxi_{n+1}\|^2 - \EE_n \| \bdxi_{n+1}\|^2 ) - \gamma_{n+1} \langle \nablaf{n}, \bdxi_{n+1} \rangle\\

\leq& -\frac{\gamma_{n+1}}{2} \|\nablaf{n}\|^2 + 8C_s\csg \lppp1-\frac{\beta_{sg}}{4}\rppp \gamma_{n+1}^2\\

&+ C_s \gamma_{n+1}^2 (\|\bdxi_{n+1}\|^2 - \EE_n \| \bdxi_{n+1}\|^2 ) - \gamma_{n+1} \langle \nablaf{n}, \bdxi_{n+1} \rangle,\\
\end{array}
\label{eq:fb2}
\end{equation}
where the 1st step is based on condition \ref{cm_1}, the 2nd to the last step is based on Lemma \ref{lemma:errorbound} and the last step is based on the assumption that $CC_s(1+2\csg\beta_{sg}) \leq \frac{1}{2}$.

For simplicity, we let
\begin{equation}
\arraycolsep=1.1pt\def\arraystretch{1.5}  
\begin{array}{rl}
\cnn \triangleq & 2\csg (3c_1)^{\frac{2+\beta_{sg}}{2}}\tiln^{\frac{2+\beta_{sg}}{2}\tau},\\

\cnnn \triangleq & \lppp \log(6\tiln^{\tau}) + \log\log \frac{\cnn}{\csg}\rppp^{\frac{1}{2}}.\\
\end{array}
\label{eq:fbdefine1}
\end{equation}
Under condition \ref{cm_2}, based on Lemma \ref{lemma:jin}, with probability at least $1-\frac{1}{3}\tiln^{-\tau}$, either
\begin{equation}
\ssum{n=0}{\tiln-1} \csg \gamma_{n+1}^2 \| \nablaf{n}\|^2 (1+\| \nablaf{n}\|^{\beta_{sg}}) \ge \cnn,
\label{eq:fb3}
\end{equation}
or
\begin{equation}
\arraycolsep=1.1pt\def\arraystretch{1.5}  
\begin{array}{rl}
& -\ssum{n=0}{\tiln-1}\gamma_{n+1}\langle \nablaf{n},\bdxi_{n+1}\rangle\\

\leq & \cjin \lpp \max \lww \ssum{n=0}{\tiln-1} \csg \gamma_{n+1}^2 \| \nablaf{n}\|^2 (1+\| \nablaf{n}\|^{\beta_{sg}}), \csg\rww  \rpp^{\frac{1}{2}}\cnnn\\

\leq & \cjin \lpp \lpp \ssum{n=0}{\tiln-1} \csg \gamma_{n+1}^2 \| \nablaf{n}\|^2 \rpp^{\frac{1}{2}} + \lpp \ssum{n=0}{\tiln-1} \csg \gamma_{n+1}^2 \| \nablaf{n}\|^{2+\beta_{sg}} \rpp^{\frac{1}{2}} + (\csg)^{\frac{1}{2}} \rpp\cnnn. \\
\end{array}
\label{eq:fb4}
\end{equation}
In fact, we can use the Markov inequality to show that (\ref{eq:fb3}) occurs with a small probability. For simplicity, we denote
$$
\arraycolsep=1.1pt\def\arraystretch{1.5}  
\begin{array}{rl}
A_{\tiln} \triangleq &\ssum{n=0}{\tiln-1}\gamma_{n+1}\| \nablaf{n}\|^2.\\
\end{array}
$$
Then, based on Lemma \ref{lemma:nablasumbound}, we have
$$
\EE A_{\tiln} \leq c_1,
$$
which implies that 
\begin{equation}
\PP \lpp A_{\tiln} <  \frac{\cnn}{2\csg} \rpp \ge 1-c_1 \lpp \frac{\cnn}{2\csg}\rpp^{-1} = 1-\frac{1}{3}\tiln^{-\tau}.
\label{eq:fb5}
\end{equation}
As $\tiln \ge (3c_1)^{-\frac{1}{\tau}}$, we have $\cnn \ge 2\csg$. Therefore, $A_{\tiln} <  \lpp \frac{\cnn}{2\csg}\rpp^{\frac{2}{2+\beta_{sg}}}$ implies $A_{\tiln} <  \frac{\cnn}{2\csg}$ and further
\begin{equation}
 \ssum{n=0}{\tiln-1} \csg \gamma_{n+1}^2 \| \nablaf{n}\|^2 (1+\| \nablaf{n}\|^{\beta_{sg}}) \leq \csg\lpp A_{\tiln} + A_{\tiln}^{\frac{2+\beta_{sg}}{2}}\rpp < \cnn.
\label{eq:fb6}
\end{equation}
Based on the results given in (\ref{eq:fb3}), (\ref{eq:fb4}), (\ref{eq:fb5}) and (\ref{eq:fb6}), we have
\begin{equation}
\PP \lpp -\ssum{n=0}{\tiln-1}\gamma_{n+1}\langle \nablaf{n},\bdxi_{n+1}\rangle \leq \cjin (\csg)^{\frac{1}{2}}\cnnn \lppp A_{\tiln}^{\frac{1}{2}} + A_{\tiln}^{\frac{2+\beta_{sg}}{4}} + 1\rppp \rpp \ge 1-\frac{2}{3}\tiln^{-\tau}.
\label{eq:fb7}
\end{equation}

Next, for any $n\in \ZZ_+$ and $t\in \fff_n$, we have
\begin{equation}
\arraycolsep=1.1pt\def\arraystretch{1.5}  
\begin{array}{rl}
& \PP_n \lppp \big| \| \bdxi_{n+1}\|^2 - \EE_n\|\bdxi_{n+1}\|^2 \big| \ge t\rppp\\

\leq & \PP_n \lppp \| \bdxi_{n+1}\|^2 \ge t - 4\csg \lppp 1+\|\nablaf{n}\|^{\beta_{sg}}\rppp \rppp\\

\leq & 2\exp \lpp -\frac{ t-4\csg \lppp 1+\|\nablaf{n}\|^{\beta_{sg}}\rppp}{2\csg \lppp 1+\|\nablaf{n}\|^{\beta_{sg}}\rppp}\rpp\\

= & 2e^2 \exp\lpp -\frac{t}{2\csg \lppp 1+\|\nablaf{n}\|^{\beta_{sg}}\rppp}\rpp,\\
\end{array}
\label{eq:expbound}
\end{equation}
where the 1st step is based on Lemma \ref{lemma:errorbound} and the 2nd step is based on condition \ref{cm_2}. Therefore, if we let $t=2\csg(1+\|\nablaf{n}\|^{\beta_{sg}}) \lppp (1+\tau) \log \tiln +\log 6+2\rppp$ in the above inequality, we have
\begin{equation}
\PP \lpp \big| \| \bdxi_{n+1}\|^2 - \EE_n\|\bdxi_{n+1}\|^2 \big| \ge 2\csg(1+\|\nablaf{n}\|^{\beta_{sg}}) \lppp (1+\tau) \log \tiln +\log 6+2\rppp \rpp \leq 1-\frac{1}{3}\tiln^{-(1+\tau)}.
    \label{eq:fb8}
\end{equation}
For simplicity, we let
$$
\arraycolsep=1.1pt\def\arraystretch{1.5}  
\begin{array}{rl}
c_{4,\tiln} \triangleq& 2C_s\csg \lppp (1+\tau) \log \tiln +\log 6+2\rppp.\\
\end{array}
$$
It implies that with probability at least $1-\frac{1}{3}\tiln^{-\tau}$,
\begin{equation}
\arraycolsep=1.1pt\def\arraystretch{1.5}  
\begin{array}{rl}
& \ssum{n=0}{\tiln-1}C_s\gamma_{n+1}^2 \lppp \| \bdxi_{n+1}\|^2 - \EE_n \| \bdxi_{n+1}\|^2\rppp\\

\leq & c_{4,\tiln} \ssum{n=0}{\tiln-1} \gamma_{n+1}^2 (1+\|\nablaf{n}\|^{\beta_{sg}})\\

\leq &c_{4,\tiln} \ssum{n=0}{\tiln-1} \gamma_{n+1}^2\lpp 1 + \frac{2-\beta_{sg}}{2} (2c_{4,\tiln}\beta_{sg})^{\frac{\beta_{sg}}{2-\beta_{sg}}} + \frac{\beta_{sg}}{2} (2c_{4,\tiln}\beta_{sg})^{-1} \|\nablaf{n}\|^2\rpp\\

= & c_{4,\tiln} \lpp 1 + \frac{2-\beta_{sg}}{2} (2c_{4,\tiln}\beta_{sg})^{\frac{\beta_{sg}}{2-\beta_{sg}}}\rpp \ssum{n=0}{\tiln-1} \gamma_{n+1}^2 + \frac{1}{4}A_{\tiln},\\
\end{array}
\label{eq:fb9}
\end{equation}
where the 2nd step is based on the Young's inequality.

Based on (\ref{eq:fb2}), (\ref{eq:fb7}) and (\ref{eq:fb9}), we know that with probability at least $1-\tiln^{-\tau}$,
\begin{equation}
\resizebox{.92\hsize}{!}{$
\arraycolsep=1.1pt\def\arraystretch{1.5}  
\begin{array}{rl}
F(\bdtheta_{\tiln}) - F(\bdtheta_0) & \leq  -\frac{1}{2} A_{\tiln} + 8C_s\csg\lppp 1-\frac{\beta_{sg}}{4}\rppp \ssum{n=0}{\tiln-1}\gamma_{n+1}^2 + c_{4,\tiln} \lpp 1 + \frac{2-\beta_{sg}}{2} (2c_{4,\tiln}\beta_{sg})^{\frac{\beta_{sg}}{2-\beta_{sg}}}\rpp \ssum{n=0}{\tiln-1} \gamma_{n+1}^2\\

&\ \ \ + \frac{1}{4}A_{\tiln}+ \cjin (\csg)^{\frac{1}{2}}\cnnn \lppp A_{\tiln}^{\frac{1}{2}} + A_{\tiln}^{\frac{2+\beta_{sg}}{4}} + 1\rppp\\

& \leq -\frac{1}{4}A_{\tiln} + \cjin (\csg)^{\frac{1}{2}}\cnnn \lppp A_{\tiln}^{\frac{1}{2}} + A_{\tiln}^{\frac{2+\beta_{sg}}{4}}\rppp+  \cjin (\csg)^{\frac{1}{2}}\cnnn\\

&\ \ \  + \lpp 8C_s\csg\lppp 1-\frac{\beta_{sg}}{4}\rppp + c_{4,\tiln} \lpp 1 + \frac{2-\beta_{sg}}{2} (2c_{4,\tiln}\beta_{sg})^{\frac{\beta_{sg}}{2-\beta_{sg}}}\rpp \rpp \ssum{n=0}{\infty} \gamma_{n+1}^2.\\
\end{array}
$}
\label{eq:fbadd0}
\end{equation}
For convenience, we denote the above event by $S_{\tiln}$ in the current proof. Then, if we let
$$
\resizebox{.98\hsize}{!}{$
\arraycolsep=1.1pt\def\arraystretch{1.5}  
\begin{array}{rl}
c_{5,\tiln} \triangleq& 4\lpp F(\bdtheta_0) - F_{\min} + \cjin (\csg)^{\frac{1}{2}}\cnnn + \lpp 8C_s\csg\lppp 1-\frac{\beta_{sg}}{4}\rppp + c_{4,\tiln} \lpp 1 + \frac{2-\beta_{sg}}{2} (2c_{4,\tiln}\beta_{sg})^{\frac{\beta_{sg}}{2-\beta_{sg}}}\rpp \rpp \ssum{n=0}{\infty} \gamma_{n+1}^2 \rpp,\\
c_{6,\tiln} \triangleq& 4\cjin (\csg)^{\frac{1}{2}}\cnnn,
\end{array}
$}
$$
as implied by (\ref{eq:fbadd0}), on $S_{\tiln}$,
\begin{equation}
A_{\tiln} \leq c_{6,\tiln} A_{\tiln}^{\frac{1}{2}} + c_{6,\tiln} A_{\tiln}^{\frac{2+\beta_{sg}}{4}} + c_{5,\tiln}.
    \label{eq:fb10}
\end{equation}
We can see that if $A_{\tiln}\ge \max \lww 1, \lppp \frac{c_{5,\tiln}}{c_{6,\tiln}} \rppp^{\frac{4}{2+\beta_{sg}}} \rww$, 
$$
c_{6,\tiln} A_{\tiln}^{\frac{2+\beta_{sg}}{4}} \ge \max \lww c_{6,\tiln} A_{\tiln}^{\frac{1}{2}}, c_{5,\tiln}\rww.
$$
Meanwhile, (\ref{eq:fb10}) implies $A_{\tiln} \leq 3c_{6,\tiln} A_{\tiln}^{\frac{2+\beta_{sg}}{4}}$ and consequently $A_{\tiln} \leq (3 c_{6,\tiln})^{\frac{4}{2-\beta_{sg}}}$. Therefore, we can conclude that on $S_{\tiln}$,
$$
A_{\tiln} \leq \max \lww 1, \lppp \frac{c_{5,\tiln}}{c_{6,\tiln}} \rppp^{\frac{4}{2+\beta_{sg}}} , (3 c_{6,\tiln})^{\frac{4}{2-\beta_{sg}}}\rww\triangleq c_{7,\tiln}.
$$
In turn, on $S_{\tiln}$, we have
\begin{equation}
\arraycolsep=1.1pt\def\arraystretch{1.5}  
\begin{array}{rl}
F(\bdtheta_{\tiln})-F_{\min}  & \leq \frac{1}{4} \lppp  c_{6,\tiln} A_{\tiln}^{\frac{1}{2}} + c_{6,\tiln} A_{\tiln}^{\frac{2+\beta_{sg}}{4}} +  c_{5,\tiln} \rppp \\

& \leq \frac{1}{4} \lppp c_{6,\tiln} c_{7,\tiln}^{\frac{1}{2}} + c_{6,\tiln}c_{7,\tiln}^{\frac{2+\beta_{sg}}{4}} + c_{5,\tiln} \rppp.\\
\end{array}
\label{eq:fbadd1}
\end{equation}
Based on the definitions given previously, we can see that 
$$
\cnnn = O((\tau\log \tiln)^{\frac{1}{2}})\ \ and\  \ c_{4,\tiln} = O(\tau \log \tiln).
$$
Therefore, 
$$
c_{5,\tiln}=O\lppp (\tau\log \tiln)^{\frac{2}{2-\beta_{sg}}}\rppp, c_{6,\tiln} = O((\tau\log \tiln)^{\frac{1}{2}})\ and\ c_{7,\tiln} = O\lppp (\tau\log \tiln)^{\frac{2}{2-\beta_{sg}}}\rppp.
$$
At last, we can conclude that with probability at least $1-\tiln^{-\tau}$,
$$
F(\bdtheta_{\tiln})-F_{\min} = O\lppp (\tau\log \tiln)^{\frac{2}{2-\beta_{sg}}}\rppp.
$$
To derive a more explicit form, based on (\ref{eq:fbsimple1}), (\ref{eq:fbsimple2}) and $\tiln \ge 6$, we firstly have
\begin{equation}
(\tau \log \tiln)^{\frac{1}{2}} \leq \cnnn\leq (3\tau \log \tiln)^{\frac{1}{2}}.
    \label{eq:fb11}
\end{equation}
As $\tiln \ge 6$, $\tiln \ge e^2$, we can see that
\begin{equation}
c_{4,\tiln} \leq 8C_s \csg \tau \log \tiln.
    \label{eq:fb12}
\end{equation}
$c_{5,\tiln}$ can also be controlled by the following simplified bound,
\begin{equation}
\resizebox{.92\hsize}{!}{$
\arraycolsep=1.1pt\def\arraystretch{1.5} 
\begin{array}{rl}
c_{5,\tiln} & \leq 4\lpp F(\bdtheta_0) - F_{\min} + \cjin (\csg)^{\frac{1}{2}}\cnnn + \lpp 8C_s\csg\lppp 1-\frac{\beta_{sg}}{4}\rppp + (2-\beta_{sg})(2\beta_{sg})^{\frac{\beta_{sg}}{2-\beta_{sg}}} (8C_s\csg \tau \log \tiln )^{\frac{2}{2-\beta_{sg}}} \rpp \ssum{n=0}{\infty}\gamma_{n+1}^{2}\rpp\\

& \leq 4\lpp F(\bdtheta_0) - F_{\min} + \cjin (\csg)^{\frac{1}{2}}\cnnn + 2 (2-\beta_{sg})(2\beta_{sg})^{\frac{\beta_{sg}}{2-\beta_{sg}}} (8C_s\csg \tau \log \tiln )^{\frac{2}{2-\beta_{sg}}}  \ssum{n=0}{\infty}\gamma_{n+1}^{2}\rpp\\

& \leq 4\lpp F(\bdtheta_0) - F_{\min} + \cjin (3\csg)^{\frac{1}{2}} (\tau \log \tiln )^{\frac{1}{2}} + \frac{4(2-\beta_{sg})\alpha C^2}{2\alpha-1} (2\beta_{sg})^{\frac{\beta_{sg}}{2-\beta_{sg}}} (8C_s\csg)^{\frac{2}{2-\beta_{sg}}} (\tau \log \tiln )^{\frac{2}{2-\beta_{sg}}} \rpp\\

& \leq \frac{48(2-\beta_{sg})\alpha C^2}{2\alpha-1} (2\beta_{sg})^{\frac{\beta_{sg}}{2-\beta_{sg}}} (8C_s\csg)^{\frac{2}{2-\beta_{sg}}} (\tau \log \tiln )^{\frac{2}{2-\beta_{sg}}},\\
\end{array}
$}
\label{eq:fb13}
\end{equation}
where the 1st step is based on (\ref{eq:fbsimple3}) and (\ref{eq:fb12}), the 2nd step is based on (\ref{eq:fbsimple4}), the 3rd step is based on (\ref{eq:fb11}) and $\ssum{n=0}{\infty}\gamma_{n+1}^{2} \leq C^2\lppp 1+\int_1^{\infty}x^{-\alpha}dx\rppp =\frac{2\alpha C^2}{2\alpha-1}$, and the last step is based on (\ref{eq:fbsimple5}) and (\ref{eq:fbsimple6}). For simplicity, we let
\begin{equation}
c_8 \triangleq \frac{48(2-\beta_{sg})\alpha C^2}{2\alpha-1} (2\beta_{sg})^{\frac{\beta_{sg}}{2-\beta_{sg}}} (8C_s\csg)^{\frac{2}{2-\beta_{sg}}}.
    \label{eq:fb14}
\end{equation}
Based on (\ref{eq:fb11}), we have
\begin{equation}
4\cjin (\csg)^{\frac{1}{2}} (\tau \log \tiln)^{\frac{1}{2}} \leq c_{6,\tiln} \leq 4\cjin (3\csg)^{\frac{1}{2}} (\tau \log \tiln)^{\frac{1}{2}}.
    \label{eq:fb15}
\end{equation}
Based on (\ref{eq:fb13}) and (\ref{eq:fb15}), we have
\begin{equation}
\lpp \frac{c_{5,\tiln}}{c_{6,\tiln}}\rpp^{\frac{4}{2+\beta_{sg}}} \leq \lpp \frac{c_8}{4\cjin (\csg)^{\frac{1}{2}}}\rpp^{\frac{4}{2+\beta_{sg}}} (\tau \log \tiln )^{\frac{2}{2-\beta_{sg}}}.
    \label{eq:fb16}
\end{equation}
Based on (\ref{eq:fb15}), we have
\begin{equation}
(3c_{6,\tiln})^{\frac{4}{2-\beta_{sg}}} \leq (12\cjin (3\csg)^{\frac{1}{2}})^{\frac{4}{2-\beta_{sg}}} (\tau \log \tiln )^{\frac{2}{2-\beta_{sg}}}.
\label{eq:fb17}
\end{equation}
Therefore, if we let
\begin{equation}
c_9 \triangleq \max \lww \lpp \frac{c_8}{4\cjin (\csg)^{\frac{1}{2}}}\rpp^{\frac{4}{2+\beta_{sg}}} , (12\cjin (3\csg)^{\frac{1}{2}})^{\frac{4}{2-\beta_{sg}}}\rww,
\label{eq:fb18}
\end{equation}
based on (\ref{eq:fbsimple7}), (\ref{eq:fb16}) and (\ref{eq:fb17}), we have
$$
c_{7,\tiln} \leq c_9 (\tau \log \tiln )^{\frac{2}{2-\beta_{sg}}}.
$$
Then, based on (\ref{eq:fbadd1}), we have
\begin{equation*}
\arraycolsep=1.1pt\def\arraystretch{1.5}  
\begin{array}{rl}
F(\bdtheta_{\tiln})-F_{\min}  & \leq \frac{1}{4} \lpp 4\cjin (3\csg)^{\frac{1}{2}} c_9^{\frac{1}{2}}  \lpp (\tau \log \tiln)^{\frac{1}{2-\beta_{sg}} +\frac{1}{2}} + c_9^{\frac{\beta_{sg}}{4}} (\tau \log \tiln)^{\frac{2+\beta_{sg}}{2(2-\beta_{sg})}+\frac{1}{2}} \rpp + c_8 (\tau \log \tiln)^{\frac{2}{2-\beta_{sg}}}\rpp \\

& \leq \frac{1}{4} \lpp 8\cjin (3\csg)^{\frac{1}{2}}c_9^{\frac{2+\beta_{sg}}{4}} + c_8\rpp (\tau \log \tiln)^{\frac{2}{2-\beta_{sg}}},\\
\end{array}
\label{eq:fbadd2}
\end{equation*}
where the last step is based on (\ref{eq:fbsimple8}).

\end{proof}

\begin{lemma}
Suppose that conditions \ref{cm_1} and \ref{cm_2} hold. For $N\in \RR_+$, suppose that $\delta$ is in the form of $N^{-\tau}$ with $\tau > 0$. We let
$$
g_{good}^L \triangleq \sup\limits_{\bdtheta \in R_{good}^L} \| \nabla F(\bdtheta)\|.
$$
We define
$$
\arraycolsep=1.1pt\def\arraystretch{1.5} 
\begin{array}{rll}
\cnlemma{7}(\tau,N) &\triangleq \max \Big\{ & \lpp \frac{2CC_s \cng}{\lambdamin{}}\rpp^{1/\alpha}, C^{\frac{1}{\alpha}} ,\lpp \frac{4C C_{ng}^{(2)}}{\lambdamin{}r^2} \rpp^{\frac{1}{\alpha}},\lppp 2^{\frac{1}{2\alpha}}-1\rppp^{-1},\\

&&\lpp \frac{4\alpha}{C\lambdamin{}}\rpp^{\frac{1}{1-\alpha}}-2,
\lpp \frac{C\lambdamin{}}{2\log 2}\rpp^{\frac{1}{\alpha}},
\lpp
    \frac{192 \log(4e^2) (1+(g_{good}^L)^{\beta_{sg}})C\csg}{\lambdamin{} r^2} \rpp^{\frac{1}{\alpha}},\\

&& \lpp \frac{192 (1+(g_{good}^L)^{\beta_{sg}})C\csg}{\lambdamin{} r^2} \rpp^{\frac{1}{\alpha}} (\tau \log N)^{\frac{1}{\alpha}},
\lpp \frac{1152(1+(g_{good}^L)^{\beta_{sg}})C\csg}{\alpha\lambdamin{}r^2} \rpp^{\frac{2}{\alpha}},\\

&& \lpp \frac{48 c_{rg}\cjin}{r^2} \rpp^{\frac{1}{\alpha}} \lpp \frac{(\log 4)C\csg}{\lambdamin{}}\rpp^{\frac{1}{2\alpha}}, \lpp \frac{48 c_{rg}\cjin}{r^2} \rpp^{\frac{1}{\alpha}} \lpp \frac{C\csg}{\lambdamin{}}\rpp^{\frac{1}{2\alpha}} (\tau \log N)^{\frac{1}{2\alpha}},\\

&& \lpp \frac{48 c_{rg}\cjin}{r^2} \rpp^{\frac{2}{\alpha}} \lpp \frac{2C\csg}{\alpha \lambdamin{}} \rpp^{\frac{1}{\alpha}}
\Big\}\\

&=\Omega\lpp \lppp\tau &\log N\rppp^{1/\alpha} \rpp.\\
\end{array}
$$
Suppose that $\tiln \ge \cnlemma{7}(\tau,N)$. Then, for any $\opt \in \Theta^{opt}$, for any fixed $r\leq \frac{r_{good}^L}{\sqrt{3}}$, we have
$$
\PP \lpp \exists n \ge \tiln, \lv \bdtheta_{n} - \opt \rv > \sqrt{3}r \Big| \lv \bdtheta_{\tiln} - \opt \rv \leq r\rpp \leq \delta = N^{-\tau}.
$$

Particularly, when $\tiln$ is greater than some $\Omega\lpp \lppp\tau \log N\rppp^{1/\alpha} \rpp$, we have
$$
\PP \lpp \exists n \ge \tiln, \lv \bdtheta_{n} - \opt \rv > r_{good}^L \Big| \lv \bdtheta_{\tiln} - \opt \rv \leq r_{good}\rpp = O\lppp N^{-\tau} \rppp.
$$
    \label{lemma:trap}
\end{lemma}
\begin{proof}
Let's firstly present the simplified implications of the sophisticated requirements on $\tiln$.
\begin{enumerate}
    \itemequationnormal[eq:goodsimp1]{}{$$ \tiln \ge \lpp \frac{2CC_s \cng}{\lambdamin{}}\rpp^{1/\alpha}
    \Longrightarrow C_s \cng \gamma_{\tiln} \leq \frac{1}{2} \lambdamin{}.$$}

    \itemequationnormal[eq:goodsimp2]{}{$$ \tiln \ge C^{\frac{1}{\alpha}} 
    \Longrightarrow \gamma_{\tiln} \leq 1.$$}

    \itemequationnormal[eq:goodsimp3]{}{$$\tiln \ge \lpp \frac{4C C_{ng}^{(2)}}{\lambdamin{}r^2} \rpp^{\frac{1}{\alpha}}
    \Longrightarrow
    \frac{4}{\lambdamin{}}C_{ng}^{(2)}\gamma_{\tiln} \leq r^2.$$}

    \itemequationnormal[eq:goodsimp4]{}{$$\tiln \ge \lppp 2^{\frac{1}{2\alpha}}-1\rppp^{-1}
    \Longrightarrow
    \lpp \frac{\tiln+1}{\tiln}\rpp^{2\alpha} \leq 2.
    $$}

    \itemequationnormal[eq:goodsimp5]{}{$$ \tiln \ge \lpp \frac{4\alpha}{C\lambdamin{}}\rpp^{\frac{1}{1-\alpha}}-2
    \Longrightarrow
    \tiln +2 \ge \lpp \frac{4\alpha}{C\lambdamin{}}\rpp^{\frac{1}{1-\alpha}}.$$}

    \itemequationnormal[eq:goodsimp6]{}{$$ \tiln \ge \lpp \frac{C\lambdamin{}}{2\log 2}\rpp^{\frac{1}{\alpha}}
    \Longrightarrow
    \exp \lpp \frac{C\lambdamin{}}{2} \tiln^{-\alpha}\rpp \leq 2.
    $$}

    \itemequationsmall[eq:goodsimp7]{}{$$
    \tiln \ge \lpp
    \frac{192 \log(4e^2) (1+(g_{good}^L)^{\beta_{sg}})C\csg}{\lambdamin{} r^2} \rpp^{\frac{1}{\alpha}}
    \Longrightarrow
    \frac{32(1+(g_{good}^L)^{\beta_{sg}})C\csg}{\lambdamin{}} \log(4e^2) \tiln^{-\alpha}\leq \frac{r^2}{6}.$$}
    
    \itemequationsmall[eq:goodsimp8]{}{$$
    \tiln \ge 
    \lpp \frac{192 (1+(g_{good}^L)^{\beta_{sg}})C\csg}{\lambdamin{} r^2} \rpp^{\frac{1}{\alpha}} (\tau \log N)^{\frac{1}{\alpha}}
    \Longrightarrow
    \frac{32(1+(g_{good}^L)^{\beta_{sg}})C\csg}{\lambdamin{}} \tau (\log N) \tiln^{-\alpha}\leq \frac{r^2}{6}.
    $$}

    \itemequationnormal[eq:goodsimp9]{}{$$
    \tiln \ge 
    \lpp \frac{1152(1+(g_{good}^L)^{\beta_{sg}})C\csg}{\alpha\lambdamin{}r^2} \rpp^{\frac{2}{\alpha}}
    \Longrightarrow
    \frac{192(1+(g_{good}^L)^{\beta_{sg}})C\csg}{\alpha\lambdamin{}} \tiln^{-\frac{\alpha}{2}} \leq \frac{r^2}{6}.$$}

    \itemequationsmall[eq:goodsimp10]{}{$$
    \tiln \ge 
    \lpp \frac{48c_{rg}\cjin}{r^2} \rpp^{\frac{1}{\alpha}} \lpp \frac{(\log 4)C\csg}{\lambdamin{}}\rpp^{\frac{1}{2\alpha}} 
    \Longrightarrow
    8c_{rg}\cjin \lpp \frac{C\csg}{\lambdamin{}}\rpp^{\frac{1}{2}}(\log 4)^{\frac{1}{2}}\tiln^{-\alpha} \leq \frac{r^2}{6}.
    $$}

    \itemequationsmall[eq:goodsimp11]{}{$$
    \tiln \ge \lpp \frac{48 c_{rg}\cjin}{r^2} \rpp^{\frac{1}{\alpha}} \lpp \frac{C\csg}{\lambdamin{}}\rpp^{\frac{1}{2\alpha}} (\tau \log N)^{\frac{1}{2\alpha}}
    \Longrightarrow
    8c_{rg}\cjin \lpp \frac{C\csg}{\lambdamin{}}\rpp^{\frac{1}{2}}(\tau \log N)^{\frac{1}{2}}\tiln^{-\alpha} \leq \frac{r^2}{6}.
    $$}

    \itemequationnormal[eq:goodsimp12]{}{$$ \tiln \ge \lpp \frac{48 c_{rg}\cjin}{r^2} \rpp^{\frac{2}{\alpha}} \lpp \frac{2C\csg}{\alpha \lambdamin{}} \rpp^{\frac{1}{\alpha}}
    \Longrightarrow
    8c_{rg}\cjin \lpp \frac{C\csg}{\lambdamin{}}\rpp^{\frac{1}{2}}\lppp \frac{2}{\alpha}\rppp ^{\frac{1}{2}}\tiln^{-\frac{\alpha}{2}} \leq \frac{r^2}{6}.
    $$}
\end{enumerate}
In the present proof, for simplicity, we denote
$$
\rgor(\opt) \triangleq \lwww \bdtheta: \lvvv \bdtheta - \opt \rvvv \leq r \rwww,\ \ \rgorl(\opt) \triangleq \lwww \bdtheta: \lvvv \bdtheta - \opt \rvvv \leq \sqrt{3}r \rwww.
$$
For $n\ge \tiln$, we let
$$
\arraycolsep=1.1pt\def\arraystretch{1.5}
\begin{array}{ccl}
\bddelta_n &\triangleq & \bdtheta_n - \opt,\\
S_n &\triangleq& \lw \bdtheta_{n'} \in \rgorl(\opt)\text{ for all } \tiln + 1 \leq n' \leq n \text{ and }\bdtheta_{\tiln} \in \rgor(\opt) \rw.\\
\end{array}
$$
We also let 
$$
S_{\tiln-1} = S_{base} \triangleq \{\bdtheta_{\tiln} \in \rgor(\opt) \}.$$ We have
\begin{equation}
\resizebox{.92\hsize}{!}{$
\arraycolsep=1.1pt\def\arraystretch{1.5} 
\begin{array}{rl}
& \EE_{n-1} \lv \bddelta_n\rv^2 \daone_{S_{n-1}} \\

= & \EE_{n-1} \lv \bdtheta_{n-1} - \gamma_n \nabla f(\bdtheta_{n-1};\bdy_n) - \opt \rv^2 \daone_{S_{n-1}}\\

= & \lv \bddelta_{n-1} \rv^2 \daone_{S_{n-1}} + \gamma_n^2 \EE_{n-1} \lv \nabla f(\bdtheta_{n-1};\bdy_n) \rv^2 \daone_{S_{n-1}} - 2\gamma_n \EE_{n-1} \langle \bddelta_{n-1},\nabla f(\bdtheta_{n-1}; \bdy_n) \rangle \daone_{S_{n-1}}\\

\leq & \lv \bddelta_{n-1} \rv^2 \daone_{S_{n-1}} + C_{ng}^{(2)}\gamma_n^2 (1+\lv\nablaf{n-1} \rv^2) \daone_{S_{n-1}} - 2\gamma_n  \langle \bddelta_{n-1},\nabla F(\bdtheta_{n-1}) \rangle \daone_{S_{n-1}}\\

\leq& \lpp \lpp 1 - \lambdamin{}\gamma_n + C_s C_{ng}^{(2)}\gamma_n^2 \rpp \lv \bddelta_{n-1}\rv^2 + C_{ng}^{(2)}\gamma_n^2 \rpp \daone_{S_{n-1}}\\

\leq & \lpp \lpp 1 - \halflambda{}\gamma_n \rpp \lv \bddelta_{n-1}\rv^2 + C_{ng}^{(2)}\gamma_n^2 \rpp \daone_{S_{n-1}}\\

\leq & \lpp 1 - \halflambda{}\gamma_n \rpp \lv \bddelta_{n-1}\rv^2 \daone_{S_{n-1}} + C_{ng}^{(2)}\gamma_n^2,\\
\end{array}
\label{eq:convex1}
$}
\end{equation}
where the 3rd step is based on Lemma \ref{lemma:nablafbound}, the 4th step is based on condition \ref{cm_1} and \ref{cm_1}, and the 5th step is due to (\ref{eq:goodsimp1}).
As $\tiln \ge \lpp \frac{4\alpha}{C\lambdamin{}} \rpp^{1/(1-\alpha)} + 1$, we have
\begin{equation}
\def\arraystretch{1.5}
\begin{array}{rc}
     & \frac{4}{\lambdamin{}}\lp n^{\alpha} - (n-1)^{\alpha} \rp \leq \frac{4\alpha}{\lambdamin{}} (n-1)^{\alpha-1} \leq C  \\
     
\Rightarrow     & C\lp \frac{n-1}{n}\rp^{\alpha} + \frac{4}{\lambdamin} \lp n^{\alpha} - (n-1)^{\alpha} \rp \leq 2C \\

\Rightarrow & \frac{C}{n^{2\alpha}} + \frac{4}{\lambdamin{}}\lp \frac{1}{(n-1)^{\alpha}} - \frac{1}{n^{\alpha}} \rp \leq \frac{2C}{n^{\alpha}(n-1)^{\alpha}} \\

\Rightarrow & \frac{C^2}{n^{2\alpha}} \leq \lp 1-\frac{C}{n^{\alpha}} \halflambda{}\rp \lp -\frac{4}{\lambdamin{}} \frac{C}{(n-1)^{\alpha}} \rp + \frac{4}{\lambdamin{}} \frac{C}{n^{\alpha}} \\

\Leftrightarrow & \gamma_n^2 \leq (1 - \halflambda{}\gamma_n) \lp -\frac{4}{\lambdamin{}}\gamma_{n-1}\rp + \frac{4}{\lambdamin{}}\gamma_{n}.
\end{array}
    \label{eq:convex2}
\end{equation}

Based on (\ref{eq:convex1}) and (\ref{eq:convex2}), as $S_{n-1} \subseteq S_{n-2}$, we have

\begin{equation}
\EE_{n-1}  \lpp \lv \bddelta_n\rv^2 \daone_{S_{n-1}} - \frac{4}{\lambdamin{}}C_{ng}^{(2)} \gamma_n\rpp   \leq \lpp 1 - \halflambda{}\gamma_n\rpp \Big( \lv \bddelta_{n-1}\rv^2 \daone_{S_{n-2}}- \frac{4}{\lambdamin{}}C_{ng}^{(2)} \gamma_{n-1}\Big) .
    \label{eq:convex3}
\end{equation}

Then, we let
$$
D_n \triangleq \lpp \prod\limits_{i=1}\limits^{n} \lpp 1 - \halflambda{}\gamma_i \rpp^{-1} \rpp \lpp \lv \bddelta_n\rv^2 \daone_{S_{n-1}} - \frac{4}{\lambdamin{}}C_{ng}^{(2)}\gamma_n\rpp .
$$

Based on (\ref{eq:convex3}), we have
\begin{equation}
\EE_{n-1} D_n \leq D_{n-1},
    \label{eq:convex4}
\end{equation}
which implies that $\{ D_n:n\ge \tiln\}$ is a supermartingale with respect to $\{\fff_n:n\ge \tiln\}$.

Then, based on (\ref{eq:convex4}), for $\barn > \tiln$, we have
\begin{equation}
\arraycolsep=1.1pt\def\arraystretch{1.5} 
\begin{array}{rl}
D_{\barn} - D_{\tiln} & = \sum\limits_{n = \tiln+1}\limits^{\barn} \lpp D_n - \EE_{n-1} D_n \rpp + \sum\limits_{n=\tiln+1}\limits^{\barn} \lpp \EE_{n-1} D_n - D_{n-1} \rpp\\

& \leq \sum\limits_{n = \tiln+1}\limits^{\barn} \lpp D_n - \EE_{n-1} D_n \rpp.\\
\end{array}
    \label{eq:convex5}
\end{equation}

We can see that $D_{\barn} - D_{\tiln}$ is upper-bounded by a martingale difference sum. For simplicity, we let 
$$
\bdxi_n \triangleq  \nabla f(\bdtheta_{n-1};\bdy_n) - \nabla F (\bdtheta_{n-1}).
$$
Then, we have
\begin{equation}
\resizebox{.92\hsize}{!}{$
\arraycolsep=1.1pt\def\arraystretch{1.5} 
\begin{array}{rl}
&D_n - \EE_{n-1} D_n\\
=& \lpp \prod\limits_{i=1}\limits^{n} \lpp 1 - \halflambda{}\gamma_i \rpp^{-1} \rpp \lpp \lv \bddelta_n \rv^2 - \EE_{n-1} \lv \bddelta_n \rv^2 \rpp \daone_{S_{n-1}} \\

=& \lpp \prod\limits_{i=1}\limits^{n} \lpp 1 - \halflambda{}\gamma_i \rpp^{-1} \rpp \lpp \lv \bddelta_{n-1} - \gamma_n \nablaf{n-1}+\gamma_n \bdxi_n \rv^2 - \EE_{n-1} \lv \bddelta_n \rv^2 \rpp \daone_{S_{n-1}} \\

=&  \lpp \prod\limits_{i=1}\limits^{n} \lpp 1 - \halflambda{}\gamma_i \rpp^{-1} \rpp \lpp \gamma_n^2 \lv \bdxi_n \rv^2 - \gamma_n^2 \EE_{n-1} \lv \bdxi_n \rv^2 -2\gamma_n \langle \bdxi_n, \bddelta_{n-1} - \gamma_n \nabla F(\bdtheta_{n-1}) \rangle \rpp \daone_{S_{n-1}}.\\
\end{array}
    \label{eq:convex6}
$}
\end{equation}
Therefore, based on (\ref{eq:convex5}) and (\ref{eq:convex6}), we have
\begin{equation}
\arraycolsep=1.1pt\def\arraystretch{1.5} 
\begin{array}{rl}
D_{\barn} - D_{\tiln}\leq & \sum\limits_{n = \tiln+1}\limits^{\barn}\lpp \prod\limits_{i=1}\limits^{n} \lpp 1 - \halflambda{}\gamma_i \rpp^{-1} \rpp \lpp \gamma_n^2 \lv \bdxi_n \rv^2 - \gamma_n^2 \EE_{n-1} \lv \bdxi_n \rv^2 \rpp \daone_{S_{n-1}}\\

& + \sum\limits_{n = \tiln+1}\limits^{\barn}\lpp \prod\limits_{i=1}\limits^{n} \lpp 1 - \halflambda{}\gamma_i \rpp^{-1} \rpp \lpp -2\gamma_n \langle \bdxi_n, \bddelta_{n-1} - \gamma_n \nabla F(\bdtheta_{n-1}) \rangle \rpp \daone_{S_{n-1}}.\\
    \label{eq:convexn0}
\end{array}
\end{equation}
On $S_{n-1}$, we have
\begin{equation}
\lv \bddelta_{n-1} \rv \leq r_{good}^L\text{ and } \lv \nabla F(\bdtheta_{n-1}) \rv \leq g_{good}^L.
\label{eq:convexconst}
\end{equation}
Based on (\ref{eq:expbound}), we know that conditional on $\fff_{\tiln}$, with probability at least $1-\frac{\delta}{2\barn^3}$,
$$
\|\bdxi_n\|^2 - \EE_{n-1} \| \bdxi_n\|^2 \leq 2\csg \lppp 1+ \| \nablaf{n-1}\|^{\beta_{sg}} \rppp \log \lppp 4e^2 \frac{\barn^3}{\delta}\rppp.
$$
Therefore, if we let
$$
c_{1,\barn} \triangleq 2\lppp 1+ \lppp g_{good}^L\rppp^{\beta_{sg}}\rppp \csg \log \lppp 4e^2 \frac{\barn^3}{\delta}\rppp,
$$
conditional on $\fff_{\tiln}$, with probability at least $1-\frac{\delta}{2\barn^2}$,
\begin{equation}
\arraycolsep=1.1pt\def\arraystretch{1.5} 
\begin{array}{rl}
&  \sum\limits_{n = \tiln+1}\limits^{\barn}\lpp \prod\limits_{i=1}\limits^{n} \lpp 1 - \halflambda{}\gamma_i \rpp^{-1} \rpp \lpp \gamma_n^2 \lv \bdxi_n \rv^2 - \gamma_n^2 \EE_{n-1} \lv \bdxi_n \rv^2 \rpp \daone_{S_{n-1}}\\

\leq &  \sum\limits_{n = \tiln+1}\limits^{\barn}\lpp \prod\limits_{i=1}\limits^{n} \lpp 1 - \halflambda{}\gamma_i \rpp^{-1} \rpp 2\csg \lppp 1+ \| \nablaf{n-1}\|^{\beta_{sg}} \rppp \log \lppp 4e^2 \frac{\barn^3}{\delta}\rppp \gamma_n^2 \daone_{S_{n-1}}\\

\leq & c_{1,\barn} \sum\limits_{n = \tiln+1}\limits^{\barn}\lpp \prod\limits_{i=1}\limits^{n} \lpp 1 - \halflambda{}\gamma_i \rpp^{-1} \rpp\gamma_n^2, \\
\end{array}
    \label{eq:convexn1}
\end{equation}
where the 2nd step is based on (\ref{eq:convexconst}). Next, we are to handle the second part in (\ref{eq:convexn0}). Based on (\ref{eq:goodsimp2}) and (\ref{eq:convexconst}), we notice that
$$
\| \bddelta_{n-1} - \gamma_n \nabla F(\bdtheta_{n-1}) \| \daone_{S_{n-1}} \leq \sqrt{3}r + \gamma_n g_{good}^L \leq \sqrt{3} +  g_{good}^L \triangleq c_{rg}.
$$
It implies that $\lpp \prod\limits_{i=1}\limits^{n} \lpp 1 - \halflambda{}\gamma_i \rpp^{-1} \rpp \lpp -2\gamma_n \langle \bdxi_n, \bddelta_{n-1} - \gamma_n \nabla F(\bdtheta_{n-1}) \rangle \rpp \daone_{S_{n-1}}$ is $4\csg c_{rg}^2 \lppp1+(g_{good}^L)^{\beta_{sg}}\rppp \gamma_n^2 \lpp \prod\limits_{i=1}\limits^{n} \lpp 1 - \halflambda{}\gamma_i \rpp^{-2} \rpp$-sub-Gaussian. Then, if we let
$$
c_{2,\barn} \triangleq 2c_{rg} \cjin (\csg)^{\frac{1}{2}} \lppp \log \frac{4\barn^2}{\delta} \rppp^{\frac{1}{2}},
$$
based on Corollary \ref{cor:jin}, conditional on $\fff_{\tiln}$, with probability at least $1-\frac{\delta}{2\barn^2}$,
\begin{equation}
\arraycolsep=1.1pt\def\arraystretch{1.5} 
\begin{array}{rl}
& \sum\limits_{n = \tiln+1}\limits^{\barn}\lpp \prod\limits_{i=1}\limits^{n} \lpp 1 - \halflambda{}\gamma_i \rpp^{-1} \rpp \lpp -2\gamma_n \langle \bdxi_n, \bddelta_{n-1} - \gamma_n \nabla F(\bdtheta_{n-1}) \rangle \rpp \daone_{S_{n-1}}\\

\leq & c_{2,\barn} \lpp   \sum\limits_{n = \tiln+1}\limits^{\barn} \gamma_n^2 \lpp \prod\limits_{i=1}\limits^{n} \lpp 1 - \halflambda{}\gamma_i \rpp^{-2} \rpp \rpp^{\frac{1}{2}}.\\
\end{array}
    \label{eq:convexn2}
\end{equation}
Now, based on (\ref{eq:convexn0}), (\ref{eq:convexn1}) and (\ref{eq:convexn2}), we know that conditional on $\fff_{\tiln}$, with probability at least $1-\frac{\delta}{\barn^2}$,
\begin{equation}
\arraycolsep=1.1pt\def\arraystretch{1.5} 
\begin{array}{rl}
D_{\barn} - D_{\tiln}& \leq  c_{1,\barn} \sum\limits_{n = \tiln+1}\limits^{\barn}\lpp \prod\limits_{i=1}\limits^{n} \lpp 1 - \halflambda{}\gamma_i \rpp^{-1} \rpp\gamma_n^2 + c_{2,\barn} \lpp   \sum\limits_{n = \tiln+1}\limits^{\barn} \gamma_n^2 \lpp \prod\limits_{i=1}\limits^{n} \lpp 1 - \halflambda{}\gamma_i \rpp^{-2} \rpp \rpp^{\frac{1}{2}}\\

& \triangleq \varepsilon_1 + \varepsilon_2.\\
\end{array}
\label{eq:convexn3}
\end{equation}
Then, with the above result, based on the definition of $D_n$, conditional on $\fff_{\tiln}$, the following results hold with probability at least $1-\frac{\delta}{\barn^2}$,
\begin{equation}
\resizebox{.92\hsize}{!}{$
\arraycolsep=1.1pt\def\arraystretch{1.5} 
\begin{array}{rl}
\lv \bddelta_{\barn}\rv^2 \daone_{S_{\barn-1}}& \leq \frac{4}{\lambdamin{}}C_{ng}^{(2)}\gamma_{\barn} + \lpp \prod\limits_{i=\tiln + 1}\limits^{\barn} \lpp 1 - \halflambda{}\gamma_i \rpp \rpp \| \bddelta_{\tiln}\|^2 \daone_{S_{\tiln-1}} + \lpp \prod\limits_{i=1}\limits^{\barn} \lpp 1 - \halflambda{}\gamma_i \rpp \rpp (\varepsilon_1 + \varepsilon_2)\\

& \leq \frac{4}{\lambdamin{}}C_{ng}^{(2)}\gamma_{\barn} + \lpp \prod\limits_{i=\tiln + 1}\limits^{\barn} \lpp 1 - \halflambda{}\gamma_i \rpp \rpp r^2 + \lpp \prod\limits_{i=1}\limits^{\barn} \lpp 1 - \halflambda{}\gamma_i \rpp \rpp (\varepsilon_1 + \varepsilon_2)\\

& \leq 2r^2 + \lpp \prod\limits_{i=1}\limits^{\barn} \lpp 1 - \halflambda{}\gamma_i \rpp \rpp (\varepsilon_1 + \varepsilon_2),\\
\end{array}
$}
\label{eq:convexn4}
\end{equation}
where the last step is based on (\ref{eq:goodsimp3}). To further simplify the last term in (\ref{eq:convexn4}), we have
\begin{equation}
\arraycolsep=1.1pt\def\arraystretch{1.5} 
\begin{array}{rl}
& \lpp \prod\limits_{i=1}\limits^{\barn} \lpp 1 - \halflambda{}\gamma_i \rpp \rpp \varepsilon_1\\

= & c_{1,\barn}  \sum\limits_{n = \tiln+1}\limits^{\barn}\lpp \prod\limits_{i=n+1}\limits^{\barn} \lpp 1 - \halflambda{}\gamma_i \rpp \rpp\gamma_n^2\\

\leq & c_{1,\barn}  \sum\limits_{n = \tiln+1}\limits^{\barn} \exp\lpp -\halflambda{} \ssum{i=n+1}{\barn} \gamma_i\rpp \gamma_n^2 \\

\leq & c_{1,\barn}  \sum\limits_{n = \tiln+1}\limits^{\barn} \exp\lpp -\halflambda{} \frac{C}{1-\alpha} \lppp (\barn+1)^{1-\alpha} - (n+1)^{1-\alpha})\rppp \rpp \gamma_n^2\\

\leq &2C^2 c_{1,\barn} \exp\lpp -\halflambda{} \frac{C}{1-\alpha} (\barn+1)^{1-\alpha} \rpp  \sum\limits_{n = \tiln+1}\limits^{\barn} n^{-2\alpha} \exp\lpp \halflambda{} \frac{C}{1-\alpha} n^{1-\alpha} \rpp\\

\leq & 2C^2 c_{1,\barn} \exp\lpp -\halflambda{} \frac{C}{1-\alpha} (\barn+1)^{1-\alpha} \rpp \frac{4}{C\lambdamin{}} (\barn +2)^{-\alpha} \exp\lpp \halflambda{} \frac{C}{1-\alpha} (\barn+2)^{1-\alpha} \rpp\\

\leq &\frac{8Cc_{1,\barn}}{\lambdamin{}} (\barn+2)^{-\alpha} \exp \lpp \frac{C\lambdamin{}}{2} (\barn+1)^{-\alpha}\rpp\\

\leq& \frac{16Cc_{1,\barn}}{\lambdamin{}} (\barn+2)^{-\alpha}\\

= & \frac{32(1+(g_{good}^L)^{\beta_{sg}})C\csg}{\lambdamin{}} (\barn+2)^{-\alpha} (\log(4e^2) + 3\log\barn +\tau \log N) \\

\leq & \frac{r^2}{3} + \frac{96(1+(g_{good}^L)^{\beta_{sg}})C\csg}{\lambdamin{}} (\barn+2)^{-\alpha} \frac{2}{\alpha}\log \lppp\barn^{\frac{\alpha}{2}}\rppp\\

\leq & \frac{r^2}{3} + \frac{192(1+(g_{good}^L)^{\beta_{sg}})C\csg}{\alpha\lambdamin{}} \tiln^{-\frac{\alpha}{2}}\\

\leq &\frac{r^2}{2},\\
\end{array}
\label{eq:convexn5}
\end{equation}
where the 4th step is based on (\ref{eq:goodsimp4}), the 5th step is based on Lemma \ref{lemma:int1} and (\ref{eq:goodsimp5}), the 7th step is based on (\ref{eq:goodsimp6}), the 9th step is based on (\ref{eq:goodsimp7}) and (\ref{eq:goodsimp8}), and the last step is based on (\ref{eq:goodsimp9}). Similarly, we have
\begin{equation}
\arraycolsep=1.1pt\def\arraystretch{1.5} 
\begin{array}{rl}
& \lpp \prod\limits_{i=1}\limits^{\barn} \lpp 1 - \halflambda{}\gamma_i \rpp \rpp \varepsilon_2\\

= & c_{2,\barn} \lpp \sum\limits_{n = \tiln+1}\limits^{\barn}\lpp \prod\limits_{i=n+1}\limits^{\barn} \lpp 1 - \halflambda{}\gamma_i \rpp \rpp^2\gamma_n^2 \rpp^{\frac{1}{2}}\\

\leq & c_{2,\barn}  \lpp \sum\limits_{n = \tiln+1}\limits^{\barn} \exp\lpp - \frac{C\lambdamin{}}{1-\alpha} \lppp (\barn+1)^{1-\alpha} - (n+1)^{1-\alpha})\rppp \rpp \gamma_n^2\rpp^{\frac{1}{2}}\\

\leq & \sqrt{2} C c_{2,\barn} \exp\lpp - \frac{1}{2}\frac{C\lambdamin{}}{1-\alpha}  (\barn+1)^{1-\alpha} \rpp \lpp \sum\limits_{n = \tiln+1}\limits^{\barn} n^{-2\alpha} \exp\lpp \frac{C\lambdamin{}}{1-\alpha}  (n+1)^{1-\alpha} \rpp \rpp^{\frac{1}{2}}\\

\leq & \sqrt{2} C c_{2,\barn} \exp\lpp - \frac{1}{2}\frac{C\lambdamin{}}{1-\alpha}  (\barn+1)^{1-\alpha} \rpp \lpp \frac{2}{C\lambdamin{}} (\barn+2)^{-\alpha} \exp \lpp \frac{C\lambdamin{}}{1-\alpha} (\barn+2)^{1-\alpha}\rpp\rpp^{\frac{1}{2}}\\

\leq & 2 \lpp \frac{C}{\lambdamin{}}\rpp^{\frac{1}{2}} c_{2,\barn} (\barn+2)^{-\frac{\alpha}{2}} \exp \lpp \frac{C\lambdamin{}}{2} (\barn+1)^{-\alpha} \rpp\\

\leq & 4 \lpp \frac{C}{\lambdamin{}}\rpp^{\frac{1}{2}} c_{2,\barn} (\barn+2)^{-\frac{\alpha}{2}}\\

\leq & 8c_{rg}\cjin \lpp \frac{C\csg}{\lambdamin{}}\rpp^{\frac{1}{2}} (\barn +2)^{-\alpha}\lppp (\log 4)^{\frac{1}{2}} + (2\log \barn )^{\frac{1}{2}} + (\tau \log N)^{\frac{1}{2}}\rppp\\

\leq & \frac{r^2}{3} +  8c_{rg}\cjin \lpp \frac{C\csg}{\lambdamin{}}\rpp^{\frac{1}{2}}  \lpp \frac{2}{\alpha}\rpp^{\frac{1}{2}} (\barn +2)^{-\alpha} \lppp \log \barn^{\alpha}\rppp^{\frac{1}{2}}\\

\leq & \frac{r^2}{3} +  8c_{rg}\cjin \lpp \frac{C\csg}{\lambdamin{}}\rpp^{\frac{1}{2}}  \lpp \frac{2}{\alpha}\rpp^{\frac{1}{2}} \tiln^{-\frac{\alpha}{2}}\\

\leq & \frac{r^2}{2},\\
\end{array}
\label{eq:convexn6}
\end{equation}
where the 4th step is based on Lemma \ref{lemma:int1} and (\ref{eq:goodsimp5}), the 6th step is based on (\ref{eq:goodsimp6}), the 8th step is based on (\ref{eq:goodsimp10}) and (\ref{eq:goodsimp11}) and the last step is based on (\ref{eq:goodsimp12}).

Based on results given in (\ref{eq:convexn4}), (\ref{eq:convexn5}) and (\ref{eq:convexn6}), we know that 
\begin{equation}
\PP \lpp \lv \bddelta_{\barn} \rv^2 \daone_{S_{\barn} - 1} \leq 3 r^2 | S_{base} \rpp \ge 1-\delta/\barn^2.
    \label{eq:convex14}
\end{equation}
Therefore,
$$
\arraycolsep=1.1pt\def\arraystretch{1.5} 
\begin{array}{rl}
& \PP \lpp \exists \barn \ge \tiln, \lv \bddelta_{\barn} \rv > \sqrt{3} r \big| \lv \bddelta_{\tiln} \rv \leq r\rpp\\

= & \sum\limits_{\barn = \tiln +1}\limits^{\infty} \PP \lpp \lv \bddelta_{\barn}\rv > \sqrt{3} r, \lv \bddelta_{n}\rv \leq \sqrt{3} r,\forall \tiln < n < \barn \big| \lv \bddelta_{\tiln} \rv \leq r\rpp\\

= & \sum\limits_{\barn = \tiln +1}\limits^{\infty} \PP \lpp \lv \bddelta_{\barn}\rv^2 \daone_{S_{\barn}-1} > (\sqrt{3} r)^2 \big| \lv \bddelta_{\tiln} \rv \leq r\rpp\\

\leq & \sum\limits_{\barn = \tiln +1}\limits^{\infty} \frac{\delta}{\barn^2} \leq \delta.\\
\end{array}
$$

\end{proof}
\begin{lemma}
Suppose that conditions \ref{cm_1} and \ref{cm_2} hold. We define
$$
\cnlemma{8} \triangleq \max \lww \lppp2C (C_s +2\beta_{sg} C_s \csg)\rppp^{\frac{1}{\alpha}}, \lppp 8(4-\beta_{sg}) C C_s \csg\rppp^{\frac{1}{2\beta_B\alpha}}\rww.
$$
Suppose that $\tiln \ge \cnlemma{8}$. Then, for any $n\ge \tiln$, on $\lww \|\nablaf{n} \| \ge B_0 \triangleq \tiln^{\lppp \frac{1}{2}-\beta_B\rppp \alpha} \rww$, we have
$$
\EE_n \lppp F(\bdtheta_{n+1}) - F(\bdtheta_n)\rppp \leq -\frac{C}{4} \tiln^{-(2-2\beta_B)\alpha}.
$$
    \label{lemma:biggradientexp}
\end{lemma}
\begin{proof}
$$
\arraycolsep=1.1pt\def\arraystretch{1.5}
\begin{array}{rl}
& \EE_n \lppp F(\bdtheta_{n+1}) - F(\bdtheta_n)\rppp\\

\leq & \EE_n \langle \nablaf{n},\bdtheta_{n+1} - \bdtheta_n \rangle +\frac{C_s}{2} \EE_n \lvvv \bdtheta_{n+1} - \bdtheta_n \rvvv\\

= & -\gamma_{n+1} \lvvv \nablaf{n} \rvvv^2 + \frac{C_s\gamma_{n+1}^2}{2} \EE_n \lvvv \nabla f(\bdtheta_n; \bdy_{n+1})\rvvv^2\\

\leq & -\gamma_{n+1} \lvvv \nablaf{n} \rvvv^2 + C_s\gamma_{n+1}^2 \lvvv \nablaf{n}\rvvv^2 + C_s\gamma_{n+1}^2 \EE_n \lvvv \nabla f(\bdtheta_n;\bdy_{n+1}) - \nablaf{n} \rvvv^2\\

\leq & -\gamma_{n+1} \lppp 1-C_s\gamma_{n+1} -2\beta_{sg}C_s \csg \gamma_{n+1} \rppp \lvvv \nablaf{n}\rvvv^2 + 2(4-\beta_{sg}) C_s\csg \gamma_{n+1}^2\\

\leq & -\frac{1}{2}\gamma_{n+1} \lvvv \nablaf{n}\rvvv^2 + 2(4-\beta_{sg}) C_s\csg \gamma_{n+1}^2\\

\leq & -\frac{1}{2}\gamma_{n+1} \tiln^{-(1-2\beta_B)\alpha} + 2(4-\beta_{sg}) C_s\csg \gamma_{n+1}^2\\

= & -\frac{C}{2} \tiln^{-(2-2\beta_B)\alpha} + 2(4-\beta_{sg})C^2 C_s \csg \tiln^{-2\alpha}\\

\leq & -\frac{C}{4} \tiln^{-(2-2\beta_B)\alpha},\\
\end{array}
$$
where the 1st step is based on condition \ref{cm_1}, the 4th step is based on Lemma \ref{lemma:errorbound}, the 5th step is based on $n \ge \lppp2C (C_s +2\beta_{sg} C_s \csg)\rppp^{\frac{1}{\alpha}}$, the 6th step is based on the assumption that $\lvvv \nablaf{n} \rvvv \ge B_0$ and the last step is based on $n\ge \lppp 8(4-\beta_{sg}) C C_s \csg\rppp^{\frac{1}{2\beta_B\alpha}}$.
    
\end{proof}

\begin{lemma}
Suppose that conditions \ref{cm_1} and \ref{cm_2} are satisfied. The stepsize parameter $C$ is chosen such that $CC_s(1+2\csg\beta_{sg}) \leq \frac{1}{2}$ and $C\leq 1$. We suppose that $F(\bdtheta_n) \leq C_F$ for a given $n\in \ZZ_+$. The short conclusions are that, for any $k\in \ZZ_+$,
$$
\begin{array}{c}
\EE_n \lppp F(\bdtheta_{n+t}) - \fmin\rppp^k= O \lpp (\log t)^{\frac{2k}{2-\beta_{sg}}} \rpp,\\

\EE_n \lpp \ssum{m=1}{t} \gamma_{n+m} \| \nablaf{n+m-1}\|^2 \rpp^k = O \lpp (\log t)^{\frac{8k}{4-\beta_{sg}^2}} \rpp.\\
\end{array}
$$
To be more detailed, we let
$$
\arraycolsep=1.1pt\def\arraystretch{1.5}
\begin{array}{ccl}
a_{mb,1} & \triangleq & C_F - \fmin + \frac{1}{2\alpha-1}C^2 C_s C_{ng}^{(2)} + \frac{1}{3},\\

a_{mb,2}&\triangleq & \frac{48(2-\beta_{sg})\alpha C^2}{2\alpha-1} (2\beta_{sg})^{\frac{\beta_{sg}}{2-\beta_{sg}}} (2C_s\csg)^{\frac{2}{2-\beta_{sg}}},\\

a_{mb,3}&\triangleq & \max \lww \lpp \frac{a_{mb,2}}{4\cjin (\csg)^{\frac{1}{2}}}\rpp^{\frac{4}{2+\beta_{sg}}} , (12\cjin (3\csg)^{\frac{1}{2}})^{\frac{4}{2-\beta_{sg}}}\rww,\\

a_{mb,4}&\triangleq & \frac{1}{4} \lpp  8\cjin (3\csg)^{\frac{1}{2}} a_{mb,3}^{\frac{2+\beta_{sg}}{4}} +a_{mb,2}\rpp 2^{\frac{\beta_{sg}}{2-\beta_{sg}}},\\

a_{mb,5}&\triangleq & a_{mb,3} 2^{\frac{4+\beta_{sg}^2}{4-\beta_{sg}^2}}.\\
\end{array}
$$
We also define a positive constant $\delta_0$ such that
$$
\arraycolsep=1.1pt\def\arraystretch{1.5} 
\begin{array}{rll}
\delta_0^{-1} &= \max \Big\{ & (3a_{mb,1})^{-\frac{1}{\delta}},6e^2,2\log \lpp 2(3a_{mb,1})^{\frac{2+\beta_{sg}}{2}}\rpp ,4(2+\beta_{sg})^2,\frac{1}{3a_{mb,1}} \lpp \frac{e}{2} \rpp^{\frac{2}{2+\beta_{sg}}},\\

&&\exp \lpp \lpp \frac{2}{2-\beta_{sg}}\rpp^{\frac{2-\beta_{sg}}{\beta_{sg}}}  (8\beta_{sg} C_s\csg )^{-1}\rpp, \exp \lppp a_{mb,3}^{-\frac{2-\beta_{sg}}{2}} 2^{-\frac{4}{2+\beta_{sg}}} \rppp,\\

&& \exp \lpp (4\beta_{sg} C_s\csg)^{-\frac{\beta_{sg}}{2}} \lppp\frac{4-\beta_{sg}}{2-\beta_{sg}}\rppp^{\frac{2-\beta_{sg}}{2}}\rpp,\\

&& \exp \lpp \lpp \frac{(2\alpha-1)(F(\bdtheta_0)-\fmin)}{4(2-\beta_{sg})\alpha C^2}\rpp^{\frac{2-\beta_{sg}}{2}} (2\beta_{sg})^{-\frac{\beta_{sg}}{2}} (4C_s\csg )^{-1}\rpp,\\

&& \exp \lpp \lpp \frac{(2\alpha-1)\cjin(3\csg)^{\frac{1}{2}}}{4(2-\beta_{sg})\alpha C^2}\rpp^{\frac{2(2-\beta_{sg})}{2+\beta_{sg}}} (2\beta_{sg})^{-\frac{2\beta_{sg}}{2+\beta_{sg}}} (4C_s\csg)^{-\frac{4}{2+\beta_{sg}}}\rpp\Big\}.
\end{array}
$$
For any $k\in \ZZ_+$, we further define
$$
\arraycolsep=1.1pt\def\arraystretch{1.5}
\begin{array}{ccl}
a_{mb,6}^{(k)}& \triangleq & a_{mb,4}^k \frac{2k}{2-\beta_{sg}} 2^{\frac{2k}{2-\beta_{sg}}} \lpp (1-\delta_0) \delta_0^{-1}  \lppp \log \delta_0^{-1} \rppp^{\frac{2k-2+\beta_{sg}}{2-\beta_{sg}}} + \int_0^{\delta_0} \lppp \log \frac{1}{\delta} \rppp^{\frac{2k-2+\beta_{sg}}{2-\beta_{sg}}} d\delta \rpp,\\

a_{mb,7}^{(k)} &\triangleq & a_{mb,5}^k \frac{2k}{2-\beta_{sg}}2^{\frac{8k}{4-\beta_{sg}^2}} \lpp (1-\delta_0) \delta_0^{-1}  \lppp \log \delta_0^{-1} \rppp^{\frac{2k-2+\beta_{sg}}{2-\beta_{sg}}} + \int_0^{\delta_0} \lppp \log \frac{1}{\delta} \rppp^{\frac{2k-2+\beta_{sg}}{2-\beta_{sg}}} d\delta \rpp.\\
\end{array}
$$
Then, if we let
$$
c_{fm,1}^{(k)} \triangleq 2^{k-1}a_{mb,4}^k,\ \ c_{fm,2}^{(k)} \triangleq 2^{k-1} a_{mb,6}^{(k)},\ \  c_{fm,3}^{(k)} \triangleq 2^{k-1}a_{mb,5}^k,\ \  c_{fm,4}^{(k)} \triangleq 2^{k-1} a_{mb,7}^{(k)},
$$
we have
$$
\EE_n \lppp F(\bdtheta_{n+t}) - \fmin\rppp^k \leq c_{fm,2}^{(k)} + c_{fm,1}^{(k)} (\log t)^{\frac{2k}{2-\beta_{sg}}},\ for\ any\ t\ge1,
$$
$$
\EE_n \lpp \ssum{m=1}{t} \gamma_{n+m} \| \nablaf{n+m-1}\|^2 \rpp^k \leq  c_{fm,4}^{(k)} + c_{fm,3}^{(k)} (\log t)^{\frac{8k}{4-\beta_{sg}^2}},\ for\ any\ t\ge1.
$$
    \label{lemma:momentbound}
\end{lemma}

\begin{proof}
Based on (\ref{eq:fb2}), for $m\ge 0$, we have
$$
\resizebox{.99\hsize}{!}{$
F(\bdtheta_{m+1}) - F(\bdtheta_m) \leq -\frac{\gamma_{m+1}}{2} \| \nablaf{m}\|^2 + 8C_s\csg \lppp 1-\frac{\beta_{sg}}{4}\rppp \gamma_{m+1}^2 + C_s \gamma_{m+1}^2 \lppp \| \bdxi_{m+1}\|^2 - \EE_m \| \bdxi_{m+1}\|^2 \rppp -\gamma_{m+1}\langle \nablaf{m},\bdxi_{m+1} \rangle,
$}
$$
and consequently,
\begin{equation}
\resizebox{.92\hsize}{!}{$
\arraycolsep=1.1pt\def\arraystretch{1.5}  
\begin{array}{rl}
F(\bdtheta_{n+t}) - F(\bdtheta_n)  \leq& -\frac{1}{2} \ssum{m=1}{t} \gamma_{n+m}  \| \nablaf{n+m-1}\|^2 + 8C_s\csg \lppp 1-\frac{\beta_{sg}}{4}\rppp \ssum{m=1}{t}  \gamma_{n+m}^2\\

&+ C_s \ssum{m=1}{t} \gamma_{n+m}^2 \lppp \| \bdxi_{n+m}\|^2 - \EE_{n+m-1} \| \bdxi_{n+m}\|^2 \rppp - \ssum{m=1}{t}\gamma_{n+m}\langle \nablaf{n+m-1},\bdxi_{n+m} \rangle.\\
\end{array}
$}
\label{eq:mombound1}
\end{equation}
For simplicity, in the present proof, we let
$$
\arraycolsep=1.1pt\def\arraystretch{1.5}  
\begin{array}{ccl}
A_{n,t} &\triangleq& \ssum{m=1}{t} \gamma_{n+m}  \| \nablaf{n+m-1}\|^2.\\
\end{array}
$$
Based on Lemma \ref{lemma:nablasumbound}, we have
$$
\PP_n\lpp A_{n,t} \ge \frac{3a_{mb,1}}{\delta} \rpp \leq \frac{\delta \EE_n A_{n,t}}{3a_{mb,1}} \leq \frac{\delta}{3}.
$$
When $A_{n,t} < \frac{3a_{mb,1}}{\delta}$,
$$
\arraycolsep=1.1pt\def\arraystretch{1.5}  
\begin{array}{rl}
& \csg \ssum{m=1}{t} \gamma_{n+m}^2  \| \nablaf{n+m-1}\|^2 (1+ \| \nablaf{n+m-1}\|^{\beta_{sg}} )\\

\leq& \csg \lpp A_{n,t} + A_{n,t}^{\frac{2+\beta_{sg}}{2}}\rpp < \csg \lpp \frac{3a_{mb,1}}{\delta} + \lppp \frac{3a_{mb,1}}{\delta} \rppp^{\frac{2+\beta_{sg}}{2}} \rpp \leq 2\csg  \lppp \frac{3a_{mb,1}}{\delta} \rppp^{\frac{2+\beta_{sg}}{2}}.\\
\end{array}
$$
Therefore, if we let
$$
\cdelta{2} \triangleq 2\csg  \lppp \frac{3a_{mb,1}}{\delta} \rppp^{\frac{2+\beta_{sg}}{2}},
$$
we have
\begin{equation}
\PP_n \lpp\csg \ssum{m=1}{t} \gamma_{n+m}^2  \| \nablaf{n+m-1}\|^2 (1+ \| \nablaf{n+m-1}\|^{\beta_{sg}} ) \ge \cdelta{2}\rpp \leq \frac{\delta}{3}.
    \label{eq:mombound2}
\end{equation}
Based on condition \ref{cm_2}, we know that conditional on $\fff_{n+m-1}$, $\gamma_{n+m} \langle \nablaf{n+m-1},\bdxi_{n+m} \rangle$ is $\csg  \gamma_{n+m}^2  \| \nablaf{n+m-1}\|^2 (1+ \| \nablaf{n+m-1}\|^{\beta_{sg}} )$-subGaussian. With (\ref{eq:mombound2}) at hand, based on Lemma \ref{lemma:jin}, if we let
$$
\cdelta{3} \triangleq \lpp \log \frac{2}{\delta} + \log\log \frac{\cdelta{2}}{\csg} \rpp^{\frac{1}{2}},
$$
then conditional on $\fff_n$, with probability at least $1-\frac{2\delta}{3}$,
\begin{equation}
\resizebox{.92\hsize}{!}{$
\arraycolsep=1.1pt\def\arraystretch{1.5}  
\begin{array}{rl}
& -\ssum{m=1}{t}\gamma_{n+m}\langle \nablaf{n+m-1},\bdxi_{n+m}\rangle\\

\leq & \cjin \max \lww \ssum{m=1}{t} \csg \gamma_{n+m}^2 \| \nablaf{n+m-1}\|^2 (1+\| \nablaf{n+m-1}\|^{\beta_{sg}}), \csg\rww^{\frac{1}{2}}\cdelta{3}\\

\leq & \cjin (\csg)^{\frac{1}{2}} \lpp \lpp \ssum{m=1}{t}  \gamma_{n+m}^2 \| \nablaf{n+m-1}\|^2 \rpp^{\frac{1}{2}} + \lpp \ssum{m=1}{t}  \gamma_{n+m}^2 \| \nablaf{n+m-1}\|^{2+\beta_{sg}} \rpp^{\frac{1}{2}} +1 \rpp\cdelta{3}\\

\leq &\cjin (\csg)^{\frac{1}{2}} \lpp A_{n,t}^{\frac{1}{2}} + A_{n,t}^{\frac{2+\beta_{sg}}{4}} + 1\rpp \cdelta{3}.\\
\end{array}
$}
\label{eq:mombound3}
\end{equation}
Next, based on the tail bound shown in (\ref{eq:expbound}), if we let
$$
\cdn{4} \triangleq 2C_s \csg \log(6e^2 t \frac{1}{\delta}),
$$
we know that conditional on $\fff_n$, with probability at least $1-\frac{\delta}{3}$,
\begin{equation}
\arraycolsep=1.1pt\def\arraystretch{1.5}  
\begin{array}{rl}
& C_s \ssum{m=1}{t} \gamma_{n+m}^2  \lppp \| \bdxi_{n+m}\|^2 - \EE_{n+m-1} \| \bdxi_{n+m}\|^2 \rppp\\

\leq & \cdn{4} \ssum{m=1}{t} \gamma_{n+m}^2 \lppp 1+ \| \nablaf{n+m-1}\|^{\beta_{sg}}\rppp\\

\leq & \cdn{4} \ssum{m=1}{t} \gamma_{n+m}^2 \lpp 1+ \frac{\beta_{sg}}{2} (2\beta_{sg}\cdn{4})^{-1} \| \nablaf{n+m-1}\|^2 + \frac{2-\beta_{sg}}{2} (2\beta_{sg} \cdn{4})^{\frac{\beta_{sg}}{2-\beta_{sg}}}\rpp\\

\leq &\frac{1}{4} A_{n,t} + \cdn{4} \lpp 1+ \frac{2-\beta_{sg}}{2} (2\beta_{sg} \cdn{4})^{\frac{\beta_{sg}}{2-\beta_{sg}}}\rpp \ssum{m=0}{\infty}\gamma_{m+1}^2.\\
\end{array} 
\label{eq:mombound4}
\end{equation}
Based on the results shown in (\ref{eq:mombound1}), (\ref{eq:mombound3}) and (\ref{eq:mombound4}), conditional on $\fff_n$, the following inequality holds with probability at least $1-\delta$,
\begin{equation}
\resizebox{.92\hsize}{!}{$
\arraycolsep=1.1pt\def\arraystretch{1.5}  
\begin{array}{rl}
F(\bdtheta_{n+t}) - F(\bdtheta_n) &  \leq -\frac{1}{4}A_{n,t} + \cjin (\csg)^{\frac{1}{2}}\cdelta{3} \lppp A_{n,t}^{\frac{1}{2}} + A_{n,t}^{\frac{2+\beta_{sg}}{4}}\rppp+  \cjin (\csg)^{\frac{1}{2}}\cdelta{3}\\

&\ \ \  + \lpp 8C_s\csg\lppp 1-\frac{\beta_{sg}}{4}\rppp + \cdn{4} \lpp 1 + \frac{2-\beta_{sg}}{2} (2\cdn{4}\beta_{sg})^{\frac{\beta_{sg}}{2-\beta_{sg}}}\rpp \rpp \ssum{m=0}{\infty} \gamma_{m+1}^2.\\

& \leq -\frac{1}{4}A_{n,t} + \cjin (\csg)^{\frac{1}{2}}\cdelta{3} \lppp A_{n,t}^{\frac{1}{2}} + A_{n,t}^{\frac{2+\beta_{sg}}{4}}\rppp+  \cjin (\csg)^{\frac{1}{2}}\cdelta{3}\\

&\ \ \  + \lpp 8C_s\csg\lppp 1-\frac{\beta_{sg}}{4}\rppp + \cdn{4} \lpp 1 + \frac{2-\beta_{sg}}{2} (2\cdn{4}\beta_{sg})^{\frac{\beta_{sg}}{2-\beta_{sg}}}\rpp \rpp \lpp C^2 + C^2\int_{1}^{\infty} x^{-2\alpha}dx\rpp\\

& = -\frac{1}{4}A_{n,t} + \cjin (\csg)^{\frac{1}{2}}\cdelta{3} \lppp A_{n,t}^{\frac{1}{2}} + A_{n,t}^{\frac{2+\beta_{sg}}{4}}\rppp+  \cjin (\csg)^{\frac{1}{2}}\cdelta{3}\\

&\ \ \  + \frac{2\alpha C^2}{2\alpha-1}\lpp 8C_s\csg\lppp 1-\frac{\beta_{sg}}{4}\rppp + \cdn{4} \lpp 1 + \frac{2-\beta_{sg}}{2} (2\cdn{4}\beta_{sg})^{\frac{\beta_{sg}}{2-\beta_{sg}}}\rpp \rpp.\\
\end{array}
$}
\label{eq:mombound5}
\end{equation}
We denote the event in (\ref{eq:mombound5}) by $S_{n,\delta,t}$. We let
$$
\resizebox{.98\hsize}{!}{$
\arraycolsep=1.1pt\def\arraystretch{1.5}  
\begin{array}{rl}
\cdn{5} \triangleq& 4\lpp C_F - F_{\min} + \cjin (\csg)^{\frac{1}{2}}\cdelta{3} + \frac{2\alpha C^2}{2\alpha-1}\lpp 8C_s\csg\lppp 1-\frac{\beta_{sg}}{4}\rppp + \cdn{4} \lpp 1 + \frac{2-\beta_{sg}}{2} (2\cdn{4}\beta_{sg})^{\frac{\beta_{sg}}{2-\beta_{sg}}}\rpp \rpp  \rpp,\\

\cdelta{6} \triangleq& 4\cjin (\csg)^{\frac{1}{2}}\cdelta{3},\\

\cdn{7} \triangleq & \lww 1, \lppp \frac{\cdn{5}}{\cdelta{6}} \rppp^{\frac{4}{2+\beta_{sg}}} , (3 \cdelta{6})^{\frac{4}{2-\beta_{sg}}}\rww.\\
\end{array}
$}
$$
Then, similar to the proof of Lemma \ref{lemma:fbound}, on $S_{n,\delta,t}$, we have
\begin{equation}
A_{n,t} \leq \cdn{7}
\label{eq:momboundadd1}
\end{equation}
and
\begin{equation}
F(\bdtheta_{n+t}) - \fmin \leq \frac{1}{4} \lpp \cdelta{6} \cdn{7}^{\frac{1}{2}} + \cdelta{6}\cdn{7}^{\frac{2+\beta_{sg}}{4}} + \cdn{5} \rpp.\\
    \label{eq:mombound6}
\end{equation}
Based on the definitions, we can see that
$$
\cdelta{3} = \Omega\lpp \lppp \log \frac{1}{\delta}\rppp^{\frac{1}{2}}\rpp,\ \cdn{4} = \Omega \lpp \log t + \log \frac{1}{\delta}\rpp,
$$
and consequently,
$$
\arraycolsep=1.1pt\def\arraystretch{1.5}
\begin{array}{ccl}
\cdn{5}& =& \Omega\lpp \lppp \log t \rppp^{\frac{2}{2-\beta_{sg}}} +  \lppp \log \frac{1}{\delta} \rppp^{\frac{2}{2-\beta_{sg}}}  \rpp,\\

\cdelta{6}& = &  \Omega\lpp \lppp \log \frac{1}{\delta}\rppp^{\frac{1}{2}}\rpp,\\

\cdn{7} & = & \Omega \lpp \lppp \log \frac{1}{\delta}\rppp^{\frac{2}{2-\beta_{sg}}} + \lppp \log t \rppp^{\frac{8}{4-\beta_{sg}^2}} \lppp \log \frac{1}{\delta}\rppp^{\frac{-2}{2+\beta_{sg}}}\rpp. \\
\end{array}
$$
Therefore, based on (\ref{eq:mombound6}), on $S_{n,\delta,t}$, we have $F(\bdtheta_{n+t})=O\lpp \lppp \log t \rppp^{\frac{2}{2-\beta_{sg}}} +  \lppp \log \frac{1}{\delta} \rppp^{\frac{2}{2-\beta_{sg}}}  \rpp$.

To be particular, When $\delta \leq \delta_0$, based on Lemma \ref{lemma:deltaconclu}, we have
\begin{equation}
\lppp \log \frac{1}{\delta}\rppp^{\frac{1}{2}} \leq \cdelta{3} \leq \lppp 3\log \frac{1}{\delta}\rppp^{\frac{1}{2}}.
    \label{eq:mombound11}
\end{equation}
As $\frac{1}{\delta} \ge 6e^2$, we can see that
\begin{equation}
\cdn{4} \leq 2C_s \csg \lpp \log t+  2\log \frac{1}{\delta}\rpp.
    \label{eq:mombound12}
\end{equation}
We can derive a simpler bound on $\cdn{5}$ as follows,
\begin{equation}
\resizebox{.92\hsize}{!}{$
\arraycolsep=1.1pt\def\arraystretch{1.5} 
\begin{array}{rl}
\cdn{5} & = 4\lpp C_F - F_{\min} + \cjin (\csg)^{\frac{1}{2}}\cdelta{3} + \frac{2\alpha C^2}{2\alpha-1}\lpp 8C_s\csg\lppp 1-\frac{\beta_{sg}}{4}\rppp + \cdn{4} \lpp 1 + \frac{2-\beta_{sg}}{2} (2\cdn{4}\beta_{sg})^{\frac{\beta_{sg}}{2-\beta_{sg}}}\rpp \rpp  \rpp\\

& \leq 4\lpp C_F - F_{\min} + \cjin (\csg)^{\frac{1}{2}}\cdelta{3} + \frac{2\alpha C^2}{2\alpha-1}\lpp 8C_s\csg\lppp 1-\frac{\beta_{sg}}{4}\rppp + (2-\beta_{sg})(2\beta_{sg})^{\frac{\beta_{sg}}{2-\beta_{sg}}} (2C_s\csg)^{\frac{2}{2-\beta_{sg}}} \lpp \log t +2 \log \frac{1}{\delta}\rpp^{\frac{2}{2-\beta_{sg}}} \rpp \rpp \\

& \leq 4\lpp F(\bdtheta_0) - F_{\min} + \cjin (\csg)^{\frac{1}{2}}\cdelta{3} + \frac{4\alpha C^2}{2\alpha-1}(2-\beta_{sg}) (2\beta_{sg})^{\frac{\beta_{sg}}{2-\beta_{sg}}} (2C_s\csg)^{\frac{2}{2-\beta_{sg}}} \lpp \log t +2 \log \frac{1}{\delta}\rpp^{\frac{2}{2-\beta_{sg}}} \rpp \\

& \leq \frac{48(2-\beta_{sg})\alpha C^2}{2\alpha-1} (2\beta_{sg})^{\frac{\beta_{sg}}{2-\beta_{sg}}} (2C_s\csg)^{\frac{2}{2-\beta_{sg}}} \lpp \log t +2 \log \frac{1}{\delta}\rpp^{\frac{2}{2-\beta_{sg}}}=a_{mb,2} \lpp \log t +2 \log \frac{1}{\delta}\rpp^{\frac{2}{2-\beta_{sg}}},\\
\end{array}
$}
\label{eq:mombound13}
\end{equation}
where the 1st step is based on Lemma \ref{lemma:deltaconclu} and (\ref{eq:mombound12}), the 2nd and the 3rd step is based on Lemma \ref{lemma:deltaconclu}, and the second to the last step is based on Lemma \ref{lemma:deltaconclu} and (\ref{eq:mombound11}). 

Based on (\ref{eq:mombound11}), we have
\begin{equation}
4\cjin (\csg)^{\frac{1}{2}} \lpp \log \frac{1}{\delta}\rpp^{\frac{1}{2}} \leq \cdelta{6} \leq 4\cjin (3\csg)^{\frac{1}{2}} \lpp \log \frac{1}{\delta}\rpp^{\frac{1}{2}}.
    \label{eq:mombound15}
\end{equation}
Based on (\ref{eq:mombound13}) and (\ref{eq:mombound15}), we have
\begin{equation}
\lpp \frac{\cdn{5}}{\cdelta{6}}\rpp^{\frac{4}{2+\beta_{sg}}} \leq \lpp \frac{a_{mb,2}}{4\cjin (\csg)^{\frac{1}{2}}}\rpp^{\frac{4}{2+\beta_{sg}}} \lpp \log t +2 \log \frac{1}{\delta}\rpp^{\frac{8}{4-\beta_{sg}^2}} \lpp \log \frac{1}{\delta} \rpp^{-\frac{2}{2+\beta_{sg}}}.
    \label{eq:mombound16}
\end{equation}
Based on (\ref{eq:mombound15}), we have
\begin{equation}
(3\cdelta{6})^{\frac{4}{2-\beta_{sg}}} \leq (12\cjin (3\csg)^{\frac{1}{2}})^{\frac{4}{2-\beta_{sg}}} \lpp \log \frac{1}{\delta} \rpp^{\frac{2}{2-\beta_{sg}}}.
\label{eq:mombound17}
\end{equation}
Therefore, based on Lemma \ref{lemma:deltaconclu}, (\ref{eq:mombound16}) and (\ref{eq:mombound17}), we have
$$
\cdn{7} \leq a_{mb,3} \lpp \log t +2 \log \frac{1}{\delta}\rpp^{\frac{8}{4-\beta_{sg}^2}} \lpp \log \frac{1}{\delta} \rpp^{-\frac{2}{2+\beta_{sg}}}.
$$
Consequently, based on (\ref{eq:momboundadd1}), when $\delta \leq \delta_0$, on $S_{n,\delta,t}$, we have
\begin{equation}
A_{n,t} \leq a_{mb,3} \lpp \log t +2 \log \frac{1}{\delta}\rpp^{\frac{8}{4-\beta_{sg}^2}} \lpp \log \frac{1}{\delta} \rpp^{-\frac{2}{2+\beta_{sg}}}.
    \label{eq:mombound19}
\end{equation}
Further, based on (\ref{eq:mombound6}), on $S_{n,\delta,t}$, we have
\begin{equation}
\resizebox{.92\hsize}{!}{$
\arraycolsep=1.1pt\def\arraystretch{1.5} 
\begin{array}{rl}
F(\bdtheta_{n+t})-\fmin & \leq \frac{1}{4} \lpp \cdelta{6} \cdn{7}^{\frac{1}{2}} + \cdelta{6}\cdn{7}^{\frac{2+\beta_{sg}}{4}} + \cdn{5} \rpp\\

&\leq \frac{1}{4} \lpp 2 \cdelta{6} a_{mb,3}^{\frac{2+\beta_{sg}}{4}} \lpp \log t+2\log \frac{1}{\delta}\rpp^{\frac{2}{2-\beta_{sg}}} \lpp \log \frac{1}{\delta}\rpp^{-\frac{1}{2}} + \cdn{5} \rpp\\

& \leq \frac{1}{4} \lpp  8\cjin (3\csg)^{\frac{1}{2}} a_{mb,3}^{\frac{2+\beta_{sg}}{4}} +a_{mb,2}\rpp \lpp \log t+2\log \frac{1}{\delta}\rpp^{\frac{2}{2-\beta_{sg}}}\\

& \leq \frac{1}{4} \lpp  8\cjin (3\csg)^{\frac{1}{2}} a_{mb,3}^{\frac{2+\beta_{sg}}{4}} +a_{mb,2}\rpp 2^{\frac{\beta_{sg}}{2-\beta_{sg}}} \lpp (\log t)^{\frac{2}{2-\beta_{sg}}} + \lpp 2\log \frac{1}{\delta}\rpp^{\frac{2}{2-\beta_{sg}}} \rpp,\\
\end{array}
$}
\end{equation}
where the 2nd step is based the fact that $a_{mb,3} \lpp \log t +2 \log \frac{1}{\delta}\rpp^{\frac{8}{4-\beta_{sg}^2}} \lpp \log \frac{1}{\delta} \rpp^{-\frac{2}{2+\beta_{sg}}} \ge 1$, supported by Lemma \ref{lemma:deltaconclu}, the 3rd step is based on (\ref{eq:mombound13}) and (\ref{eq:mombound15}), and the last step is based on the Jensen's inequality.

Therefore, if we let $c_{11} \triangleq 2^{\frac{2}{2-\beta_{sg}}}a_{mb,4}$, we have
$$
\PP_n \lpp F(\bdtheta_{n+t}) - \fmin \ge a_{mb,4}  \lppp \log t \rppp^{\frac{2}{2-\beta_{sg}}} + c_{11} \lppp \log \frac{1}{\delta} \rppp^{\frac{2}{2-\beta_{sg}}} \rpp \leq \delta,\ \text{for}\ \delta \leq \delta_0.
$$
Then, for any $k\in \ZZ_+$, we have
$$
\arraycolsep=1.1pt\def\arraystretch{1.5}  
\begin{array}{rl}
& \EE_n \lpp F(\bdtheta_{n+t}) -\fmin - a_{mb,4}  \lppp \log t \rppp^{\frac{2}{2-\beta_{sg}}}\rpp_+^k \\

=&  \bigintsss_0^{\infty}\PP_n \lpp  \lpp F(\bdtheta_{n+t}) -\fmin - a_{mb,4}  \lppp \log t \rppp^{\frac{2}{2-\beta_{sg}}}\rpp_+ \ge x^{\frac{1}{k}} \rpp dx \\

=&  \bigintsss_0^{\infty}\PP_n \lpp  F(\bdtheta_{n+t}) -\fmin - a_{mb,4}  \lppp \log t \rppp^{\frac{2}{2-\beta_{sg}}} \ge x^{\frac{1}{k}} \rpp dx\\

=& c_{11}^k \frac{2k}{2-\beta_{sg}} \bigintsss_0^1 \frac{1}{\delta} \lppp \log \frac{1}{\delta} \rppp^{\frac{2k-2+\beta_{sg}}{2-\beta_{sg}}} \PP_n \lpp  F(\bdtheta_{n+t}) -\fmin - a_{mb,4}  \lppp \log t \rppp^{\frac{2}{2-\beta_{sg}}} \ge c_{11} \lppp \log \frac{1}{\delta} \rppp^{\frac{2}{2-\beta_{sg}}} \rpp d\delta \\

\leq & c_{11}^k \frac{2k}{2-\beta_{sg}} \lpp (1-\delta_0) \delta_0^{-1}  \lppp \log \delta_0^{-1} \rppp^{\frac{2k-2+\beta_{sg}}{2-\beta_{sg}}} + \bigintsss_0^{\delta_0} \lppp \log \frac{1}{\delta} \rppp^{\frac{2k-2+\beta_{sg}}{2-\beta_{sg}}} d\delta \rpp=a_{mb,6}^{(k)}.\\
\end{array}
$$
Therefore,
$$
\arraycolsep=1.1pt\def\arraystretch{1.5}  
\begin{array}{rl}
\EE_n \lppp F(\bdtheta_{n+t}) -\fmin \rppp^k & = \EE_n \lpp F(\bdtheta_{n+t}) -\fmin - a_{mb,4}  \lppp \log t \rppp^{\frac{2}{2-\beta_{sg}}} +a_{mb,4}  \lppp \log t \rppp^{\frac{2}{2-\beta_{sg}}} \rpp^k\\

& \leq \EE_n \lpp \lpp F(\bdtheta_{n+t}) -\fmin - a_{mb,4}  \lppp \log t \rppp^{\frac{2}{2-\beta_{sg}}}\rpp_+ +a_{mb,4}  \lppp \log t \rppp^{\frac{2}{2-\beta_{sg}}} \rpp^k\\

& \leq 2^{k-1} \EE_n \lpp F(\bdtheta_{n+t}) -\fmin - a_{mb,4}  \lppp \log t \rppp^{\frac{2}{2-\beta_{sg}}}\rpp_+^k + 2^{k-1} \lpp a_{mb,4}  \lppp \log t \rppp^{\frac{2}{2-\beta_{sg}}} \rpp^k\\

& \leq 2^{k-1} a_{mb,6}^{(k)} + 2^{k-1} a_{mb,4}^k (\log t)^{\frac{2k}{2-\beta_{sg}}}.
\end{array}
$$
Based on (\ref{eq:mombound19}), we have
$$
A_{n,t} \leq a_{mb,3} 2^{\frac{4+\beta_{sg}^2}{4-\beta_{sg}^2}} \lpp (\log t)^{\frac{8}{4-\beta_{sg}^2}} + 2^{\frac{8}{4-\beta_{sg}^2}} \lpp \log \frac{1}{\delta}\rpp^{\frac{2}{2-\beta_{sg}}} \rpp.
$$
Then, likewise, we have
$$
\EE_n A_{n,t}^k \leq 2^{k-1} a_{mb,7}^{(k)} + 2^{k-1} a_{mb,5}^{(k)} (\log t)^{\frac{8k}{4-\beta_{sg}^2}}.
$$

\end{proof}

\begin{lemma}
For any $c>0$, the following conclusions are true.
\begin{enumerate}
    \itemequationnormals[]{}{$$ x \ge \max \lww 2\log \lpp 2(3c)^{\frac{2+\beta_{sg}}{2}}\rpp ,4(2+\beta_{sg})^2 \rww
    \Longrightarrow \log x \ge  \log \log \lpp 2(3c)^{\frac{2+\beta_{sg}}{2}} \lppp x\rppp^{\frac{2+\beta_{sg}}{2}} \rpp.$$}

    \itemequationnormals[]{}{$$ x \ge \frac{1}{3c} \lpp \frac{e}{2} \rpp^{\frac{2}{2+\beta_{sg}}}
    \Longrightarrow \log \lpp 2(3c)^{\frac{2+\beta_{sg}}{2}}  x^{\frac{2+\beta_{sg}}{2}} \rpp \ge 1.$$}

    \itemequationnormals[]{}{$$ x \ge
    \exp \lpp \lpp \frac{2}{2-\beta_{sg}}\rpp^{\frac{2-\beta_{sg}}{\beta_{sg}}}  (8\beta_{sg} C_s\csg )^{-1}\rpp
    \Longrightarrow
    \frac{2-\beta_{sg}}{2} \lppp8\beta_{sg}C_s\csg\log x\rppp^{\frac{\beta_{sg}}{2-\beta_{sg}}} \ge 1.
    $$}

    \itemequationnormals[]{}{$$ \arraycolsep=1.1pt\def\arraystretch{1.5}  
    \begin{array}{rl}
    & x \ge
    \exp \lpp (4\beta_{sg} C_s\csg)^{-\frac{\beta_{sg}}{2}} \lppp\frac{4-\beta_{sg}}{2-\beta_{sg}}\rppp^{\frac{2-\beta_{sg}}{2}}\rpp\\
    \Longrightarrow&
    8C_s\csg\lppp 1-\frac{\beta_{sg}}{4}\rppp \leq  (2-\beta_{sg})(2\beta_{sg})^{\frac{\beta_{sg}}{2-\beta_{sg}}} (4C_s\csg  \log x )^{\frac{2}{2-\beta_{sg}}}.\\
    \end{array}
    $$}

    \itemequationnormals[]{}{$$\arraycolsep=1.1pt\def\arraystretch{1.5}  
    \begin{array}{rl}
    & x \ge \exp \lpp \lpp \frac{(2\alpha-1)(C_F-\fmin)}{4(2-\beta_{sg})\alpha C^2}\rpp^{\frac{2-\beta_{sg}}{2}} (2\beta_{sg})^{-\frac{\beta_{sg}}{2}} (4C_s\csg )^{-1}\rpp\\
    \Longrightarrow&
    C_F - \fmin \leq \frac{4(2-\beta_{sg})\alpha C^2}{2\alpha-1} (2\beta_{sg})^{\frac{\beta_{sg}}{2-\beta_{sg}}} (4C_s\csg)^{\frac{2}{2-\beta_{sg}}} ( \log x )^{\frac{2}{2-\beta_{sg}}}.\\
    \end{array} $$}

    \itemequationnormals[]{}{$$ \arraycolsep=1.1pt\def\arraystretch{1.5}  
    \begin{array}{rl}
    & x \ge \exp \lpp \lpp \frac{(2\alpha-1)\cjin(3\csg)^{\frac{1}{2}}}{4(2-\beta_{sg})\alpha C^2}\rpp^{\frac{2(2-\beta_{sg})}{2+\beta_{sg}}} (2\beta_{sg})^{-\frac{2\beta_{sg}}{2+\beta_{sg}}} (4C_s\csg)^{-\frac{4}{2+\beta_{sg}}}\rpp\\
    \Longrightarrow&
    \cjin (3\csg)^{\frac{1}{2}} \leq \frac{4(2-\beta_{sg})\alpha C^2}{2\alpha-1} (2\beta_{sg})^{\frac{\beta_{sg}}{2-\beta_{sg}}} (4C_s\csg)^{\frac{2}{2-\beta_{sg}}} ( \log x )^{\frac{2+\beta_{sg}}{2(2-\beta_{sg})}}.\\
    \end{array} $$}

    \itemequationnormals[]{}{$$ 
    x \ge \exp \lppp c^{-\frac{2-\beta_{sg}}{2}} \rppp 
    \Longrightarrow
    c(\log x)^{\frac{2}{2-\beta_{sg}}} \ge 1. $$}

    \itemequationnormals[]{}{$$ 
    x \ge \exp \lpp c^{-\frac{2-\beta_{sg}}{2}}\rpp
    \Longrightarrow
    (\log x)^{\frac{1}{2-\beta_{sg}}} \leq c^{\frac{\beta_{sg}}{4}} (\log x)^{\frac{2+\beta_{sg}}{2(2-\beta_{sg})}}. $$}

\end{enumerate}
\label{lemma:deltaconclu}
\end{lemma}


\begin{lemma}
Suppose that conditions \ref{cm_1} and \ref{cm_2} hold. We let
$$
\arraycolsep=1.1pt\def\arraystretch{1.5}
\begin{array}{ccl}
\delta_1 &=& \tiln^{-d_1},\\
\delta_2 &=& \tiln^{-d_2},\\
T &=& \tiln^{\alpha} \log (\tiln^{\beta_T \alpha}),\\

B_0 &=& \tiln^{-(\frac{1}{2} - \beta_B)\alpha},\\

B_1 &=& 2\lpp1+\frac{1}{C_s}\rpp\tiln^{-(\frac{1}{2} - \beta_B - CC_s\beta_T)\alpha},\\

B_2 &=& 160C_{hl}\lpp 1+\frac{1}{C_s}\rpp^3 \tiln^{-(1-2\beta_B-4CC_s\beta_T)\alpha},\\

B_3 &=& 160CC_{hl}\lpp 1+\frac{1}{C_s}\rpp^3 \tiln^{-(1-2\beta_B-4CC_s\beta_T)\alpha}\log(\tiln^{\beta_T\alpha}),\\

r_b &=& \sup \lwww \|\bdtheta - \opt\|: \opt \in \Theta^{opt}, \bdtheta \in \{ \bdtheta: \| \nabla F(\bdtheta)\| \leq 2\} \rwww.\\
\end{array}
$$
We define
$$
\arraycolsep=1.1pt\def\arraystretch{1.5}
\begin{array}{ccl}
\cnlemma{11}(d_1,d_2) &\triangleq & \max \Big\{ (CC_s)^{\frac{1}{\alpha}}, \lppp \frac{\alpha}{CC_s}\rppp^{\frac{1}{1-\alpha}} ,\lpp 2\lpp 1+\frac{1}{C_s}\rpp \rpp^{\frac{2}{1-2\beta_B-2CC_s\beta_T}},\\

&& \ \ \ \ \ \ \ \lpp\frac{80C_{hl}(1+C_s)^2}{C_s^2} \rpp^{\frac{1}{(1-\beta_B-3CC_s\beta_T)\alpha}},\lpp \frac{160CC_{hl}(1+C_s)^2}{C_s^2} \frac{\beta_T}{1-\beta_B-3CC_s\beta_T}\rpp^{\frac{2}{(1-\beta_B-3CC_s\beta_T)\alpha}},\\

&& \ \ \ \ \ \ \ \lpp 2^{\frac{\beta_B\alpha}{d_1}-1} \lppp 1+ 2^{\beta_{sg}}\rppp \frac{d_1CC_s\csg}{\beta_B\alpha} \rpp^{\frac{1}{\beta_B\alpha}}, \lpp \frac{16 \lppp 1+\frac{1}{C_s}\rppp^2}{\csg \lppp 1+2^{\beta_{sg}}\rppp} \rpp^{\frac{1}{(1-2\beta_B - 2CC_s\beta_T)\alpha}},\\

&& \ \ \ \ \ \ \ \lpp \frac{CC_s}{10(1+C_s)^2} \rpp^{\frac{1}{2CC_s\beta_T \alpha}}, \lpp \frac{\lppp 1+2^{\beta_{sg}}\rppp C^2 C_s^3 \csg}{160 (1+C_s)^4} \rpp^{\frac{1}{(2\beta_B + 4CC_s\beta_T)\alpha}},\\

&& \ \ \ \ \ \ \ \exp \lpp \frac{1}{\beta_T\alpha}\rpp, \exp \lpp \frac{1}{\alpha+d_2} \lpp \frac{\sqrt{2}}{\sqrt{2} - 1}\rpp^2 \rpp,\lpp \frac{160CC_{hl}^2 (1+C_s)^2}{C_s^3}\rpp^{\frac{1}{(2-2\beta_B-4CC_s\beta_T)\alpha}},\\

&& \ \ \ \ \ \ \ \lpp \frac{CC_s}{40(1+C_s)^2} \rpp^{\frac{1}{(1+2CC_s\beta_T)\alpha}},\lpp \frac{\lppp1+2^{\beta_{sg}} \rppp C\csg C_s^3}{40 (1+C_s)^4} \rpp^{\frac{1}{(2\beta_B + 4CC_s\beta_T)\alpha}},\\

&& \ \ \ \ \ \ \ \lpp \frac{\lppp 1+2^{\beta_{sg}}\rppp^2 C^3 (\csg)^2 C_s^5}{640 (1+C_s)^6} \rpp^{\frac{1}{(1+4\beta_B+6CC_s\beta_T)\alpha}},\lpp \frac{4CC_{hl}}{C_s}\rpp^{\frac{1}{\lppp \frac{3}{2} +\beta_B +CC_s\beta_T\rppp\alpha}},\\

&& \ \ \ \ \ \ \ \lpp \frac{\lppp 1+2^{\beta_{sg}} \rppp^{\frac{3}{2}} C^2 \lppp \csg\rppp^{\frac{3}{2}} C_s^4}{80 (1+C_s)^5} \rpp^{\frac{1}{\lppp \frac{1}{2} + 3\beta_B + 5CC_s\beta_T \rppp\alpha}},\lpp \frac{\lppp1+2^{\beta_{sg}}\rppp C \csg C_s^2}{4(1+C_s)^2}\rpp^{\frac{1}{(2\beta_B + 2CC_s\beta_T)\alpha}},\\

&& \ \ \ \ \ \ \ \lpp \frac{\lppp1+2^{\beta_{sg}}\rppp C^2 \csg C_s^3}{40 (1+C_s)^4}\rpp^{\frac{1}{(1+2\beta_B+4CC_s\beta_T)\alpha}},\lpp \frac{8^{CC_s}(\alpha+\beta_T\alpha+d_2)C_s}{80 CC_s\beta_T\alpha (1+C_s)^2}\rpp^{\frac{1}{CC_s\beta_T\alpha}},\\

&& \ \ \ \ \ \ \ \lpp \frac{2(\alpha+\beta_T\alpha +d_2)8^{\frac{1+4\beta_B+8CC_s\beta_T}{2+2\beta_T}}\lppp 1+2^{\beta_{sg}}\rppp^{\frac{1}{2}}CC_s^{\frac{3}{2}}\lppp\csg\rppp^{\frac{1}{2}}}{\sqrt{80}(1+4\beta_B+8CC_s\beta_T)\alpha\cjin^{\frac{1}{2}}(1+C_s)^2}  \rpp^{\frac{2}{(1+4\beta_B+8CC_s\beta_T)\alpha}},\\

&& \ \ \ \ \ \ \ \lpp \frac{4\lppp1+2^{\beta_{sg}}\rppp^2 d_2 C^2 C_s^4 \lppp \csg\rppp^2 8^{\frac{(\beta_B+CC_s\beta_T)\alpha}{d_2}} }{25(\beta_B+CC_s\beta_T)\alpha (1+C_s)^4}\rpp^{\frac{1}{(\beta_B+CC_s\beta_T)\alpha}},\\

&& \ \ \ \ \ \ \ \lpp \frac{4\lppp 1+2^{\beta_{sg}}\rppp d_2C\cjin C_s^3 \csg 8^{\frac{(\beta_B+CC_s\beta_T)\alpha}{d_2}}}{25 (1+C_s)^2 (\beta_B+CC_s\beta_T)\alpha} \rpp^{\frac{1}{(\beta_B+CC_s\beta_T)\alpha}}\rww
\end{array}
$$

Suppose $d_1 \ge \beta_B\alpha$, $2\beta_B + 4CC_s\beta_T<1$, $\tiln \ge \cnlemma{11}(d_1,d_2)$. For $n\ge \tiln$, consider the second-order approximation of $F(\cdot)$ at $\bdtheta_n$,
\begin{equation}
\tilde{F}_n(\bdtheta) \triangleq F(\bdtheta_n) + \langle \nabla F (\bdtheta_n), \bdtheta - \bdtheta_n \rangle + \frac{1}{2}(\bdtheta - \bdtheta_n)^T \nabla^2 F(\bdtheta_n) (\bdtheta - \bdtheta_n).
    \label{eq:saddle1}
\end{equation}
and an auxiliary stochastic gradient descent sequence based on $\tilde{F}_n(\cdot)$,
\begin{equation}
\tilde{\bdtheta}_{n+t} = \tilde{\bdtheta}_{n+t-1} - \gamma_{n+t} (\nabla \tilde{F}_n (\tilde{\bdtheta}_{n+t-1}) + \nabla f(\bdtheta_{n+t-1};\bdy_{n+t}) - \nabla F(\bdtheta_{n+t-1})), t\ge 1,
    \label{eq:saddle2}
\end{equation}
with $\tilde{\bdtheta}_n = \bdtheta_n$. For $t\ge0$, we have the following decomposition
$$
\resizebox{.92\hsize}{!}{$
\arraycolsep=1.1pt\def\arraystretch{1.5}
\begin{array}{ccl}
\tildetheta{n+t} - \bdtheta_n & = & -\sum\limits_{j=1}\limits^t \gamma_{n+j} \lpp \pprod{i=j+1}{t} (I - \gamma_{n+i}\hcal_n) \rpp \bdxi_{n+j} - \sum\limits_{j=0}\limits^{t-1}\gamma_{n+j+1} \lpp \prod\limits_{i=1}\limits^j (I - \gamma_{n+i} \hcal_n)\rpp \nabla F(\bdtheta_n).\\
\end{array}
$}
$$
Further, let
$$
\resizebox{.98\hsize}{!}{$
S_{n,T} \triangleq \Big\{ \big\| \tildetheta{n+t} - \bdtheta_n \big\| \leq B_1, \big\| \nablatilde{n+t}\big\| \leq B_1, \big\| \nablaf{n+t} - \nablatilde{n+t}\big\| \leq B_2, \big\| \tildetheta{n+t}-\bdtheta_{n+t} \big\| \leq B_3,\forall 0\leq t \leq T \Big\}.
$}
$$
We let
$$
\rgb \triangleq \lw \bdtheta: \lv \nabla F(\bdtheta) \rv \leq B_0, \lambda_{min} (\nabla^2 F(\bdtheta )) \leq -\Tilde{\lambda} \rw,
$$
where $\tilde{\lambda}$ is defined in condition \ref{cm_4} (but note that we don't need condition \ref{cm_4} in the current Lemma). On $\{\bdtheta_n \in \rgb{}\}$, we have
\begin{equation*}
    \PP_n\lpp S_{n,T}  \rpp \ge 1-T(\delta_1+\delta_2).
\end{equation*}
\label{lemma:run}
\end{lemma}

\begin{proof}
Let's firstly present the simplified implications of the sophisticated requirements on $\tiln$.
\begin{enumerate}
    \itemequationnormal[eq:simple0]{}{$$ \tiln \ge \lpp 2\lpp 1+\frac{1}{C_s}\rpp \rpp^{\frac{2}{1-2\beta_B-2CC_s\beta_T}}
    \Longrightarrow B_1 \leq 1.
    $$}
    
    \itemequationnormal[eq:simple1]{}{$$ \tiln \ge \lpp 80C_{hl} \lppp 1+\frac{1}{C_s}\rppp^2 \rpp^{\frac{1}{(1-\beta_B-3CC_s\beta_T)\alpha}}
    \Longrightarrow B_2 \leq B_1.
    $$}

    \itemequationnormal[eq:simple1_1]{}{$$ \tiln \ge \lpp 160CC_{hl} \lppp1+\frac{1}{C_s}\rppp^2 \frac{\beta_T}{1-\beta_B-3CC_s\beta_T}\rpp^{\frac{2}{(1-\beta_B-3CC_s\beta_T)\alpha}} \Longrightarrow B_3 \leq B_1.
    $$}

    \itemequationsmall[eq:simple3]{}{$$
    \tiln \ge \lpp 2^{\frac{\beta_B\alpha}{d_1}-1} \lppp 1+ 2^{\beta_{sg}}\rppp \frac{d_1CC_s\csg}{\beta_B\alpha} \rpp^{\frac{1}{\beta_B\alpha}}
     \Longrightarrow \lpp \frac{C\csg}{2C_s}(1+2^{\beta_{sg}})\rpp^{\frac{1}{2}} \lpp \log \frac{2}{\delta_1} \rpp^{\frac{1}{2}}  \leq \frac{1}{C_s}\tiln^{\beta_B\alpha}.
    $$}

    \itemequationnormal[eq:simple4]{}{$$
    \tiln \ge \lpp \frac{16 \lppp 1+\frac{1}{C_s}\rppp^2}{\csg \lppp 1+2^{\beta_{sg}}\rppp} \rpp^{\frac{1}{(1-2\beta_B - 2CC_s\beta_T)\alpha}}
    \Longrightarrow 4B_1^2 \leq \csg (1+2^{\beta_{sg}}).
    $$}

    \itemequationnormal[eq:simple5]{}{$$ \tiln \ge \lpp \frac{CC_s}{10(1+C_s)^2} \rpp^{\frac{1}{2CC_s\beta_T \alpha}}
    \Longrightarrow 4C_{hl}\gamma_{n+j} B_1^2 \leq (1+C_s)B_2.
    $$}

    \itemequationnormal[eq:simple6]{}{$$\tiln \ge \lpp \frac{\lppp 1+2^{\beta_{sg}}\rppp C^2 C_s^3 \csg}{160 (1+C_s)^4} \rpp^{\frac{1}{(2\beta_B + 4CC_s\beta_T)\alpha}}
    \Longrightarrow (1+2^{\beta_{sg}})C_{hl}\csg \gamma_{\tiln}^2 \leq (1+C_s)B_2.
    $$}

    \itemequationnormal[eq:simple7]{}{$$ \tiln \ge \max \lww \exp \lpp \frac{1}{\beta_T\alpha}\rpp, \exp \lpp \frac{1}{\alpha+d_2} \lpp \frac{\sqrt{2}}{\sqrt{2} - 1}\rpp^2 \rpp \rww
    \Longrightarrow \frac{\sqrt{2}}{\sqrt{2}-1} \leq \lpp \log \frac{8T}{\delta_2}\rpp^{\frac{1}{2}}.
    $$}

    \itemequationnormal[eq:simple8]{}{$$
    \tiln \ge \lpp \frac{160CC_{hl}^2 (1+C_s)^2}{C_s^3}\rpp^{\frac{1}{(2-2\beta_B-4CC_s\beta^T)\alpha}}
    \Longrightarrow
    C_{hl}B_2\gamma_{\tiln} \leq 1+C_s.
    $$}

    \itemequationnormal[eq:simple9]{}{$$
    \tiln \ge \lpp \frac{CC_s}{40(1+C_s)^2} \rpp^{\frac{1}{(1+2CC_s\beta_T)\alpha}}
    \Longrightarrow
    C_{hl}B_1^2 \gamma_{\tiln} \leq (1+C_s)B_2.
    $$}

    \itemequationnormal[eq:simple10]{}{$$
    \tiln \ge \lpp \frac{\lppp1+2^{\beta_{sg}} \rppp C\csg C_s^3}{40 (1+C_s)^4} \rpp^{\frac{1}{(2\beta_B + 4CC_s\beta_T)\alpha}}
    \Longrightarrow
    4 \lppp1+2^{\beta_{sg}} \rppp C_{hl}\csg \gamma_{\tiln} \leq (1+C_s)B_2.
    $$}

    \itemequationscript[eq:simple11]{}{$$
    \tiln \ge \lpp \frac{\lppp 1+2^{\beta_{sg}}\rppp^2 C^3 (\csg)^2 C_s^5}{640 (1+C_s)^6} \rpp^{\frac{1}{(1+4\beta_B+6CC_s\beta_T)\alpha}}
    \Longrightarrow
    \lppp 1+2^{\beta_{sg}}\rppp^2 C_{hl} (\csg)^2 \gamma_{\tiln} \leq (1+C_s)B_1^2 B_2.
    $$}

    \itemequationnormal[eq:simple12]{}{$$
    \tiln \ge \lpp \frac{4CC_{hl}}{C_s}\rpp^{\frac{1}{\lppp \frac{3}{2} +\beta_B +CC_s\beta_T\rppp\alpha}}
    \Longrightarrow
    2C_{hl}\gamma_{\tiln}B_1 \leq (1+C_s).
    $$}

    \itemequationnormal[eq:simple13]{}{$$
    \begin{array}{rl}
    & \tiln \ge \lpp \frac{\lppp 1+2^{\beta_{sg}} \rppp^{\frac{3}{2}} C^2 \lppp \csg\rppp^{\frac{3}{2}} C_s^4}{80 (1+C_s)^5} \rpp^{\frac{1}{\lppp \frac{1}{2} + 3\beta_B + 5CC_s\beta_T \rppp\alpha}}\\
    \Longrightarrow &
    4 \lppp 1+2^{\beta_{sg}} \rppp^{\frac{3}{2}} C_{hl} \lppp \csg\rppp^{\frac{3}{2}} \gamma_{\tiln}^2 \leq (1+C_s)B_1B_2.
    \end{array}
    $$}

    \itemequationnormal[eq:simple14]{}{$$
    \tiln \ge \lpp \frac{\lppp1+2^{\beta_{sg}}\rppp C \csg C_s^2}{4(1+C_s)^2}\rpp^{\frac{1}{(2\beta_B + 2CC_s\beta_T)\alpha}}
    \Longrightarrow
    \lppp 1+2^{\beta_{sg}}\rppp \csg \gamma_{\tiln} \leq B_1^2.
    $$}

    \itemequationnormal[eq:simple15]{}{$$
    \tiln \ge \lpp \frac{\lppp1+2^{\beta_{sg}}\rppp C^2 \csg C_s^3}{40 (1+C_s)^4}\rpp^{\frac{1}{(1+2\beta_B+4CC_s\beta_T)\alpha}}
    \Longrightarrow
    4\lppp 1+2^{\beta_{sg}}\rppp C_{hl}\csg \gamma_{\tiln}^2 \leq (1+C_s)B_2.
    $$}

    \itemequationnormal[eq:simple16]{}{$$
    \tiln \ge \lpp \frac{8^{CC_s}(\alpha+\beta_T\alpha+d_2)C_s}{80 CC_s\beta_T\alpha (1+C_s)^2}\rpp^{\frac{1}{CC_s\beta_T\alpha}}
    \Longrightarrow
    C_{hl}B_1^2 \log \lpp\frac{8T}{\delta_2}\rpp\leq 2(1+C_s)B_2 \lpp \log \frac{8}{\delta_2}\rpp^{\frac{1}{2}}.
    $$}

    \itemequationnormal[eq:simple17]{}{$$
    \arraycolsep=1.1pt\def\arraystretch{1.5}
    \begin{array}{rl}
    & \tiln \ge  \lpp \frac{2(\alpha+\beta_T\alpha +d_2)8^{\frac{1+4\beta_B+8CC_s\beta_T}{2+2\beta_T}}\lppp 1+2^{\beta_{sg}}\rppp^{\frac{1}{2}}CC_s^{\frac{3}{2}}\lppp\csg\rppp^{\frac{1}{2}}}{\sqrt{80}(1+4\beta_B+8CC_s\beta_T)\alpha\cjin^{\frac{1}{2}}(1+C_s)^2}  \rpp^{\frac{2}{(1+4\beta_B+8CC_s\beta_T)\alpha}}\\
    \Longrightarrow &
    2\lppp1+2^{\beta_{sg}}\rppp C_{hl}\csg \gamma_{\tiln}^2 \lpp \log \frac{8T}{\delta_2}\rpp^{\frac{1}{2}} \leq \cjin (1+C_s)B_2  \lpp \log \frac{8}{\delta_2}\rpp^{\frac{1}{2}}.
    \end{array}
    $$}

    \itemequationnormal[eq:simple18]{}{$$
    \arraycolsep=1.1pt\def\arraystretch{1.5}
    \begin{array}{rl}
    & \tiln \ge \lpp \frac{4\lppp1+2^{\beta_{sg}}\rppp^2 d_2 C^2 C_s^4 \lppp \csg\rppp^2 8^{\frac{(\beta_B+CC_s\beta_T)\alpha}{d_2}} }{25(\beta_B+CC_s\beta_T)\alpha (1+C_s)^4}\rpp^{\frac{1}{(\beta_B+CC_s\beta_T)\alpha}}\\
    \Longrightarrow & \frac{8}{5} \lppp1+2^{\beta_{sg}}\rppp C\csg \tiln^{-\alpha} \lpp \log \frac{8}{\delta_2}\rpp^{\frac{1}{2}} \leq B_1^2.
    \end{array}
    $$}

    \itemequationnormal[eq:simple19]{}{$$
    \arraycolsep=1.1pt\def\arraystretch{1.5}
    \begin{array}{rl}
    & \tiln \ge \lpp \frac{4\lppp 1+2^{\beta_{sg}}\rppp d_2C\cjin C_s^3 \csg 8^{\frac{(\beta_B+CC_s\beta_T)\alpha}{d_2}}}{25 (1+C_s)^2 (\beta_B+CC_s\beta_T)\alpha} \rpp^{\frac{1}{(\beta_B+CC_s\beta_T)\alpha}}\\
    \Longrightarrow & \frac{4}{5} \lppp1+2^{\beta_{sg}}\rppp^{\frac{1}{2}} \cjin \lppp CC_s\csg\rppp^{\frac{1}{2}} \tiln^{-\frac{\alpha}{2}} \lpp \log \frac{8}{\delta_2}\rpp^{\frac{1}{2}} \leq B_1.
    \end{array}
    $$}

\end{enumerate}

For simplicity, we add some notations in the current proof,
$$
\begin{array}{c}
\hcal_n \triangleq \nabla^2 F(\bdtheta_n),\\
\bdxi_{n} \triangleq \nabla f(\bdtheta_{n-1};\bdy_n) - \nabla F(\bdtheta_{n-1}).
\end{array}
$$

Based on the definiton given in (\ref{eq:saddle1}), we have
$$
\nabla \tilde{F}_n(\bdtheta)  = \nabla F(\bdtheta_n) + \hcal_n(\bdtheta - \bdtheta_n),
$$
which implies
$$
\nablatilde{n+t} - \nablatilde{n+t-1} = \hcal_n (\tildetheta{n+t} - \tildetheta{n+t-1}) = -\gamma_{n+t} \hcal_n (\nablatilde{_{n+t-1}} + \bdxi_{n+t}),
$$
for any $t\in \ZZ_+$. Hence, we have
\begin{equation}
\arraycolsep=1.1pt\def\arraystretch{1.5}
\begin{array}{rl}
\nablatilde{n+t} & = (I - \gamma_{n+t} \hcal_n) \nablatilde{{n+t-1}} - \gamma_{n+t} \hcal_n \bdxi_{n+t}\\

& = \lpp \prod\limits_{j=1}\limits^t (I - \gamma_{n+j} \hcal_n) \rpp \nablatilde{n} - \hcal_n \sum\limits_{j=1}\limits^t \lpp \gamma_{n+j} \prod\limits_{i=j+1}\limits^t(I - \gamma_{n+i}\hcal_n) \rpp \bdxi_{n+j}\\

& = \lpp \prod\limits_{j=1}\limits^t (I - \gamma_{n+j} \hcal_n) \rpp \nabla F(\bdtheta_n) - \hcal_n \sum\limits_{j=1}\limits^t \lpp \gamma_{n+j} \prod\limits_{i=j+1}\limits^t(I - \gamma_{n+i}\hcal_n) \rpp \bdxi_{n+j},\\
\end{array}
    \label{eq:saddle3}
\end{equation}
and further,
\begin{equation}
\arraycolsep=1.1pt\def\arraystretch{1.5}
\begin{array}{rl}
\tildetheta{n+t}- \bdtheta_n &= \sum\limits_{j=1}\limits^t (\tildetheta{n+j} - \tildetheta{n+j-1})\\

& = \sum\limits_{j=1}\limits^t -\gamma_{n+j} (\nablatilde{n+j-1} + \bdxi_{n+j})\\

& = \sum\limits_{j=0}\limits^{t-1} -\gamma_{n+j+1} (\nablatilde{n+j} + \bdxi_{n+j+1})\\

& = \sum\limits_{j=0}\limits^{t-1} -\gamma_{n+j+1} \bdxi_{n+j+1} - \sum\limits_{j=0}\limits^{t-1}\gamma_{n+j+1} \lpp \prod\limits_{i=1}\limits^j (I - \gamma_{n+i} \hcal_n)\rpp \nabla F(\bdtheta_n)\\

& \ \ +\hcal_n \sum\limits_{j=0}\limits^{t-1} \gamma_{n+j+1} \lpp \sum\limits_{k=1}\limits^j \gamma_{n+k} \lpp \prod\limits_{i=k+1}\limits^j (I - \gamma_{n+i}\hcal )\rpp \bdxi_{n+k} \rpp,\\
\end{array}
\label{eq:saddle4}
\end{equation}
The third term on the right-hand side of (\ref{eq:saddle4}) can be simplified as
$$
\arraycolsep=1.1pt\def\arraystretch{1.5} 
\begin{array}{rl}
& \hcal_n \sum\limits_{j=0}\limits^{t-1} \gamma_{n+j+1} \lpp \sum\limits_{k=1}\limits^j \gamma_{n+k} \lpp \prod\limits_{i=k+1}\limits^j (I - \gamma_{n+i}\hcal )\rpp \bdxi_{n+k} \rpp\\

= & \hcal_n \sum\limits_{k=1}\limits^{t-1} \sum\limits_{j=k}\limits^{t-1} \gamma_{n+j+1} \gamma_{n+k} \lpp \prod\limits_{i=k+1}\limits^j (I - \gamma_{n+i} \hcal_n)\rpp \bdxi_{n+k}\\

= & \hcal_n \sum\limits_{k=1}\limits^{t-1} \gamma_{n+k} \lpp \sum\limits_{j=k}\limits^{t-1} \gamma_{n+j+1} \lpp \prod\limits_{i=k+1}\limits^j (I - \gamma_{n+i} \hcal_n)\rpp \rpp \bdxi_{n+k}\\

= & \sum\limits_{k=1}\limits^{t-1} \gamma_{n+k} \lpp I - \prod\limits_{j=k+1}\limits^t (I - \gamma_{n+j} \hcal_n) \rpp \bdxi_{n+k}.\\
\end{array}
$$
Therefore, we have
\begin{equation}
\resizebox{.92\hsize}{!}{$
\arraycolsep=1.1pt\def\arraystretch{1.5}
\begin{array}{rl}
\tildetheta{n+t} - \bdtheta_n & = -\sum\limits_{j=1}\limits^t \gamma_{n+j} \lpp \pprod{i=j+1}{t} (I - \gamma_{n+i}\hcal_n) \rpp \bdxi_{n+j} - \sum\limits_{j=0}\limits^{t-1}\gamma_{n+j+1} \lpp \prod\limits_{i=1}\limits^j (I - \gamma_{n+i} \hcal_n)\rpp \nabla F(\bdtheta_n)\\

& \triangleq T_{1,t} + T_{2,t},
\end{array}
    \label{eq:saddle5}
$}
\end{equation}
with
$$
\arraycolsep=1.1pt\def\arraystretch{1.5}
\begin{array}{ccl}
T_{1,t} &\triangleq& -\sum\limits_{j=1}\limits^t \gamma_{n+j} \lpp \pprod{i=j+1}{t} (I - \gamma_{n+i}\hcal_n) \rpp \bdxi_{n+j},\\

T_{2,t} & \triangleq & - \sum\limits_{j=0}\limits^{t-1}\gamma_{n+j+1} \lpp \prod\limits_{i=1}\limits^j (I - \gamma_{n+i} \hcal_n)\rpp \nabla F(\bdtheta_n).\\
\end{array}
$$
As the second term is relatively easy to control, we analyze it firstly. When $\lv \nabla F(\bdtheta_n) \rv \leq B_0$, $\lv T_{2,t} \rv$ can be bounded as
\begin{equation}
\arraycolsep=1.1pt\def\arraystretch{1.5}
\begin{array}{rl}
\lv T_{2,t} \rv & \leq \lpp \sum\limits_{j=0}\limits^{t-1}\gamma_{n+j+1} \lpp \prod\limits_{i=1}\limits^j (1 + C_s \gamma_{n+i})\rpp \rpp B_0\\

&=  \lpp \sum\limits_{j=0}\limits^{t-1}\gamma_{n+j+1} \lpp \prod\limits_{i=1}\limits^j (1 + CC_s(n+i)^{-\alpha} )\rpp \rpp B_0\\

&\leq  \lpp \sum\limits_{j=0}\limits^{t-1}\gamma_{n+j+1} \exp\lpp CC_s \lpp \ssum{i=1}{j}(n+i)^{-\alpha} \rpp\rpp B_0\\

&\leq  \lpp \sum\limits_{j=0}\limits^{t-1}\gamma_{n+j+1} \exp\lpp \frac{CC_s}{1-\alpha} ((n+j)^{1-\alpha} - n^{1-\alpha}) \rpp B_0\\

&\leq  \exp\lpp -\frac{CC_s}{1-\alpha}n^{1-\alpha} \rpp CB_0 \ssum{j=0}{t-1} (n+j)^{-\alpha} \exp\lpp \frac{CC_s}{1-\alpha} (n+j)^{1-\alpha} \rpp\\

&\leq  \exp\lpp -\frac{CC_s}{1-\alpha}n^{1-\alpha} \rpp CB_0 \frac{1}{CC_s}\exp\lpp \frac{CC_s}{1-\alpha} (n+t)^{1-\alpha} \rpp\\ 

&\leq  \frac{B_0}{C_s} \exp\lpp CC_s n^{-\alpha}t \rpp,\\
\end{array}
    \label{eq:saddle6}
\end{equation}
where the 1st step is true under condition \ref{cm_1} and the 6th step is based on Lemma \ref{lemma:int0} (notice that $\tiln \ge \lppp \frac{\alpha}{CC_s}\rppp^{\frac{1}{1-\alpha}})$.

For $t\ge0$, we define events
$$
\arraycolsep=1.1pt\def\arraystretch{1.5}
\begin{array}{ccl}
A_{n,t,0}& \triangleq & \lwww \lvvv \tildetheta{n+j} - \bdtheta_n \rvvv \leq B_1,j=0,1,\ldots,t\rwww,\\

A_{n,t,1} &\triangleq &\lwww\lvvv \nablatilde{n+j} \rvvv \leq B_1,j=0,1,\ldots,t\rwww,\\

A_{n,t,2}& \triangleq & \lwww \lvvv \nabla F(\bdtheta_{n+j}) - \nablatilde{n+j} \rvvv \leq B_2,j=0,1,\ldots,t\rwww,\\

A_{n,t,3}& \triangleq & \lwww \lvvv \tildetheta{n+j} - \bdtheta_{n+j} \rvvv \leq B_3,j=0,1,\ldots,t\rwww.\\
\end{array}
$$
For completeness, we let $A_{n,-1,0}$, $A_{n,-1,1}$, $A_{n,-1,2}$ and $A_{n,-1,3}$ be the whole space, i.e. $\daone_{A_{n,-1,i}}\equiv 1$ for $i=0,1,2,3$.

Now, we use induction to show that for $t=0,1,\ldots, T$, on \{$\bdtheta_n \in \rgb$\}, 
$$
\PP_n \lpp A_{n,t,0}\cap A_{n,t,1} \cap A_{n,t,2} \cap A_{n,t,3} \rpp \ge 1 - t(\delta_1+\delta_2).
$$
The above inequality is obviously true for $t=0$. Suppose that it is true for $t=k-1$, $1\leq k \leq T$. For simplicity, for $t=0,1,\ldots, T$, we let
$$
S_{n,t} \triangleq A_{n,t,0}\cap A_{n,t,1} \cap A_{n,t,2} \cap A_{n,t,3}.
$$
\twoline{To control $\lvvv \tildetheta{n+k} - \bdtheta_n\rvvv$}, based on (\ref{eq:saddle5}) and (\ref{eq:saddle6}), we just need to obtain a bound on $\lv T_{1,k}\rv$. Firstly, as all $I-\gamma_{n+i}\hcal_n$ is positive semidefinite, we can see that the largest eigenvalue of $\pprod{i=j+1}{k} (I - \gamma_{n+i}\hcal_n) $ can be bounded as
$$
\arraycolsep=1.1pt\def\arraystretch{1.5}
\begin{array}{rl}
\lambda_{max} \lpp \pprod{i=j+1}{k} (I - \gamma_{n+i}\hcal_n)  \rpp &\leq \pprod{i=j+1}{k}\lambda_{max}(I - \gamma_{n+i}\hcal_n) \leq \pprod{i=j+1}{k}(1+C_s\gamma_{n+i})\\
& \leq \exp\lpp C_s \ssum{i=j+1}{k}\gamma_{n+i} \rpp \leq \exp \lpp \frac{CC_s}{1-\alpha} ((n+k)^{1-\alpha} - (n+j)^{1-\alpha})\rpp\\

& \triangleq b_{n,k,j}.\\
\end{array}
$$
Then, based on the above result and condition \ref{cm_2}, we know that $-\gamma_{n+j} \lpp \pprod{i=j+1}{k} (I - \gamma_{n+i}\hcal_n) \rpp \bdxi_{n+j}$ is $\csg \gamma_{n+j}^2 b_{n,k,j}^2 \lppp 1+ \|\nablaf{n+j-1}\|^{\beta_{sg}} \rppp$-sub-Gaussian. Therefore, by Lemma \ref{lemma:jin}, conditional on $\fff_n$, with probability at least $1-\delta_1$, either
$$
\csg \ssum{j=1}{k}\gamma_{n+j}^2 b_{n,k,j}^2 \lppp 1+ \| \nablaf{n+j-1}\|^{\beta_{sg}}\rppp \leq 2\lppp 1 + 2^{\beta_{sg}}\rppp \csg \ssum{j=1}{k}\gamma_{n+j}^2 b_{n,k,j}^2,
$$
or
$$
\resizebox{.98\hsize}{!}{$
\| T_{1,k}\| \leq \max \lww \csg \ssum{j=1}{k}\gamma_{n+j}^2 b_{n,k,j}^2 \lppp 1+ \| \nablaf{n+j-1}\|^{\beta_{sg}}\rppp, \lppp 1 + 2^{\beta_{sg}}\rppp \csg \ssum{j=1}{k}\gamma_{n+j}^2 b_{n,k,j}^2 \rww^{\frac{1}{2}}\lpp \log \frac{2}{\delta_1} + \log \log 2\rpp^{\frac{1}{2}}.
$}
$$
Actually, on $S_{n,k-1}$, due to (\ref{eq:simple0}) and (\ref{eq:simple1}), we have
\begin{equation}
\|\nablaf{n+j-1}\| \leq \| \nablaf{n+j-1} - \nablatilde{n+j-1}\| + \|\nablatilde{n+j-1}\| \leq B_1 + B_2 \leq 2B_1 \leq 2,
\label{eq:saddlen1}
\end{equation}
for $j=1,2,\ldots,k$. Therefore, with the induction assumption that $\PP_n(S_{n,k-1}) \ge 1-(k-1)(\delta_1+\delta_2)$, we know that
\begin{equation}
\PP_n \lpp \| T_{1,k} \| \leq \lpp  \lppp 1 + 2^{\beta_{sg}}\rppp \csg \ssum{j=1}{k}\gamma_{n+j}^2 b_{n,k,j}^2 \rpp^{\frac{1}{2}} \lpp \log \frac{2}{\delta_1}\rpp^{\frac{1}{2}}, S_{n,k-1}\rpp \ge 1-k\delta_1-(k-1)\delta_2.
    \label{eq:saddlen2}
\end{equation}

To simplify the bound in (\ref{eq:saddlen2}), we have
$$
\arraycolsep=1.1pt\def\arraystretch{1.5} 
\begin{array}{rl}
& (1+2^{\beta_{sg}})\lppp \log \frac{2}{\delta_1} \rppp \csg 
 \ssum{j=1}{k}  \gamma_{n+j}^2 b_{n,k,j}^2\\

 = & C^2 \csg (1+2^{\beta_{sg}}) \lppp\log \frac{2}{\delta_1}\rppp 
   \exp\lpp \frac{2CC_s}{1-\alpha}(n+k)^{1-\alpha}\rpp \ssum{j=1}{k}  \frac{1}{(n+j)^{2\alpha}} \exp\lpp -\frac{2CC_s}{1-\alpha}(n+j)^{1-\alpha}\rpp\\

\leq & C^2 \csg (1+2^{\beta_{sg}}) \lppp\log \frac{2}{\delta_1}\rppp
   \exp\lpp \frac{2CC_s}{1-\alpha}(n+k)^{1-\alpha}\rpp \frac{n^{-\alpha}}{2CC_s} \exp\lpp-\frac{2CC_s}{1-\alpha} n^{1-\alpha}\rpp\\

\leq & \frac{C\csg }{2C_s}(1+2^{\beta_{sg}}) \lppp\log \frac{2}{\delta_1} \rppp n^{-\alpha} \exp (2CC_s k n^{-\alpha}),\\
\end{array}
$$
where the 2nd step is based on Lemma \ref{lemma:int2}. Therefore, if we denote
$$
\arraycolsep=1.1pt\def\arraystretch{1.5}
\begin{array}{rcl}
c_1 & \triangleq & \lpp \frac{C\csg }{2C_s}(1+2^{\beta_{sg}}) \rpp^{\frac{1}{2}},\\

U_{n,k} &\triangleq& \lw \lv T_{1,k} \rv \leq c_1\lppp \log \frac{2}{\delta_1}  \rppp^{\frac{1}{2}} n^{-\frac{\alpha}{2}} \exp(CC_s k n^{-\alpha}) \rw, \\
\end{array}
$$
we have
\begin{equation}
\PP_n \lpp 
 U_{n,k} \cap  S_{n,k-1}  \rpp \ge 1-k\delta_1 - (k-1)\delta_2,
    \label{eq:saddle9}
\end{equation}
on $\{\bdtheta_n\in \rgb{}\}$. Based on (\ref{eq:saddle5}) and (\ref{eq:saddle6}), on $U_{n,k}$, we have
$$
\arraycolsep=1.1pt\def\arraystretch{1.5}
\begin{array}{rl}
\lvvv \tildetheta{n+k} - \bdtheta_n\rvvv & \leq \lv T_{1,k} \rv + \lv T_{2,k}\rv \\

& \leq c_1\lppp \log \frac{2}{\delta_1}  \rppp^{\frac{1}{2}} n^{-\frac{\alpha}{2}} \exp(CC_s k n^{-\alpha}) + \frac{B_0}{C_s} \exp\lpp CC_s n^{-\alpha}k \rpp\\
& \leq \lpp c_1\lppp \log \frac{2}{\delta_1}  \rppp^{\frac{1}{2}} n^{-\frac{\alpha}{2}} + \frac{B_0}{C_s} \rpp \exp(CC_s T n^{-\alpha})\\
& = \lpp c_1\lppp\log \frac{2}{\delta_1}  \rppp^{\frac{1}{2}} n^{-\frac{\alpha}{2}} + \frac{1}{C_s} \tiln^{-\frac{\alpha}{2}} \tiln^{\beta_B \alpha} \rpp \exp(CC_s \beta_T \alpha (\tiln/n)^{\alpha}(\log \tiln ) )\\

& \leq \frac{2}{C_s} \tiln^{-\frac{\alpha}{2}} \tiln^{\beta_B \alpha} \tiln^{CC_s \beta_T \alpha} \leq B_1,\\
\end{array}
$$
where the second to the last step is based on (\ref{eq:simple3}). Hence,
\begin{equation}
U_{n,k} \subseteq A_{n,k,0}.
    \label{eq:saddle10}
\end{equation}

\twoline{To control $\lvvv \nablatilde{n+k}\rvvv$}, we can use the characterization given in (\ref{eq:saddle3}). On $U_{n,k}$,
\begin{equation}
\resizebox{.92\hsize}{!}{$
\arraycolsep=1.1pt\def\arraystretch{1.5}
\begin{array}{rl}
\lvvv \nablatilde{n+k}\rvvv & \leq \lvv \lpp \prod\limits_{j=1}\limits^k (I - \gamma_{n+j} \hcal_n) \rpp \nabla F(\bdtheta_n) \rvv +  \lv \hcal_n \rv \lvv \sum\limits_{j=1}\limits^k \lpp \gamma_{n+j} \prod\limits_{i=j+1}\limits^k(I - \gamma_{n+i}\hcal_n) \rpp \bdxi_{n+j} \rvv\\

& \leq B_0 \lvv \prod\limits_{j=1}\limits^k (I - \gamma_{n+j} \hcal_n) \rvv + C_s \lv T_{1,k} \rv\\

& \leq B_0 \pprod{j=1}{k}(1+C_s\gamma_{n+k}) + C_s \lv T_{1,k} \rv\\

& \leq B_0 \exp\lpp CC_s \ssum{j=1}{k}(n+j)^{-\alpha}\rpp +C_s \lv T_{1,k} \rv\\

& \leq B_0 \exp(CC_s k n^{-\alpha}) +C_s \lv T_{1,k} \rv\\

& \leq B_0 \exp(CC_s T \tiln^{-\alpha}) + C_s c_1 \lpp \log \frac{2}{\delta_1} \rpp^{\frac{1}{2}} \tiln^{-\frac{\alpha}{2}} \exp(CC_sT \tiln^{-\alpha})\\

& = \lpp \tiln^{\beta_B\alpha} + C_s c_1 \lpp \log \frac{2}{\delta_1} \rpp^{\frac{1}{2}} \rpp \tiln^{-\frac{\alpha}{2}} \exp(CC_sT \tiln^{-\alpha})\\

& \leq 2\tiln^{(-\frac{1}{2} + \beta_B + CC_s\beta_T)\alpha}\leq B_1,
\end{array}
$}
\label{eq:saddle11}
\end{equation}
where the 2nd step is based on condition \ref{cm_1}. The results in (\ref{eq:saddle11}) imply that
\begin{equation}
U_{n,k} \subseteq A_{n,k,1}.
    \label{eq:saddle12}
\end{equation}

\twoline{To bound $\lvvv \nabla F(  \bdtheta_{n+k}) - \nablatilde{n+k}\rvvv$}, we introduce 2 additional notations to save some space. We denote
$$
\begin{array}{cc}
\hcal_{n,j} \triangleq  \nabla^2 F(\bdtheta_{n+j}), & \deltah{n}{j} \triangleq \hcal_{n,j} - \hcal_{n,0}.
\end{array}
$$
Under condition \ref{cm_1}, on $A_{n,k-1,0}\cap A_{n,k-1,3}$, for $j\leq k-1$, we have
\begin{equation}
\lv \deltah{n}{j}\rv \leq C_{hl} \lv \bdtheta_{n+j} - \bdtheta_n\rv \leq C_{hl}(\| \tildetheta{n+j} - \bdtheta_{n+j}\| + \| \tildetheta{n+j} - \bdtheta_n \|) \leq C_{hl}(B_1 + B_3) \leq 2C_{hl}B_1.
    \label{eq:saddle13}
\end{equation}
We also let
$$
R_{n,j} = \int_0^1 \lpp\nabla^2 F(\bdtheta_{n+j-1} + u(\bdtheta_{n+j} - \bdtheta_{n+j-1})) - \nabla^2 F(\bdtheta_{n+j-1}) \rpp du(\bdtheta_{n+j} - \bdtheta_{n+j-1}).
$$
Again, under condition \ref{cm_1}, we have
\begin{equation}
\arraycolsep=1.1pt\def\arraystretch{1.5}
\begin{array}{rl}
\|R_{n,j}\| &\leq \frac{C_{hl}}{2} \| \bdtheta_{n+j} - \bdtheta_{n+j-1}\|^2\\

&= \frac{C_{hl}}{2}\gamma_{n+j} \| \nabla F(\bdtheta_{n+j-1}) + \bdxi_{n+j}\|^2\\

& \leq C_{hl}\gamma_{n+j}^2 (\| \nablaf{n+j-1}\|^2 + \| \bdxi_{n+j} \|^2)\\

& \leq C_{hl}\gamma_{n+j}^2 (2\| \nablaf{n+j-1} - \nablatilde{n+j-1}\|^2 + 2\|\nablatilde{n+j-1}\|^2 + \|\bdxi_{n+j}\|^2)\\

& \leq 2C_{hl}\gamma_{n+j}^2 (\| \nablaf{n+j-1} - \nablatilde{n+j-1}\|^2 + \|\nablatilde{n+j-1}\|^2)\\

& \ \ + C_{hl}\csg \gamma_{n+j}^2 (1+ \|\nablaf{n+j-1}\|^{\beta_{sg}})),\\
\end{array}
    \label{eq:saddleadd2}
\end{equation}
where the last step is based on condition \ref{cm_2}. Then, based on (\ref{eq:saddleadd2}), (\ref{eq:saddlen1}) and the definition of $S_{n,j-1}$, we have
\begin{equation}
\| R_{n,j}\| \yis{j-1} \leq 2C_{hl}\gamma_{n+j}^2 (B_2^2 + B_1^2) + (1+2^{\beta_{sg}})C_{hl}\csg \gamma_{n+j}^2  \leq 2 (1+2^{\beta_{sg}})C_{hl}\csg \gamma_{n+j}^2,
    \label{eq:saddleadd3}
\end{equation}
where the last step is based on (\ref{eq:simple1}) and (\ref{eq:simple4}). With the definition of $R_{n,j}$, we have
\begin{equation}
\arraycolsep=1.1pt\def\arraystretch{1.5}
\begin{array}{rl}
\nabla F(\bdtheta_{n+j}) & = \nabla F(\bdtheta_{n+j-1}) + \int_0^1 \nabla^2 F(\bdtheta_{n+j-1} + u(\bdtheta_{n+j} - \bdtheta_{n+j-1})) du (\bdtheta_{n+j} - \bdtheta_{n+j-1})\\

& = \nabla F(\bdtheta_{n+j-1})  + \hcal_{n,j-1} (\bdtheta_{n+j} - \bdtheta_{n+j-1}) + R_{n,j}\\

& = \nabla F(\bdtheta_{n+j-1})  + (\hcal_{n} + \deltah{n}{j-1})(-\gamma_{n+j}(\nabla F(\bdtheta_{n+j-1})+\bdxi_{n+j}))+R_{n,j}\\

& = (I-\gamma_{n+j}\hcal_{n})\nabla F(\bdtheta_{n+j-1}) - \gamma_{n+j}\hcal_{n}\bdxi_{n+j} \\

&\ \ \ - \gamma_{n+j}\deltah{n}{j-1}(\nabla F(\bdtheta_{n+j-1}) + \bdxi_{n+j}) + R_{n,j}.\\
\end{array}
    \label{eq:saddle14}
\end{equation}
Since
$$
\nablatilde{n+j} = (I - \gamma_{n+j}\hcal_n)\nablatilde{n+j-1} - \gamma_{n+j} \hcal_n \bdxi_{n+j},
$$
we have
\begin{equation}
\arraycolsep=1.1pt\def\arraystretch{1.5}
\begin{array}{rl}
\deltad{n+j} & \triangleq \nablaf{n+j} -  \nablatilde{n+j}\\

& = (I-\gamma_{n+j} \hcal_n)\deltad{n+j-1} - \gamma_{n+j} \deltah{n}{j-1}(\nablaf{n+j-1} + \bdxi_{n+j}) + R_{n,j}\\

& = (I-\gamma_{n+j} \hcal_n)\deltad{n+j-1} - \gamma_{n+j}\deltah{n}{j-1}\deltad{n+j-1}\\

& \ \ \ -\gamma_{n+j}\deltah{n}{j-1}\nablatilde{n+j-1} -\gamma_{n+j}\deltah{n}{j-1}\bdxi_{n+j} + R_{n,j}.\\
\end{array}
    \label{eq:saddle15}
\end{equation}
Recall that our third goal is to show that with a large probability,
$$
\| \deltad{n+k}\| \yis{k-1} \leq B_2.
$$
To achieve this goal, we firstly derive a bound on $\EE_{n+j-1} (\| \deltad{n+j}\|^2 \yis{j-1})$ for $j\leq k$. With the results in (\ref{eq:saddle15}), we have
\begin{equation}
\resizebox{.92\hsize}{!}{$
\arraycolsep=1.1pt\def\arraystretch{1.5} 
\begin{array}{rl}
& \EE_{n+j-1} \| \deltad{n+j}\|^2 \\

=&  \deltad{n+j-1}^T(I-\gamma_{n+j}\hcal_n)^2 \deltad{n+j-1} \textcolor{red}{+} \gamma_{n+j}^2\deltad{n+j-1}^T (\deltah{n}{j-1})^2 \deltad{n+j-1} \\

& \textcolor{red}{+}\gamma_{n+j}^2
 (\nablatilde{n+j-1})^T (\deltah{n}{j-1})^2 \nablatilde{n+j-1} \textcolor{red}{+} \gamma_{n+j}^2 \EE_{n+j-1} (\bdxi_{n+j}^T (\deltah{n}{j-1})^2 \bdxi_{n+j})\textcolor{red}{+} \EE_{n+j-1} \| R_{n,j}\|^2\\

 & \textcolor{red}{-} 2\gamma_{n+j} \deltad{n+j-1}^T (I-\gamma_{n+j}\hcal_n) \deltah{n}{j-1}\deltad{n+j-1} \textcolor{red}{+} 2\gamma_{n+j}^2 \deltad{n+j-1}^T (\deltah{n}{j-1})^2 \nablatilde{n+j-1}\\

 & \textcolor{red}{-} 2\gamma_{n+j} \deltad{n+j-1}^T (I-\gamma_{n+j}\hcal_n) \deltah{n}{j-1} \nablatilde{n+j-1}\textcolor{red}{-} 2\gamma_{n+j} \EE_{n+j-1}(\bdxi_{n+j}^T \deltah{n}{j-1}R_{n,j})  \\

  &  \textcolor{red}{+} 2(\EE_{n+j-1} R_{n,j})^T((I-\gamma_{n+j}\hcal_n)\deltad{n+j-1} - \gamma_{n+j}\deltah{n}{j-1}\deltad{n+j-1} - \gamma_{n+j}\deltah{n}{j-1}\nablatilde{n+j-1}) \\

 \triangleq & T_1 + T_2+T_3+T_4+T_5 - T_{12} +T_{23} - T_{13} - T_{45} + T_{5}'.\\
\end{array}
$}
    \label{eq:saddle16}
\end{equation}
On $S_{n,j-1}$, each term on the right-hand side of (\ref{eq:saddle16}) can be controlled as follows.
{                                  
\setlength{\leftmargini}{0.5cm}      
\begin{itemize}
\Item 
\begin{equation}
\arraycolsep=1.1pt\def\arraystretch{1.5}
\begin{array}{rl}
T_1\yis{j-1} & = \deltad{n+j-1}^T(I-\gamma_{n+j}\hcal_n)^2 \deltad{n+j-1}\yis{j-1}\\
& \leq (1+C_s \gamma_{n+j})^2 \| \deltad{n+j-1}\|^2 \yis{j-1}.\\
\end{array}
\label{eq:saddle17}
\end{equation}

\Item \begin{equation}
\arraycolsep=1.1pt\def\arraystretch{1.5}
\begin{array}{rl}
T_2\yis{j-1} & = \gamma_{n+j}^2\deltad{n+j-1}^T (\deltah{n}{j-1})^2 \deltad{n+j-1}\yis{j-1}\\
& \leq \gamma_{n+j}^2 (2C_{hl}B_1)^2 \| \deltad{n+j-1}\|^2 \yis{j-1}\\
& \leq 4C_{hl}^2 B_1^2 B_2^2 \gamma_{n+j}^2 \yis{j-1}\propto \gamma_{n+j}^2 B_1^2 B_2^2 \yis{j-1},\\
\end{array}
\label{eq:saddle18}
\end{equation}
were the 2nd step is based on (\ref{eq:saddle13}).

\Item
\begin{equation}
\arraycolsep=1.1pt\def\arraystretch{1.5}
\begin{array}{rl}
T_3\yis{j-1} & = \gamma_{n+j}^2
 (\nablatilde{n+j-1})^T (\deltah{n}{j-1})^2 \nablatilde{n+j-1}\yis{j-1}\\

& \leq \gamma_{n+j}^2 (2C_{hl}B_1)^2 \| \nablatilde{n+j-1}\|^2 \yis{j-1}\\

& \leq 4C_{hl}^2 B_1^4 \gamma_{n+j}^2 \yis{j-1} \propto \gamma_{n+j}^2 B_1^4 \yis{j-1},\\
\end{array}
    \label{eq:saddle19}
\end{equation}
where the 2nd step is based on (\ref{eq:saddle13}).

\Item 
\begin{equation}
\arraycolsep=1.1pt\def\arraystretch{1.5}
\begin{array}{rl}
T_4\yis{j-1} & = \gamma_{n+j}^2 \EE_{n+j-1} (\bdxi_{n+j}^T (\deltah{n}{j-1})^2 \bdxi_{n+j})\yis{j-1}\\

& \leq \gamma_{n+j}^2 (2C_{hl}B_1)^2 \EE_{n+j-1}(\| \bdxi_{n+j}\|^2 \yis{j-1})\\

& \leq 16 C_{hl}^2 \csg B_1^2 \gamma_{n+j}^2\lppp 1+\| \nablaf{n+j-1}\|^{\beta_{sg}}\rppp \yis{j-1}\\

& \leq 16\lppp 1+ 2^{\beta_{sg}}\rppp C_{hl}^2 \csg B_1^2 \gamma_{n+j}^2 \yis{j-1} \propto \gamma_{n+j}^2 B_1^2 \yis{j-1},\\
\end{array}
    \label{eq:saddle20}
\end{equation}
where the 2nd step is based on (\ref{eq:saddle13}), the 3rd step is based on Lemma \ref{lemma:errorbound} and the 4th step is based on (\ref{eq:saddlen1}).

\Item 
\begin{equation}
\resizebox{.86\hsize}{!}{$
\arraycolsep=1.1pt\def\arraystretch{1.5}
\begin{array}{rl}
T_5\yis{j-1} & = \EE_{n+j-1}\|R_{n,j}\|^2\yis{j-1}\leq \lppp 2(1+2^{\beta_{sg}}) C_{hl}\csg \gamma_{n+j}^2 \rppp^2\yis{j-1} \propto \gamma_{n+j}^4 \yis{j-1},\\
\end{array}
    \label{eq:saddle21}
$}
\end{equation}
where the 2nd step is based on (\ref{eq:saddleadd3}).

\Item 
\begin{equation}
\arraycolsep=1.1pt\def\arraystretch{1.5}
\begin{array}{rl}
T_{12}\yis{j-1} & = 2\gamma_{n+j} \deltad{n+j-1}^T (I-\gamma_{n+j}\hcal_n) \deltah{n}{j-1}\deltad{n+j-1}\yis{j-1}\\

& \leq 2\gamma_{n+j}(1+C_s\gamma_{n+j}) \| \deltad{n+j-1}\|^2 (2C_{hl}B_1)\yis{j-1}\\

& \leq 4C_{hl}(1+C_s)\gamma_{n+j}B_1B_2^2\yis{j-1} \propto \gamma_{n+j}B_1B_2^2 \yis{j-1},\\
\end{array}
    \label{eq:saddle22}
\end{equation}
where the 2nd step is based on (\ref{eq:saddle13}) and condition \ref{cm_1}.

\Item 
\begin{equation}
\arraycolsep=1.1pt\def\arraystretch{1.5}
\begin{array}{rl}
T_{23}\yis{j-1} & = 2\gamma_{n+j}^2 \deltad{n+j-1}^T (\deltah{n}{j-1})^2 \nablatilde{n+j-1}\yis{j-1}\\

& \leq 2\gamma_{n+j}^2(2C_{hl}B_1)^2\| \deltad{n+j-1}\| \| \nablatilde{n+j-1}\|\yis{j-1}\\

& \leq 8C_{hl}^2\gamma_{n+j}^2B_1^3B_2\yis{j-1}\propto \gamma_{n+j}^2B_1^3B_2 \yis{j-1},\\
\end{array}
    \label{eq:saddle23}
\end{equation}
where the 2nd step is based on (\ref{eq:saddle13}).

\Item 
\begin{equation}
\arraycolsep=1.1pt\def\arraystretch{1.5}
\begin{array}{rl}
T_{13}\yis{j-1} & = 2\gamma_{n+j} \deltad{n+j-1}^T (I-\gamma_{n+j}\hcal_n) \deltah{n}{j-1} \nablatilde{n+j-1}\yis{j-1}\\

& \leq 2\gamma_{n+j}(1+C_s \gamma_{n+j})(2C_{hl}B_1)\| \deltad{n+j-1}\| \| \nablatilde{n+j-1}\|\yis{j-1}\\

& \leq 4C_{hl}(1+C_s)\gamma_{n+j} B_1^2 B_2\yis{j-1}  \propto \gamma_{n+j} B_1^2 B_2 \yis{j-1},\\
\end{array}
    \label{eq:saddle24}
\end{equation}
where the 2nd step is based on (\ref{eq:saddle13}) and condition \ref{cm_1}.

\Item  
\begin{equation}
\arraycolsep=1.1pt\def\arraystretch{1.5}
\begin{array}{rl}
T_{45}\yis{j-1} & = 2\gamma_{n+j} \EE_{n+j-1}(\bdxi_{n+j}^T \deltah{n}{j-1}R_{n,j})\yis{j-1}\\

& \leq 2\gamma_{n+j} \EE_{n+j-1}(\|\bdxi_{n+j}\|  \|\deltah{n}{j-1}\| \|R_{n,j}\|\yis{j-1})\\

& \leq 2\gamma_{n+j} \lppp \EE_{n+j-1} \| \bdxi_{n+j}\|^2\rppp^{\frac{1}{2}} (2C_{hl}B_1) \lppp 2(1+2^{\beta_{sg}})C_{hl}\csg \gamma_{n+j}^2 \rppp \yis{j-1}\\

& \leq 8 (1+2^{\beta_{sg}}) C_{hl}^2 \csg B_1 \gamma_{n+j}^3 \lppp 4\csg (1+\| \nablaf{n+j-1}\|^{\beta_{sg}}\rppp^{\frac{1}{2}} \yis{j-1}\\ 

& \leq 16 (1+2^{\beta_{sg}})^{\frac{3}{2}} C_{hl}^2 (\csg)^{\frac{3}{2}} B_1 \gamma_{n+j}^3  \yis{j-1} \propto \gamma_{n+j}^3 B_1 \yis{j-1},\\
\end{array}
    \label{eq:saddle25}
\end{equation}
where the 3rd step is based on (\ref{eq:saddle13}) and (\ref{eq:saddleadd3}), the 4th step is based on condition \ref{cm_2} and the 5th step is based on (\ref{eq:saddlen1}).
\Item  
\begin{equation}
\resizebox{.87\hsize}{!}{$
\arraycolsep=1.1pt\def\arraystretch{1.5}
\begin{array}{rl}
T_{5}'\yis{j-1} & = 2(\EE_{n+j-1} R_{n,j})^T((I-\gamma_{n+j}\hcal_n)\deltad{n+j-1} - \gamma_{n+j}\deltah{n}{j-1}\deltad{n+j-1}\\

&\ \ \ \ \ \ \ \ \ \ \ \ \ \ \ \ \ \ \ \ \ \ \ \ \ \ - \gamma_{n+j}\deltah{n}{j-1}\nablatilde{n+j-1})\yis{j-1}\\

& \leq 2 \EE_{n+j-1} (\| R_{n,j}\|\yis{j-1}) ((1+C_s\gamma_{n+j})B_2 + \gamma_{n+j} (2C_{hl}B_1) (B_2+B_1))\\

& \leq 4 C_{hl}\csg \gamma_{n+j}^2 (1+2^{\beta_{sg}})((1+C_s\gamma_{n+j})B_2 + \gamma_{n+j} (2C_{hl}B_1) (B_2+B_1)) \yis{j-1}\\

& \propto \gamma_{n+j}^2 B_2 \yis{j-1},\\
\end{array}
$}
    \label{eq:saddle26}
\end{equation}
where the 2nd step is based on (\ref{eq:saddle13}) and condition \ref{cm_1}, and the 3rd step is based on (\ref{eq:saddleadd3}).
\end{itemize}
}  

Based on the results given in (\ref{eq:saddle17})-(\ref{eq:saddle26}), we can see that
$$
\resizebox{.98\hsize}{!}{$
\arraycolsep=1.1pt\def\arraystretch{1.5} 
\begin{array}{rl}
& \EE_{n+j-1}(\| \deltad{n+j}\|^2 \yis{j-1})\\

\leq &(1+C_s \gamma_{n+j})^2 \| \deltad{n+j-1}\|^2 \yis{j-1}\\

& \ + O(\gamma_{n+j}^2 B_1^2 B_2^2 + \gamma_{n+j}^2 B_1^4 + \gamma_{n+j}^2B_1^2 + \gamma_{n+j}^4 + \gamma_{n+j} B_1 B_2^2 + \gamma_{n+j}^2 B_1^3 B_2 + \gamma_{n+j} B_1^2 B_2 + \gamma_{n+j}^3 B_1 + \gamma_{n+j}^2 B_2) \yis{j-1}\\

= & (1+C_s \gamma_{n+j})^2 \| \deltad{n+j-1}\|^2 \yis{j-1} + O(\gamma_{n+j}B_1^2 B_2)\yis{j-1}.\\
\end{array}
$}
$$
Actually, based on (\ref{eq:simple8})-(\ref{eq:simple15}), $4C_{hl}(1+C_s)\gamma_{n+j} B_1^2 B_2\yis{j-1}$ in (\ref{eq:saddle24}) is the dominant term and we have
\begin{equation*}
\EE_{n+j-1}(\| \deltad{n+j}\|^2 \yis{j-1}) \leq ((1+C_s \gamma_{n+j})^2 \| \deltad{n+j-1}\|^2 + 40 C_{hl}(1+C_s)\gamma_{n+j} B_1^2 B_2)\yis{j-1}.\\
    \label{eq:saddle27}
\end{equation*}
Therefore, if we let
$$
c_2 \triangleq  \frac{20C_{hl}(1+C_s)}{C_s},
$$
we have
\begin{equation}
\arraycolsep=1.1pt\def\arraystretch{1.5}
\begin{array}{rl}
\EE_{n+j-1}(\| \deltad{n+j}\|^2 +c_2B_1^2B_2)\yis{j-1} & \leq (1+C_s \gamma_{n+j})^2 (\| \deltad{n+j-1}\|^2 + c_2 B_1^2 B_2)\yis{j-1}\\

& \leq (1+C_s \gamma_{n+j})^2 (\| \deltad{n+j-1}\|^2 + c_2 B_1^2 B_2)\yis{j-2}.\\
\end{array}
    \label{eq:saddle28}
\end{equation}
Now, we let
\begin{equation*}
G_{n,j} \triangleq \pprod{i=1}{j}(1+C_s\gamma_{n+i})^{-2} (\| \deltad{n+j}\|^2 + c_2 B_1^2 B_2)\yis{j-1}.
    \label{eq:saddle29}
\end{equation*}
Based on (\ref{eq:saddle28}), we have
\begin{equation}
\EE_{n+j-1}G_{n,j} \leq G_{n,j-1},
    \label{eq:saddle30}
\end{equation}
which implies that $\{G_{n,j}:j\ge 0\}$ is a supermartingale with respect to $\{\fff_{n+j}:j\ge 0\}$. With this property at hand, we have
\begin{equation}
\arraycolsep=1.1pt\def\arraystretch{1.5}
\begin{array}{rl}
G_{n,k} - G_{n,0} & = \ssum{j=1}{k}(G_{n,j} - \EE_{n+j-1}G_{n,j}) + \ssum{j=1}{k}(\EE_{n+j-1}G_{n,j} - G_{n,j-1})\\

& \leq \ssum{j=1}{k}(G_{n,j} - \EE_{n+j-1}G_{n,j}),\\
\end{array}
\label{eq:saddle31}
\end{equation}
where the right-hand side is a martingale difference sum. Actually, if we denote
$$
L_{n,j} \triangleq (I-\gamma_{n+j+1} \hcal_n)\deltad{n+j} - \gamma_{n+j+1}\deltah{n}{j}(\deltad{n+j} + \nablatilde{n+j}),
$$
based on (\ref{eq:saddle15}), we have
\begin{equation}
\resizebox{.92\hsize}{!}{$
\arraycolsep=1.1pt\def\arraystretch{1.5} 
\begin{array}{rl}
& \ssum{j=1}{k} G_{n,j} - \EE_{n+j-1} G_{n,j}\\

= & \ssum{j=1}{k} \pprod{i=1}{j}(1+C_s\gamma_{n+i})^{-2} \lppp \| \deltad{n+j}\|^2 - \EE_{n+j-1} \| \deltad{n+j}\|^2 \rppp\yis{j-1}\\

= & \ssum{j=1}{k}\pprod{i=1}{j}(1+C_s\gamma_{n+i})^{-2} \lpp \lppp \gamma_{n+j}^2 \bdxi_{n+j}^T (\deltah{n}{j-1})^2 \bdxi_{n+j} - \EE_{n+j-1}( \gamma_{n+j}^2 \bdxi_{n+j}^T (\deltah{n}{j-1})^2 \bdxi_{n+j})\rppp\\

& \ \ \ \ \ \ \ \ \ \ \ \ \ \ \ \ \ \ \ \ \ \ \ \ + \big( \|R_{n,j}\|^2 - \EE_{n+j-1}\|R_{n,j}\|^2\big) + \big( 2R_{n,j}^T L_{n,j-1} - \EE_{n+j-1}( 2R_{n,j}^T L_{n,j-1})\big)\\

& \ \ \ \ \ \ \ \ \ \ \ \ \ \ \ \ \ \ \ \ \ \ \ \ + \big(- 2\gamma_{n+j} \bdxi_{n+j}^T \deltah{n}{j-1}R_{n,j} + \EE_{n+j-1}(2\gamma_{n+j} \bdxi_{n+j}^T \deltah{n}{j-1}R_{n,j})\big)\\

& \ \ \ \ \ \ \ \ \ \ \ \ \ \ \ \ \ \ \ \ \ \ \ \ + \big(-2\gamma_{n+j}\bdxi_{n+j}^T \deltah{n}{j-1}L_{n,j-1} \big)\rpp\yis{j-1}\\

\triangleq & \ssum{j=1}{k} \pprod{i=1}{j}(1+C_s\gamma_{n+i})^{-2} |\tilde{T}_1 + \tilde{T}_2 + \tilde{T}_3 + \tilde{T}_4 + \tilde{T}_5| \yis{j-1}.\\
\end{array}
$}
    \label{eq:saddle32}
\end{equation}
We have the following bounds,
{                                  
\setlength{\leftmargini}{0.4cm}    
\begin{itemize}
\item \ \vspace{-0.8cm}
\begin{equation}
\resizebox{.88\hsize}{!}{$
\arraycolsep=1.1pt\def\arraystretch{1.5} 
\begin{array}{rl}
\| L_{n,j-1}\| \yis{j-1} & \leq \big((1+C_s \gamma_{n+j})\| \deltad{n+j-1} \| + \gamma_{n+j} \| \deltah{n}{j-1}\| (\| \deltad{n+j-1}\| + \| \nablatilde{n+j-1}\|)\big) \yis{j-1}\\

& \leq (1+C_s ) B_2 + \gamma_{n+j} (2C_{hl}B_1) (B_2+B_1)\leq (1+C_s ) B_2 + 4C_{hl}\gamma_{n+j}B_1^2 \leq 2(1+C_s)B_2,\\
\end{array}
$}
\label{eq:saddleL}
\end{equation}
where the 1st step is based on condition \ref{cm_1}, the 2nd step is based on (\ref{eq:saddle13}) and the definition of $S_{n,j-1}$ and the last step is based on (\ref{eq:simple5}).
\item For any $t\in \fff_{n+j-1}$, we have
\begin{equation}
\arraycolsep=1.1pt\def\arraystretch{1.5} 
\begin{array}{rl}
& \PP_{n+j-1} \lppp \tilde{T}_1 \yis{j-1}\ge t\rppp\\

\leq & \PP_{n+j-1} \lppp \gamma_{n+j}^2 \bdxi_{n+j}^T (\deltah{n}{j-1})^2 \bdxi_{n+j} \yis{j-1}\ge t\rppp\\

\leq & \PP_{n+j-1} \lppp \gamma_{n+j}^2 \|\bdxi_{n+j}\|^2 \|\deltah{n}{j-1}\|^2  \yis{j-1}\ge t\rppp\\

\leq & \PP_{n+j-1} \lppp (2C_{hl} B_1)^2  \gamma_{n+j}^2 \| \bdxi_{n+j}\|^2 \yis{j-1}\ge t\rppp\\

\leq & 2 \exp \lpp \frac{-t}{8C_{hl}^2 \csg B_1^2 \gamma_{n+j}^2 \lppp 1+ \|\nablaf{n+j-1}\|^{\beta_{sg}}\rppp} \rpp \yis{j-1}\\

\leq & 2 \exp \lpp \frac{-t}{8\lppp 1+ 2^{\beta_{sg}}\rppp C_{hl}^2 \csg B_1^2 \gamma_{n+j}^2 } \rpp, \\
\end{array}
\label{eq:saddlet11}
\end{equation}
where the 3rd step is based on (\ref{eq:saddle13}), the 4th step is based on condition \ref{cm_2} and the last step is based on (\ref{eq:saddlen1}). Therefore, based on (\ref{eq:saddlet11}), we have
\begin{equation*}
\PP_n \lpp \tilde{T}_1 \yis{j-1} \ge 8\lppp 1+ 2^{\beta_{sg}}\rppp C_{hl}^2 \csg B_1^2 \gamma_{n+j}^2 \log \lpp \frac{8T}{\delta_2} \rpp \rpp \leq \frac{\delta_2}{4T}.
    \label{eq:saddlet12}
\end{equation*}
Furthermore, if we let
$$
c_3 \triangleq 8\lppp 1+ 2^{\beta_{sg}}\rppp C_{hl}^2 \csg,
$$
we have
\begin{equation}
\resizebox{.88\hsize}{!}{$
\PP_n \lpp \ssum{j=1}{k} \pprod{i=1}{j}(1+C_s\gamma_{n+i})^{-2} \tilde{T}_1 \yis{j-1} \leq c_3 B_1^2 \log \lpp \frac{8T}{\delta_2} \rpp \ssum{j=1}{k} \lpp \gamma_{n+j}^2 \pprod{i=1}{j}(1+C_s\gamma_{n+i})^{-2} \rpp \rpp \ge 1-\frac{\delta_2}{4}.
    \label{eq:saddlet13}
$}
\end{equation}
\item Based on the bound given in (\ref{eq:saddleadd3}), we have
$$
\tilde{T}_2 \yis{j-1} \leq 2\lppp 2(1+2^{\beta_{sg}}) C_{hl}\csg \gamma_{n+j}^2 \rppp^2 \leq 8(1+2^{\beta_{sg}})(1+C_s) C_{hl}\csg B_2 \gamma_{n+j}^2,
$$
where the last step is based on (\ref{eq:simple6}). Based on (\ref{eq:saddleadd3}) and (\ref{eq:saddleL}), we have
$$
\tilde{T}_3 \yis{j-1} \leq 8(1+2^{\beta_{sg}})(1+C_s) C_{hl}\csg B_2 \gamma_{n+j}^2.
$$
Therefore, if we let
$$
c_4 \triangleq 16 (1+2^{\beta_{sg}})(1+C_s) C_{hl}\csg,
$$
we have
$$
\pprod{i=1}{j}(1+C_s\gamma_{n+i})^{-2} (\tilde{T}_2 + \tilde{T}_3)\yis{j-1} \leq c_4 B_2 \gamma_{n+j}^2 \pprod{i=1}{j}(1+C_s\gamma_{n+i})^{-2}.
$$
By applying Lemma \ref{lemma:afmartingale}, we have
\begin{equation}
\resizebox{.88\hsize}{!}{$
\PP_n \lpp \ssum{j=1}{k} \pprod{i=1}{j}(1+C_s\gamma_{n+i})^{-2} (\tilde{T}_2 + \tilde{T}_3)\yis{j-1} \leq c_4 B_2 \lpp \log \frac{8}{\delta_2} \rpp^{\frac{1}{2}} \lpp \ssum{j=1}{k}\gamma_{n+j}^4 \pprod{i=1}{j}(1+C_s\gamma_{n+i})^{-4}\rpp^{\frac{1}{2}} \rpp \ge 1-\frac{\delta_2}{4}.
    \label{eq:saddlet21}
$}
\end{equation}
\item To control the part containing $\tilde{T}_4$, we firstly have
\begin{equation}
\arraycolsep=1.1pt\def\arraystretch{1.5} 
\begin{array}{rl}
& \EE_{n+j-1} \lppp 2\gamma_{n+j} \bdxi_{n+j}^T \deltah{n}{j-1}R_{n,j} \rppp \yis{j-1}\\

\leq & 2\gamma_{n+j} \EE_{n+j-1} \| \bdxi_{n+j}\| \| \deltah{n}{j-1} \| \| R_{n,j}\| \yis{j-1}\\

\leq & 2\gamma_{n+j} \lppp \EE_{n+j-1} \| \bdxi_{n+j}\|^2 \rppp^{\frac{1}{2}} (2C_{hl}B_1) \lppp 2(1+2^{\beta_{sg}})C_{hl}\csg \gamma_{n+j}^2 \rppp  \yis{j-1}\\

\leq & 16 C_{hl}^2 (\csg)^{\frac{3}{2}} (1+2^{\beta_{sg}})^{\frac{3}{2}} B_1 \gamma_{n+j}^3,\\
\end{array}
    \label{eq:saddlet41}
\end{equation}
where the 2nd step is based on (\ref{eq:saddle13}) and (\ref{eq:saddleadd3}), and the 3rd step is based on the Lemma \ref{lemma:errorbound} and (\ref{eq:saddlen1}). In addition, for any $t\in \fff_{n+j-1}$, we have
\begin{equation}
\arraycolsep=1.1pt\def\arraystretch{1.5} 
\begin{array}{rl}
& \PP_{n+j-1} \lppp -2\gamma_{n+j} \bdxi_{n+j}^T \deltah{n}{j-1}R_{n,j} \yis{j-1} \ge t\rppp\\

\leq & \PP_{n+j-1} \lppp 2\gamma_{n+j}\| \bdxi_{n+j}\| \| \deltah{n}{j-1} \| \| R_{n,j}\| \yis{j-1} \ge t \rppp \\

\leq & \PP_{n+j-1} \lppp 8(1+2^{\beta_{sg}}) C_{hl}^2 \csg B_1 \gamma_{n+j}^3 \| \bdxi_{n+j}\| \ge t \rppp \yis{j-1}\\

\leq & 2\exp \lpp -\frac{t^2}{\lppp 8(1+2^{\beta_{sg}}) C_{hl}^2 \csg B_1 \gamma_{n+j}^3 \rppp^2 2\csg (1+\|\nablaf{n+j-1}\|^{\beta_{sg}})} \rpp \yis{j-1}\\

\leq & 2\exp \lpp -\frac{t^2}{\lppp 8(1+2^{\beta_{sg}}) C_{hl}^2 \csg B_1 \gamma_{n+j}^3 \rppp^2 2\csg (1+2^{\beta_{sg}})} \rpp \yis{j-1},\\
\end{array}
    \label{eq:saddlet42}
\end{equation}
where the 2nd step is based on (\ref{eq:saddle13}) and (\ref{eq:saddleadd3}), the 3rd step is based on condition \ref{cm_2} and the last step is based on (\ref{eq:saddlen1}). Therefore, conditional on $\fff_n$, with probability at least $1-\frac{\delta_2}{4T}$,
\begin{equation}
-2\gamma_{n+j} \bdxi_{n+j}^T \deltah{n}{j-1}R_{n,j} \yis{j-1} \leq 8\sqrt{2} (1+2^{\beta_{sg}})^{\frac{3}{2}} C_{hl}^2 (\csg)^{\frac{3}{2}} B_1 \gamma_{n+j}^3 \lpp \log \frac{8T}{\delta_2}\rpp^{\frac{1}{2}}.
\label{eq:saddlet43}
\end{equation}
Therefore, if we let
$$
c_5 \triangleq 16 (1+2^{\beta_{sg}})^{\frac{3}{2}}C_{hl}^2 (\csg)^{\frac{3}{2}},
$$
based on (\ref{eq:simple7}), (\ref{eq:saddlet41}) and(\ref{eq:saddlet43}), we have
$$
\PP_n \lpp \tilde{T}_4 \yis{j-1}  \leq c_5 B_1 \gamma_{n+j}^3  \lpp \log \frac{8T}{\delta_2}\rpp^{\frac{1}{2}} \rpp \ge 1-\frac{\delta_2}{4T}.
$$
Furthermore, we have
\begin{equation}
\resizebox{.88\hsize}{!}{$
\PP_n \lpp \ssum{j=1}{k}\pprod{i=1}{j}(1+C_s\gamma_{n+i})^{-2} \tilde{T}_4 \yis{j-1} \leq c_5 B_1 \lpp \log \frac{8T}{\delta_2}\rpp^{\frac{1}{2}} \ssum{j=1}{k} \lpp \gamma_{n+j}^3  \pprod{i=1}{j}(1+C_s\gamma_{n+i})^{-2} \rpp \rpp  \ge 1-\frac{\delta_2}{4}.
    \label{eq:saddlet44}
$}
\end{equation}
\item As 
$$
\tilde{T}_5=-2\gamma_{n+j}\bdxi_{n+j}^T (\deltah{n}{j-1}) L_{n,j-1},
$$
based on the bounds on $\|\deltah{n}{j-1}\|\yis{j-1}$ and $\|L_{n,j-1}\|\yis{j-1}$ given respectively in (\ref{eq:saddle13}) and (\ref{eq:saddleL}), we know that $\pprod{i=1}{j}(1+C_s\gamma_{n+i})^{-2} \tilde{T}_5 \yis{j-1}$ is sub-Gaussian with parameter
$$
\arraycolsep=1.1pt\def\arraystretch{1.5} 
\begin{array}{rl}
& 4 \pprod{i=1}{j}(1+C_s\gamma_{n+i})^{-4}\gamma_{n+2}^2 (2C_{hl}B_1)^2 (2(1+C_s)B_2)^2 \csg \lppp 1+ \| \nablaf{n+j-1}\|^{\beta_{sg}}\rppp \yis{j-1}\\

\leq & 64 (1+2^{\beta_{sg}})\csg C_{hl}^2 (1+C_s)^2 B_1^2 B_2^2 \gamma_{n+j}^2 \pprod{i=1}{j}(1+C_s\gamma_{n+i})^{-4}.
\end{array}
$$
We let
$$
c_6 \triangleq 8(1+2^{\beta_{sg}})^{\frac{1}{2}}  \cjin(\csg)^{\frac{1}{2}} C_{hl}(1+C_s).
$$
Then, based on Corollary \ref{cor:jin}, conditional on $\fff_n$, with probability at least $1-\frac{\delta_2}{4}$,
\begin{equation}
\arraycolsep=1.1pt\def\arraystretch{1.5} 
\begin{array}{rl}
& \ssum{j=1}{k}\pprod{i=1}{j}(1+C_s\gamma_{n+i})^{-2}\tilde{T}_5 \yis{j-1}\\

\leq & 8(1+2^{\beta_{sg}})^{\frac{1}{2}} \cjin (\csg)^{\frac{1}{2}} C_{hl}(1+C_s)B_1B_2 \lpp \log \frac{8}{\delta_2}\rpp^{\frac{1}{2}} \lpp \ssum{j=1}{k}\gamma_{n+j}^2 \pprod{i=1}{j}(1+C_s\gamma_{n+i})^{-4} \rpp^{\frac{1}{2}}\\
= & c_6 B_1B_2 \lpp \log \frac{8}{\delta_2}\rpp^{\frac{1}{2}} \lpp \ssum{j=1}{k}\gamma_{n+j}^2 \pprod{i=1}{j}(1+C_s\gamma_{n+i})^{-4} \rpp^{\frac{1}{2}}.\\
\end{array}
\label{eq:saddlet5}
\end{equation}
\end{itemize}
}

Based on (\ref{eq:saddlet13}), (\ref{eq:saddlet21}), (\ref{eq:saddlet44}) and (\ref{eq:saddlet5})
\begin{equation}
\arraycolsep=1.1pt\def\arraystretch{1.5} 
\begin{array}{rl}
1-\delta_2 & \leq \PP_n \lpp G_{n,k}-G_{n,0} \leq c_3 B_1^2 \log \lpp \frac{8T}{\delta_2} \rpp \ssum{j=1}{k} \lpp \gamma_{n+j}^2 \pprod{i=1}{j}(1+C_s\gamma_{n+i})^{-2} \rpp\\

& \ \ \ \ \ \ \ \ \ \ \ \ \ \ \ \ \ \ \ \ \ \ \ \ \ \ \ \ + c_4 B_2 \lpp \log \frac{8}{\delta_2} \rpp^{\frac{1}{2}} \lpp \ssum{j=1}{k}\gamma_{n+j}^4 \pprod{i=1}{j}(1+C_s\gamma_{n+i})^{-4}\rpp^{\frac{1}{2}}\\

& \ \ \ \ \ \ \ \ \ \ \ \ \ \ \ \ \ \ \ \ \ \ \ \ \ \ \ \ + c_5 B_1 \lpp \log \frac{8T}{\delta_2}\rpp^{\frac{1}{2}} \ssum{j=1}{k} \lpp \gamma_{n+j}^3  \pprod{i=1}{j}(1+C_s\gamma_{n+i})^{-2} \rpp\\

& \ \ \ \ \ \ \ \ \ \ \ \ \ \ \ \ \ \ \ \ \ \ \ \ \ \ \ \ + c_6 B_1B_2 \lpp \log \frac{8}{\delta_2}\rpp^{\frac{1}{2}} \lpp \ssum{j=1}{k}\gamma_{n+j}^2 \pprod{i=1}{j}(1+C_s\gamma_{n+i})^{-4} \rpp^{\frac{1}{2}}\rpp\\

& \leq \PP_n \lpp G_{n,k}-G_{n,0} \leq 2c_4 B_2 \lpp \log \frac{8}{\delta_2} \rpp^{\frac{1}{2}} \ssum{j=1}{k} \lpp \gamma_{n+j}^2 \pprod{i=1}{j}(1+C_s\gamma_{n+i})^{-2} \rpp\\

& \ \ \ \ \ \ \ \ \ \ \ \ \ \ \ \ \ \ \ \ \ \ \ \ \ \ \ \ + 2c_6 B_1B_2 \lpp \log \frac{8}{\delta_2}\rpp^{\frac{1}{2}} \lpp \ssum{j=1}{k}\gamma_{n+j}^2 \pprod{i=1}{j}(1+C_s\gamma_{n+i})^{-4} \rpp^{\frac{1}{2}}\rpp,\\
\end{array}
\label{eq:saddle39}
\end{equation}
where the 2nd step is based on (\ref{eq:simple16}) and (\ref{eq:simple17}). Therefore, conditional on $\fff_n$, with probability at least $1-\delta_2$,
\begin{equation}
\arraycolsep=1.1pt\def\arraystretch{1.5} 
\begin{array}{rl}
\|\deltad{n+k}\|^2 \leq \pprod{i=1}{k}(1+C_s\gamma_{n+i})^2 \lpp &  G_{n,0} + 2c_4 B_2 \lpp \log \frac{8}{\delta_2}\rpp^{\frac{1}{2}} \ssum{j=1}{k} \lpp \gamma_{n+j}^2 \pprod{i=1}{j}(1+C_s\gamma_{n+i})^{-2} \rpp\\

& + 2c_6 B_1B_2 \lpp \log \frac{8}{\delta_2}\rpp^{\frac{1}{2}} \lpp \ssum{j=1}{k}\gamma_{n+j}^2 \pprod{i=1}{j}(1+C_s\gamma_{n+i})^{-4} \rpp^{\frac{1}{2}}\rpp\\
    \label{eq:saddle40}
\end{array}
\end{equation}
or $S_{n,k-1}^c$ occurs. To simplify the result in (\ref{eq:saddle40}), we have
\begin{equation*}
\arraycolsep=1.1pt\def\arraystretch{1.5} 
\begin{array}{rll}
 \pprod{i=1}{k}(1+C_s\gamma_{n+i})^2G_{n,0}  & =c_2B_1^2B_2 \pprod{i=1}{k}(1+C_s\gamma_{n+i})^2 & \leq  c_2B_1^2 B_2\exp\lpp 2CC_s \ssum{i=1}{k}(n+i)^{-\alpha}\rpp\\

& \leq  c_2B_1^2 B_2\exp( 2CC_sT\tiln^{-\alpha}) & = c_2B_1^2B_2 \tiln^{2CC_s\beta_T\alpha},
\end{array}
\label{eq:saddle41}
\end{equation*}
\begin{equation*}
\arraycolsep=1.1pt\def\arraystretch{1.5} 
\begin{array}{rl}
& \pprod{i=1}{k}(1+C_s\gamma_{n+i})^2  \ssum{j=1}{k}\lpp \pprod{i=1}{j}(1+C_s\gamma_{n+i})^{-2}\rpp \gamma_{n+j}^2 \\

= &   \ssum{j=1}{k}\lpp \pprod{i=j+1}{k}(1+C_s\gamma_{n+i})^2\rpp \gamma_{n+j}^2 \\

\leq &  \ssum{j=1}{k} \exp\lpp 2CC_s \ssum{i=j+1}{k}(n+i)^{-\alpha}\rpp \gamma_{n+j}^2\\

\leq &   \ssum{j=1}{k} \exp\lpp \frac{2CC_s }{1-\alpha} ((n+k)^{1-\alpha}-(n+j)^{1-\alpha})\rpp \gamma_{n+j}^2 \\

= & C^2 \exp \lpp \frac{2CC_s}{1-\alpha} (n+k)^{1-\alpha}\rpp   \ssum{j=1}{k} \exp\lpp -\frac{2CC_s}{1-\alpha}(n+j)^{1-\alpha} \rpp (n+j)^{-2\alpha} \\

\leq & C^2 \exp \lpp \frac{2CC_s}{1-\alpha} (n+k)^{1-\alpha}\rpp  \frac{n^{-\alpha}}{2CC_s} \exp\lpp -\frac{2CC_s}{1-\alpha}n^{1-\alpha}\rpp \\

\leq &  \frac{C}{2C_s}  \tiln^{-{\alpha}} \exp(2CC_sT\tiln^{-\alpha})\\

= &  \frac{C}{2C_s}   \tiln^{-\alpha+2CC_s\beta_T\alpha},\\
\end{array}
\label{eq:saddle42}
\end{equation*}
where the 5th step is based on Lemma \ref{lemma:int2}, and
\begin{equation*}
\arraycolsep=1.1pt\def\arraystretch{1.5} 
\begin{array}{rl}
& \pprod{i=1}{k}(1+C_s\gamma_{n+i})^2 \lpp \ssum{j=1}{k}\lpp \pprod{i=1}{j}(1+C_s\gamma_{n+i})^{-4}\rpp \gamma_{n+j}^2 \rpp^{\frac{1}{2}}\\

= &  \lpp \ssum{j=1}{k}\lpp \pprod{i=j+1}{k}(1+C_s\gamma_{n+i})^4\rpp \gamma_{n+j}^2 \rpp^{\frac{1}{2}}\\

\leq & \lpp \ssum{j=1}{k} \exp\lpp 4CC_s \ssum{i=j+1}{k}(n+i)^{-\alpha}\rpp \gamma_{n+j}^2 \rpp^{\frac{1}{2}}\\

\leq &  \lpp \ssum{j=1}{k} \exp\lpp \frac{4CC_s }{1-\alpha} ((n+k)^{1-\alpha}-(n+j)^{1-\alpha})\rpp \gamma_{n+j}^2 \rpp^{\frac{1}{2}}\\

= & C \exp \lpp \frac{2CC_s}{1-\alpha} (n+k)^{1-\alpha}\rpp \lpp  \ssum{j=1}{k} \exp\lpp -\frac{4CC_s}{1-\alpha}(n+j)^{1-\alpha} \rpp (n+j)^{-2\alpha} \rpp^{\frac{1}{2}}\\

\leq & C \exp \lpp \frac{2CC_s}{1-\alpha} (n+k)^{1-\alpha}\rpp \lpp \frac{n^{-\alpha}}{4CC_s} \exp\lpp -\frac{4CC_s}{1-\alpha}n^{1-\alpha}\rpp \rpp^{\frac{1}{2}}\\

\leq &  \lpp \frac{C}{4C_s}\rpp^{\frac{1}{2}}  \tiln^{-\frac{\alpha}{2}} \exp(2CC_sT\tiln^{-\alpha})\\

= &  \lpp \frac{C}{4C_s}\rpp^{\frac{1}{2}}  \tiln^{-\frac{\alpha}{2}+2CC_s\beta_T\alpha},\\
\end{array}
\label{eq:saddle43}
\end{equation*}
where the 5th step is based on Lemma \ref{lemma:int2}. Therefore, (\ref{eq:saddle40}) implies
\begin{equation*}
\resizebox{.95\hsize}{!}{$
\PP_n\lpp\| \deltad{n+k} \|^2 \leq  B_1 \tiln^{2CC_s\beta_T\alpha}\lpp c_2B_1 B_2 + \frac{c_4 C B_2}{C_s B_1} \tiln^{-\alpha}\lpp \log \frac{8}{\delta_2}\rpp^{\frac{1}{2}} + c_6 \lpp \frac{C}{C_s}\rpp^{\frac{1}{2}} B_2 \tiln^{-\frac{\alpha}{2}} \lpp \log \frac{8}{\delta_2}\rpp^{\frac{1}{2}}\rpp\ or\ S_{n,k-1}^c \rpp \ge 1- \delta_2.
    \label{eq:saddle433}
$}
\end{equation*}
Based on (\ref{eq:simple18}) and (\ref{eq:simple19}), we have
$$
\PP_n\lpp\| \deltad{n+k} \|^2 \leq  2c_2B_1^2B_2 \tiln^{2CC_s\beta_T\alpha}\ or\ S_{n,k-1}^c \rpp \ge 1- \delta_2.
$$
Noticing that $B_2 = 2c_2B_1^2 \tiln^{2CC_s\beta_T\alpha}$, we have
\begin{equation}
\PP_n(A_{n,k,2} \cup S_{n,k-1}^c ) \ge 1-\delta_2.
    \label{eq:saddle44}
\end{equation}

At last, \twoline{we need to bound $\| \tildetheta{n+k} - \bdtheta_{n+k}\|$}. Based on the definition given in (\ref{eq:saddle2}), on $A_{n,k,2}$, we have

\begin{equation*}
\resizebox{.94\hsize}{!}{$
\| \tildetheta{n+k} - \bdtheta_{n+k}\| = \Big\| \ssum{j=1}{k}\gamma_{n+j} \deltad{n+j-1}\big\|\leq \ssum{j=1}{k}\gamma_{n+j} \| \deltad{n+j-1}\|\leq CB_2 \ssum{j=1}{k}(n+j)^{-\alpha} \leq CB_2T\tiln^{-\alpha} =B_3,
    \label{eq:saddle45}
$}
\end{equation*}
which, combined with (\ref{eq:saddle44}), implies that
\begin{equation}
\PP_n((A_{n,k,2} \cap A_{n,k,3}) \cup S_{n,k-1}^c ) \ge 1-\delta_2.
    \label{eq:saddle46}
\end{equation}

Now, based on (\ref{eq:saddle9}), (\ref{eq:saddle10}), (\ref{eq:saddle12}), (\ref{eq:saddle46}) and the assumption that $\PP_n(S_{n,k-1}) \ge 1-(k-1)(\delta_1+\delta_2)$, we have
$$
\PP_n(S_{n,k} ) \ge \PP_n(S_{n,k-1} ) - \PP_n(S_{n,k-1}\backslash S_{n,k} ) \ge 1-(k-1)(\delta_1+\delta_2) - (\delta_1+\delta_2) = 1-k(\delta_1+\delta_2).
$$

\end{proof}


\begin{lemma}
Suppose that conditions in Lemma \ref{lemma:run} hold. We define the following constants for convenience
$$
\arraycolsep=1.1pt\def\arraystretch{1.5} 
\begin{array}{ccl}
c_{sd,1} &\triangleq& 4(d-1)CC_s\csg,\\

c_{sd,2} &\triangleq & \frac{C\sigma_{\min}^2}{2e^23^{\alpha}},\\

c_{sd,3} &\triangleq& 512 \beta_{sg} C^2 (\csg)^2,\\

c_{sd,4} &\triangleq& 1280CC_sC_{hl} \lppp 1+\frac{1}{C_s}\rppp^4.\\
\end{array}
$$
In addition, we require that $3\beta_B + 5 CC_s\beta_T - C\tilde{\lambda}\beta_T < \frac{1}{2}$ and $\tiln \ge  \cnlemma{12}$,
where
$$
\arraycolsep=1.1pt\def\arraystretch{1.5}
\begin{array}{ccl}
\cnlemma{12}&\triangleq & \max \Big\{ \lpp \frac{4\beta_T\alpha^2}{1-\alpha}\rpp^{\frac{2}{1-\alpha}}, e^{\frac{1}{\beta_T \alpha}} , \lpp \frac{2\beta_T\alpha}{1-\alpha}\rpp^{\frac{2}{1-\alpha}}, \lppp 432 c_{fm,4}^{(2)} \beta_T^3\rppp^{\frac{2}{\alpha}},\\

&& \ \ \ \ \ \ \ \lpp 2 c_{fm,3}^{(2)}\beta_T^3(1+\beta_T)^{\frac{16}{4-\beta_{sg}^2}} \lpp \frac{56-6\beta_{sg}^2}{4-\beta_{sg}^2}\rpp^{\frac{28-3\beta_{sg}^2}{4-\beta_{sg}^2}} \rpp^{\frac{2}{\alpha}},\lpp \frac{256(4-\beta_{sg})\beta_T^4C^2}{\beta_{sg}}\rpp^{\frac{1}{\alpha}},\\

&& \ \ \ \ \ \ \ \lppp 32 \beta_{sg} C^2 (\csg)^2 \tilde{\lambda}^4 \rppp^{-\frac{1}{(4-4\beta_B)\alpha}}, \lpp 2\sqrt{2} c_{sd,3}^{\frac{1}{4}} \lpp \frac{\beta_T\alpha}{1+C\tilde{\lambda}\beta_T}\rpp^{\frac{1}{2}} \rpp^{\frac{2}{(1+C\tilde{\lambda}\beta_T})\alpha},\\

&& \ \ \ \ \ \ \ \lpp 2\sqrt{2} C_sc_{sd,3}^{\frac{1}{2}} \lpp \frac{2\beta_T}{\beta_B}\rpp^{\frac{1}{2}} \rpp^{\frac{4}{\beta_B\alpha}}, \lpp \frac{1}{2(C_s+1)}\rpp^{\frac{1}{CC_s\beta_T\alpha}},\\

&& \ \ \ \ \ \ \ \lpp \frac{80 \beta_T CC_{hl}\lppp1+\frac{1}{C_s}\rppp^{\frac{1}{2}}}{1-\beta_B-3CC_s\beta_T} \rpp^{\frac{2}{(1-\beta_B-3CC_s\beta_T)\alpha}}, \lpp \frac{9}{80} \lppp 1+\frac{1}{C_s}\rppp^{-1} \frac{1}{CC_s}\rpp^{\frac{1}{2CC_s\beta_T\alpha}},\\

&& \ \ \ \ \ \ \ \exp \lpp \lpp \frac{c_{fm,2}^{(2)}}{c_{fm,1}^{(2)}}\rpp^{\frac{2-\beta_{sg}}{4}} \alpha^{-1}\rpp, \lpp 2 \lppp c_{fm,1}^{(2)} \rppp^{\frac{1}{2}} (1+\beta_T)^{\frac{2}{2-\beta_{sg}}} \beta_T^{\frac{1}{2}} \lpp \frac{6-\beta_{sg}}{2-\beta_{sg}}\rpp^{\frac{6-\beta_{sg}}{2(2-\beta_{sg})})} \rpp^{\frac{2}{\alpha}}\\

&& \ \ \ \ \ \ \ \lpp \frac{16}{c_{sd,2}}\rpp^{\frac{1}{C\tilde{\lambda}\beta_T\alpha}}, \lpp \frac{32c_{sd,4}\beta_T}{\lpp \frac{1}{2}-3\beta_B-5CC_s\beta_T + C\tilde{\lambda}\beta_T\rpp c_{sd,2}}\rpp^{\frac{2}{\lpp \frac{1}{2}-3\beta_B-5CC_s\beta_T+ C\tilde{\lambda}\beta_T\rpp\alpha}} ,\\

&& \ \ \ \ \ \ \ \cnlemma{11}((5+\beta_B+C\tilde{\lambda}\beta_T)\alpha, (5+\beta_B+C\tilde{\lambda}\beta_T)\alpha)\Big\}.\\
\end{array}
$$

In the above definition, $c_{fm,1}^{(2)},\ldots,c_{fm,4}^{(2)}$ are constants defined in Lemma \ref{lemma:momentbound}. Then, for any $\tiln \leq n \leq 2\tiln$, on $\{\bdtheta_n \in \rgb{}\}$, we have  
$$
\EE_n \lppp F(\bdtheta_{n+T}) - F(\bdtheta_n)\rppp \leq -\frac{C\sigma_{\min}^2}{32e^23^{\alpha}} \tiln^{-\alpha}\tiln^{C\tilde{\lambda} \beta_T \alpha}.
$$
\label{lemma:baddecrease}
\end{lemma}

\begin{proof}
For simplicity, we abbreviate the constants $c_{sd,i}$ as $c_i$, for $i=1,2,3,4$. Let's firstly present the simplified implications of the sophisticated requirements on $\tiln$.
\begin{enumerate}
    \itemequationnormal[eq:ssimple0]{}{$$ \tiln \ge \max \lww \lpp \frac{4\beta_T\alpha^2}{1-\alpha}\rpp^{\frac{2}{1-\alpha}}, e^{\frac{1}{\beta_T \alpha}} \rww
    \Longrightarrow \big| (T-1)(\tiln + T+1)^{-\alpha} - T \tiln^{-\alpha} \big| \leq 2.
    $$}

    \itemequationnormal[eq:ssimple1]{}{$$\tiln \ge 
    \lpp \frac{2\beta_T\alpha}{1-\alpha}\rpp^{\frac{2}{1-\alpha}}
    \Longrightarrow T \leq \tiln.
    $$}

    \itemequationnormal[eq:ssimple2]{}{$$
    \tiln \ge 
    \lppp 432 c_{fm,4}^{(2)} \beta_T^3\rppp^{\frac{2}{\alpha}}
    \Longrightarrow
    c_{fm,4}^{(2)} \lppp \log \tiln^{\beta_T\alpha}\rppp^3 \leq \frac{1}{2} \tiln^{\alpha}.
    $$}

    \itemequationnormal[eq:ssimple3]{}{$$
    \resizebox{.83\hsize}{!}{$
    \tiln \ge 
    \lpp 2 c_{fm,3}^{(2)}\beta_T^3(1+\beta_T)^{\frac{16}{4-\beta_{sg}^2}} \lpp \frac{56-6\beta_{sg}^2}{4-\beta_{sg}^2}\rpp^{\frac{28-3\beta_{sg}^2}{4-\beta_{sg}^2}} \rpp^{\frac{2}{\alpha}}
    \Longrightarrow
    c_{fm,3}^{(2)} \lppp \log \tiln^{\beta_T\alpha}\rppp^3 \lppp \alpha \log \tiln + \log \log \tiln^{\beta_T\alpha}\rppp^{\frac{16}{4-\beta_{sg}^2}} \leq \frac{1}{2} \tiln^{\alpha}.
    $}
    $$}

    \itemequationnormal[eq:ssimple4]{}{$$
    \tiln \ge 
    \lpp \frac{256(4-\beta_{sg})\beta_T^4C^2}{\beta_{sg}}\rpp^{\frac{1}{\alpha}}
    \Longrightarrow
    \lpp 2-\frac{\beta_{sg}}{2}\rpp \lppp \log \tiln^{\beta_T\alpha}\rppp^4 \leq \frac{\beta_{sg}}{2C^2}\tiln^{2\alpha}.
    $$}

    \itemequationnormal[eq:ssimple5]{}{$$
    \resizebox{.83\hsize}{!}{$
    \tiln \ge \lppp 32 \beta_{sg} C^2 (\csg)^2 \tilde{\lambda}^4 \rppp^{-\frac{1}{(4-4\beta_B)\alpha}} \Longrightarrow
    32\beta_{sg} C^2 (\csg)^2 \tiln^{(2+4C\tilde{\lambda} \beta_T)\alpha} \ge  \frac{1}{\tilde{\lambda}^4} \tiln^{-(2-4\beta_B-4C\tilde{\lambda}\beta_T)\alpha}.
    $}
    $$}

    \itemequationnormal[eq:ssimple6]{}{$$ \resizebox{.83\hsize}{!}{$
    \tiln \ge \lpp 2\sqrt{2} c_{sd,3}^{\frac{1}{4}} \lpp \frac{\beta_T\alpha}{1+C\tilde{\lambda}\beta_T}\rpp^{\frac{1}{2}} \rpp^{\frac{2}{(1+C\tilde{\lambda}\beta_T})\alpha} \Longrightarrow
    \sqrt{2} c_{sd,3}^{\frac{1}{4}} \lppp \log(\tiln^{\beta_T\alpha})\rppp^{\frac{1}{2}} \tiln^{-(2+ C\tilde{\lambda}\beta_T)\alpha} \leq \frac{1}{2}\tiln^{-\alpha}.
     $}
    $$}

    \itemequationnormal[eq:ssimple7]{}{$$
    \tiln \ge \lpp 2\sqrt{2} C_s c_{sd,3}^{\frac{1}{2}} \lpp \frac{2\beta_T}{\beta_B}\rpp^{\frac{1}{2}} \rpp^{\frac{4}{\beta_B\alpha}} \Longrightarrow
    \sqrt{2} C_s c_{sd,3}^{\frac{1}{2}} \lppp \log(\tiln^{\beta_T\alpha})\rppp^{\frac{1}{2}} \tiln^{-\frac{1}{2}(2+ \beta_B)\alpha} \leq \frac{1}{2}\tiln^{-\alpha}.
    $$}

    \itemequationnormal[eq:ssimple8]{}{$$ \tiln \ge \lpp \frac{1}{2(C_s+1)}\rpp^{\frac{1}{CC_s\beta_T\alpha}}
    \Longrightarrow B_0  \leq C_sB_1.
    $$}

    \itemequationnormal[eq:ssimple9]{}{$$\tiln \ge 
    \lpp \frac{80 \beta_T CC_{hl}\lppp1+\frac{1}{C_s}\rppp^{\frac{1}{2}}}{1-\beta_B-3CC_s\beta_T} \rpp^{\frac{2}{(1-\beta_B-3CC_s\beta_T)\alpha}}
    \Longrightarrow B_3  \leq 2B_1.
    $$}

    \itemequationnormal[eq:ssimple10]{}{$$ \tiln \ge
    \max \lww e^{\frac{1}{\beta_T\alpha}} , \lpp \frac{9}{80} \lppp 1+\frac{1}{C_s}\rppp^{-1} \frac{1}{CC_s}\rpp^{\frac{1}{2CC_s\beta_T\alpha}}\rww
    \Longrightarrow \frac{9}{2}C_{hl} B_1^2 \leq C_sB_3.
    $$}

    \itemequationnormal[eq:ssimple11]{}{$$ \tiln \ge
    \max \lww e^{\frac{1}{\beta_T\alpha}} , \exp \lpp \lpp \frac{c_{fm,2}^{(2)}}{c_{fm,1}^{(2)}}\rpp^{\frac{2-\beta_{sg}}{4}} \alpha^{-1}\rpp \rww
    \Longrightarrow c_{fm,2}^{(2)} \leq c_{fm,1}^{(2)} (\log T)^{\frac{4}{2-\beta_{sg}}}.
    $$}

    \itemequationnormal[eq:ssimple12]{}{$$
    \arraycolsep=1.1pt\def\arraystretch{1.5} 
    \begin{array}{rl} & \tiln \ge
    \max \lww e^{\frac{1}{\beta_T\alpha}} , \lpp 2 \lppp c_{fm,1}^{(2)} \rppp^{\frac{1}{2}} (1+\beta_T)^{\frac{2}{2-\beta_{sg}}} \beta_T^{\frac{1}{2}} \lpp \frac{6-\beta_{sg}}{2-\beta_{sg}}\rpp^{\frac{6-\beta_{sg}}{2(2-\beta_{sg})})} \rpp^{\frac{2}{\alpha}} \rww\\
    \Longrightarrow &
    2 \lppp c_{fm,1}^{(2)}\rppp^{\frac{1}{2}} (\log T)^{\frac{2}{2-\beta_{sg}}} \tiln^{-\frac{1}{2}(4+\beta_B +4C\tilde{\lambda}\beta_T)\alpha} \lppp \log \tiln^{\beta_T\alpha}\rppp^{\frac{1}{2}} \leq \tiln^{-\alpha}.\\
    \end{array}
    $$}

    \itemequationnormal[eq:ssimple13]{}{$$ \tiln \ge
    \lpp \frac{16}{c_{sd,2}}\rpp^{\frac{1}{C\tilde{\lambda}\beta_T\alpha}}
    \Longrightarrow \frac{c_{sd,2}}{8}\tiln^{C\tilde{\lambda}\beta_T\alpha}\ge 2.
    $$}

    \itemequationnormal[eq:ssimple14]{}{$$ 
    \resizebox{.83\hsize}{!}{$
    \tiln \ge
    \lpp \frac{32c_{sd,4}\beta_T}{\lpp \frac{1}{2}-3\beta_B-5CC_s\beta_T+ C\tilde{\lambda}\beta_T\rpp c_{sd,2}}\rpp^{\frac{2}{\lpp \frac{1}{2}-3\beta_B-5CC_s\beta_T+ C\tilde{\lambda}\beta_T\rpp\alpha}}
    \Longrightarrow c_{sd,4} \log \lppp \tiln^{\beta_T\alpha}\rppp \leq \frac{c_{sd,2}}{16} \tiln^{\lpp \frac{1}{2}-3\beta_B-5CC_s\beta_T+ C\tilde{\lambda}\beta_T\rpp\alpha}.
    $}
    $$}
\end{enumerate}
For simplicity, we let
$$
\arraycolsep=1.1pt\def\arraystretch{1.5} 
\begin{array}{ccl}
\hcal_n &\triangleq& \nabla^2 F(\bdtheta_n),\\

\sb & \triangleq & \{\bdtheta_n \in \rgb{}\}.\\
\end{array}
$$
Based on condition \ref{cm_1}, for any $\bdtheta\in \Theta$, we have
    \begin{equation*}
    F(\bdtheta) \leq F(\bdtheta_n) + \nablaf{n}^T(\bdtheta-\bdtheta_n) + \frac{1}{2}(\bdtheta-\bdtheta_n)^T\hcal_n(\bdtheta-\bdtheta_n) + \frac{1}{6}C_{hl}\|\bdtheta-\bdtheta_n\|^3,
        \label{eq:bd1}
    \end{equation*}
which implies
\begin{equation}
\resizebox{.92\hsize}{!}{$
\arraycolsep=1.1pt\def\arraystretch{1.5} 
\begin{array}{rl}
F(\bdtheta_{n+T}) - F(\bdtheta_n) & \leq \nablaf{n}^T (\bdtheta_{n+T}-\bdtheta_n) + \frac{1}{2}(\bdtheta_{n+T}-\bdtheta_n)^T \hcal_n (\bdtheta_{n+T}-\bdtheta_n) + \frac{1}{6} C_{hl} \| \bdtheta_{n+T} - \bdtheta_n\|^3\\

& \leq \lppp \nablaf{n}^T(\tildetheta{n+T}-\bdtheta_n) + \frac{1}{2} (\tildetheta{n+T}-\bdtheta_n)^T \hcal_n (\tildetheta{n+T}-\bdtheta_n) \rppp\\

& \ \ \ + \lppp \nablaf{n}^T (\bdtheta_{n+T} - \tildetheta{n+T}) + \frac{1}{2} (\bdtheta_{n+T} - \tildetheta{n+T})^T \hcal_n (\bdtheta_{n+T} - \tildetheta{n+T})\\

&\ \ \ \ \ \ + (\bdtheta_{n+T} - \tildetheta{n+T})^T \hcal_n (\tildetheta{n+T}-\bdtheta_n) + \frac{1}{6}C_{hl}\|\bdtheta_{n+T} - \bdtheta_n\|^3\rppp\\

&\triangleq F_1 + F_2,\\
\end{array}
$}
    \label{eq:bd2}
\end{equation}
where
$$
\arraycolsep=1.1pt\def\arraystretch{1.5} 
\begin{array}{ccl}
F_1 & \triangleq & \nablaf{n}^T(\tildetheta{n+T}-\bdtheta_n) + \frac{1}{2} (\tildetheta{n+T}-\bdtheta_n)^T \hcal_n (\tildetheta{n+T}-\bdtheta_n),\\

F_2 & \triangleq & \nablaf{n}^T (\bdtheta_{n+T} - \tildetheta{n+T}) + \frac{1}{2} (\bdtheta_{n+T} - \tildetheta{n+T})^T \hcal_n (\bdtheta_{n+T} - \tildetheta{n+T}) \\

& & + (\bdtheta_{n+T} - \tildetheta{n+T})^T \hcal_n (\tildetheta{n+T}-\bdtheta_n) + \frac{1}{6}C_{hl}\|\bdtheta_{n+T} - \bdtheta_n\|^3.\\
\end{array}
$$
Then, based on (\ref{eq:bd2}), we have
\begin{equation}
\arraycolsep=1.1pt\def\arraystretch{1.5} 
\begin{array}{rl}
&\EE_n \lppp F(\bdtheta_{n+T}) - F(\bdtheta_n)\rppp\\

=&  \EE_n \lppp(F(\bdtheta_{n+T}) - F(\bdtheta_n))\yis{T}| \sb\rppp + \EE_n \lppp(F(\bdtheta_{n+T}) - F(\bdtheta_n))\yisc{T}\rppp\\

\leq& \EE_n(F_1 \yis{T}) + \EE_n(F_2 \yis{T}) + \EE_n\lppp(F(\bdtheta_{n+T}) - \fmin)\yisc{T}\rppp\\ 

= & \EE_n F_1 - \EE_n(F_1 \yisc{T}) + \EE_n(F_2 \yis{T}) + \EE_n\lppp(F(\bdtheta_{n+T}) - \fmin)\yisc{T}\rppp\\ 

\triangleq & T_1 - T_2 + T_3 +T_4.\\
\end{array}
    \label{eq:bd3}
\end{equation}

\twoline{To bound $T_1$}, based on Lemma \ref{lemma:run}, we firstly have
\begin{equation}
\arraycolsep=1.1pt\def\arraystretch{1.5} 
\begin{array}{rl}
& \EE_n\lppp \nablaf{n}^T (\tildetheta{n+T}-\bdtheta_n) \rppp\\

= & -\nablaf{n}^T \EE_n \lpp \sum\limits_{j=1}\limits^T \gamma_{n+j} \lpp \pprod{i=j+1}{T} (I - \gamma_{n+i}\hcal_n) \rpp \bdxi_{n+j} \rpp\\

& - \nablaf{n}^T\lpp \sum\limits_{j=0}\limits^{T-1}\gamma_{n+j+1} \lpp \prod\limits_{i=1}\limits^j (I - \gamma_{n+i} \hcal_n)\rpp \rpp\nabla F(\bdtheta_n)\\

= & - \nablaf{n}^T\lpp \sum\limits_{j=0}\limits^{T-1}\gamma_{n+j+1} \lpp \prod\limits_{i=1}\limits^j (I - \gamma_{n+i} \hcal_n)\rpp \rpp \nabla F(\bdtheta_n).\\
\end{array}
\label{eq:bd4}
\end{equation}

We also have
\begin{equation}
\resizebox{.92\hsize}{!}{$
\arraycolsep=1.1pt\def\arraystretch{1.5} 
\begin{array}{rl}
& \EE_n \lppp (\tildetheta{n+T}-\bdtheta_n)^T \hcal_n (\tildetheta{n+T}-\bdtheta_n)  \rppp\\

=&  \EE_n \lpp  \lpp \sum\limits_{j=1}\limits^T \gamma_{n+j} \lpp \pprod{i=j+1}{T} (I - \gamma_{n+i}\hcal_n) \rpp \bdxi_{n+j} \rpp^T \hcal_n \lpp \sum\limits_{j=1}\limits^T \gamma_{n+j} \lpp \pprod{i=j+1}{T} (I - \gamma_{n+i}\hcal_n) \rpp \bdxi_{n+j} \rpp \rpp\\

&+  \lpp \sum\limits_{j=0}\limits^{T-1}\gamma_{n+j+1} \lpp \prod\limits_{i=1}\limits^j (I - \gamma_{n+i} \hcal_n)\rpp \nabla F(\bdtheta_n) \rpp^T \hcal_n \lpp \sum\limits_{j=0}\limits^{T-1}\gamma_{n+j+1} \lpp \prod\limits_{i=1}\limits^j (I - \gamma_{n+i} \hcal_n)\rpp \nabla F(\bdtheta_n) \rpp.\\
\end{array}
$}
\label{eq:bd5}
\end{equation}

Suppose the singular value decomposition of $\hcal_n$ is
$$
\hcal_n = \Gamma_n^T \Lambda_n \Gamma_n,
$$
where $\Gamma_n$ is an orthogonal matrix and $\Lambda_n = \text{diag} \{ \lambda_{n,1}, \lambda_{n,2},\ldots, \lambda_{n,d}\}$, $\lambda_{n,1} \ge \lambda_{n,2}\ge \ldots\ge\lambda_{n,d}$. On $\sb$, $\lambda_{n,d} \leq -\tilde{\lambda}$. We also let
$$
\bar{\bdxi}_{n+j} = (\bar{\xi}_{n+j,1}, \bar{\xi}_{n+j,2},\ldots, \bar{\xi}_{n+j,d})^T \triangleq \Gamma_n\bdxi_{n+j}.
$$
We can handle the first term on the right-hand side of \ref{eq:bd5}) as follows,
\begin{equation}
\resizebox{.92\hsize}{!}{$
\arraycolsep=1.1pt\def\arraystretch{1.5}  
\begin{array}{rl}
& \EE_n \lpp  \lpp \sum\limits_{j=1}\limits^T \gamma_{n+j} \lpp \pprod{i=j+1}{T} (I - \gamma_{n+i}\hcal_n) \rpp \bdxi_{n+j} \rpp^T \hcal_n \lpp \sum\limits_{j=1}\limits^T \gamma_{n+j} \lpp \pprod{i=j+1}{T} (I - \gamma_{n+i}\hcal_n) \rpp \bdxi_{n+j} \rpp \rpp\\

= & \EE_n \lpp \ssum{j=1}{T} \gamma_{n+j}^2 \bdxi_{n+j}^T \lpp \pprod{i=j+1}{T}(I - \gamma_{n+i}\hcal_n) \rpp \hcal_n \lpp \pprod{i=j+1}{T}(I - \gamma_{n+i}\hcal_n) \rpp \bdxi_{n+j} \rpp\\

= & \EE_n \lpp \ssum{j=1}{T} \gamma_{n+j}^2 \bdxi_{n+j}^T \Gamma_n^T \lpp \pprod{i=j+1}{T}(I - \gamma_{n+i}\Lambda_n) \rpp \Lambda_n \lpp \pprod{i=j+1}{T}(I - \gamma_{n+i}\Lambda_n) \rpp \Gamma_n\bdxi_{n+j} \rpp\\

= & \ssum{j=1}{T} \ssum{k=1}{d} \gamma_{n+j}^2 \lambda_{n,k} \pprod{i=j+1}{T} (1-\gamma_{n+i}\lambda_{n,k})^2 \EE_n \tilde{\xi}_{n+j,k}^2\\

\leq & (d-1)C_s \ssum{j=1}{T}\gamma_{n+j}^2 \EE_n \| \tilde{\bdxi}_{n+j}\|^2 - \tilde{\lambda}\sigma_{\min}^2\ssum{j=1}{T}\gamma_{n+j}^2\lpp \pprod{i=j+1}{T}(1+\gamma_{n+i} \tilde{\lambda})^2\rpp \\

= & (d-1)C_s \ssum{j=1}{T}\gamma_{n+j}^2 \EE_n \| {\bdxi}_{n+j}\|^2 - \tilde{\lambda}\sigma_{\min}^2\ssum{j=1}{T}\gamma_{n+j}^2\lpp \pprod{i=j+1}{T}(1+\gamma_{n+i} \tilde{\lambda})^2\rpp \\

\leq & 4(d-1)C_s \csg\ssum{j=1}{T}\gamma_{n+j}^2 \lpp \frac{4-\beta_{sg}}{2} + \frac{\beta_{sg}}{2} \EE_n\|\nablaf{n+j-1}\|^2\rpp - \tilde{\lambda}\sigma_{\min}^2\ssum{j=1}{T}\gamma_{n+j}^2 \lpp \pprod{i=j+1}{T}(1+\gamma_{n+i} \tilde{\lambda})^2\rpp.\\
\end{array}
$}
\label{eq:bd6}
\end{equation}
where the 4th step is based on conditions \ref{cm_1}, the last step is based on condition \ref{cm_4} and Lemma \ref{lemma:errorbound}. In fact, we have
\begin{equation}
\arraycolsep=1.1pt\def\arraystretch{1.5} 
\begin{array}{rl}
& \ssum{j=1}{T}\gamma_{n+j}^2 \lpp \frac{4-\beta_{sg}}{2} + \frac{\beta_{sg}}{2} \EE_n\|\nablaf{n+j-1}\|^2\rpp\\

\leq & \frac{4-\beta_{sg}}{2}C^2 tn^{-2\alpha} + \frac{\beta_{sg}}{2} C n^{-\alpha} \ssum{j=1}{T}\gamma_{n+j} \EE_n \| \nablaf{n+j-1}\|^2\\

\leq & \frac{4-\beta_{sg}}{2}C^2 T\tiln^{-2\alpha} + \frac{\beta_{sg}}{2}C \lppp F(\bdtheta_n) - F_{\min} + \frac{1}{2\alpha-1}C^2 C_s C_{ng}^{(2)}\rppp n^{-\alpha}\\

\leq &\tiln^{-\alpha}\lppp \frac{4-\beta_{sg}}{2}C^2\log (\tiln^{\beta_T\alpha}) + \frac{\beta_{sg}}{2}C \lppp F(\bdtheta_n) - F_{\min} + \frac{1}{2\alpha-1}C^2C_sC_{ng}^{(2)}\rppp \rppp\\

\leq & \beta_{sg} C F_b \tiln^{-\alpha},\\
\end{array}
\label{eq:bd7}
\end{equation}
where the 2nd step is based on Lemma \ref{lemma:nablasumbound} and the last step is based on the assumption that 
$$
F_b \ge \max \lww \frac{4}{2\alpha-1}C^2 C_s C_{ng}^{(2)} - 4\fmin, 4\frac{4-\beta_{sg}}{\beta_{sg}} C \log (\tiln^{\beta_T\alpha}) \rww.
$$
We also have
\begin{equation}
\resizebox{.92\hsize}{!}{$
\arraycolsep=1.1pt\def\arraystretch{1.5} 
\begin{array}{rl}
& \ssum{j=1}{T}\gamma_{n+j}^2 \lpp \pprod{i=j+1}{T}(1+\gamma_{n+i} \tilde{\lambda})^2\rpp \\

\ge & \ssum{j=1}{T} \gamma_{n+j}^2 \exp \lpp C\tilde{\lambda} \ssum{i=j+1}{T}(n+i)^{-\alpha}\rpp\\

\ge & \ssum{j=1}{T} \gamma_{n+j}^2 \exp \lpp \frac{C\tilde{\lambda}}{1-\alpha} \lppp (n+T+1)^{1-\alpha} - (n+j+1)^{1-\alpha}\rppp \rpp\\


\ge &C^2 (n+T)^{-\alpha} \exp \lpp \frac{C\tilde{\lambda}}{1-\alpha} (n+T+1)^{1-\alpha}\rpp \ssum{j=1}{T} (n+j+1)^{-\alpha}\exp\lpp \frac{-C\tilde{\lambda}}{1-\alpha} (n+j+1)^{1-\alpha} \rpp\\

\ge &C^2 (n+T)^{-\alpha} \exp \lpp \frac{C\tilde{\lambda}}{1-\alpha} (n+T+1)^{1-\alpha}\rpp  \int_{n+T}^{n+T+2} x^{-\alpha} \exp \lpp \frac{-C\tilde{\lambda}}{1-\alpha}x^{1-\alpha} \rpp dx\\

= & \frac{C}{\tilde{\lambda}} (n+T)^{-\alpha} \exp \lpp \frac{C\tilde{\lambda}}{1-\alpha} (n+T+1)^{1-\alpha}\rpp \lpp \exp \lpp \frac{-C\tilde{\lambda}}{1-\alpha} (n+2)^{1-\alpha}\rpp - \exp \lpp \frac{-C\tilde{\lambda}}{1-\alpha} (n+T+2)^{1-\alpha}\rpp \rpp\\

\ge & \frac{C}{\tilde{\lambda}} (n+T)^{-\alpha} \lppp \exp\lppp C\tilde{\lambda}(T-1)(n+T+1)^{-\alpha}\rppp -1 \rppp\\

\ge & \frac{C}{\tilde{\lambda}} (n+T)^{-\alpha} \lppp \exp(-2) \exp(C\tilde{\lambda} T\tiln^{-\alpha}) - 1 \rppp\\

\ge & \frac{C}{2e^2 3^{\alpha} \tilde{\lambda}} \tiln^{(C\tilde{\lambda}\beta_T - 1)\alpha},\\
\end{array}
$}
\label{eq:bd8}
\end{equation}
where the 7th step is based on (\ref{eq:ssimple0}) and the 8th step is based on $\tiln \ge (2e^2)^{\frac{1}{C\tilde{\lambda}\beta_T\alpha}}$, $n \leq 2\tiln$ and (\ref{eq:ssimple1}). Combining (\ref{eq:bd6}), (\ref{eq:bd7}) and (\ref{eq:bd8}), recalling that
$$
\arraycolsep=1.1pt\def\arraystretch{1.5} 
\begin{array}{ccl}
c_{sd,1} &=& 4(d-1) C C_s \csg,\\
c_{sd,2} &= & \frac{C\sigma_{\min}^2}{2e^23^{\alpha}},
\end{array}
$$
we have
\begin{equation}
\resizebox{.92\hsize}{!}{$
\arraycolsep=1.1pt\def\arraystretch{1.5} 
\begin{array}{rl}
& \EE_n \lpp  \lpp \sum\limits_{j=1}\limits^T \gamma_{n+j} \lpp \pprod{i=j+1}{T} (I - \gamma_{n+i}\hcal_n) \rpp \bdxi_{n+j} \rpp^T \hcal_n \lpp \sum\limits_{j=1}\limits^T \gamma_{n+j} \lpp \pprod{i=j+1}{T} (I - \gamma_{n+i}\hcal_n) \rpp \bdxi_{n+j} \rpp \rpp\\

\leq & \tiln^{-\alpha} (c_{sd,1} F_b - c_{sd,2} \tiln^{C\tilde{\lambda}\beta_T \alpha}) \leq -\frac{c_{sd,2}}{2} \tiln^{-\alpha}\tiln^{C\tilde{\lambda}\beta_T \alpha},\\ 
\end{array}
$}
    \label{eq:bd9}
\end{equation}
where the last step is based on the assumption that
$$
F_b \leq \frac{\sigma_{\min}^2}{16e^2 (d-1)3^{\alpha}C_s \csg} \tiln^{C\tilde{\lambda}\beta_T \alpha}.
$$

The second term on the right-hand side of (\ref{eq:bd5}) can be simplified as
\begin{equation}
\resizebox{.92\hsize}{!}{$
\arraycolsep=1.1pt\def\arraystretch{1.5} 
\begin{array}{rl}
& \lpp \sum\limits_{j=0}\limits^{T-1}\gamma_{n+j+1} \lpp \prod\limits_{i=1}\limits^j (I - \gamma_{n+i} \hcal_n)\rpp \nabla F(\bdtheta_n) \rpp^T \hcal_n \lpp \sum\limits_{j=0}\limits^{T-1}\gamma_{n+j+1} \lpp \prod\limits_{i=1}\limits^j (I - \gamma_{n+i} \hcal_n)\rpp \nabla F(\bdtheta_n) \rpp\\

= & (\nablaf{n})^T \lpp \sum\limits_{j=0}\limits^{T-1}\gamma_{n+j+1} \lpp \prod\limits_{i=1}\limits^j (I - \gamma_{n+i} \hcal_n)\rpp \hcal_n \rpp\lpp \sum\limits_{j=0}\limits^{T-1}\gamma_{n+j+1} \lpp \prod\limits_{i=1}\limits^j (I - \gamma_{n+i} \hcal_n)\rpp \nabla F(\bdtheta_n) \rpp\\

= & (\nablaf{n})^T \lpp I - \pprod{j=1}{T}(I-\gamma_{n+j}\hcal) \rpp \lpp \sum\limits_{j=0}\limits^{T-1}\gamma_{n+j+1} \lpp \prod\limits_{i=1}\limits^j (I - \gamma_{n+i} \hcal_n)\rpp \nabla F(\bdtheta_n) \rpp\\

\leq & \nablaf{n}^T\lpp \sum\limits_{j=0}\limits^{T-1}\gamma_{n+j+1} \lpp \prod\limits_{i=1}\limits^j (I - \gamma_{n+i} \hcal_n)\rpp \rpp \nabla F(\bdtheta_n).\\
\end{array}
$}
\label{eq:bd10}
\end{equation}

Therefore, based on (\ref{eq:bd5}), (\ref{eq:bd9}) and (\ref{eq:bd10}), on $\sb$, we have
\begin{equation}
\arraycolsep=1.1pt\def\arraystretch{1.5} 
\begin{array}{rl}
& \EE_n \lppp (\tildetheta{n+T}-\bdtheta_n)^T \hcal_n (\tildetheta{n+T}-\bdtheta_n)  \rppp\\

\leq & -\frac{c_{sd,2}}{2} \tiln^{-\alpha}\tiln^{C\tilde{\lambda}\beta_T \alpha} + \nablaf{n}^T\lpp \sum\limits_{j=0}\limits^{T-1}\gamma_{n+j+1} \lpp \prod\limits_{i=1}\limits^j (I - \gamma_{n+i} \hcal_n)\rpp \rpp \nabla F(\bdtheta_n).\\
\end{array}
\label{eq:bd11}
\end{equation}
Then, based on (\ref{eq:bd4}) and (\ref{eq:bd11}), on $\sb$, we have
\begin{equation}
\arraycolsep=1.1pt\def\arraystretch{1.5} 
\begin{array}{rl}
T_1 & = \EE_n\lppp \nablaf{n}^T (\tildetheta{n+T}-\bdtheta_n) \rppp + \frac{1}{2} \EE_n \lppp (\tildetheta{n+T}-\bdtheta_n)^T \hcal_n (\tildetheta{n+T}-\bdtheta_n)\rppp \leq -\frac{c_{sd,2}}{4} \tiln^{-\alpha}\tiln^{C\tilde{\lambda}\beta_T \alpha}.\\
\end{array}
\label{eq:bd12}
\end{equation}

\twoline{To bound $T_2$}, we firstly have the following decomposition on $S_{bad}$
\begin{equation}
\arraycolsep=1.1pt\def\arraystretch{1.5} 
\begin{array}{rl}
|T_2| & \leq \big| \EE_n \nablaf{n}^T (\tildetheta{n+T} - \bdtheta_n)\daone_{S_{n,T}^c}\big| + \frac{1}{2} \big| \EE_n (\tildetheta{n+T} - \bdtheta_n)^T \hcal_n (\tildetheta{n+T} - \bdtheta_n) \daone_{S_{n,T}^c} \big|\\

& \leq B_0 \EE_n \lppp \lvvv \tildetheta{n+T} - \bdtheta_n \rvvv \daone_{S_{n,T}^c} \rppp + C_s \EE_n \lppp \lvvv \tildetheta{n+T} - \bdtheta_n \rvvv^2 \daone_{S_{n,T}^c} \rppp\\

& \leq B_0 \lppp \EE_n \lvvv \tildetheta{n+T} - \bdtheta_n \rvvv^4 \rppp^{\frac{1}{4}} \PP_n^{\frac{1}{2}} \lppp S_{n,T}^c\rppp + C_s \EE_n \lvvv \tildetheta{n+T} - \bdtheta_n \rvvv^4 \rppp^{\frac{1}{2}} \PP_n^{\frac{1}{2}} \lppp S_{n,T}^c\rppp,\\
\end{array}
\label{eq:bd13}
\end{equation}
where the 2nd step is based on condition \ref{cm_1} and the last step is based on the Cauchy's inequality. Based on the decomposition given in Lemma \ref{lemma:run}, we have
\begin{equation}
\resizebox{.92\hsize}{!}{$
\EE_n \lvvv \tildetheta{n+T} - \bdtheta_n \rvvv^4 \leq 8 \EE_n \lvvv \sum\limits_{j=1}\limits^T \gamma_{n+j} \lpp \pprod{i=j+1}{T} (I - \gamma_{n+i}\hcal_n) \rpp \bdxi_{n+j} \rvvv^4  +  8\EE_n \lvvv \sum\limits_{j=0}\limits^{T-1}\gamma_{n+j+1} \lpp \prod\limits_{i=1}\limits^j (I - \gamma_{n+i} \hcal_n)\rpp \nabla F(\bdtheta_n) \rvvv^4.
$}
\label{eq:bd14}
\end{equation}


The first part on the right-hand side of (\ref{eq:bd14}) can be controlled as follows.
\begin{equation}
\resizebox{.92\hsize}{!}{$
\arraycolsep=1.1pt\def\arraystretch{1.5} 
\begin{array}{rl}
& \EE_n \lvvv \sum\limits_{j=1}\limits^T \gamma_{n+j} \lpp \pprod{i=j+1}{T} (I - \gamma_{n+i}\hcal_n) \rpp \bdxi_{n+j} \rvvv^4\\

\leq & T^3 \EE_n \lpp\sum\limits_{j=1}\limits^T \gamma_{n+j}^4 \lvv \pprod{i=j+1}{T} (I - \gamma_{n+i}\hcal_n) \rvv^4 \EE_{n+j-1} \| \bdxi_{n+j}\|^4\rpp\\

\leq & 32(\csg)^2 T^3 \EE_n \lpp \sum\limits_{j=1}\limits^T \gamma_{n+j}^4 \lvv \pprod{i=j+1}{T} (I - \gamma_{n+i}\hcal_n) \rvv^4 \lpp 2-\frac{\beta_{sg}}{2} + \frac{\beta_{sg}}{2} \| \nablaf{n+j-1}\|^4 \rpp \rpp\\

\leq & 32(\csg)^2 T^3 \EE_n \lpp \sum\limits_{j=1}\limits^T \gamma_{n+j}^4 \exp \lpp 4C\tilde{\lambda} \ssum{i=j+1}{T}(n+i)^{-\alpha}\rpp \lpp 2-\frac{\beta_{sg}}{2} + \frac{\beta_{sg}}{2} \| \nablaf{n+j-1}\|^4 \rpp \rpp\\

\leq & 32(\csg)^2 T^3 \exp (4C\tilde{\lambda} T n^{-\alpha})\lpp \lpp 2-\frac{\beta_{sg}}{2}\rpp C^4 T n^{-4\alpha} + \frac{\beta_{sg}}{2} C^2 n^{-2\alpha} \EE_n\ssum{j=1}{T} \gamma_{n+j}^2 \| \nablaf{n+j-1}\|^4 \rpp \\

\leq & 32(\csg)^2 T^3 \exp (4C\tilde{\lambda} T n^{-\alpha})\lpp \lpp 2-\frac{\beta_{sg}}{2}\rpp C^4 T n^{-4\alpha} + \frac{\beta_{sg}}{2} C^2 n^{-2\alpha} \EE_n \lpp \ssum{j=1}{T}  \gamma_{n+j} \| \nablaf{n+j-1}\|^2\rpp^2 \rpp \\

\leq & 32(\csg)^2 T^3 \exp (4C\tilde{\lambda} T n^{-\alpha})\lpp \lpp 2-\frac{\beta_{sg}}{2}\rpp C^4 T n^{-4\alpha} + \frac{\beta_{sg}}{2} C^2 n^{-2\alpha} \lpp c_{fm,4}^{(2)} + c_{fm,3}^{(2)} (\log T)^{\frac{16}{4-\beta_{sg}^2}}\rpp \rpp\\

\leq & 32 C^4 (\csg)^2 \tiln^{4C\tilde{\lambda} \beta_T \alpha} \lpp \lpp 2-\frac{\beta_{sg}}{2}\rpp  \lppp \log \tiln^{\beta_T\alpha}\rppp^4 +\frac{\beta_{sg}}{2C^2} \tiln^{\alpha} (\log \tiln^{\beta_T \alpha}\rppp^3 \lpp c_{fm,4}^{(2)} + c_{fm,3}^{(2)} \lpp \alpha \log \tiln + \log \log \tiln^{\beta_T \alpha}\rpp^{\frac{16}{4-\beta_{sg}^2}}\rpp \rpp\\

\leq & 32\beta_{sg} C^2 (\csg)^2 \tiln^{(2+4C\tilde{\lambda} \beta_T)\alpha},\\
\end{array}
$}
\label{eq:bd16}
\end{equation}
where the 2nd step is based on Lemma \ref{lemma:errorbound}, the 6th step is based on Lemma \ref{lemma:momentbound} ($c_{fm,3}^{(2)}$ and $c_{fm,4}^{(2)}$ are constants defined in Lemma \ref{lemma:momentbound}) and the last step is based on (\ref{eq:ssimple2}), (\ref{eq:ssimple3}) and (\ref{eq:ssimple4}). The second part on the right-hand side of (\ref{eq:bd14}) can be handled as follows.
\begin{equation}
\arraycolsep=1.1pt\def\arraystretch{1.5} 
\begin{array}{rl}
& \EE_n \lvvv \sum\limits_{j=0}\limits^{T-1}\gamma_{n+j+1} \lpp \prod\limits_{i=1}\limits^j (I - \gamma_{n+i} \hcal_n)\rpp \nabla F(\bdtheta_n) \rvvv^4\\

\leq & \lpp \sum\limits_{j=0}\limits^{T-1}\gamma_{n+j+1} \exp \lpp C\tilde{\lambda} \ssum{i=1}{j}(n+i)^{-\alpha}\rpp \rpp^4 B_0^4\\

\leq & \lpp \sum\limits_{j=0}\limits^{T-1}\gamma_{n+j+1} \exp \lpp \frac{C\tilde{\lambda}}{1-\alpha} \lppp (n+j)^{1-\alpha}-n^{1-\alpha}\rppp \rpp \rpp^4 B_0^4\\

= & C^4 \exp \lpp -\frac{4C\tilde{\lambda}}{1-\alpha}n^{1-\alpha}\rpp \lpp \sum\limits_{j=0}\limits^{T-1} (n+j+1)^{-\alpha} \exp \lpp \frac{C\tilde{\lambda}}{1-\alpha}(n+j)^{1-\alpha}\rpp \rpp^4 B_0^4\\

\leq & C^4 \exp \lpp -\frac{4C\tilde{\lambda}}{1-\alpha}n^{1-\alpha}\rpp \lpp \sum\limits_{j=0}\limits^{T-1} (n+j)^{-\alpha} \exp \lpp \frac{C\tilde{\lambda}}{1-\alpha}(n+j)^{1-\alpha}\rpp \rpp^4 B_0^4\\

\leq & \frac{B_0^4}{\tilde{\lambda}^4} \exp \lpp \frac{4C\tilde{\lambda}}{1-\alpha}\lpp (n+T)^{1-\alpha}-n^{1-\alpha}\rpp \rpp \\

\leq & \frac{B_0^4}{\tilde{\lambda}^4} \exp \lppp 4C\tilde{\lambda}T \tiln^{-\alpha}\rppp = \frac{1}{\tilde{\lambda}^4} \tiln^{-(2-4\beta_B-4C\tilde{\lambda}\beta_T)\alpha},\\
\end{array}
\label{eq:bd17}
\end{equation}
where the 5th step is based on Lemma \ref{lemma:int0} and the assumption that $\tiln \ge \lpp \frac{\alpha}{C\tilde{\lambda}}\rpp^{\frac{1}{1-\alpha}}$.

Based on (\ref{eq:bd14}), (\ref{eq:bd16}) and (\ref{eq:bd17}), recalling that
$$
c_{sd,3} = 512\beta_{sg} C^2 (\csg)^2.
$$
we have
\begin{equation}
\arraycolsep=1.1pt\def\arraystretch{1.5} 
\begin{array}{rl}
\EE_n \lvvv \tildetheta{n+T} - \bdtheta_n \rvvv^4 & \leq 8 \lpp 32\beta_{sg} C^2 (\csg)^2 \tiln^{(2+4C\tilde{\lambda} \beta_T)\alpha} + \frac{1}{\tilde{\lambda}^4} \tiln^{-(2-4\beta_B-4C\tilde{\lambda}\beta_T)\alpha} \rpp\\

& \leq c_{sd,3} \tiln^{(2+4C\tilde{\lambda} \beta_T)\alpha},\\
\end{array}
\label{eq:bd18}
\end{equation}
where the last step is based on (\ref{eq:ssimple5}). Based on Lemma \ref{lemma:run} and condition that $\tiln \ge \cnlemma{11}((5+\beta_B+4C\tilde{\lambda}\beta_T)\alpha, (5+\beta_B+4C\tilde{\lambda}\beta_T)\alpha)$, on $S_{bad}$, we have
\begin{equation}
\PP_n \lppp S_{n,T}^c\rppp \leq 2T \tiln^{-(5+\beta_B+4C\tilde{\lambda}\beta_T)\alpha}.
    \label{eq:bd19}
\end{equation}
Putting (\ref{eq:bd18}) back to (\ref{eq:bd13}), we have
\begin{equation}
\arraycolsep=1.1pt\def\arraystretch{1.5} 
\begin{array}{rl}
|T_2| & \leq \lppp B_0 c_{sd,3}^{\frac{1}{4}} \tiln^{\lppp \frac{1}{2}+C\tilde{\lambda}\beta_T\rppp \alpha} + C_s c_{sd,3}^{\frac{1}{2}} \tiln^{(1+2 C\tilde{\lambda}\beta_T)\alpha} \rppp \PP_n^{\frac{1}{2}} \lppp S_{n,T}^c\rppp\\

& \leq \lpp c_{sd,3}^{\frac{1}{4}} \tiln^{\lppp \frac{\beta_B}{2} + C\tilde{\lambda}\beta_T\rppp \alpha}  + C_s c_{sd,3}^{\frac{1}{2}} \tiln^{(1+2 C\tilde{\lambda}\beta_T)\alpha} \rpp \sqrt{2} \tiln^{\frac{\alpha}{2}} \lppp \log \lppp \tiln^{\beta_T\alpha}\rppp\rppp^{\frac{1}{2}} \tiln^{-\frac{1}{2}(5+\beta_B+4C\tilde{\lambda}\beta_T)\alpha}\\

& \leq \tiln^{-\alpha},\\
\end{array}
\label{eq:bd20}
\end{equation}
where the 2nd step is based on (\ref{eq:bd19}) and the last step is based on (\ref{eq:ssimple6}) and (\ref{eq:ssimple7}).

\twoline{To bound $T_3$}, on $\sb$, as
$$
c_{sd,4} = 1280CC_sC_{hl} \lppp 1+\frac{1}{C_s}\rppp^4,
$$
we have
\begin{equation}
\arraycolsep=1.1pt\def\arraystretch{1.5} 
\begin{array}{rl}
T_3 & \leq \| \nablaf{n} \| \EE_n \lppp \| \bdtheta_{n+T} - \tildetheta{n+T}\| \daone_{S_{n,T}}\rppp + \frac{C_s}{2} \EE_n \lppp \| \bdtheta_{n+T} - \tildetheta{n+T}\|^2 \daone_{S_{n,T}}\rppp\\

&\ \ + C_s \EE_n \lppp \| \bdtheta_{n+T} - \tildetheta{n+T}\| \| \tildetheta{n+T} - \bdtheta_n \| \daone_{S_{n,T}}\rppp + \frac{1}{6} C_{hl} \EE_n \lppp \| \bdtheta_{n+T} - \bdtheta_n \|^3 \daone_{S_{n,T}}\rppp\\

& \leq B_0B_3 + \frac{C_s}{2} B_3^2 + C_s B_1B_3 +\frac{1}{6}C_{hl}(B_1 + B_3)^3\\

& \leq 4C_s B_1B_3 = c_{sd,4} \tiln^{-(\frac{3}{2}-3\beta_B - 5CC_s\beta_T )\alpha} \log(\tiln^{\beta_T\alpha}),\\
\end{array}
\label{eq:bd21}
\end{equation}
where the 3rd step is based on (\ref{eq:ssimple8}), (\ref{eq:ssimple9}) and (\ref{eq:ssimple10}).

\twoline{To bound $T_4$},
\begin{equation}
\arraycolsep=1.1pt\def\arraystretch{1.5} 
\begin{array}{rl}
T_4 & = \EE_n \lppp (F(\bdtheta_{n+T}) - \fmin ) \yis{T}\rppp \leq \lppp \EE_n (F(\bdtheta_{n+T}) - \fmin )^2 \rppp^{\frac{1}{2}} \PP_n^{\frac{1}{2}} \lppp S_{n,T}^c\rppp\\ 

& \leq \sqrt{2} \lpp c_{fm,2}^{(2)} + c_{fm,1}^{(2)} (\log T)^{\frac{4}{2-\beta_{sg}}} \rpp^{\frac{1}{2}} \tiln^{-\frac{1}{2}(4+\beta_B +4C\tilde{\lambda}\beta_T)\alpha} \lppp \log \tiln^{\beta_T\alpha}\rppp^{\frac{1}{2}}\\

& \leq 2 \lppp c_{fm,1}^{(2)}\rppp^{\frac{1}{2}} (\log T)^{\frac{2}{2-\beta_{sg}}} \tiln^{-\frac{1}{2}(4+\beta_B +4C\tilde{\lambda}\beta_T)\alpha} \lppp \log \tiln^{\beta_T\alpha}\rppp^{\frac{1}{2}}\\

& \leq \tiln^{-\alpha},\\
\end{array}
\label{eq:bd22}
\end{equation}
where the 3rd step is based on Lemma \ref{lemma:momentbound} and Lemma \ref{lemma:run}, the 4th step is based on (\ref{eq:ssimple11}) and the last step is based on (\ref{eq:ssimple12}).

Now, based on (\ref{eq:bd3}), (\ref{eq:bd12}), (\ref{eq:bd20}), (\ref{eq:bd21}) and (\ref{eq:bd22}), on $S_{bad}$, we have
\begin{equation}
\arraycolsep=1.1pt\def\arraystretch{1.5} 
\begin{array}{rl}
&\EE_n \lppp F(\bdtheta_{n+T}) - F(\bdtheta_n)\rppp\\

\leq & -\frac{c_{sd,2}}{4} \tiln^{-\alpha}\tiln^{C\tilde{\lambda}\beta_T \alpha} +2\tiln^{-\alpha} + c_{sd,4} \tiln^{-(\frac{3}{2}-3\beta_B - 5CC_s\beta_T )\alpha} \log(\tiln^{\beta_T\alpha})\\

\leq & -\frac{c_{sd,2}}{8} \tiln^{-\alpha}\tiln^{C\tilde{\lambda} \beta_T \alpha} + c_{sd,4} \tiln^{-(\frac{3}{2}-3\beta_B - 5CC_s\beta_T )\alpha} \log(\tiln^{\beta_T\alpha})\\

\leq & -\frac{c_{sd,2}}{16} \tiln^{-\alpha}\tiln^{C\tilde{\lambda} \beta_T \alpha}, \\
\end{array}
\label{eq:bd23}
\end{equation}
where the 2nd step is based on (\ref{eq:ssimple13}) and the last step is based on (\ref{eq:ssimple14}).

\end{proof}

\begin{lemma}
Suppose that conditions in Lemma \ref{lemma:biggradientexp} and \ref{lemma:baddecrease} hold. We additionally require condition \ref{cm_4} and $2\beta_B=C\tilde{\lambda}\beta_T > 2-\frac{1}{\alpha}$. (To be compatible with condition given in Lemma \ref{lemma:baddecrease}, we actually need $\alpha < \frac{1}{2-\frac{\tilde{\lambda}}{10C_s+\tilde{\lambda}}}$.) We let
$$
\arraycolsep=1.1pt\def\arraystretch{1.5} 
\begin{array}{rl}
B_4 & \triangleq \frac{C\sigma_{\min}^2}{32e^23^{\alpha}}\tiln^{-(2-C\tilde{\lambda}\beta_T)\alpha} \frac{1}{\log \lppp \tiln^{\beta_T\alpha}\rppp}, \\

\tilde{T} &\triangleq \tiln^{(2-C\tilde{\lambda}\beta_T)\alpha} \tiln^{\beta_{\epsilon}\alpha} \log \lppp \tiln^{\beta_T\alpha}\rppp,\\
\end{array}
$$
for some positive constant $\beta_{\epsilon} < \frac{1}{\alpha} + C\tilde{\lambda}\beta_T -2$. We also define
$$
\resizebox{.98\hsize}{!}{$
\arraycolsep=1.1pt\def\arraystretch{1.5}
\begin{array}{ccl}
\cnlemma{13}&\triangleq & \max \Big\{ \exp \lpp \frac{\sigma_{\min}^2}{8e^23^{\alpha}\beta_T \alpha } \rpp, \lpp \frac{128 e^2 3^{\alpha}}{C\sigma_{\min}^2} \rpp^{\frac{1}{(1-C\tilde{\lambda}\beta_T + \beta_{\epsilon})\alpha}}, \lpp \frac{128 e^23^{\alpha}c_{fb}}{C\sigma_{\min}^2} 2^{\frac{\beta_{\epsilon}\alpha}{2}} \lpp \frac{4\tau}{(2-\beta_{sg})\beta_{\epsilon}\alpha} \rpp^{\frac{2}{2-\beta_{sg}}} \rpp^{\frac{2}{\beta_{\epsilon}\alpha}}, b_0^{-\frac{1}{\lppp \frac{1}{2}-\beta_B\rppp\alpha}} \Big\},\\
\end{array}
$}
$$
where $c_{fb}$ is the constant defined in Lemma \ref{lemma:fbound} and $\tau$ is some positive constant. We assume that $\tiln \ge \max \{ \cnlemma{8}, \cnlemma{12}, \cnlemma{13}\}$. We additionally require 
$$
\tiln \ge \lpp \frac{2\beta_T\alpha}{1- (2-C\tilde{\lambda}\beta_T + \beta_{\epsilon})\alpha} \rpp^{\frac{2}{1- (2-C\tilde{\lambda}\beta_T + \beta_{\epsilon})\alpha}}
$$
such that $\tilde{T} \leq \tiln$. For any $\tiln \leq n \leq 2\tiln - \tilde{T}$, we suppose that $F(\bdtheta_n) - \fmin \leq c_{fb} \lppp \tau \log n\rppp^{\frac{2}{2-\beta_{sg}}}$. Then we have
$$
\PP_n \lppp \bdtheta_l \notin \rgo, n \leq l \leq n + \tilde{T} \rppp \leq \frac{1}{2}.
$$
    \label{lemma:probgood}
\end{lemma}

\begin{proof}
Let's firstly present the simplified implications of the sophisticated requirements on $\tiln$.
\begin{enumerate}
    \itemequationnormal[eq:pgsimple1]{}{$$
    \tiln \ge \exp \lpp \frac{\sigma_{\min}^2}{8e^23^{\alpha}\beta_T \alpha } \rpp
    \Longrightarrow 
    \frac{C\sigma_{\min}^2}{32e^23^{\alpha}} \frac{1}{\log \lppp \tiln^{\beta_T\alpha}\rppp} \leq \frac{C}{4}.
    $$}

    \itemequationnormal[eq:pgsimple2]{}{$$
    \tiln \ge \lpp \frac{128 e^2 3^{\alpha}}{C\sigma_{\min}^2} \rpp^{\frac{1}{(1-C\tilde{\lambda}\beta_T + \beta_{\epsilon})\alpha}}
    \Longrightarrow 
    T \leq \frac{1}{4}\tilde{T}.
    $$}

    \itemequationsmall[eq:pgsimple3]{}{$$ 
    \tiln \ge \lpp \frac{128 e^23^{\alpha}c_{fb}}{C\sigma_{\min}^2} 2^{\frac{\beta_{\epsilon}\alpha}{2}} \lpp \frac{4\tau}{(2-\beta_{sg})\beta_{\epsilon}\alpha} \rpp^{\frac{2}{2-\beta_{sg}}} \rpp^{\frac{2}{\beta_{\epsilon}\alpha}}
    \Longrightarrow 
    \frac{32e^23^{\alpha}}{C\sigma_{\min}^2} \frac{2c_{fb}}{\tiln^{\beta_{\epsilon}\alpha}} \lppp \tau \log (2\tiln) \rppp^{\frac{2}{2-\beta_{sg}}} \leq \frac{1}{2}.
    $$}

    \itemequationnormal[eq:pgsimple4]{}{$$ 
    \tiln \ge b_0^{-\frac{1}{\lppp \frac{1}{2}-\beta_B\rppp\alpha}}
    \Longrightarrow 
    B_0 \leq b_0.
    $$}
\end{enumerate}
We let $\tsigma_0 = n$,
$$
\tsigma_{i+1} = \left\{ \begin{array}{ll}
\tsigma_i + T, & \text{if } \bdtheta_{\tsigma_i} \in \rgb,\\

\tsigma_i + 1, & \text{if } \bdtheta_{\tsigma_i} \notin \rgb,
\end{array} \right.
$$
for $i\in \ZZ_+$. We also let
$$
K = \min \{ k : \tsigma_k \ge n + \tilde{T} - 2T\},
$$
and
$$
M = \tsigma_K.
$$
It is not hard to see that $M$ is a stopping time. For simplicity, we let
$$
S_{n,i} \triangleq \lwww \bdtheta_j \notin \rgo,n\leq j \leq \tsigma_i \rwww.
$$
For completeness, we suppose $\yis{-1}\equiv 1$. We also let $\bar{F}(\cdot) \triangleq F(\cdot) - \fmin$,
$$
\tilde{B}_4 \triangleq \min \lww \frac{C}{4} \tiln^{-(2-2\beta_B)\alpha} , \frac{C\sigma_{\min}^2}{32e^23^{\alpha}}\tiln^{-(2-C\tilde{\lambda}\beta_T)\alpha} \frac{1}{\log \lppp \tiln^{\beta_T\alpha}\rppp} \rww.
$$
As $2\beta_B = C\tilde{\lambda}\beta_T$, based on (\ref{eq:pgsimple1}),
$$
\tilde{B}_4 = \frac{C\sigma_{\min}^2}{32e^23^{\alpha}}\tiln^{-(2-C\tilde{\lambda}\beta_T)\alpha} \frac{1}{\log \lppp \tiln^{\beta_T\alpha}\rppp} = B_4.
$$
Then we have
\begin{equation}
\arraycolsep=1.1pt\def\arraystretch{1.5} 
\begin{array}{rl}
& \EE_n \lppp \barf(\bdtheta_{\tsigma_{K+1}}) \yis{K} \rppp - \EE_n \lppp \barf(\bdtheta_{\tsigma_0})\yis{-1} \rppp\\

= & \ssum{i=0}{\infty} \EE_n \lppp \lppp \barf(\bdtheta_{\tsigma_{i+1}}) \yis{i} - \barf(\bdtheta_{\tsigma_{i}}) \yis{i-1} \rppp \daone \{ i \leq K\} \rppp\\

= & \ssum{i=0}{\infty} \ssum{j=n}{\infty} \EE_n \lppp \lppp \barf(\bdtheta_{\tsigma_{i+1}}) \yis{i} - \barf(\bdtheta_{\tsigma_{i}}) \yis{i-1} \rppp \daone \{ i \leq K\} \daone \{ \tsigma_i = j\} \rppp\\

= & \ssum{i=0}{\infty} \ssum{j=n}{\infty} \EE_n \lppp \lppp \barf(\bdtheta_{\tsigma_{i+1}}) \yis{i} - \barf(\bdtheta_{\tsigma_{i}}) \yis{i-1} \rppp \daone \{ j \leq M\} \daone \{ \tsigma_i = j\} \rppp\\

= & \ssum{i=0}{\infty} \ssum{j=n}{\infty} \EE_n \lppp \lppp ( F(\bdtheta_{\tsigma_{i+1}}) - F(\bdtheta_{\tsigma_{i}})) \yis{i} - \barf(\bdtheta_{\tsigma_{i}}) \lppp\yis{i-1} - \yis{i}\rppp \rppp \daone \{ j \leq M\} \daone \{ \tsigma_i = j\} \rppp\\

\leq & \ssum{i=0}{\infty} \ssum{j=n}{\infty} \EE_n \lppp ( F(\bdtheta_{\tsigma_{i+1}}) - F(\bdtheta_{\tsigma_{i}})) \yis{i}   \daone \{ j \leq M\} \daone \{ \tsigma_i = j\} \rppp\\

= & \ssum{i=0}{\infty} \ssum{j=n}{\infty} \EE_n \lppp ( F(\bdtheta_{v(j)}) - F(\bdtheta_{j})) \daone\{ \bdtheta_l \notin \rgo, n \leq l \leq j\}   \daone \{ j \leq M\} \daone \{ \tsigma_i = j\} \rppp\\

= & \ssum{i=0}{\infty} \ssum{j=n}{\infty} \EE_n \lppp \EE_j( F(\bdtheta_{v(j)}) - F(\bdtheta_{j})) \daone\{ \bdtheta_l \notin \rgo, n \leq l \leq j\}   \daone \{ j \leq M\} \daone \{ \tsigma_i = j\} \rppp\\

\leq & -\tilde{B}_4 \ssum{i=0}{\infty} \ssum{j=n}{\infty} \EE_n \lppp v(j) \daone\{ \bdtheta_l \notin \rgo, n \leq l \leq j\}   \daone \{ j \leq M\} \daone \{ \tsigma_i = j\} \rppp\\

= & -B_4 \ssum{i=0}{\infty} \ssum{j=n}{\infty} \EE_n \lppp (\tsigma_{i+1} -\tsigma_i) \yis{i} \daone\{i\leq K\} \daone \{ \tsigma_i = j\} \rppp\\

= & -B_4 \ssum{i=0}{\infty}\EE_n \lppp (\tsigma_{i+1} -\tsigma_i) \yis{i} \daone\{i\leq K\} \rppp\\

\leq & -B_4 \ssum{i=0}{\infty}\EE_n \lppp (\tsigma_{i+1} -\tsigma_i) \daone \{ \bdtheta_l \notin \rgo, n \leq l \leq n + \tilde{T}\} \daone\{i\leq K\} \rppp\\

= & -B_4 \EE_n\lppp(\tsigma_{K+1} - \tsigma_0)\daone \{ \bdtheta_l \notin \rgo, n \leq l \leq n + \tilde{T}\}\rppp\\

\leq & -B_4 (\tilde{T}-2T) \PP_n \lppp \bdtheta_l \notin \rgo, n \leq l \leq n + \tilde{T} \rppp,\\ 
\end{array}
\label{eq:probgood1}
\end{equation}
where the 5th step is based on the positivity of $\barf$, the 7th step is due to the fact that $\daone\{ \bdtheta_l \notin \rgo, n \leq l \leq j\}   \daone \{ j \leq M\} \daone \{ \tsigma_i = j\} \in \fff_j$ and the 8th step is based on condition \ref{cm_4}, (\ref{eq:pgsimple4}), Lemma \ref{lemma:biggradientexp} and \ref{lemma:baddecrease}. Based on (\ref{eq:probgood1}), we have
\begin{equation}
\arraycolsep=1.1pt\def\arraystretch{1.5} 
\begin{array}{rl}
\PP_n \lppp \bdtheta_l \notin \rgo, n \leq l \leq n + \tilde{T} \rppp & \leq \frac{1}{B_4 (\tilde{T}-2T)} \lpp  \EE_n \lppp \barf(\bdtheta_{\tsigma_0})\yis{-1} \rppp - \EE_n\lppp \barf(\bdtheta_{\tsigma_{K+1}}) \yis{K} \rppp \rpp\\

& \leq \frac{1}{B_4 (\tilde{T}-2T)} \EE_n \lppp \barf(\bdtheta_{\tsigma_0})\yis{-1} \rppp  = \frac{1}{B_4 (\Tilde{T}-2T)}  \barf(\bdtheta_n)\\

& \leq \frac{c_{fb}}{B_4 (\tilde{T}-2T)} \lppp \tau \log n \rppp^{\frac{2}{2-\beta_{sg}}}\\

& \leq \frac{2c_{fb}}{B_4 \tilde{T}} \lppp \tau \log n \rppp^{\frac{2}{2-\beta_{sg}}}\\

& = \frac{32e^23^{\alpha}}{C\sigma_{\min}^2} \frac{2c_{fb}}{\tiln^{\beta_{\epsilon}\alpha}} \lppp \tau \log (2\tiln) \rppp^{\frac{2}{2-\beta_{sg}}}\\

& \leq \frac{1}{2},\\
\end{array}
\label{eq:probgood2}
\end{equation}
where the 5th step is based on (\ref{eq:pgsimple2}) and the last step is based on (\ref{eq:pgsimple3}).
\end{proof}


\begin{lemma}
Suppose that conditions in Lemma \ref{lemma:probgood} hold. In addition, we assume that $\tau \ge 1$. Then, we have
$$
\PP \lppp \bdtheta_n \notin \rgo , \tiln \leq n \leq 2\tiln \rppp = O \lppp \tiln^{1-(2-C\tilde{\lambda}\beta_T + \beta_{\epsilon})\alpha - \tau} \rppp
$$
when $\tiln$ is greater than some constant depending on $\beta_{\epsilon}$ and $\tau$. Particularly, we can let $\beta_{\epsilon}=\frac{1}{2} \lppp \frac{1}{\alpha} + C\tilde{\lambda}\beta_T - 2 \rppp$ and $\tau=\frac{1}{2} - \lppp 1-\frac{1}{2}C\tilde{\lambda}\beta_T \rppp\alpha + 6\alpha$, we have
$$
\PP \lppp \bdtheta_n \notin \rgo , \tiln \leq n \leq 2\tiln \rppp = O\lppp \tiln^{-6\alpha}\rppp.
$$
    \label{lemma:entergood}
\end{lemma}

\begin{proof}
For simplicity we temporarily let $\bar{F}(\cdot) \triangleq F(\cdot) - \fmin$, $\varphi(j) \triangleq c_{fb} \lppp \tau \log (\tiln +j \tilde{T}) \rppp^{\frac{2}{2-\beta_{sg}}}$, where $c_{fb}$ is a constant defined in Lemma \ref{lemma:fbound}. For any $k\in \ZZ_+$, if $k\tilde{T} \leq \tiln$, we have
\begin{equation}
\resizebox{.92\hsize}{!}{$
\arraycolsep=1.1pt\def\arraystretch{1.5} 
\begin{array}{rl}
& \PP \lpp \lwww \bdtheta_n \notin \rgo , \tiln \leq n \leq \tiln + k \tilde{T} \rwww \cap \lwww \barf(\bdtheta_{\tiln + j\tilde{T}}) \leq \varphi(j), 0 \leq j \leq k \rwww \rpp\\

= & \pprod{j=0}{k-1} \PP \lpp \lwww \bdtheta_n \notin \rgo , \tiln + j \tilde{T} \leq n \leq \tiln + (j+1) \tilde{T} \rwww \cap \lwww \barf (\bdtheta_{\tiln + (j+1)\tilde{T}}) \leq \varphi(j+1) \rwww\\

& \ \ \ \ \ \ \ \ \Big| \bigcap\limits_{i=0}\limits^{j-1} \lppp \lwww \bdtheta_n \notin \rgo , \tiln + i \tilde{T} \leq n \leq \tiln + (i+1) \tilde{T} \rwww \cap \lwww \barf (\bdtheta_{\tiln + (i+1)\tilde{T}}) \leq \varphi(i+1) \rwww \rppp \cap \lwww \barf (\bdtheta_{\tiln }) \leq \varphi(0) \rwww \rpp\\

\leq & \pprod{j=0}{k-1} \PP \lpp \lwww \bdtheta_n \notin \rgo , \tiln + j \tilde{T} \leq n \leq \tiln + (j+1) \tilde{T} \rwww \\

& \ \ \ \ \ \ \ \ \Big| \bigcap\limits_{i=0}\limits^{j-1} \lppp \lwww \bdtheta_n \notin \rgo , \tiln + i \tilde{T} \leq n \leq \tiln + (i+1) \tilde{T} \rwww \cap \lwww \barf (\bdtheta_{\tiln + (i+1)\tilde{T}}) \leq \varphi(i+1) \rwww \rppp \cap \lwww \barf (\bdtheta_{\tiln }) \leq \varphi(0) \rwww \rpp\\

\leq & \frac{1}{2^k},\\
\end{array}
\label{eq:eg1}
$}
\end{equation}
where the last step is based on Lemma \ref{lemma:probgood}. If we let $k = \lfloor \frac{\tiln}{\tilde{T}} \rfloor$, we have
$$
\PP \lpp \lwww \bdtheta_n \notin \rgo , \tiln \leq n \leq \tiln + \lfloor \frac{\tiln}{\tilde{T}} \rfloor \tilde{T} \rwww \cap \lwww \barf(\bdtheta_{\tiln + j\tilde{T}}) \leq \varphi(j), 0 \leq j \leq \lfloor \frac{\tiln}{\tilde{T}} \rfloor \rwww \rpp \leq \lpp \frac{1}{2}\rpp^{\lfloor \frac{\tiln}{\tilde{T}} \rfloor}.
$$
Further,
\begin{equation}
\arraycolsep=1.1pt\def\arraystretch{1.5} 
\begin{array}{rl}
& \PP \lppp \bdtheta_n \notin \rgo , \tiln \leq n \leq 2\tiln \rppp\\

\leq & \PP \lpp \lwww \bdtheta_n \notin \rgo , \tiln \leq n \leq 2\tiln \rwww \cap \lww \barf(\bdtheta_{\tiln+j\tilde{T}}) \leq \varphi(j), 0\leq j \leq \lfloor \frac{\tiln}{\tilde{T}} \rfloor \rww \rpp\\

& + \PP \lpp \exists 0\leq j \leq \lfloor \frac{\tiln}{\tilde{T}} \rfloor, \barf(\bdtheta_{\tiln+j\tilde{T}}) > \varphi(j) \rpp\\

\leq & \lpp \frac{1}{2} \rpp^{\lfloor \frac{\tiln}{\tilde{T}} \rfloor} + \ssum{j=0}{\lfloor \frac{\tiln}{\tilde{T}} \rfloor} \PP \lppp \barf(\bdtheta_{\tiln+j\tilde{T}}) > \varphi(j) \rppp\\

\leq & \lpp \frac{1}{2} \rpp^{\lfloor \frac{\tiln}{\tilde{T}} \rfloor} + \ssum{j=0}{\lfloor \frac{\tiln}{\tilde{T}} \rfloor} \tiln^{-\tau}\\

= & \lpp \frac{1}{2} \rpp^{\lfloor \frac{\tiln}{\tilde{T}} \rfloor} + \lpp \lfloor \frac{\tiln}{\tilde{T}} \rfloor + 1 \rpp \tiln^{-\tau}\\

= & O \lppp \tiln^{1-(2-C\tilde{\lambda}\beta_T + \beta_{\epsilon})\alpha - \tau} \rppp,
\end{array}
\label{eq:eg2}
\end{equation}
where according to Lemma \ref{lemma:fbound}, the 3rd step is true when $\tiln \ge \cnlemma{6}$.

\end{proof}
\begin{lemma}
Under conditions \ref{cm_1} and \ref{cm_2},
$$
\PP \lpp \lwww \lvvv \bdtheta_n - \opt \rvvv \leq r_{good}^L, \forall n \ge N\rwww \cap \lww \limk \bdtheta_k = \opt \rww^c \rpp=0.
$$
Likewise, we have
$$
\PP \lpp \lwww \lvvv \bbdtheta_n - \opt \rvvv \leq r_{good}^L, \forall n \ge N\rwww \cap \lww \limk \bbdtheta_k = \opt \rww^c \rpp=0.
$$
    \label{lemma:mustconverge}
\end{lemma}

\begin{proof}
In the present proof, for $n\ge N$, we let
$$
\bddelta_n \triangleq \bdtheta_k - \opt.
$$
Based on (\ref{eq:thm3_5}), we know there exists a positive integer $N_0$ such that for $n\ge N_0$, on $\lwww \bdtheta_n \in \rgo^L(\opt)\rwww$,
$$
\EE_n \| \bddelta_{n+1} \|^4 \leq \lpp1-\frac{1}{2} \tlambda \gamma_{n+1} \rpp \| \bddelta_n\|^4 + c_1 \gamma_{n+1}^3,
$$
where
$$
c_1 \triangleq 1024 (\csg)^2 \lpp \frac{36}{\tlambda}+3\rpp.
$$
Let $N' = N \vee N_0$. For $n\ge N'$, for convenience, we let
$$
S_n \triangleq \lwww \lvvv \bdtheta_k - \opt \rvvv \leq r_{good}^L, \forall n \ge k \ge N'\rwww.
$$
Then, for $n\ge N'+1$,
\begin{equation*}
\resizebox{.92\hsize}{!}{$
\arraycolsep=1.1pt\def\arraystretch{1.5} 
\begin{array}{rl}
& \EE \lsss \| \bddelta_n \|^4 \daone_{S_{n-1}} \rsss\\

\leq & \lpp1-\frac{1}{2} \tlambda \gamma_{n} \rpp \EE \lsss \| \bddelta_{n-1}\|^4 \daone_{S_{n-2}} \rsss + c_1 \gamma_n^3\\

& \cdots\\

\leq & \lpp \pprod{k=N'+1}{n} \lpp 1-\frac{1}{2}\tlambda \gamma_k \rpp \rpp \lppp r_{good}^L \rppp^4 + c_1 \ssum{k=N'+1}{n} \lpp \gamma_k^3 \pprod{t=k+1}{n} \lpp 1-\frac{1}{2}\tlambda \gamma_t \rpp \rpp\\

\leq & \exp \lpp -\frac{1}{2} C\tlambda \ssum{k=N'+1}{n} k^{-\alpha} \rpp \lppp r_{good}^L\rppp^4 + c_1 C^3 \ssum{k=N'+1}{n} k^{-3\alpha} \exp \lpp -\frac{1}{2} C\tlambda \ssum{t=k+1}{n}t^{-\alpha} \rpp\\

\leq &\exp \lpp -\frac{1}{2} C\tlambda (n-N')n^{-\alpha}\rpp \lppp r_{good}^L\rppp^4 + c_1C^3 \ssum{k=N'+1}{n} k^{-3\alpha}\exp \lpp \frac{1}{2}\frac{C\tlambda}{1-\alpha} \lppp(k+1)^{1-\alpha} - (n+1)^{1-\alpha} \rppp\rpp\\

\leq & \exp \lppp -C\tlambda 2^{\alpha-2} \lppp (n-N')^{\alpha+1} + (n-N') (N')^{\alpha} \rppp \rppp \lppp r_{good}^L\rppp^4\\

& + 2^{3\alpha} c_1 C^3 \ssum{k=N'+2}{n+1} k^{-3\alpha}\exp \lpp \frac{1}{2}\frac{C\tlambda}{1-\alpha} \lppp k^{1-\alpha} - (n+1)^{1-\alpha} \rppp\rpp\\

\leq & c_2(N') (n-N')^{-2\alpha} + 2^{3\alpha} c_1 C^3 \frac{4}{C\tlambda} (n+2)^{-2\alpha} \exp \lpp \frac{1}{2} \frac{C\lambda}{1-\alpha} \lppp (n+2)^{1-\alpha} - (n+2)^{1-\alpha}\rppp \rpp\\

= & O \lppp (n-N')^{-2\alpha}\rppp,\\
\end{array}
\label{eq:lmmc1}
$}
\end{equation*}
where the 5th step is based on the Jensen's inequality and the 6th step is based on Lemma \ref{lemma:int1}. Based on the above result, we know that, for any $j\in \ZZ_+$,
\begin{equation*}
\arraycolsep=1.1pt\def\arraystretch{1.5} 
\begin{array}{rl}
& \ssum{n=1}{\infty} \PP \lpp \lww \| \bddelta_n \| > \frac{1}{j}\rww \cap \lwww \lvvv \bdtheta_k - \opt \rvvv \leq r_{good}^L, \forall k \ge N\rwww \rpp \\

\leq & N' + \ssum{n=N'+1}{\infty} \PP \lpp \lww \| \bddelta_n \| > \frac{1}{j}\rww \cap \lwww \lvvv \bdtheta_k - \opt \rvvv \leq r_{good}^L, \forall k \ge N'\rwww \rpp \\

\leq & N' + j^4 \ssum{n=N'+1}{\infty} \EE \lss \| \bddelta_n \|^4 \daone \lwww \lvvv \bdtheta_k - \opt \rvvv \leq r_{good}^L, \forall k \ge N'\rwww  \rss\\

\leq & N' + j^4 \ssum{n=N'+1}{\infty} \EE \lss \| \bddelta_n \|^4 \daone_{S_{n-1}} \rss < \infty.\\
\end{array}
\label{eq:lmmc2}
\end{equation*}
Therefore, based on the Borel-Cantelli lemma, we have
$$
\PP \lpp \lpp \bigcap\limits_{m=1}\limits^{\infty} \bigcup\limits_{n=m}\limits^{\infty} \lww \| \bddelta_n\| > \frac{1}{j} \rww \rpp \cap \lwww \lvvv \bdtheta_k - \opt \rvvv \leq r_{good}^L, \forall k \ge N\rwww \rpp=0.
$$
Further, we have
$$
\arraycolsep=1.1pt\def\arraystretch{1.5} 
\begin{array}{rl}
& \PP \lpp \lwww \lvvv \bdtheta_n - \opt \rvvv \leq r_{good}^L, \forall n \ge N\rwww \cap \lww \limk \bdtheta_k = \opt \rww^c \rpp\\

= & \lim\limits_{j\rightarrow \infty} \PP \lpp \lpp \bigcap\limits_{m=1}\limits^{\infty} \bigcup\limits_{n=m}\limits^{\infty} \lww \| \bddelta_n\| > \frac{1}{j} \rww \rpp \cap \lwww \lvvv \bdtheta_k - \opt \rvvv \leq r_{good}^L, \forall k \ge N\rwww \rpp=0.
\end{array}
$$

\end{proof}

\begin{lemma}
Under conditions \ref{cm_1}, \ref{cm_2} and \ref{cm_4}, for any $\opt \in \Theta^{opt}$, $\tau >0$,
$$
\PP \lpp \lwww \lvvv \bdtheta_n - \opt \rvvv \leq \sqrt{3}r_{good}, \forall n \ge N\rwww^c \cap \lww \limk \bdtheta_k = \opt \rww \rpp = O\lppp N^{-\tau}\rppp.
$$
Likewise, we have
$$
\PP \lpp \lwww \lvvv \bbdtheta_n - \opt \rvvv \leq \sqrt{3}r_{good}, \forall n \ge N\rwww^c \cap \lww \limk \bbdtheta_k = \opt \rww \rpp = O\lppp N^{-\tau}\rppp.
$$
    \label{lemma:mustbound}
\end{lemma}

\begin{proof}
For simplicity, we denote
$$
A_N(\bdtheta, r) \triangleq \lwww \lvvv \bdtheta_n - \bdtheta \rvvv \leq r, \forall n \ge N \rwww, \bdtheta \in \Theta, r > 0, N \in \ZZ_+,
$$
and
$$
B_{good, N}(\bdtheta) \triangleq \lww \exists \frac{N}{2} \leq n \leq N,\bdtheta_n \in \rgo\lppp \bdtheta \rppp \lww, \bdtheta \in \Theta,  N \in \ZZ_+.
$$
For any $\tau >0$,
\begin{equation*}
\arraycolsep=1.1pt\def\arraystretch{1.5} 
\begin{array}{rl}
& \PP \lpp A_N(\opt, \sqrt{3} r_{good})^c \cap \lww \limk \bdtheta_k = \opt \rww \rpp\\

\leq & \PP \lpp B_{good, N}(\opt) \cap A_N(\opt, \sqrt{3} r_{good})^c\rpp + \PP \lpp \sopt \cap \lww \forall \frac{N}{2} \leq n \leq N,\bdtheta_n \notin \rgo\lppp \opt \rppp \rww \rpp\\

= & O \lppp N^{-\tau} \rppp + \PP \lpp \sopt \cap \lww \forall \frac{N}{2} \leq n \leq N,\bdtheta_n \notin \rgo\lppp \opt \rppp \rww \rpp\\

\leq & O \lppp N^{-\tau} \rppp + \PP \lpp \forall \frac{N}{2} \leq n \leq N,\bdtheta_n \notin \rgo \rpp + \sum\limits_{\bdtheta' \in \Theta^{opt},\bdtheta' \ne \opt} \PP \lpp B_{good, N}(\bdtheta')\cap \sopt \rpp\\

= &  O \lppp N^{-\tau} \rppp + \sum\limits_{\bdtheta' \in \Theta^{opt},\bdtheta' \ne \opt} \PP \lpp B_{good, N}(\bdtheta')\cap \sopt \rpp\\

\leq &  O \lppp N^{-\tau} \rppp + \sum\limits_{\bdtheta' \in \Theta^{opt},\bdtheta' \ne \opt}  \PP \lpp A_N(\bdtheta', \sqrt{3} r_{good}) \cap \sopt \big| B_{good,N}(\bdtheta')\rpp \PP \lpp B_{good,N}(\bdtheta')\rpp\\

& + \sum\limits_{\bdtheta' \in \Theta^{opt},\bdtheta' \ne \opt}  \PP \lpp A_N(\bdtheta', \sqrt{3} r_{good})^c  \big| B_{good,N}(\bdtheta')\rpp \PP \lpp B_{good,N}(\bdtheta')\rpp\\

= & O \lppp N^{-\tau} \rppp + \sum\limits_{\bdtheta' \in \Theta^{opt},\bdtheta' \ne \opt}  \PP \lpp A_N(\bdtheta', \sqrt{3} r_{good}) \cap \sopt \big| B_{good,N}(\bdtheta')\rpp \PP \lpp B_{good,N}(\bdtheta')\rpp\\

= & O \lppp N^{-\tau} \rppp,\\
\end{array}
\label{eq:lmmb1}
\end{equation*}
where the 2nd step is based on Lemma \ref{lemma:trap}, the 4th step is based on Lemma \ref{lemma:entergood}, the 6th step is based on Lemma \ref{lemma:trap} and the last step is based on Lemma \ref{lemma:mustconverge}.
\end{proof}

\section{Results Supporting Theorem \ref{thm:covmat}}
\label{sec:supportcovmat}

\begin{corollary}
Under conditions given in Theorem \ref{thm:secbound},
$$
\EE \lppp \| \bbdtheta_N - \opt\|^p \daone \lwww \limk \bbdtheta_k = \opt\rwww \rppp = O\lppp N^{-\frac{p}{2}\alpha}\rppp,\ p=2,4.
$$
    \label{cor:momentbound}
\end{corollary}

\subsection{Proof of Theorem \ref{thm:covmat}}
\begin{proof}
For any fixed $\nu \in (0,1)$, we let $N' \triangleq  N^{\nu}$ and 
$$
\check{\bdtheta}_N \triangleq  \frac{1}{N - N'}\ssum{n=N'+1}{N}\bdtheta_n,\ \check{\bdtheta}_N^* \triangleq  \frac{1}{N - N'}\ssum{n=N'+1}{N}\bdtheta_n^*.
$$
For simplicity, we let
$$
S_{opt} \triangleq  \lww \limk \bdtheta_k = {\opt}\rww,\ S_{opt}^* \triangleq  \lww \limk \bbdtheta_k = {\opt}\rww,
$$
$$
\snopt \triangleq \lww \exists \frac{N'}{4} \leq n \leq \frac{N'}{2}, \bdtheta_n \in \rgo(\opt)\rww,\ \ssnopt \triangleq \lww \exists \frac{N'}{4} \leq n \leq \frac{N'}{2}, \bbdtheta_n \in \rgo(\opt)\rww.
$$

Our first goal is to derive a bound on
$$
\EE \lvv \EE_* \lpp \lpp\lww \sqrt{N-N'} A (\check{\bdtheta}_N^* - \check{\bdtheta}_N) \rww\lww\sqrt{N-N'} A (\check{\bdtheta}_N^* - \check{\bdtheta}_N) \rww^T  - U\rpp \daone_{S_{opt}} \daone_{S_{opt}^*} \rpp \rvv,
$$
which is a close approximation to our true target. To facilitate our analysis, we consider the following decomposition.
For $n\in \NN$, we denote
$$
\arraycolsep=1.1pt\def\arraystretch{1.5} 
\begin{array}{ccl}
\bdrho_n  &\triangleq& \nablaf{n} - A(\bdtheta_n - \opt),\\

\bdxi_{n+1} &\triangleq& \nablaf{n} - \nabla f(\bdtheta_n; \bdy_{n+1}),\\

\bdrho_n^*  &\triangleq& \nabla F(\bdtheta_n^*) - A(\bdtheta_n^* - \opt),\\

\bdxi_{n+1}^* &\triangleq& \nabla F(\bdtheta_n^*) - w_{n+1}\nabla f(\bdtheta_n^*; \bdy_{n+1}).
\end{array}
$$

Based on the update rule given in (\ref{mainupdate}), we have
$$
\frac{\bdtheta_n - \opt{}}{\gamma_{n+1}} - \frac{\bdtheta_{n+1} - \opt{}}{\gamma_{n+1}} = \nabla f(\bdtheta_n; \bdy_{n+1}) = \bdrho_n - \bdxi_{n+1} + A(\bdtheta_n - \opt{}).
$$
Therefore, we have
\begin{equation}
 A(\bdtheta_n - \opt{}) = 
\frac{\bdtheta_n - \opt{}}{\gamma_{n+1}} - \frac{\bdtheta_{n+1} - \opt{}}{\gamma_{n+1}} - \bdrho_n + \bdxi_{n+1}.
    \label{eq:decomp1}
\end{equation}
Summing up (\ref{eq:decomp1}) from $N'+1$ to $N$, we have
$$
\arraycolsep=1.1pt\def\arraystretch{1.5} 
\begin{array}{rl}
     &  (N-N') A(\check{\bdtheta}_N - \opt{})\\
     
    = & \sum\limits_{n=N'+1}\limits^{N} \frac{\bdtheta_n - \opt{}}{\gamma_{n+1}} -  \sum\limits_{n=N'+1}\limits^{N} \frac{\bdtheta_{n+1} - \opt{}}{\gamma_{n+1}} - \sum\limits_{n=N'+1}\limits^{N} \bdrho_n + \sum\limits_{n=N'+1}\limits^{N} \bdxi_{n+1}\\
    
    = & \sum\limits_{n=N'+1}\limits^{N} \frac{\bdtheta_n - \opt{}}{\gamma_{n+1}} -  \sum\limits_{n=N'+2}\limits^{N+1} \frac{\bdtheta_{n} - \opt{}}{\gamma_{n}} - \sum\limits_{n=N'+1}\limits^{N} \bdrho_n + \sum\limits_{n=N'+1}\limits^{N} \bdxi_{n+1}\\
    
    = & \frac{\bdtheta_{N'+1} - \opt}{\gamma_{N'+2}} - \frac{\bdtheta_{N+1} - \opt}{\gamma_{N+1}} + \sum\limits_{n=N'+2}\limits^{N} \left(\frac{1}{\gamma_{n+1}} - \frac{1}{\gamma_n} \right)(\bdtheta_n - \opt)- \sum\limits_{n=N'+1}\limits^{N} \bdrho_n + \sum\limits_{n=N'+1}\limits^{N} \bdxi_{n+1}\\
    
    \triangleq & E_{N,1} - E_{N,2} + E_{N,3} - E_{N,4} + E_{N,5}.
\end{array}
$$
Likewise, we have
$$
\arraycolsep=1.1pt\def\arraystretch{1.5} 
\begin{array}{rl}
     &  (N-N') A(\check{\bdtheta}_N^* - \opt{})\\
     
    = & \frac{\bbdtheta_{N'+1} - \opt}{\gamma_{N'+2}} - \frac{\bdtheta_{N+1}^* - \opt}{\gamma_{N+1}} + \sum\limits_{n=N'+2}\limits^{N} \left(\frac{1}{\gamma_{n+1}} - \frac{1}{\gamma_n} \right)(\bdtheta_n^* - \opt)- \sum\limits_{n=N'+1}\limits^{N} \bdrho_n^* + \sum\limits_{n=N'+1}\limits^{N} \bdxi_{n+1}^*\\
    
    \triangleq & E_{N,1}^* - E_{N,2}^* + E_{N,3}^* - E_{N,4}^* + E_{N,5}^*.
\end{array}
$$
As a result, we have the following decomposition
\begin{equation}
(N-N') A(\check{\bdtheta}_N^* - \check{\bdtheta}_N) = \Delta_{N,1} - \Delta_{N,2} + \Delta_{N,3} - \Delta_{N,4} + \Delta_{N,5},
    \label{eq:decomp2}
\end{equation}
where
$$
\arraycolsep=1.1pt\def\arraystretch{1.5} 
\begin{array}{ccl}
\Delta_{N,1} &=& E_{N,1}^* - E_{N,1} = \frac{\bbdtheta_{N'+1} - \bdtheta_{N'+1}}{\gamma_{N'+2}},\\

\Delta_{N,2} &=& E_{N,2}^* - E_{N,2} = \frac{\bbdtheta_{N+1} - \bdtheta_{N+1}}{\gamma_{N+1}},\\

\Delta_{N,3} &=& E_{N,3}^* - E_{N,3} = \sum\limits_{n=N'+2}\limits^{N}\left( \frac{1}{\gamma_{n+1}} - \frac{1}{\gamma_n}\right)(\bbdtheta_n - \bdtheta_n),\\

\Delta_{N,4} &=& E_{N,4}^* - E_{N,4} = \sum\limits_{n=N'+1}\limits^{N} (\bdrho^*_n - \bdrho_n),\\

\Delta_{N,5} &=& E_{N,5}^* - E_{N,5} = \sum\limits_{n=N'+1}\limits^{N} (\bdxi^*_{n+1} - \bdxi_{n+1}).\\
\end{array}
$$

Now, based on the decomposition given in (\ref{eq:decomp2}), we have

\begin{equation}
\resizebox{.92\hsize}{!}{$
\arraycolsep=1.1pt\def\arraystretch{1.5} 
\begin{array}{rl}
& \EE \lvv \EE_* \lppp \lppp\lwww \sqrt{N-N'} A (\check{\bdtheta}_N^* - \check{\bdtheta}_N) \rwww\lwww\sqrt{N-N'} A (\check{\bdtheta}_N^* - \check{\bdtheta}_N) \rwww^T  - U\rppp \daone_{S_{opt}} \daone_{S_{opt}^*} \rppp \rvv\\

\leq & \EE \lvv\frac{1}{N-N'}\Delta_{N,1} \Delta_{N,1}^T \daone_{S_{opt}} \daone_{S_{opt}^*}\rvv + \EE\lvv \frac{1}{N-N'}\Delta_{N,2} \Delta_{N,2}^T \daone_{S_{opt}} \daone_{S_{opt}^*}\rvv + \EE\lvv \frac{1}{N-N'}\Delta_{N,3} \Delta_{N,3}^T \daone_{S_{opt}} \daone_{S_{opt}^*} \rvv\\

& + \EE\lvv  \frac{1}{N-N'}\Delta_{N,4} \Delta_{N,4}^T \daone_{S_{opt}} \daone_{S_{opt}^*}\rvv + \EE\lvv \EE_* \lpp \lpp \lpp \frac{1}{N-N'}\Delta_{N,5} \Delta_{N,5}^T \rpp - U \rpp \daone_{S_{opt}} \daone_{S_{opt}^*} \rpp\rvv.
\end{array}
    \label{eq:cp1}
$}
\end{equation}

\noindent\textit{Part 1.} 
We have
\begin{equation}
\arraycolsep=1.1pt\def\arraystretch{1.5} 
\begin{array}{rl}
& \EE \lvv\frac{1}{N-N'}\Delta_{N,1} \Delta_{N,1}^T \daone_{S_{opt}} \daone_{S_{opt}^*}\rvv\\

\leq & \frac{2}{N-N'} \lppp \gamma_{N'+2}\rppp^{-2} \lpp \EE \lpp \lvvv \bbdtheta_{N'+1} - \opt\rvvv^2 \daone_{S_{opt}^*}\rpp +  \EE \lpp \lvvv \bdtheta_{N'+1} - \opt\rvvv^2 \daone_{S_{opt}}\rpp \rpp\\

= & O \lppp N^{-1} (N')^{\alpha} \rppp,\\

= & O \lppp N^{\alpha-1} \rppp,\\
\end{array}
\label{eq:cp2}
\end{equation}
where the last step is based on Theorem \ref{thm:secbound} and Corollary \ref{cor:momentbound}.

\noindent\textit{Part 2.}
We have

\begin{equation}
\arraycolsep=1.1pt\def\arraystretch{1.5} 
\begin{array}{rl}
&  \EE\lvv \frac{1}{N-N'}\Delta_{N,2} \Delta_{N,2}^T \daone_{S_{opt}} \daone_{S_{opt}^*}\rvv\\

= & \EE\lvv \frac{1}{N-N'} \lpp \frac{\bbdtheta_{N+1} - \opt}{\gamma_{N+1}} - \frac{\bdtheta_{N+1} - \opt}{\gamma_{N+1}} \rpp \lpp \frac{\bbdtheta_{N+1} - \opt}{\gamma_{N+1}} - \frac{\bdtheta_{N+1} - \opt}{\gamma_{N+1}} \rpp^T  \daone_{S_{opt}} \daone_{S_{opt}^*} \rvv \\ 

= & \frac{1}{N-N'} \EE\lpp\lvv \frac{\bbdtheta_{N+1} - \opt}{\gamma_{N+1}} - \frac{\bdtheta_{N+1} - \opt}{\gamma_{N+1}} \rvv^2 \daone_{S_{opt}} \daone_{S_{opt}^*}\rpp\\

= & O\lppp N^{2\alpha - 1}\rppp \lpp  \EE\lppp \lvvv \bbdtheta_{N+1} - \opt \rvvv^2 \daone_{S_{opt}^*}\rppp + \EE\lppp\lvvv \bdtheta_{N+1} - \opt \rvvv^2 \daone_{S_{opt}}\rppp \rpp \\

= & O \lppp N^{\alpha - 1 } \rppp,
\end{array}
\label{eq:cp32}
\end{equation}
where the last step is based on Theorem \ref{thm:secbound} and Corollary \ref{cor:momentbound}.

\noindent\textit{Part 3.} We have

\begin{equation}
\resizebox{.92\hsize}{!}{$
\arraycolsep=1.1pt\def\arraystretch{1.5} 
\begin{array}{rl}
& \EE \lpp\lvv \lpp \sum\limits_{n=N'+2}\limits^N \lppp (n+1)^{\alpha} - n^{\alpha}\rppp \lppp \bbdtheta_n - \bdtheta_n\rppp\rpp   \rvv^2 \daone_{S_{opt}} \daone_{S_{opt}^*} \rpp\\

\leq &  \EE\lpp \lpp  \sum\limits_{n=N'+2}\limits^N \lvvv \lppp (n+1)^{\alpha} - n^{\alpha}\rppp \lppp \bbdtheta_n - \bdtheta_n\rppp   \rvvv \rpp^2 \daone_{S_{opt}} \daone_{S_{opt}^*} \rpp\\

\leq &  \EE\lpp \lpp  \sum\limits_{n=N'+2}\limits^N \lvvv \lppp \alpha n^{\alpha-1}\rppp \lppp \bbdtheta_n - \bdtheta_n\rppp   \rvvv \rpp^2 \daone_{S_{opt}} \daone_{S_{opt}^*} \rpp\\

= & \sum\limits_{n=N'+2}\limits^N\alpha^2n^{2\alpha-2}\EE\lppp \lvvv \bbdtheta_n - \bdtheta_n \rvvv^2 \daone_{S_{opt}} \daone_{S_{opt}^*} \rppp + 2\sum\limits_{N'+2\leq n_1<n_2\leq N}\alpha^2 (n_1 n_2)^{\alpha-1} \EE \lppp \lvvv \bbdtheta_{n_1} - \bdtheta_{n_1} \rvvv \lvvv \bbdtheta_{n_2} - \bdtheta_{n_2} \rvvv \daone_{S_{opt}} \daone_{S_{opt}^*} \rppp \\

= & O(1) +  2\sum\limits_{N'+2\leq n_1<n_2\leq N}\alpha^2 (n_1 n_2)^{\alpha-1} \EE \lppp \lvvv \bbdtheta_{n_1} - \bdtheta_{n_1} \rvvv \lvvv \bbdtheta_{n_2} - \bdtheta_{n_2} \rvvv \daone_{S_{opt}} \daone_{S_{opt}^*} \rppp \\

\leq & O(1) + 2\alpha^2 \sum\limits_{N'+2\leq n_1<n_2\leq N} \lpp n_1^{\alpha-1} n_2^{\alpha-1} \lppp \EE \lvvv \bbdtheta_{n_1} - \bdtheta_{n_1} \rvvv^2 \daone_{S_{opt}} \daone_{S_{opt}^*}\rppp^{\frac{1}{2}} \lppp \EE \lvvv \bbdtheta_{n_2} - \bdtheta_{n_2} \rvvv^2 \daone_{S_{opt}} \daone_{S_{opt}^*}\rppp^{\frac{1}{2}} \rpp \\

\leq & O(1) +  2\alpha^2 \lpp \sum\limits_{n=N'+2}\limits^N \lpp n^{\alpha-1} \lppp \EE \lvvv \bbdtheta_n - \bdtheta_n\rvvv^2 \daone_{S_{opt}} \daone_{S_{opt}^*} \rppp^{\frac{1}{2}}\rpp\rpp^2\\

= & O(1) + O(N^{\alpha})\\
= & O(N^{\alpha}),
\end{array}
\label{eq:cp3}
$}
\end{equation}

where the 4th and 7th steps are based on Theorem \ref{thm:secbound} and Corollary \ref{cor:momentbound}. Therefore, we have
\begin{equation}
\resizebox{.92\hsize}{!}{$
\arraycolsep=1.1pt\def\arraystretch{1.5} 
\begin{array}{rl}
&  \EE\lvv \frac{1}{N-N'}\Delta_{N,3} \Delta_{N,3}^T \daone_{S_{opt}} \daone_{S_{opt}^*} \rvv\\

= &  \EE\lvv \frac{1}{N-N'}\lpp \sum\limits_{n=N'+2}\limits^N \lpp \frac{1}{\gamma_{n+1}} - \frac{1}{\gamma_{n}}\rpp \lppp \bbdtheta_n - \bdtheta_n\rppp\rpp  \lpp \sum\limits_{n=N'+2}\limits^N \lpp \frac{1}{\gamma_{n+1}} - \frac{1}{\gamma_{n}}\rpp \lppp \bbdtheta_n - \bdtheta_n\rppp\rpp^T \daone_{S_{opt}} \daone_{S_{opt}^*} \rvv\\  

\leq & \frac{1}{(N-N') C^2} \EE \lpp\lvv \lpp \sum\limits_{n=N'+2}\limits^N \lppp (n+1)^{\alpha} - n^{\alpha}\rppp \lppp \bbdtheta_n - \bdtheta_n\rppp\rpp   \rvv^2 \daone_{S_{opt}} \daone_{S_{opt}^*} \rpp\\

= & O\lppp N^{\alpha - 1 } \rppp, \\
\end{array}
\label{eq:cp4}
$}
\end{equation}
where the last step is based on (\ref{eq:cp3}). 

\noindent\textit{Part 4.} We firstly have
\begin{equation}
\arraycolsep=1.1pt\def\arraystretch{1.5}
\begin{array}{rl}
     &  \EE \lpp \sum\limits_{n=1}\limits^N \lvvv \bdtheta_n - \opt \rvvv^2 \daone_{S_{opt}} \rpp^2\\
     
=     &  \sum\limits_{n=1}\limits^N \EE \lppp\lvvv \bdtheta_n - \opt \rvvv^4 \daone_{S_{opt}} \rppp+ \sum\limits_{1\leq n_1 < n_2 \leq N} \EE\lppp \lvvv \bdtheta_{n_1} - \opt \rvvv^2 \lvvv \bdtheta_{n_2} - \opt \rvvv^2 \daone_{S_{opt}} \rppp\\

= & O(1) + \sum\limits_{1\leq n_1 < n_2 \leq N} \EE\lppp \lvvv \bdtheta_{n_1} - \opt \rvvv^2 \lvvv \bdtheta_{n_2} - \opt \rvvv^2 \daone_{S_{opt}} \rppp\\

\leq & O(1) + \sum\limits_{1\leq n_1 < n_2 \leq N} \lppp\EE\lppp  \lvvv \bdtheta_{n_1} - \opt \rvvv^4 \daone_{S_{opt}}\rppp\rppp^{\frac{1}{2}} \lppp\EE \lppp \lvvv \bdtheta_{n_1} - \opt \rvvv^4 \daone_{S_{opt}}\rppp\rppp^{\frac{1}{2}}\\

\leq & O(1) + \lpp\sum\limits_{n=1}\limits^N \lppp \EE \lppp\lvvv \bdtheta_{n} - \opt \rvvv^4 \daone_{S_{opt}} \rppp\rppp^{\frac{1}{2}} \rpp^2\\

= & O(N^{2-2\alpha}),
\end{array}
    \label{eq:cp6_0}
\end{equation}

where the 2nd and 5th steps are based on Theorem \ref{thm:secbound} and Corollary \ref{cor:momentbound}. Likewise, we have
\begin{equation}
\EE \lpp \sum\limits_{n=1}\limits^N \lvvv \bbdtheta_n - \opt \rvvv^2 \daone_{S_{opt}^*} \rpp^2 = O(N^{2-2\alpha}).
    \label{eq:cpadd1}
\end{equation}
Then,
\begin{equation}
\arraycolsep=1.1pt\def\arraystretch{1.5}
\resizebox{.92\hsize}{!}{$
\begin{array}{rl}
&  \EE\lvv \frac{1}{N-N'}\Delta_{N,4} \Delta_{N,4}^T \daone_{S_{opt}} \daone_{S_{opt}^*}\rvv\\

= & \frac{1}{N-N'} \EE\lpp \lvv \sum\limits_{n=N'+1}\limits^N \lppp \nabla F(\bbdtheta_n) - A(\bbdtheta_n - \opt)\rppp - \sum\limits_{n=N'+1}\limits^N \lppp \nabla F(\bdtheta_n) - A(\bdtheta_n - \opt)\rppp \rvv^2 \daone_{S_{opt}} \daone_{S_{opt}^*}\rpp \\  

\leq &\frac{2}{N-N'} \EE\lvv \sum\limits_{n=N'+1}\limits^N \lppp \nabla F(\bbdtheta_n) - A(\bbdtheta_n - \opt)\rppp \daone_{S_{opt}^*} \rvv^2   + \frac{4}{N} \EE\lvv \sum\limits_{n=N'+1}\limits^N \lppp \nabla F(\bdtheta_n) - A(\bdtheta_n - \opt)\rppp\daone_{S_{opt}} \rvv^2 \\

\leq & \frac{2}{N-N'} \EE\lpp \sum\limits_{n=N'+1}\limits^N C_{hl}\lvvv \bbdtheta_n - \opt\rvvv^2 \daone_{S_{opt}^*}\rpp^2   + \frac{4}{N} \EE\lpp \sum\limits_{n=N'+1}\limits^N C_{hl} \lvvv \bdtheta_n - \opt\rvvv^2 \daone_{S_{opt}} \rpp^2 \\

= & O\lp N^{1-2\alpha}\rp,\\
\end{array}
\label{eq:cp6}
$}
\end{equation}
where the 3rd step relies on the \ref{cm_1} condition and the last step is based on (\ref{eq:cp6_0}) and (\ref{eq:cpadd1}).

\noindent\textit{Part 5.} To better use the property of $\{w_{n}:n\in \ZZ_+\}$, we have the following decomposition
\begin{equation}
\arraycolsep=1.1pt\def\arraystretch{1.5} 
\begin{array}{rl}
&  \EE\lvv \EE_* \lpp \lpp \lpp \frac{1}{N-N'}\Delta_{N,5} \Delta_{N,5}^T \rpp - U \rpp \daone_{S_{opt}} \daone_{S_{opt}^*} \rpp\rvv \\

\leq & \EE\lvv \EE_* \lpp \lpp \lpp \frac{1}{N-N'}\Delta_{N,5} \Delta_{N,5}^T \rpp - U \rpp \daone_{\snopt} \daone_{\ssnopt} \rpp\rvv \\

 & + \EE\lvv \EE_* \lpp \lpp \lpp \frac{1}{N-N'}\Delta_{N,5} \Delta_{N,5}^T \rpp - U \rpp \daone_{S_{opt}\backslash \snopt} \rpp\rvv \\

& + \EE\lvv \EE_* \lpp \lpp \lpp \frac{1}{N-N'}\Delta_{N,5} \Delta_{N,5}^T \rpp - U \rpp \daone_{S_{opt}^*\backslash \ssnopt} \rpp\rvv . \\
\end{array}
\label{eq:thm4add_1}
\end{equation}

By the Taylor's expansion, we have
\begin{equation}
\arraycolsep=1.1pt\def\arraystretch{1.5}
\begin{array}{rl}
& \nabla F(\bbdtheta_n) - w_{n+1}\nabla  f(\bbdtheta_n; \bdy_{n+1})\\

= & \lppp\nabla F(\bdtheta_n) - w_{n+1}\nabla  f(\bdtheta_n; \bdy_{n+1})\rppp + \lppp \nabla^2 F(\bdtheta_n) - w_{n+1} \nabla^2 f(\bdtheta_n; \bdy_{n+1})\rppp \lppp \bbdtheta_n - \bdtheta_n\rppp\\

& + \lppp \nabla^3 F(\bdtheta_n) - w_{n+1} \nabla^3 f(\bdtheta_n; \bdy_{n+1})\rppp \bigotimes^2 \lppp \bbdtheta_n - \bdtheta_n\rppp\\

& + \lppp \nabla^4 F(\bdtheta_n) - w_{n+1} \nabla^4 f(\bdtheta_n; \bdy_{n+1})\rppp \bigotimes^3 \lppp \wwtheta{n} - \bdtheta_n\rppp,\\
\end{array}
\label{eq:part5_1}
\end{equation}
where $\wwtheta{n}$ depends on $\bdtheta_n$, $\bdy_{n+1}$ and $w_{n+1}$ and satisfies that
$$
\lvvv \wwtheta{n} - \bdtheta_n \rvvv \leq \lvvv {\bbdtheta_n} - \bdtheta_n \rvvv.
$$
For simplicity, we let
\begin{equation*}
\arraycolsep=1.1pt\def\arraystretch{1.5}
\begin{array}{rl}
G_{n,1} \triangleq & (1- w_{n+1})\nabla  f(\bdtheta_n; \bdy_{n+1}),\\

G_{n,2} \triangleq & \lppp \nabla^2 F(\bdtheta_n) - w_{n+1} \nabla^2 f(\bdtheta_n; \bdy_{n+1})\rppp \lppp \bbdtheta_n - \bdtheta_n\rppp,\\

G_{n,3} \triangleq & \lppp \nabla^3 F(\bdtheta_n) - w_{n+1} \nabla^3 f(\bdtheta_n; \bdy_{n+1})\rppp \bigotimes^2 \lppp \bbdtheta_{n} - \bdtheta_n\rppp,\\

G_{n,4} \triangleq & \lppp \nabla^4 F(\bdtheta_n) - w_{n+1} \nabla^4 f(\bdtheta_n; \bdy_{n+1})\rppp \bigotimes^3 \lppp \wwtheta{n} - \bdtheta_n\rppp.\\
\end{array}
\end{equation*}
Then, the first term on the right-hand side of (\ref{eq:thm4add_1}) can be decomposed as follows,
\begin{equation}
\arraycolsep=1.1pt\def\arraystretch{1.5}
\begin{array}{rl}
 & \EE\lvv \EE_* \lpp \lpp \lpp \frac{1}{N-N'}\Delta_{N,5} \Delta_{N,5}^T \rpp - U \rpp \daone_{\snopt} \daone_{\ssnopt} \rpp\rvv \\

 \leq & \EE\lvv \EE_* \lpp \lpp \frac{1}{N-N'} \lpp  \ssum{n=N'+1}{N} G_{n,1} \rpp \lpp \ssum{n=N'+1}{N} G_{n,1} \rpp^T  - U \rpp \daone_{\snopt} \daone_{\ssnopt} \rpp\rvv \\

 & + \EE\lvv \EE_* \lpp \frac{1}{N-N'} \lpp   \ssum{n=N'+1}{N} G_{n,2} \rpp \lpp \ssum{n=N'+1}{N} G_{n,2} \rpp^T  \daone_{\snopt} \daone_{\ssnopt} \rpp\rvv \\

  & + \EE\lvv \EE_* \lpp \frac{1}{N-N'} \lpp   \ssum{n=N'+1}{N} G_{n,3} \rpp \lpp \ssum{n=N'+1}{N} G_{n,3} \rpp^T  \daone_{\snopt} \daone_{\ssnopt} \rpp\rvv \\

   & + \EE\lvv \EE_* \lpp \frac{1}{N-N'} \lpp   \ssum{n=N'+1}{N} G_{n,4} \rpp \lpp \ssum{n=N'+1}{N} G_{n,4} \rpp^T  \daone_{\snopt} \daone_{\ssnopt} \rpp\rvv \\

   & + \sum\limits_{1\leq i < j \leq 4} \EE\lvv \EE_* \lpp \frac{2}{N-N'} \lpp   \ssum{n=N'+1}{N} G_{n,i} \rpp \lpp \ssum{n=N'+1}{N} G_{n,j} \rpp^T  \daone_{\snopt} \daone_{\ssnopt} \rpp\rvv \\

   \triangleq & A_{11} + A_{22} + A_{33} + A_{44} + \sum\limits_{1\leq i < j \leq 4}A_{ij}.\\
\end{array}
\label{eq:part5_2}
\end{equation}
\noindent\textit{Part 5.1.} We use $\EE_{*,k}$ to denote the conditional expectation conditional of the sigma algebra generated by $\bdtheta_0$, $\{\bdy_n:n\in \ZZ_+\}$ and $\{ w_n: n\leq k\}$. To bound $A_{11}$, we have
\begin{equation}
\arraycolsep=1.1pt\def\arraystretch{1.5}
\begin{array}{rl}
& \EE_* \lpp \lpp \frac{1}{N-N'} \lpp  \ssum{n=N'+1}{N} G_{n,1} \rpp \lpp \ssum{n=N'+1}{N} G_{n,1} \rpp^T  - U \rpp \daone_{\snopt} \daone_{\ssnopt} \rpp\\

= & \EE_* \lpp \lpp \frac{1}{N-N'} \EE_{*,N'}\lpp  \ssum{n=N'+1}{N} G_{n,1} \rpp \lpp \ssum{n=N'+1}{N} G_{n,1} \rpp^T  - U \rpp \daone_{\snopt} \daone_{\ssnopt} \rpp\\

= & \EE_* \lpp \lpp \frac{1}{N-N'} \ssum{n=N'+1}{N} \nabla f(\bdtheta_n;\bdy_{n+1}) \nabla f(\bdtheta_n;\bdy_{n+1})^T - U \rpp \daone_{\snopt} \daone_{\ssnopt} \rpp.\\
\end{array}
\label{eq:part5_3}
\end{equation}
By the Taylor's expansion, we have
\begin{equation}
\nabla f(\bdtheta_n;\bdy_{n+1}) = \nabla f(\opt;\bdy_{n+1}) + \nabla^2 f(\opt;\bdy_{n+1}) \lppp \wtheta{n} - \opt\rppp,
\label{eq:part5_4}
\end{equation}
where $\wtheta{n}$ depends on $\bdtheta_n$ and $\bdy_{n+1}$ such that
$$
\lvvv \wtheta{n} - \opt\rvvv \leq \| \bdtheta_n -\opt\|.
$$
Based on (\ref{eq:part5_3}), we have
\begin{equation}
\arraycolsep=1.1pt\def\arraystretch{1.5}
\begin{array}{rl}
& A_{11}\\

\leq & \EE \lvv  \lpp \frac{1}{N-N'} \ssum{n=N'+1}{N} \nabla f(\bdtheta_n;\bdy_{n+1}) \nabla f(\bdtheta_n;\bdy_{n+1})^T - U \rpp \daone_{\snopt} \daone_{\ssnopt} \rvv\\

\leq & \EE \lvv \frac{1}{N-N'} \ssum{n=N'+1}{N} \nabla f(\opt;\bdy_{n+1}) \nabla f(\opt;\bdy_{n+1})^T - U \rvv\\

& + \EE \lvv \frac{1}{N-N'} \ssum{n=N'+1}{N} \nabla^2 f(\opt;\bdy_{n+1}) \lppp \wtheta{n} - \opt\rppp \lppp \wtheta{n} - \opt\rppp^T \nabla^2 f(\opt;\bdy_{n+1}) \daone_{\snopt} \rvv\\

& +  \EE \lvv \frac{2}{N-N'} \ssum{n=N'+1}{N} \nabla f(\opt;\bdy_{n+1}) \lppp \wtheta{n} - \opt\rppp^T \nabla^2 f(\opt;\bdy_{n+1}) \daone_{\snopt} \rvv,\\
\end{array}
\label{eq:part5_5}
\end{equation}
where the 2nd step is based on (\ref{eq:part5_4}).

According to \cite{tropp2016expected}, we can have
\begin{equation}
\EE \lvv \frac{1}{N-N'} \ssum{n=N'+1}{N} \nabla f(\opt;\bdy_{n+1}) \nabla f(\opt;\bdy_{n+1})^T - U \rvv = O\lppp N^{-\frac{1}{2}}\rppp.
    \label{eq:part5_6}
\end{equation}
The second term on the right-hand side of (\ref{eq:part5_5}) can be controlled as follows,
\begin{equation}
\arraycolsep=1.1pt\def\arraystretch{1.5}
\begin{array}{rl}
&  \EE \lvv \frac{1}{N-N'} \ssum{n=N'+1}{N} \nabla^2 f(\opt;\bdy_{n+1}) \lppp \wtheta{n} - \opt\rppp \lppp \wtheta{n} - \opt\rppp^T \nabla^2 f(\opt;\bdy_{n+1}) \daone_{\snopt} \rvv\\

\leq & \frac{1}{N-N'} \ssum{n=N'+1}{N} \EE \lppp \| \nabla^2 f(\opt;\bdy_{n+1}) \|^2 \lvvv \wtheta{n} - \opt \rvvv^2 \daone_{\snopt} \rppp\\

\leq & \frac{1}{N-N'} \ssum{n=N'+1}{N} \EE \lppp \| \nabla^2 f(\opt;\bdy_{n+1}) \|^2 \lvvv \bdtheta_{n} - \opt \rvvv^2 \daone_{\snopt} \rppp\\

\leq & \frac{C_{sh}^{(2)}}{N-N'}  \ssum{n=N'+1}{N} \EE \lppp \lvvv \bdtheta_{n} - \opt \rvvv^2 \daone_{\snopt} \rppp\\

\leq & \frac{C_{sh}^{(2)}}{N-N'}  \ssum{n=N'+1}{N} \EE \lppp \lvvv \bdtheta_{n} - \opt \rvvv^2 \daone_{S_{opt}} \rppp + \frac{C_{sh}^{(2)}}{N-N'}  \ssum{n=N'+1}{N} \EE \lppp \lvvv \bdtheta_{n} - \opt \rvvv^2 \daone_{\snopt\backslash S_{opt}} \rppp\\

= & O\lppp N^{-\alpha}\rppp + \frac{C_{sh}^{(2)}}{N-N'}  \ssum{n=N'+1}{N} \EE \lppp \lvvv \bdtheta_{n} - \opt \rvvv^2 \daone_{\snopt\backslash S_{opt}} \rppp\\

\leq & O\lppp N^{-\alpha}\rppp + \frac{C_{sh}^{(2)}}{N-N'}  \ssum{n=N'+1}{N} \lpp \EE \lvvv \bdtheta_{n} - \opt \rvvv^4\rpp^{\frac{1}{2}} \PP^{\frac{1}{2}}\lppp \snopt\backslash S_{opt}\rppp\\

= & O\lppp N^{-\alpha}\rppp,
\end{array}
\label{eq:part5_7}
\end{equation}
where the 3rd step is based on condition \ref{cm_5}, the 5th step is based on Theorem \ref{thm:secbound} and the last step is based on condition \ref{cm_3}, Lemma \ref{lemma:trap} and Lemma \ref{lemma:momentbound}. 
The third term on the right-hand side of (\ref{eq:part5_5}) can be bounded as follows
\begin{equation}
\arraycolsep=1.1pt\def\arraystretch{1.5}
\resizebox{.92\hsize}{!}{$
\begin{array}{rl}
& \EE \lvv \frac{2}{N-N'} \ssum{n=N'+1}{N} \nabla f(\opt;\bdy_{n+1}) \lppp \wtheta{n} - \opt\rppp^T \nabla^2 f(\opt;\bdy_{n+1}) \daone_{\snopt} \rvv\\

\leq & \frac{2}{N-N'} \ssum{n=N'+1}{N} \EE \lvvv \nabla f(\opt;\bdy_{n+1}) \lppp \wtheta{n} - \opt\rppp^T \nabla^2 f(\opt;\bdy_{n+1}) \daone_{\snopt} \rvvv\\

\leq & \frac{2}{N-N'} \lppp \EE \| \nabla f(\opt;\bdy)\|^2\rppp^{\frac{1}{2}}\lppp \EE \| \nabla^2 f(\opt;\bdy)\|^2\rppp^{\frac{1}{2}} \ssum{n=N'+1}{N} \lppp \EE \lppp \lvvv \wtheta{n} - \opt\rvvv^2 \daone_{\snopt}\rppp \rppp^{\frac{1}{2}}\\

\leq & \frac{2}{N-N'} \lppp \EE \| \nabla f(\opt;\bdy)\|^2\rppp^{\frac{1}{2}}\lppp \EE \| \nabla^2 f(\opt;\bdy)\|^2\rppp^{\frac{1}{2}} \ssum{n=N'+1}{N} \lppp \EE \lppp \lvvv \bdtheta_{n} - \opt\rvvv^2 \daone_{\snopt}\rppp \rppp^{\frac{1}{2}}\\

\leq & \frac{2\sqrt{C_{sh}^{(2)}}}{N-N'} \lppp \EE \| \nabla f(\opt;\bdy)\|^2\rppp^{\frac{1}{2}} \ssum{n=N'+1}{N} \lppp \EE \lppp \lvvv \bdtheta_{n} - \opt\rvvv^2 \daone_{\snopt}\rppp \rppp^{\frac{1}{2}}\\

= & O(N^{-1}) \lpp \ssum{n=N'+1}{N} \lppp \EE \lppp \lvvv \bdtheta_{n} - \opt\rvvv^2 \daone_{S_{opt}}\rppp \rppp^{\frac{1}{2}} + \ssum{n=N'+1}{N} \lppp \EE \lppp \lvvv \bdtheta_{n} - \opt\rvvv^2 \daone_{\snopt\backslash S_{opt}}\rppp \rppp^{\frac{1}{2}} \rpp\\

= & O \lppp N^{-\frac{1}{2}\alpha}\rppp + O(N^{-1}) \ssum{n=N'+1}{N} \lpp \EE \lvvv \bdtheta_{n} - \opt\rvvv^4\rpp^{\frac{1}{4}} \PP^{\frac{1}{4}}\lppp \snopt\backslash S_{opt} \rppp\\

= & O \lppp N^{-\frac{1}{2}\alpha}\rppp,\\
\end{array}
\label{eq:part5_13}
$}
\end{equation}
where the 2nd step is based on the H\"{o}lder's inequality and the 4th step is based on condition \ref{cm_5}, the 6th step is based on Theorem \ref{thm:secbound} and the last step is based on condition \ref{cm_3}, Lemma \ref{lemma:trap} and Lemma \ref{lemma:momentbound}. 
Based on (\ref{eq:part5_5}), (\ref{eq:part5_6}), (\ref{eq:part5_7}) and (\ref{eq:part5_13}), we have
\begin{equation}
A_{11} = O\lppp N^{-\frac{1}{2}}\rppp + O \lppp N^{-\alpha}\rppp + O \lppp N^{-\frac{1}{2}\alpha}\rppp = O \lppp N^{-\frac{1}{2}\alpha}\rppp.
\label{eq:part5_15}
\end{equation}

\noindent\textit{Part 5.2.} To bound $A_{22}$, we have
\begin{equation}
\arraycolsep=1.1pt\def\arraystretch{1.5}
\resizebox{.92\hsize}{!}{$
\begin{array}{rl}
A_{22} & = \EE\lvv \EE_* \lpp \frac{1}{N-N'} \lpp   \ssum{n=N'+1}{N} G_{n,2} \rpp \lpp \ssum{n=N'+1}{N} G_{n,2} \rpp^T  \daone_{\snopt} \daone_{\ssnopt} \rpp\rvv \\

&\leq \frac{1}{N-N'} \EE \lpp \lvv \ssum{n=N'+1}{N} G_{n,2} \rvv^2 \daone_{\snopt} \daone_{\ssnopt} \rpp\\

& = \frac{1}{N-N'} \ssum{n=N'+1}{N} \EE \lppp \lvvv (\nabla^2 F(\bdtheta_n) - w_{n+1} \nabla^2 f(\bdtheta_n;\bdy_{n+1})) (\bbdtheta_n - \bdtheta_n) \rvvv^2 \daone_{\snopt} \daone_{\ssnopt} \rppp\\

& \ \ + \frac{2}{N-N'} \sum\limits_{N'+1 \leq n_1 < n_2 \leq N} \EE \lppp \langle (\nabla^2 F(\bdtheta_{n_1}) - w_{n_1+1} \nabla^2 f(\bdtheta_{n_1};\bdy_{n_1+1})) (\bbdtheta_{n_1} - \bdtheta_{n_1}),\\

&\ \ \ \ \ \ \ \ \ \ \ \ \ \ \ \ \ \ \ \ \ \ \ \ \ \ \ \ (\nabla^2 F(\bdtheta_{n_2}) - w_{n_2+1} \nabla^2 f(\bdtheta_{n_2};\bdy_{n_2+1})) (\bbdtheta_{n_2} - \bdtheta_{n_2}) \rangle \daone_{\snopt} \daone_{\ssnopt} \rppp\\

& = \frac{1}{N-N'} \ssum{n=N'+1}{N} \EE \lppp \lvvv (\nabla^2 F(\bdtheta_n) - w_{n+1} \nabla^2 f(\bdtheta_n;\bdy_{n+1})) (\bbdtheta_n - \bdtheta_n) \rvvv^2 \daone_{\snopt} \daone_{\ssnopt} \rppp\\

& \leq \frac{2}{N-N'} \ssum{n=N'+1}{N} \EE \lppp \lppp C_s^2 + w_{n+1}^2 \| \nabla^2 f(\bdtheta_n;\bdy_{n+1})\|^2 \rppp \| \bbdtheta_n - \bdtheta_n \|^2 \daone_{\snopt} \daone_{\ssnopt} \rppp\\

& = \frac{2}{N-N'} \ssum{n=N'+1}{N} \EE \lppp \lppp C_s^2 + 2 \| \nabla^2 f(\bdtheta_n;\bdy_{n+1})\|^2 \rppp \| \bbdtheta_n - \bdtheta_n \|^2 \daone_{\snopt} \daone_{\ssnopt} \rppp\\

& \leq \frac{2}{N-N'} \lppp C_s^2 + C_{sh}^{(2)}\rppp \ssum{n=N'+1}{N} \EE \lppp \| \bbdtheta_n - \bdtheta_n \|^2 \daone_{\snopt} \daone_{\ssnopt} \rppp\\

& = O\lppp N^{-\alpha}\rppp,\\
\end{array}
\label{eq:part5_16}
$}
\end{equation}
where the second to the last step is based on condition \ref{cm_5} and the last step can be shown in a way similar to (\ref{eq:part5_7}).

\noindent\textit{Part 5.3.} To bound $A_{33}$, we have
\begin{equation}
\arraycolsep=1.1pt\def\arraystretch{1.5}
\begin{array}{rl}
A_{33} & = \EE\lvv \EE_* \lpp \frac{1}{N-N'} \lpp   \ssum{n=N'+1}{N} G_{n,3} \rpp \lpp \ssum{n=N'+1}{N} G_{n,3} \rpp^T  \daone_{\snopt} \daone_{\ssnopt} \rpp\rvv \\

&\leq \frac{1}{N-N'} \EE \lpp \lvv \ssum{n=N'+1}{N} G_{n,3} \rvv^2 \daone_{\snopt} \daone_{\ssnopt} \rpp\\

&= \frac{1}{N-N'} \ssum{n=N'+1}{N} \EE \lppp \lvvv \lppp \nabla^3 F(\bdtheta_n) - w_{n+1} \nabla^3 f(\bdtheta_n;\bdy_{n+1})\rppp \bigotimes^2 \lppp \bbdtheta_n - \bdtheta_n \rppp \rvvv^2 \daone_{\snopt} \daone_{\ssnopt} \rppp\\

& \leq \frac{1}{N-N'} \ssum{n=N'+1}{N} \EE \lppp \lvvv  \nabla^3 F(\bdtheta_n) - w_{n+1} \nabla^3 f(\bdtheta_n;\bdy_{n+1})\rvvv^2  \lvvv \bbdtheta_n - \bdtheta_n  \rvvv^4 \daone_{\snopt} \daone_{\ssnopt} \rppp\\

& = O \lpp \frac{1}{N}  \ssum{n=N'+1}{N}   \EE \lppp  \lvvv \bbdtheta_{n} - \bdtheta_{n} \rvvv^4 \daone_{\snopt} \daone_{\ssnopt} \rppp \rpp,\\
\end{array}
\label{eq:part5_17}
\end{equation}
where the last step is based on condition \ref{cm_5}. 
In fact, we can see that
$$
\arraycolsep=1.1pt\def\arraystretch{1.5}
\begin{array}{rl}
& \ssum{n=N'+1}{N}   \EE \lppp  \lvvv \bdtheta_{n} - \opt \rvvv^4 \daone_{\snopt}  \rppp\\

\leq & \ssum{n=N'+1}{N}   \EE \lppp  \lvvv \bdtheta_{n} - \opt \rvvv^4 \daone_{S_{opt}}  \rppp +  \ssum{n=N'+1}{N}   \EE \lppp  \lvvv \bdtheta_{n} - \opt \rvvv^4 \daone_{\snopt\backslash S_{opt}}  \rppp\\

= & O(1) + \ssum{n=N'+1}{N}   \EE \lppp  \lvvv \bdtheta_{n} - \opt \rvvv^4 \daone_{\snopt\backslash S_{opt}}  \rppp\\

\leq &  O(1) + \ssum{n=N'+1}{N}   \lpp \EE  \lvvv \bdtheta_{n} - \opt \rvvv^8\rpp^{\frac{1}{2}} \PP^{\frac{1}{2}}\lppp \snopt\backslash S_{opt}\rppp\\

= &  O(1),\\
\end{array}
$$
where the second step is based on Theorem \ref{thm:fourbound} and the last step is based on condition \ref{cm_3}, Lemma \ref{lemma:trap} and Lemma \ref{lemma:momentbound}. 
Likewise, we can have
$$
\ssum{n=\frac{N}{2}+1}{N}  \EE \lpp \| \bbdtheta_{n} - \opt\|^4 \daone_{\ssnopt} \rpp =  O(1),
$$
and consequently
$$
\ssum{n=\frac{N}{2}+1}{N}   \EE \lppp  \lvvv \bbdtheta_{n} - \bdtheta_{n} \rvvv^4 \daone_{\snopt} \daone_{\ssnopt} \rppp =  O(1).
$$
Therefore, from (\ref{eq:part5_17}), we have
\begin{equation}
A_{33} = O(N^{-1}).
    \label{eq:part5_20}
\end{equation}

\noindent\textit{Part 5.4.} To bound $A_{44}$, we have
\begin{equation}
\arraycolsep=1.1pt\def\arraystretch{1.5}
\resizebox{.92\hsize}{!}{$
\begin{array}{rl}
A_{44} & = \EE\lvv \EE_* \lpp \frac{1}{N-N'} \lpp   \ssum{n=N'+1}{N} G_{n,4} \rpp \lpp \ssum{n=N'+1}{N} G_{n,4} \rpp^T  \daone_{\snopt} \daone_{\ssnopt} \rpp\rvv \\

&\leq \frac{1}{N-N'} \EE \lpp \lvv \ssum{n=N'+1}{N} G_{n,4} \rvv^2 \daone_{\snopt} \daone_{\ssnopt} \rpp\\

&\leq \frac{1}{N-N'} \EE \lpp \lpp \ssum{n=N'+1}{N} \lvvv \nabla^4 F(\bdtheta_n) - w_{n+1} \nabla^4 f(\bdtheta_n;\bdy_{n+1})\rvvv \lvvv \wwtheta{n} - \bdtheta_n \rvvv^3 \rpp^2 \daone_{\snopt} \daone_{\ssnopt} \rpp\\

&\leq \frac{1}{N-N'} \EE \lpp \lpp \ssum{n=N'+1}{N} \lvvv \nabla^4 F(\bdtheta_n) - w_{n+1} \nabla^4 f(\bdtheta_n;\bdy_{n+1})\rvvv \lvvv \bbdtheta_{n} - \bdtheta_n \rvvv^3 \rpp^2 \daone_{\snopt} \daone_{\ssnopt} \rpp\\

& \leq \frac{1}{N-N'} \ssum{n=N'+1}{N} \EE \lpp \lvvv \nabla^4 F(\bdtheta_n) - w_{n+1} \nabla^4 f(\bdtheta_n;\bdy_{n+1})\rvvv^2 \lvvv \bbdtheta_{n} - \bdtheta_n \rvvv^6 \daone_{\snopt} \daone_{\ssnopt} \rpp\\

& \ \ \ + \frac{2}{N-N'} \sum\limits_{N' + 1 \leq n_1 < n_2 \leq N} \lpp \EE \lpp \lvvv \nabla^4 F(\bdtheta_{n_1}) - w_{{n_1}+1} \nabla^4 f(\bdtheta_{n_1};\bdy_{{n_1}+1})\rvvv^2 \lvvv \bbdtheta_{{n_1}} - \bdtheta_{n_1} \rvvv^6 \daone_{\snopt} \daone_{\ssnopt} \rpp \rpp^{\frac{1}{2}}\\

&\ \ \ \ \ \ \ \ \ \ \ \ \ \ \ \ \ \ \ \ \ \ \ \ \ \ \ \ \cdot \lpp \EE\lpp \lvvv \nabla^4 F(\bdtheta_{n_2}) - w_{{n_2}+1} \nabla^4 f(\bdtheta_{n_2};\bdy_{{n_2}+1})\rvvv^2 \lvvv \bbdtheta_{{n_2}} - \bdtheta_{n_2} \rvvv^6 \daone_{\snopt} \daone_{\ssnopt} \rpp \rpp^{\frac{1}{2}}\\

& = O\lpp \frac{1}{N} \lpp \ssum{n=N'+1}{N} \lpp \EE \lpp  \lvvv \bbdtheta_{n} - \bdtheta_n \rvvv^6 \daone_{\snopt} \daone_{\ssnopt} \rpp \rpp^{\frac{1}{2}} \rpp^2 \rpp,\\
\end{array}
\label{eq:part5_21}
$}
\end{equation}
where the 5th step is based on the Cauchy's inequality and the last step is based on condition \ref{cm_5}. For simplicity, we temporarily define
$$
S_{i,j} \triangleq \lwww \bdtheta_k \in \rgo^L(\opt),i\leq k \leq j\rwww.
$$
For $N'+1 \leq n \leq N$, based on Lemma \ref{lemma:sixbound}, we know there exist constants $c_7$ and $c_8$ such that when $N\ge c_7$,
\begin{equation}
\arraycolsep=1.1pt\def\arraystretch{1.5}
\resizebox{.92\hsize}{!}{$
\begin{array}{rl}
& \EE \lppp \|\bdtheta_n - \opt \|^6 \daone_{S_{\frac{N}{3},n-1}} \rppp\\

\leq & (1- \tlambda \gamma_n) \EE \lppp \|\bdtheta_{n-1} - \opt \|^6 \daone_{S_{\frac{N'}{2},n-1}} \rppp + c_8 \gamma_{n}^4\\

\leq & \lpp \pprod{k=\frac{N'}{2}+1}{n} (1-\tlambda \gamma_k)\rpp \EE \lvvv \bdtheta_{\frac{N'}{2}} - \opt\rvvv^6 + c_8 \ssum{k=\frac{N'}{2}+1}{n} \lpp \gamma_k^4 \pprod{j=k+1}{n} (1-\tlambda \gamma_j)\rpp\\

\leq & (1-\tlambda \gamma_n)^{n-\frac{N'}{2}} \EE \lvvv \bdtheta_{\frac{N'}{2}} - \opt\rvvv^6 + c_8  \ssum{k=\frac{N'}{2}+1}{n} \gamma_k^4 (1-\tlambda \gamma_{n})^{n-k}\\

= & O \lpp \exp \lpp-C\tlambda \lppp n-\frac{N'}{2}\rppp n^{-\alpha}\rpp \lppp \log N'\rppp^{\frac{12}{2-\beta_{sg}}}\rpp + c_8  \ssum{k=\frac{N'}{2}+1}{n} \gamma_k^4 (1-\tlambda \gamma_{n})^{n-k}\\

= & O \lpp \exp \lpp-C\tlambda  n^{1-\alpha}\rpp \lppp \log N'\rppp^{\frac{12}{2-\beta_{sg}}}\rpp + O\lpp \ssum{k=\frac{N'}{2}+1}{n} k^{-4\alpha} \exp \lppp -C\tlambda(n-k)n^{-\alpha}\rppp \rpp\\

= & O \lpp \exp \lpp-C\tlambda  n^{1-\alpha}\rpp \lppp \log N'\rppp^{\frac{12}{2-\beta_{sg}}}\rpp + O\lpp \ssum{k=\frac{N'}{2}+1}{n} k^{-4\alpha} \exp \lppp C\tlambda k^{1-\alpha}\rppp \rpp  \exp \lppp C\tlambda n^{1-\alpha}\rppp\\

= & O \lpp \exp \lpp-C\tlambda  n^{1-\alpha}\rpp \lppp \log N'\rppp^{\frac{12}{2-\beta_{sg}}}\rpp + O\lppp n^{-\alpha}\rppp\\

= & O\lppp n^{-3\alpha}\rppp,\\
\end{array}
\label{eq:part5_22}
$}
\end{equation}
where the 4th step is based on condition \ref{cm_3} and Lemma \ref{lemma:momentbound}, the 7th step is based on thee Lemma \ref{lemma:int1}.

Consequently, based on (\ref{eq:part5_21}) and (\ref{eq:part5_22}), we have
\begin{equation}
A_{44} = O\lppp N^{1-3\alpha}\rppp.
    \label{eq:part5_23}
\end{equation}

\noindent\textit{Part 5.4.} In this part, we provide bounds on the cross terms in (\ref{eq:part5_2}). We see that
\begin{equation}
\arraycolsep=1.1pt\def\arraystretch{1.5}
\begin{array}{rl}
& \EE \lpp \lvv \ssum{n=N'+1}{N} (1-w_{n+1}) \nabla f(\bdtheta_n;\bdy_{n+1}) \rvv^2 \daone_{\snopt} \rpp\\

= & \ssum{n=N'+1}{N} \EE \lppp (1-w_{n+1})^2 \| \nabla f(\bdtheta_n;\bdy_{n+1}) \|^2 \daone_{\snopt} \rppp\\

= & \ssum{n=N'+1}{N} \EE \lppp \| \nabla f(\bdtheta_n;\bdy_{n+1}) \|^2 \daone_{\snopt} \rppp\\

\leq & 2 \ssum{n=N'+1}{N} \EE \lppp \| \nabla f(\bdtheta_n;\bdy_{n+1}) -\nablaf{n} \|^2 \daone_{\snopt} \rppp + 2 \ssum{n=N'+1}{N} \EE \lppp \| \nablaf{n} \|^2 \daone_{\snopt} \rppp\\ 

\leq & 2(4-\beta_{sg})\csg N + (4\beta_{sg} \csg + 2) \ssum{n=N'+1}{N} \EE \lppp \| \nablaf{n} \|^2 \daone_{\snopt} \rppp\\ 

\leq & O(N) + (4\beta_{sg} \csg + 2)C_s^2 \ssum{n=N'+1}{N}  \EE \lppp \| \bdtheta_n - \opt \|^2 \daone_{\snopt} \rppp\\ 

= & O(N),\\
\end{array}
\label{eq:part5_24}
\end{equation}
where the 4th step is based on Lemma \ref{lemma:errorbound}, the 5th step is based on condition \ref{cm_1} and the last step is similar to (\ref{eq:part5_16}). Therefore, we can have
\begin{equation}
\arraycolsep=1.1pt\def\arraystretch{1.5}
\resizebox{.92\hsize}{!}{$
\begin{array}{rl}
A_{12} = &  \EE\lvv \EE_* \lpp 
\frac{2}{N-N'} \lpp   \ssum{n=N'+1}{N} G_{n,1} \rpp \lpp \ssum{n=N'+1}{N} G_{n,2} \rpp^T  \daone_{\snopt} \daone_{\ssnopt} \rpp\rvv \\

\leq & 
\frac{2}{N-N'} \lpp \EE \lpp \lvv \ssum{n=N'+1}{N} (1-w_{n+1}) \nabla f(\bdtheta_n;\bdy_{n+1}) \rvv^2 \daone_{\snopt} \rpp \rpp^{\frac{1}{2}} \lpp \EE \lpp \lvv \ssum{n=N'+1}{N} G_{n,2} \rvv^2 \daone_{\snopt} \daone_{\ssnopt} \rpp \rpp^{\frac{1}{2}}\\

= & O \lppp N^{-\frac{\alpha}{2}}\rppp,\\
\end{array}
\label{eq:part5_25}
$}
\end{equation}
where the last step is based on (\ref{eq:part5_16}) and (\ref{eq:part5_24}). Likewise, based on (\ref{eq:part5_20}), (\ref{eq:part5_23}) and (\ref{eq:part5_24}), we can obtain that
\begin{equation}
A_{13} = O\lppp N^{-\alpha} \rppp,\ A_{14} = O\lppp N^{\frac{1}{2} - \frac{3}{2}\alpha} \rppp.
    \label{eq:part5_26}
\end{equation}
Based on (\ref{eq:part5_16}), (\ref{eq:part5_20}) and (\ref{eq:part5_23}), we can check that
\begin{equation}
A_{23} = O\lppp N^{-\frac{3}{2}\alpha} \rppp,\ A_{24} = O\lppp N^{\frac{1}{2} - 2\alpha} \rppp,\ A_{34} = O\lppp N^{\frac{1}{2} - \frac{5}{2}\alpha} \rppp.
    \label{eq:part5_27}
\end{equation}

To sum up, based on (\ref{eq:part5_2}), (\ref{eq:part5_15}), (\ref{eq:part5_16}), (\ref{eq:part5_20}), (\ref{eq:part5_23}), (\ref{eq:part5_26}) and (\ref{eq:part5_27}), we have
$$
\EE\lvv \EE_* \lpp \lpp \lpp \frac{1}{N-N'}\Delta_{N,5} \Delta_{N,5}^T \rpp - U \rpp \daone_{\snopt} \daone_{\ssnopt} \rpp\rvv = O\lppp N^{-\frac{\alpha}{2}} + N^{\frac{1}{2} - \frac{3}{2}\alpha} \rppp.
$$
Putting the above result back to (\ref{eq:thm4add_1}), by applying Lemma \ref{lemma:momentbound}, Lemma \ref{lemma:trap} and Lemma \ref{lemma:mustconverge}, it is not hard to check that
\begin{equation}
\EE\lvv \EE_* \lpp \lpp \lpp \frac{1}{N-N'}\Delta_{N,5} \Delta_{N,5}^T \rpp - U \rpp \daone_{S_{opt}} \daone_{S_{opt}^*} \rpp\rvv = O\lppp N^{-\frac{\alpha}{2}} + N^{\frac{1}{2} - \frac{3}{2}\alpha} \rppp.
    \label{eq:part5_28}
\end{equation}

Finally, with (\ref{eq:cp1}), (\ref{eq:cp2}), (\ref{eq:cp32}), (\ref{eq:cp4}), (\ref{eq:cp6}) and (\ref{eq:part5_28}), we have the following bound
\begin{equation}
\resizebox{.92\hsize}{!}{$
\EE \lvv \EE_* \lpp \lpp\lww \sqrt{N-N'} A (\check{\bdtheta}_N^* - \check{\bdtheta}_N) \rww\lww\sqrt{N-N'} A (\check{\bdtheta}_N^* - \check{\bdtheta}_N) \rww^T  - U\rpp \daone_{S_{opt}} \daone_{S_{opt}^*} \rpp \rvv = O \lppp N^{\alpha-1} + N^{1-2\alpha} + N^{-\frac{\alpha}{2}}  \rppp.
$}
    \label{eq:part5_29}
\end{equation}

We observe that
\begin{equation}
\arraycolsep=1.1pt\def\arraystretch{1.5}
\resizebox{.92\hsize}{!}{$
\begin{array}{rl}
& \EE \lvvv \lppp (N-N') A \lppp \check{\bdtheta}_N^* - \check{\bdtheta}_N \rppp \lppp \check{\bdtheta}_N^* - \check{\bdtheta}_N \rppp^T A - N A \lppp \bar{\bdtheta}_N^* - \bar{\bdtheta}_N \rppp \lppp \bar{\bdtheta}_N^* - \bar{\bdtheta}_N \rppp^T A \rppp \daone_{S_{opt}} \daone_{S_{opt}^*}\rvvv\\

\leq & \lpp \frac{1}{N-N'} - \frac{1}{N} \rpp \EE \lvv A \lpp \ssum{n=N'+1}{N}(\bbdtheta_n - \bdtheta_n)\rpp \lpp \ssum{n=N'+1}{N}(\bbdtheta_n - \bdtheta_n)\rpp^T A \daone_{S_{opt}} \daone_{S_{opt}^*}\rvv\\

& + \frac{1}{N} \EE \lvv  A \lpp \ssum{n=1}{N'}(\bbdtheta_n - \bdtheta_n)\rpp \lpp \ssum{n=1}{N'}(\bbdtheta_n - \bdtheta_n)\rpp^T A \daone_{S_{opt}} \daone_{S_{opt}^*}\rvv\\

& + \frac{2}{N} \EE \lvv  A \lpp \ssum{n=N'+1}{N}(\bbdtheta_n - \bdtheta_n)\rpp \lpp \ssum{n=1}{N'}(\bbdtheta_n - \bdtheta_n)\rpp^T A \daone_{S_{opt}} \daone_{S_{opt}^*}\rvv\\

= & O\lppp N^{\nu-2}\rppp \EE \lpp \ssum{n=N'+1}{N}\|\bbdtheta_n - \bdtheta_n\| \daone_{S_{opt}} \daone_{S_{opt}^*} \rpp^2 + O\lppp N^{-1}\rppp \EE \lpp \ssum{n=1}{N'}\|\bbdtheta_n - \bdtheta_n\| \daone_{S_{opt}} \daone_{S_{opt}^*} \rpp^2\\

& + O\lppp N^{-1}\rppp \EE \lpp \ssum{n=1}{N'}\|\bbdtheta_n - \bdtheta_n\| \daone_{S_{opt}} \daone_{S_{opt}^*} \rpp \lpp \ssum{n=N'+1}{N}\|\bbdtheta_n - \bdtheta_n\| \daone_{S_{opt}} \daone_{S_{opt}^*} \rpp\\

=& O \lppp N^{\nu-\alpha}\rppp + O \lppp N^{-1+(2-\alpha)\nu}\rppp + O \lppp N^{-\frac{\alpha}{2} + \nu \lppp 1-\frac{\alpha}{2}\rppp}\rppp,\\
\end{array}
$}
    \label{eq:part5_30}
\end{equation}
where the last step is based on Theorem \ref{thm:secbound}, Theorem \ref{thm:fourbound} and Corollary \ref{cor:momentbound}. Since $\nu$ can be arbitrary close to $0$, based on (\ref{eq:part5_29}) and (\ref{eq:part5_30}), we know that for any small positive $\epsilon$,
$$
\resizebox{.92\hsize}{!}{$
\EE \lvv \EE_* \lpp \lpp\lww \sqrt{N} A (\bar{\bdtheta}_N^* - \bar{\bdtheta}_N) \rww\lww\sqrt{N} A (\bar{\bdtheta}_N^* - \bar{\bdtheta}_N) \rww^T  - U\rpp \daone_{S_{opt}} \daone_{S_{opt}^*} \rpp \rvv = O \lppp N^{\alpha-1} + N^{1-2\alpha} + N^{-\frac{\alpha}{2}+\epsilon}  \rppp.
$}
$$

\end{proof}

\section{Auxiliary Technical Lemmas}
\begin{lemma}[Lemma 5.6 in \cite{fort2015central}]
Let $(\Omega,\mathcal{A},\PP,\{\fff_n,n\ge 0\})$ be a filtered probability space and let $\fff_{\infty} = \sigma(\fff_n,n \ge 1)$. For any $B\in \fff_{\infty}$, there exists a $\fff_n$-adapted sequence $\{A_n,n\ge 0\}$ such that $\lim\limits_n \daone_{A_n} = \daone_B$ $\PP$-a.s.. 
\label{lemma:setapprox1}
\end{lemma}


\begin{lemma}
Suppose $\{x_n\}_{n\ge 1}$ is a sequence of almost-surely finite random variables, i.e. $\PP(|x_n|<\infty)=1,\forall n\ge1$. Let $A$ be an event set an $\{A_n\}_{n\ge 1}$ be a sequence of sets such that $\lim\limits_n \daone_{A_n} = \daone_{A}$ almost-surely. Suppose that $f:\ZZ_+ \rightarrow [0,\infty)$ satisfies that $\lim\limits_n f(n) = 0$. Then,
$$
\lim\limits_{N\rightarrow \infty} f(N)\sum\limits_{n=1}\limits^{N}[x_n (\daone_{A_n} - \daone_{A})] = 0,\ a.s..
$$
\label{lemma:setapprox2}
\end{lemma}

\begin{proof}
Since $\lim\limits_n \daone_{A_n} = \daone_{A}$ a.s., there exists an almost-surely finite random variable $\xi$ such that when $n>\xi$, $\daone_{A_n} = \daone_{A}$. Then, we have
$$
\begin{array}{rl}
& f(N)\sum\limits_{n=1}\limits^{N}[x_n (\daone_{A_n} - \daone_{A})] = f(N)\sum\limits_{n=1}\limits^{N}[x_n (\daone_{A_n} - \daone_{A})(\daone\{n>\xi\} + \daone\{n\leq \xi\})]\\

= & f(N)\sum\limits_{n=1}\limits^{N\wedge \xi}[x_n (\daone_{A_n} - \daone_{A})].
\end{array}
$$
Based on
$$
f(N)\sum\limits_{n=1}\limits^{N\wedge \xi}|x_n (\daone_{A_n} - \daone_{A})| \leq 2f(N) \sum\limits_{n=1}\limits^{\xi}|x_n| \rightarrow 0\ a.s.,
$$
we can conclude that
$$
\lim\limits_{N\rightarrow \infty} f(N)\sum\limits_{n=1}\limits^{N}[x_n (\daone_{A_n} - \daone_{A})] = \lim\limits_{N\rightarrow \infty} f(N)\sum\limits_{n=1}\limits^{N\wedge \xi}[x_n (\daone_{A_n} - \daone_{A})] = 0,\ a.s..
$$
\end{proof}


\begin{lemma}
Suppose that $\{x_n\}_{n\ge 1}$ is a sequence of univariate random variables such that $\lim\limits_{n\rightarrow \infty} \EE e^{ix_n} =\alpha$, where $\alpha$ is some constant. Let $\{y_n\}_{n\ge 1}$ is a sequence of random variables such that $y_n\rightarrow 0$ a.s.. Then we have
$$
\lim\limits_{n\rightarrow \infty} \EE e^{i(x_n+y_n)} = \alpha.
$$
\label{lemma:asconverge1}
\end{lemma}

\begin{proof}
Note that $|e^{i(x_n+y_n)} - e^{ix_n}|\leq 2$. Then, based on the dominated convergence theorem, we have
$$
\lim\limits_{n\rightarrow \infty} \EE (e^{i(x_n+y_n)} - e^{ix_n}) =  \EE \left[\lim\limits_{n\rightarrow \infty} (e^{i(x_n+y_n)} - e^{ix_n})\right] =0.
$$
This results in
$$
\lim\limits_{n\rightarrow \infty} \EE e^{i(x_n+y_n)} = \lim\limits_{n\rightarrow \infty} \EE e^{ix_n} + \lim\limits_{n\rightarrow \infty} \EE (e^{i(x_n+y_n)} - e^{ix_n}) = \alpha.
$$
\end{proof}

\begin{lemma}[Theorem 1 in \cite{robbins1971convergence}] Suppose that $(\Omega,\fff,\PP)$ is a probability space and $\fff_1\subset \fff_2 \subset\ldots$ a sequence of sub-$\sigma$-algebras of $\fff$. For each $n=1,2,\ldots$, let $z_n,\beta_n,\xi_n,\zeta_n$ be non-negative $\fff_n$-measurable random variables such that
$$
\EE\left[z_{n+1}|\fff_n\right] \leq z_n(1+\beta_n)+\xi_n-\zeta_n.
$$
Then, $\lim\limits_{n\rightarrow }z_n$ exists and is finite and $\sum\limits_{1}\limits^{\infty}\zeta_n <\infty$ a.s. on
$$
\left\{ \sum\limits_{1}\limits^{\infty}\beta_n <\infty,\sum\limits_{1}\limits^{\infty}\xi_n <\infty\right\}.
$$
\label{rs71}
\end{lemma}
\begin{lemma}[Theorem 1 in \cite{mei2018landscape}]
Let $l(\cdot;\cdot):\RR^p\times \RR^d\rightarrow \RR$ be a loss function. Given data $\{\bdz_1,\bdz_2,\ldots,\bdz_n\}$, The empirical loss is $\hat{L}_n(\bdtheta)\triangleq \frac{1}{n}\sum\limits_{i=1}\limits^{n}l(\bdtheta;\bdz_i)$. The population loss is $L(\bdtheta)\triangleq \EE\left[ \hat{L}_n(\bdtheta)\right]$. Suppose the following assumptions hold for some radius $r>0$:
\begin{enumerate}
\item (Gradient statistical noise). The gradient of the loss is $\tau^2$-sub-Gaussian. Namely for any $\bdlambda\in \RR^p$, and $\bdtheta\in B_d(\textbf{0},r)$
$$
\EE\left\{\exp\left(\langle \bdlambda,\nabla l(\bdtheta;\bdZ)-\EE\left[ \nabla l(\bdtheta;\bdZ)\right]\rangle \right)\right\} \leq \exp\left(\frac{\tau^2\|\bdlambda\|_2^2}{2}\right).
$$
\item (Hessian statistical noise). The Hessian of the loss, evaluated on a unit vector, is $\tau^2$-sub-exponential. Namely, for any $\bdlambda\in \RR^p$ with $\|\bdlambda\|_2=1$, and $\bdtheta\in B_p(\textbf{0},r)$,
$$
\EE\left\{\exp\left( \frac{1}{\tau^2}|\langle \bdlambda,\nabla^2l(\bdtheta;\bdZ)\bdlambda\rangle - \EE\left[ \langle \bdlambda,\nabla^2l(\bdtheta;\bdZ)\bdlambda\rangle\right] |\right)\right\} \leq 2.
$$
\item (Hessian regularity). The Hessian of the population loss is bounded at one point. Namely, there exists $\bdtheta_*\in B_p(\textbf{0},r)$ and $H$ such that $\|\nabla^2 L(\bdtheta_*)\|_{op} \leq H$. Further, the Hessian of the loss function is Lipschitz continuous with integrable Lipschitz constant. Namely, there exists $J_*$ such that
$$
J(\bdz)\triangleq \sup\limits_{\bdtheta_1\ne \bdtheta_2\in B_p(\textbf{0},r)}\frac{\|\nabla^2 l(\bdtheta_1;\bdz) - \nabla^2 l(\bdtheta_1;\bdz)\|_{op}}{\|\bdtheta_1-\bdtheta_2\|_2},
$$
$$
\EE\left[J(\bdZ)\right]\leq J_*.
$$
Further, there exists a constant $c_h$ such that $H\leq \tau^2 p^{c_h}$, $J_*\leq \tau^3 p^{c_h}$.
\end{enumerate}
Then, there exists a universal constant $C_0$, such that letting $C=C_0(c_h\vee \log(r\tau/\delta)\vee1)$, the following results hold:
\begin{enumerate}[label=(\alph*)]
\item The sample gradient converges uniformly to the population gradient in Euclidean norm. Namely, if $n\ge Cp\log p$, we have
$$
\PP\left( \sup\limits_{\bdtheta\in B_p(\textbf{0},r)}\|\nabla \hat{L}_n(\bdtheta)-\nabla L(\bdtheta)\|_2\leq \tau\sqrt{\frac{Cp\log n}{n}}\right)\ge 1-\delta.
$$
\item The sample Hessian converges uniformly to the population Hessian in operator norm. Namely, if $n\ge Cp\log p$, we have
$$
\PP\left( \sup\limits_{\bdtheta\in B_p(\textbf{0},r)}\|\nabla^2 \hat{L}_n(\bdtheta)-\nabla^2 L(\bdtheta)\|_{op}\leq \tau^2\sqrt{\frac{Cp\log n}{n}}\right)\ge 1-\delta.
$$
\end{enumerate}
\label{mei1}
\end{lemma}

\begin{lemma}
\label{lemma:gradientsmooth}
Under condition \ref{cm_1}, for any $\bdtheta \in \Theta$ and $\opt \in \Theta^{opt}$,
$$
\nabla F(\bdtheta) - A(\bdtheta-\opt) \leq \frac{C_{hl}}{2} \| \bdtheta- \opt\|^2.
$$
\end{lemma}

\begin{proof}
We have
$$
\def\arraystretch{1.5}
\begin{array}{rl}
& \nabla F(\bdtheta) - \nabla F(\opt)\\
= & \nabla F(\opt + (\bdtheta - \opt)) - \nabla F(\opt)\\
= & \int_0^1 \nabla^2 F(\opt + t(\bdtheta - \opt)) (\bdtheta - \opt) dt\\
= & A(\bdtheta - \opt) + \int_0^1 (\nabla^2 F(\opt + t(\bdtheta - \opt)) -A)(\bdtheta - \opt) dt.\\
\end{array}
$$
As $\nabla F(\opt)=0$, we have
$$
\def\arraystretch{1.5}
\begin{array}{rl}
& \lv \nabla F(\theta) - A(\bdtheta-\opt) \rv\\
\leq &\int_0^1 \lv(\nabla^2 F(\opt + t(\bdtheta - \opt)) -A)(\bdtheta - \opt)\rv dt\\
\leq & C_{hl} \lv \bdtheta - \opt \rv^2 \int_0^1 t dt\\
= & \frac{C_{hl}}{2} \lv \bdtheta - \opt \rv^2,\\
\end{array}
$$
where the 2nd step is based on condition \ref{cm_1}.
\end{proof}


\begin{lemma}[Azuma-Hoeffding]
Suppose that $\{r_n:n\ge 1\}$ is a martingale difference sequence with respect to a filtration $\{ \mathscr{G}_n:n\ge 0\}$, i.e.
$$
\EE(r_n|\mathscr{G}_{n-1}) = 0, n\ge 1.
$$
For each $n\ge 1$, there's a positive constant $b_n$ such that
$$
|r_n| \leq b_n.
$$
Then, for any $t > 0$,
$$
\PP\lpp \sum\limits_{n=1}\limits^{N} r_n \ge t \rpp \leq \exp\lpp -t^2 / \sum\limits_{n=1}\limits^{N} b_n^2 \rpp.
$$
    \label{lemma:afinequality}
\end{lemma}

\begin{lemma}
Suppose that $\{\bdv_n:n\ge 1\}$ is a martingale difference vector sequence with respect to a filtration $\{ \mathscr{G}_n:n\ge 0\}$, i.e.
$$
\EE(\bdv_n|\mathscr{G}_{n-1}) = \mathbf{0}, n\ge 1.
$$
For each $n\ge 1$, there's a positive $\mathscr{G}_{n-1}$-measurable random variable $u_n$ such that
$$
\lv \bdv_n\rv \leq u_n.
$$
Then, for any $\lambda,\delta > 0$,
$$
\PP\lpp \exists 1\leq k\leq N, \lv \ssum{n=1}{k}\bdv_n\rv \ge \lambda \ssum{n=1}{k} u_n^2 + \frac{1}{\lambda}\log \frac{2}{\delta}\rpp \leq \delta.
$$
    \label{lemma:afmartingale}
\end{lemma}

\begin{lemma}
Suppose that $a,b\in \ZZ_+$, $b \ge a+1$, $c\in \RR_+$ and $\alpha \in (0,1)$. If
$$
a \ge \lpp \frac{\alpha}{c(1-\alpha)} \rpp^{1/(1-\alpha)},
$$
we have
$$
\sum\limits_{n=a}\limits^{b-1} \frac{1}{n^{\alpha}}\exp(cn^{1-\alpha}) \leq \frac{1}{c(1-\alpha)} \lpp  \exp (cb^{1-\alpha}) -  \exp (ca^{1-\alpha}) \rpp.
$$
\label{lemma:int0}
\end{lemma}
\begin{proof}
We let
$$
h(x) = \frac{1}{x^{\alpha}}\exp(cx^{1-\alpha}).
$$
The derivative of $h$ is
$$
h'(x) = x^{-2\alpha}\exp(cx^{1-\alpha})(-\alpha x^{\alpha-1} + c(1-\alpha)),
$$
which is non-negative when $x\ge \lpp \frac{\alpha}{c(1-\alpha)} \rpp^{1/(1-\alpha)}$. It means that $h(x)$ is monotonically increasing on $\Big[ \lpp \frac{1\alpha}{c(1-\alpha)} \rpp^{1/(1-\alpha)}, \infty \rpp $. Therefore,
$$
\sum\limits_{n=a}\limits^{b-1} \frac{1}{n^{\alpha}}\exp(cn^{1-\alpha}) \leq \int_a^b \frac{1}{x^{\alpha}} \exp(cx^{1-\alpha}) dx  = \frac{1}{c(1-\alpha)} \exp(cx^{1-\alpha}) \big|_a^b. 
$$
\end{proof}

\begin{lemma}
Suppose that $a,b\in \ZZ_+$, $b \ge a+1$, $c\in \RR_+$ and $\alpha \in (0,1)$. For any $k\in \ZZ$ such that $k\ge 2$, if
$$
a \ge \lpp \frac{2(k-1)\alpha}{c(1-\alpha)} \rpp^{1/(1-\alpha)},
$$
we have
$$
\int_a^b \frac{1}{x^{k\alpha}} \exp(cx^{1-\alpha}) dx \leq \frac{2}{c(1-\alpha)} \lpp b^{-(k-1)\alpha} \exp (cb^{1-\alpha}) - a^{-(k-1)\alpha} \exp (ca^{1-\alpha}) \rpp
$$
and
$$
\sum\limits_{n=a}\limits^{b-1} \frac{1}{n^{k\alpha}}\exp(cn^{1-\alpha}) \leq \frac{2}{c(1-\alpha)} \lpp b^{-(k-1)\alpha} \exp (cb^{1-\alpha}) - a^{-(k-1)\alpha} \exp (ca^{1-\alpha}) \rpp.
$$
    \label{lemma:int1}
\end{lemma}
\begin{proof}
We let
$$
S = \int_a^b \frac{1}{x^{k\alpha}} \exp(cx^{1-\alpha}) dx.
$$
Then,
$$
\arraycolsep=1.1pt\def\arraystretch{1.5}
\begin{array}{rl}
S & = \int_a^b  \frac{x^{-(k-1)\alpha}}{c(1-\alpha)}d\exp(cx^{1-\alpha})\\

& = \frac{x^{-(k-1)\alpha}}{c(1-\alpha)} \exp(cx^{1-\alpha}) \big|_a^b + \int_a^b \frac{(k-1)\alpha x^{-(1+(k-1)\alpha)}}{c(1-\alpha)} \exp(cx^{1-\alpha}) dx\\

& = \frac{x^{-(k-1)\alpha}}{c(1-\alpha)} \exp(cx^{1-\alpha}) \big|_a^b + \int_a^b \frac{(k-1)\alpha}{c(1-\alpha)} x^{\alpha-1} x^{-k\alpha} \exp(cx^{1-\alpha}) dx\\

& \leq \frac{x^{-(k-1)\alpha}}{c(1-\alpha)} \exp(cx^{1-\alpha}) \big|_a^b + \frac{(k-1)\alpha}{c(1-\alpha)} a^{\alpha - 1} S\\

& \leq \frac{x^{-(k-1)\alpha}}{c(1-\alpha)} \exp(cx^{1-\alpha}) \big|_a^b + \frac{1}{2}S.\\
\end{array}
$$
Therefore, we have our first result
$$
S \leq \frac{2x^{-(k-1)\alpha}}{c(1-\alpha)} \exp(cx^{1-\alpha}) \big|_a^b.
$$
Next, we let
$$
h(x) = \frac{1}{x^{k\alpha}}\exp (cx^{1-\alpha}).
$$
The derivative of $h$ is
$$
h'(x) = x^{-2\alpha} \exp (cx^{1-\alpha}) (c(1-\alpha)x^{-\alpha} - k\alpha x^{-1}),
$$
which is non-negative when $x\ge \lpp \frac{2(k-1)\alpha}{c(1-\alpha)} \rpp^{1/(1-\alpha)}$. It implies that $h(x)$ is monotonically increasing on $\Big[ \lpp \frac{2(k-1)\alpha}{c(1-\alpha)} \rpp^{1/(1-\alpha)}, \infty \rpp $. Therefore,
$$
\sum\limits_{n=a}\limits^{b-1} \frac{1}{n^{k\alpha}}\exp(cn^{1-\alpha}) \leq \int_a^b \frac{1}{x^{k\alpha}} \exp(cx^{1-\alpha}) dx  \leq \frac{2x^{-(k-1)\alpha}}{c(1-\alpha)} \exp(cx^{1-\alpha}) \big|_a^b. 
$$
\end{proof}

\begin{lemma}
Suppose that $a,b,k\in \ZZ_+$, $b\ge a+1$, $c\in \RR_+$ and $\alpha \in (0,1)$. Then,
$$
\ssum{n=a+1}{b} \frac{1}{n^{k\alpha}} \exp(-cn^{1-\alpha}) \leq \frac{a^{-(k-1)\alpha}}{c(1-\alpha)} \exp (-ca^{1-\alpha}).
$$
    \label{lemma:int2}
\end{lemma}
\begin{proof}
We let
$$
h(x) = \frac{1}{x^{k\alpha}} \exp(-cx^{1-\alpha}).
$$
It is not hard to see that $h(x)$ is monotonically decreasing on $(0,\infty)$. Therefore,
$$
\begin{array}{rl}
\arraycolsep=1.1pt\def\arraystretch{1.5}
& \ssum{n=a+1}{b} \frac{1}{n^{k\alpha}} \exp(-cn^{1-\alpha})\\

\leq & \int_a^b \frac{1}{x^{k\alpha}} \exp(-cx^{1-\alpha}) dx\\

= & \int_a^b - \frac{1}{c(1-\alpha)}\frac{1}{x^{(k-1)\alpha}} d exp(-cx^{1-\alpha})\\

= & - \frac{1}{c(1-\alpha)}\frac{1}{x^{(k-1)\alpha}}  exp(-cx^{1-\alpha})\Big|_a^b - \int_a^b \frac{(k-1)\alpha}{c(1-\alpha)} x^{-(k-1)\alpha-1} \exp(cx^{1-\alpha}) dx\\

\leq & - \frac{1}{c(1-\alpha)}\frac{1}{x^{(k-1)\alpha}}  exp(-cx^{1-\alpha})\Big|_a^b\\

\leq & \frac{a^{-(k-1)\alpha}}{c(1-\alpha)} \exp (-ca^{1-\alpha}).\\
\end{array}
$$
\end{proof}
\begin{lemma}
Suppose that condition \ref{cm_2} holds. For $n\ge 0$, we have
$$
\EE_n\| \nabla f(\bdtheta_{n};\bdy_{n+1}) - \nablaf{n}\|^2 \leq 4\csg (1+\|\nablaf{n}\|^{\beta_{sg}}) \leq 4\csg\lppp \frac{4-\beta_{sg}}{2} + \frac{\beta_{sg}}{2} \|\nablaf{n}\|^2\rppp,
$$
and
$$
\EE_n\| \nabla f(\bdtheta_{n};\bdy_{n+1}) - \nablaf{n}\|^4 \leq 32(\csg)^2 \lpp 2-\frac{\beta_{sg}}{2} + \frac{\beta_{sg}}{2} \| \nablaf{n}\|^4\rpp.
$$

\label{lemma:errorbound}
\end{lemma}

\begin{proof}
For $n\ge0$, we have
\begin{equation*}
\arraycolsep=1.1pt\def\arraystretch{1.5}  
\begin{array}{rl}
\EE_n\| \nabla f(\bdtheta_{n};\bdy_{n+1}) - \nablaf{n}\|^2 &=\int_0^{\infty} \PP_n(\|\nabla f(\bdtheta_{n};\bdy_{n+1}) - \nablaf{n}\|^2 \ge x) dx\\

& \leq \int_0^{\infty} 2\exp \lppp -\frac{x}{2\csg (1+\|\nablaf{n}\|^{\beta_{sg}})} \rppp dx\\

& = 4\csg (1+\|\nablaf{n}\|^{\beta_{sg}})\\

& \leq 4\csg\lppp \frac{4-\beta_{sg}}{2} + \frac{\beta_{sg}}{2} \|\nablaf{n}\|^2\rppp,\\
\end{array}
\end{equation*}
where the 2nd step is based on condition \ref{cm_2} and the last step is based on the Young's inequality. We also have 
\begin{equation*}
\arraycolsep=1.1pt\def\arraystretch{1.5} 
\begin{array}{rl}
&\EE_{n} \| \nabla f(\bdtheta_{n};\bdy_{n+1}) - \nablaf{n} \|^4 \\

=&  \int_0^{\infty} \PP_{n} \lppp \| \nabla f(\bdtheta_{n};\bdy_{n+1}) - \nablaf{n}\|^4 \ge x\rppp dx  \\

\leq& 2 \int_0^{\infty} \exp \lpp -\frac{\sqrt{x}}{2\csg (1+\| \nablaf{n}\|^{\beta_{sg}})}\rpp dx\\

=& 16(\csg)^2 (1+\| \nablaf{n}\|^{\beta_{sg}})^2\\

\leq& 32 (\csg)^2 (1+\| \nablaf{n}\|^{2\beta_{sg}})\\

\leq& 32(\csg)^2 \lpp 2-\frac{\beta_{sg}}{2} + \frac{\beta_{sg}}{2} \| \nablaf{n}\|^4\rpp,\\
\end{array}
\end{equation*}
where the 2nd step is based on condition \ref{cm_2} and the last step is based on the Young's inequality.
\end{proof}

\begin{lemma}
Under condition \ref{cm_2}, we have,
$$
\EE \lv \nabla f(\bdtheta;\bdy) \rv^2 \leq \cng (1 + \| \nabla F(\bdtheta) \|^2),\ \forall\ \bdtheta \in \Theta,
$$
with $\cng = \max\lwww 8\csg \lpp 2 - \frac{\beta_{sg}}{2}\rpp , 2+4\beta_{sg} \csg\rwww$.
    \label{lemma:nablafbound}
\end{lemma}
\begin{proof}
We have
$$
\arraycolsep=1.1pt\def\arraystretch{1.5}  
\begin{array}{rl}
\EE \| \nabla f(\bdtheta;\bdy)\|^2 & \leq 2\| \nabla F(\bdtheta)\|^2 + 2\EE\| \nabla f(\bdtheta;\bdy) - \nabla F(\bdtheta)\|^2\\

& = 2\| \nabla F(\bdtheta)\|^2 + 2\int_0^{\infty}\PP(\| \nabla f(\bdtheta;\bdy) - \nabla F(\bdtheta)\|^2 \ge t )dt\\

& \leq 2\| \nabla F(\bdtheta)\|^2 + 4\int_0^{\infty} \exp \lppp \frac{-t}{2\csg (1+\|\nabla F(\bdtheta)\|^{\beta_{sg}})} \rppp dt\\

&= 2\| \nabla F(\bdtheta)\|^2 + 8\csg (1+\|\nabla F(\bdtheta)\|^{\beta_{sg}})\\

& \leq 2\| \nabla F(\bdtheta)\|^2 + 8\csg \lppp 1+\frac{\beta_{sg}}{2}\|\nabla F(\bdtheta)\|^2 + \frac{2-\beta_{sg}}{2} \rppp\\

&\leq \cng(1+\|\nabla F(\bdtheta)\|^2),
\end{array}
$$
where the 3rd step is based on condition \ref{cm_2} and the 5th step is based on the Young's inequality.
\end{proof}
\begin{corollary}
Suppose that condition \ref{cm_2} holds. For $p\in \ZZ_+$, there exist a constant $C_{ng}^{(p)}$ such that
$$
\EE \| \nabla f(\bdtheta;\bdy)\|^2 \leq C_{ng}^{(p)} (1+ \| \nabla F(\bdtheta)\|^p),\ \forall \bdtheta \in \Theta.
$$
\label{cor:fmomentbound}
\end{corollary}

\begin{lemma}
Under conditions \ref{cm_1} and \ref{cm_2}, suppose that $C \leq (C_s C_{ng}^{(2)})^{-1}$ and $n\in \ZZ_+$. Assume that
$$
F(\bdtheta_n) \leq C_F,
$$
with some constant $C_F\in \RR$. Then, for any $t\ge 1$, we have
$$
\ssum{j=1}{t}\gamma_{n+j} \EE_n \| \nablaf{n+j-1}\|^2 \leq C_F - F_{\min} + \frac{1}{2\alpha-1}C^2C_s C_{ng}^{(2)},
$$
where $F_{\min}$ and $C_s$ are constants defined in the assumptions and $\cng$ is a constant defined in Lemma \ref{lemma:nablafbound}.
\label{lemma:nablasumbound}
\end{lemma}
\begin{proof}
For $j\ge 1$, we have
$$
\arraycolsep=1.1pt\def\arraystretch{1.5}  
\begin{array}{rl}
& \EE_n \lppp F(\bdtheta_{n+j}) - F(\bdtheta_{n+j-1}) \rppp\\

\leq & \EE_n\langle \bdtheta_{n+j} - \bdtheta_{n+j-1}, \nablaf{n+j-1}\rangle +\frac{C_s}{2} \EE_n\| \bdtheta_{n+j} - \bdtheta_{n+j-1} \|^2\\

= & -\gamma_{n+j} \EE_n \| \nablaf{n+j-1}\|^2 +\frac{C_s}{2} \gamma_{n+j}^2 \EE_n \| \nabla f(\bdtheta_{n+j-1};\bdy_{n+j}) \|^2\\

\leq & -\gamma_{n+j} \EE_n \| \nablaf{n+j-1}\|^2 +\frac{C_s\cng}{2} \gamma_{n+j}^2 (1+\EE_n\| \nablaf{n+j-1}\|^2)\\

\leq & -\frac{\gamma_{n+j}}{2} \EE_n \| \nablaf{n+j-1}\|^2 + \frac{1}{2} C_s C_{ng}^{(2)}\gamma_{n+j}^2,\\
\end{array}
$$
where the 1st step is based on te condition \ref{cm_1}, the 3rd step is based on Lemma \ref{lemma:nablafbound} and the last step is based on the assumption that $C\leq (C_s \cng)^{-1}$. Then, we have
$$
\arraycolsep=1.1pt\def\arraystretch{1.5}  
\begin{array}{rl}
\ssum{j=1}{t}\gamma_{n+j} \EE_n \| \nablaf{n+j-1}\|^2 & \leq F(\bdtheta_n) - \EE_n F(\bdtheta_{n+t}) + C_s C_{ng}^{(2)}  \ssum{j=1}{t}\gamma_{n+j}^2\\

& \leq C_F - F_{\min} + \frac{1}{2\alpha-1}C^2C_s C_{ng}^{(2)} n^{1-2\alpha}\\

& \leq C_F - F_{\min} + \frac{1}{2\alpha-1}C^2C_s C_{ng}^{(2)}.\\
\end{array}
$$
\end{proof}
\begin{lemma}
Suppose that conditions \ref{cm_1} and \ref{cm_1} hold. $\bdtheta_n\in \rgo(\opt)$ for some $\opt \in \Theta^{opt}$. Then, we have
$$
\EE_n \|\bdtheta_{n+1}-\opt\|^6 \leq (1-2\gamma_{n+1} + O(\gamma_{n+1}^2)) \|\bdtheta_{n}-\opt\|^6 + O(\gamma_{n+1}^4).
$$
\label{lemma:sixbound}
\end{lemma}

\begin{proof}
For simplicity, in the current proof, we let $\Delta_n = \bdtheta_n - \opt$ and abbreviate $\nabla f(\bdtheta;\bdy_{n+1})$ as $\nabla f_{n+1}$. We firstly have
$$
\resizebox{.98\hsize}{!}{$
\arraycolsep=1.1pt\def\arraystretch{1.5}  
\begin{array}{rl}
& \| \Delta_{n+1}\|^6\\

= & \lppp \| \Delta_n - \gamma_{n+1} \nabla f_{n+1} \|^2 \rppp^3\\

= & \lppp \| \Delta_n\|^2 + \gamma_{n+1}^2 \| \nabla f_{n+1} \|^2 - 2\gamma_{n+1} \langle \Delta_n, \nabla f_{n+1} \rangle \rppp^3\\

= & \| \Delta_n\|^6 + \gamma_{n+1}^6 \| \nabla f_{n+1}\|^6 - 8\gamma_{n+1}^3 \langle \Delta_n, \nabla f_{n+1}\rangle^3 + 3\gamma_{n+1}^2 \|\Delta_n\|^4 \| \nabla f_{n+1}\|^2\\ 

& +3\gamma_{n+1}^4 \|\Delta_n\|^2 \| \nabla f_{n+1} \|^4 - 6\gamma_{n+1} \| \Delta_n\|^4 \langle \Delta_n, \nabla f_{n+1} \rangle + 12\gamma_{n+1}^2 \|\Delta_n\|^2 \langle \Delta_n, \nabla f_{n+1} \rangle^2\\

& -6\gamma_{n+1}^5 \|\nabla f_{n+1}\|^4 \langle \Delta_n, \nabla f_{n+1}\rangle +12 \gamma_{n+1}^4\|\nabla f_{n+1}\|^2\langle \Delta_n, \nabla f_{n+1}\rangle^2 - 12 \gamma_{n+1}^3 \|\Delta_n\|^2 \|\nabla f_{n+1}\|^2\langle \Delta_n, \nabla f_{n+1}\rangle\\

\leq & \| \Delta_n\|^6 + \gamma_{n+1}^6 \| \nabla f_{n+1}\|^6 + 8\gamma_{n+1}^3 \| \Delta_n\|^3 \|\nabla f_{n+1}\|^3 + 3\gamma_{n+1}^2 \|\Delta_n\|^4 \| \nabla f_{n+1}\|^2\\ 

& +3\gamma_{n+1}^4 \|\Delta_n\|^2 \| \nabla f_{n+1} \|^4 - 6\gamma_{n+1} \| \Delta_n\|^4 \langle \Delta_n, \nabla f_{n+1} \rangle + 12\gamma_{n+1}^2 \|\Delta_n\|^4 \| \nabla f_{n+1} \|^2\\

& +6\gamma_{n+1}^5 \|\Delta_n\| \|\nabla f_{n+1}\|^5  +12 \gamma_{n+1}^4 \| \Delta_n\|^2 \|\nabla f_{n+1}\|^4  + 12 \gamma_{n+1}^3 \|\Delta_n\|^3 \|\nabla f_{n+1}\|^3\\

= & \| \Delta_n\|^6 - 6\gamma_{n+1} \| \Delta_n\|^4 \langle \Delta_n, \nabla f_{n+1} \rangle + 15\gamma_{n+1}^2 \|\Delta_n\|^4 \| \nabla f_{n+1}\|^2 + 20\gamma_{n+1}^3 \| \Delta_n\|^3 \|\nabla f_{n+1}\|^3\\

&+ 15 \gamma_{n+1}^4 \| \Delta_n\|^2 \|\nabla f_{n+1}\|^4 +6\gamma_{n+1}^5 \|\Delta_n\| \|\nabla f_{n+1}\|^5 + \gamma_{n+1}^6 \| \nabla f_{n+1}\|^6.\\
\end{array}
$}
$$
Then,
$$
\resizebox{.98\hsize}{!}{$
\arraycolsep=1.1pt\def\arraystretch{1.5}  
\begin{array}{rl}
& \EE_n \| \Delta_{n+1}\|^6\\

\leq & \| \Delta_n\|^6 - 6\gamma_{n+1} \| \Delta_n\|^4 \langle \Delta_n, \nablaf{n} \rangle  + 15 C_{ng}^{(2)} \gamma_{n+1}^2 \|\Delta_n\|^4 (1+\|\nablaf{n}\|^2) + 20 C_{ng}^{(3)} \gamma_{n+1}^3 \|\Delta_n\|^3 (1+\|\nablaf{n}\|^3)\\

& + 15 C_{ng}^{(4)} \gamma_{n+1}^4 \|\Delta_n\|^2 (1+\|\nablaf{n}\|^4) + 6 C_{ng}^{(5)} \gamma_{n+1}^5 \|\Delta_n\| (1+\|\nablaf{n}\|^5) + C_{ng}^{(6)} \gamma_{n+1}^6 (1+\|\nablaf{n}\|^6)\\

\leq & (1-3\tlambda \gamma_{n+1})\|\Delta_n\|^6   + 15 C_{ng}^{(2)} \gamma_{n+1}^2 \|\Delta_n\|^4 (1+\|\nablaf{n}\|^2) + 20 C_{ng}^{(3)} \gamma_{n+1}^3 \|\Delta_n\|^3 (1+\|\nablaf{n}\|^3)\\

& + 15 C_{ng}^{(4)} \gamma_{n+1}^4 \|\Delta_n\|^2 (1+\|\nablaf{n}\|^4) + 6 C_{ng}^{(5)} \gamma_{n+1}^5 \|\Delta_n\| (1+\|\nablaf{n}\|^5) + C_{ng}^{(6)} \gamma_{n+1}^6 (1+\|\nablaf{n}\|^6)\\

\leq & (1-3\tlambda \gamma_{n+1})\|\Delta_n\|^6   + 15 C_{ng}^{(2)} \gamma_{n+1}^2 \|\Delta_n\|^4 (1+C_s^2 \|\Delta_{n}\|^2) + 20 C_{ng}^{(3)} \gamma_{n+1}^3 \|\Delta_n\|^3 (1+C_s^3\|\Delta_{n}\|^3)\\

& + 15 C_{ng}^{(4)} \gamma_{n+1}^4 \|\Delta_n\|^2 (1+C_s^4\|\Delta_{n}\|^4) + 6 C_{ng}^{(5)} \gamma_{n+1}^5 \|\Delta_n\| (1+C_s^5\|\Delta_{n}\|^5) + C_{ng}^{(6)} \gamma_{n+1}^6 (1+C_s^6\|\Delta_{n}\|^6)\\

\leq & \lppp 1-3\tlambda \gamma_{n+1} + 15C_{ng}^{(2)}C_s^2\gamma_{n+1}^2 + 20 C_{ng}^{(3)}C_s^3\gamma_{n+1}^3 + 15 C_{ng}^{(4)}C_s^4\gamma_{n+1}^4 +6C_{ng}^{(5)}C_s^5\gamma_{n+1}^5 +  C_{ng}^{(6)}C_s^6\gamma_{n+1}^6 \rppp\|\Delta_n\|^6\\

&+ 15 C_{ng}^{(2)}  \lpp \frac{1}{3} \lpp \lppp 10 C_{ng}^{(2)}\rppp^{\frac{2}{3}} \lppp \tlambda \rppp^{-\frac{2}{3}} \gamma_{n+1}^{\frac{4}{3}} \rpp^3 + \frac{2}{3} \lpp \lppp 10 C_{ng}^{(2)}\rppp^{-\frac{2}{3}}  \lppp \tlambda \rppp^{\frac{2}{3}} \gamma_{n+1}^{\frac{2}{3}} \| \Delta_n\|^4 \rpp^{\frac{3}{2}} \rpp + 10 C_{ng}^{(3)} \lppp \gamma_{n+1}^4 + \gamma_{n+1}^2 \| \Delta_n\|^6 \rppp\\

& + 15 C_{ng}^{(4)} \lpp \frac{1}{3} \lpp \gamma_{n+1}^{\frac{4}{3}} \|\Delta_n\|^2 \rpp^3 + \frac{2}{3} \lpp \gamma_{n+1}^{\frac{8}{3}} \rpp^{\frac{3}{2}} \rpp + 6 C_{ng}^{(5)} \lpp \frac{1}{6} \lpp \gamma_{n+1}^{\frac{5}{3}} \| \Delta_n\|\rpp^6 + \frac{5}{6} \lpp \gamma_{n+1}^{\frac{10}{3}} \rpp^{\frac{6}{5}} \rpp + C_{ng}^{(6)} \gamma_{n+1}^6\\

= & (1-2\tlambda \gamma_{n+1} + O(\gamma_{n+1}^2)) \|\Delta_n\|^6 + O(\gamma_{n+1}^4),\\
\end{array}
$}
$$
where the 1st step is base on Corollary \ref{cor:fmomentbound}, the 2rd step is based on condition \ref{cm_1}, the 3rd step is based on condition \ref{cm_1} and the 4th step is based on the Young's inequality.
\end{proof}


\begin{lemma}[\cite{jin2019short}]
Suppose that random vectors $\bdX_1,\ldots,\bdX_n \in \RR^d$ and corresponding flirtations $\fff_i =\sigma(\bdX_1,\ldots, \bdX_i)$ satisfy that $\bdX_i|\fff_{i-1}$ is zero-mean $\sigma_i^2$-norm-subGaussian with $\sigma_i\in \fff_{i-1}$, for $i= 0,1,2,\ldots,n$. For any fixed $\delta>0$, and $B>b>0$, there's an absolute constant $\cjin$ such that, with probability at least $1-\delta$, either
$$
\ssum{i=1}{n}\sigma_i^2 \ge B,
$$
or
$$
\lvvv \ssum{i=1}{n}\bdX_i\rvvv \leq \cjin \sqrt{\max\lwww \ssum{i=1}{n}\sigma_i^2,b\rwww \lppp \log \frac{2d}{\delta}+\log \log \frac{B}{b} \rppp}.
$$
    \label{lemma:jin}
\end{lemma}

\begin{corollary}
Under the conditions given in Lemma \ref{lemma:jin}, if we additionally assume that $\sigma_i$ is a constant, for $i=1,2,\ldots, n$, with probability at least $1-\delta$,
$$
\lvvv \ssum{i=1}{n}\bdX_i\rvvv \leq \cjin \lpp \ssum{i=1}{n}\sigma_i^2 \rpp^{\frac{1}{2}} \lpp \log \frac{2d}{\delta} \rpp^{\frac{1}{2}}.
$$
    \label{cor:jin}
\end{corollary}



\section{Supplementary Results for Application Examples}


\subsection{Example 1: Gaussian Mixture Model}
\label{suppexample2}
\subsubsection{Verification of Conditions}
\begin{itemize} 
\item \ref{LC1}
Based on the Gaussian tail assumption, we can immediately know that $\EE\|\nabla f(\opt;\bdz)\|^{2+\tau}<\infty$, for any $\tau>0$. Next, we show that the second part of condition \ref{LC1} is also valid. For any $\tau>0$, $\delta>0$, $\|\bdtheta^*-\opt\|\leq \delta$, we have
$$
\arraycolsep=1.1pt\def\arraystretch{1.5}
\begin{array}{rl}
& \EE\|\nabla f(\opt;\bdz) - \nabla f(\bdtheta^*;\bdz)\|^{2+\tau}\\

= & \EE \left\| \frac{1}{\sigma^2}\lppp\opt+(1-2\phi(\bdy,\opt))\bdy\rppp - \frac{1}{\sigma^2}\lppp\bdtheta^*+(1-2\phi(\bdy,\bdtheta^*))\bdy\rppp \right\|^{2+\tau}\\

\leq & C_1 \lppp \|\opt-\bdtheta^*\|^{2+\tau} + \EE\lvvv\frac{2}{\sigma^2} (\phi(\bdy,\opt) - \phi(\bdy,\bdtheta^*))\bdy\rvvv^{2+\tau}\rppp\\

\leq & C_1 \lppp \|\opt-\bdtheta^*\|^{2+\tau} + \lppp\frac{2}{\sigma^2}\rppp^{2(2+\tau)} \EE \lppp |\langle y,\opt\rangle - \langle y,\bdtheta^*\rangle |^{2+\tau} \| \bdy\|^{2+\tau} \rppp\\

\leq & C_2 \|\opt-\bdtheta^*\|^{2+\tau}, \\
\end{array}
$$
where $C_1$ and $C_2$ are some constants potentially dependent on $\tau,\sigma^2,\delta$ and $\mathcal{P}_Y$.\\

\item \textbf{(Lower Bound)} Recall that
$$
F(\bdtheta) \triangleq \frac{1}{\sigma^2}\left(\frac{1}{2}\|\bdtheta\|^2 - 2\EE\left[s(\langle y,\bdtheta\rangle\right] \right),
$$
where $s(t)\triangleq \int^t_0 \frac{1}{1+e^{-\frac{2}{\sigma^2}u}}du$. Then, We have
\begin{equation}
\arraycolsep=1.1pt\def\arraystretch{1.5}
\begin{array}{rcl}
\sigma^2F(\bdtheta) & \ge & \frac{1}{2}\|\bdtheta\|^2 - 2\EE|\langle \bdy,\bdtheta\rangle | =  \frac{1}{2}\|\bdtheta\|^2 - 2\EE|\langle (\opt+\bdv),\bdtheta\rangle | \\

& \ge & \frac{1}{2}\|\bdtheta\|^2 - 2 |\langle \opt,\bdtheta\rangle | - 2 \EE|\langle \bdv,\bdtheta \rangle |  \ge  \frac{1}{2}\|\bdtheta\|^2 - 2 \|\opt\| \| \bdtheta\| - 2(\EE\langle \bdv,\bdtheta\rangle^2)^{\frac{1}{2}} \\

& = & \frac{1}{2}\|\bdtheta\|^2 - 2 \|\opt\| \| \bdtheta\| - 2\frac{1}{\sigma} \|\bdtheta\| \ge  \frac{1}{2}\|\bdtheta\|^2 - \frac{1}{4} \| \bdtheta\|^2 - 4\lppp \| \opt \| + \frac{1}{\sigma}\rppp^2\\

& = & \frac{1}{4}\|\bdtheta\|^2  - 4\lppp \| \opt \| + \frac{1}{\sigma}\rppp^2,\\
\end{array}
\label{gmmsupp0}
\end{equation}
where $\bdv$ is a random vector following $N\lppp\bm{0},\frac{1}{\sigma^2} I_d\rppp$. Based on (\ref{gmmsupp0}), we can see that $F$ is lower-bounded by $-\frac{4}{\sigma^2} \lppp \| \opt \| + \frac{1}{\sigma}\rppp^2$.

\item \ref{cm_1} We can check that
$$
\nabla^2F(\bdtheta) = \frac{1}{\sigma^2}\Bigg( I_d - \frac{4}{\sigma^2}\EE\frac{\bdy\bdy^T}{\lppp e^{-\frac{1}{\sigma^2}\langle \bdy,\bdtheta\rangle} + e^{\frac{1}{\sigma^2}\langle \bdy,\bdtheta\rangle}\rppp^2}\Bigg),
$$
which indicates that $\| \nabla^2 F(\bdtheta)\| \leq \frac{1}{\sigma^2}$.

To show Hessian smoothness, we temporarily let $g(x)\triangleq (e^{-x}+e^x)^{-2}$. Then, we have
$$
|g'(x)| = 2\big| (e^{-x} + e^x)(e^x - e^{-x})\big|^{-3} \leq 2 (e^{-x}+e^x)^{-2} \leq \frac{1}{2}.
$$
Then, we can have
$$
\arraycolsep=1.1pt\def\arraystretch{1.5}
\begin{array}{rl}
& \lvvv \nabla^2 F(\bdtheta_1) - \nabla^2 F(\bdtheta_2)\rvvv\\

= & \frac{4}{\sigma^4} \lvv \EE \lpp \lppp g\lppp \frac{1}{\sigma^2} \langle \bdy, \bdtheta_1\rangle \rppp - g\lppp \frac{1}{\sigma^2} \langle \bdy, \bdtheta_2\rangle \rppp \rppp \bdy \bdy^T \rvv\\

\leq & \frac{4}{\sigma^4} \EE \lppp \big|  g\lppp \frac{1}{\sigma^2} \langle \bdy, \bdtheta_1\rangle \rppp - g\lppp \frac{1}{\sigma^2} \langle \bdy, \bdtheta_2\rangle \rppp\big| \|\bdy\|^2 \rppp\\

\leq & \frac{2}{\sigma^6} \|\bdtheta_1 - \bdtheta_2\| \EE \|\bdy\|^3,\\
\end{array}
$$
which validates condition \ref{cm_1}.

As we have mentioned in the main body, according to \cite{balakrishnan2017statistical}, when $\| \opt\|$ is sufficiently large, $\nabla^2 F(\opt)=\nabla^2 F(-\opt)$ will be positive definite. It is also not hard to see that the smallest eigenvalue of $\nabla^2 F(\bdtheta)$ is continuous in $\bdtheta$. 

\item \ref{cm_2} We can see that
$$
\arraycolsep=1.1pt\def\arraystretch{1.5}
\begin{array}{rl}
& \lvvv \nabla f(\bdtheta;\bdz) - \nabla F(\bdtheta)\rvvv\\

= & \lvvv \frac{1}{\sigma^2}(1-2\phi(\bdy, \bdtheta))\bdy - \EE \lppp \frac{1}{\sigma^2}(1-2\phi(\bdy, \bdtheta))\bdy \rppp + \bdxi \rvvv\\

\leq &  \frac{1}{\sigma^2} \|\bdy\| +  \frac{1}{\sigma^2} \EE \| \bdy\| + \EE\|\bdxi\|,
\end{array}
$$
from which we can easily know that there exists a constant $\csg$ independent of $\bdtheta$ such that $\nabla f(\bdtheta;\bdz) - \nabla F(\bdtheta)$ is $\csg$-norm-subGaussian.

As we show in the proof of Proposition \ref{gaussianprop}, which will be presented shortly, we can see that
$$
\inf\limits_{\bdtheta} \EE|\langle \nabla f(\bdtheta;\bdz) - \nabla F(\bdtheta),\bdnu\rangle|
$$
is lower-bounded by some constant. We also know that
$$
\inf\limits_{\bdtheta} \EE|\langle \nabla f(\bdtheta;\bdz) - \nabla F(\bdtheta),\bdnu\rangle|^2 \ge \lpp \inf\limits_{\bdtheta} \EE|\langle \nabla f(\bdtheta;\bdz) - \nabla F(\bdtheta),\bdnu\rangle| \rpp^2,
$$
which can immediately show the validity of \ref{cm_2}.

\item \ref{cm_3}
Based on (\ref{gmmsupp0}), we can immediately know that condition \ref{cm_3} is true.

\item \ref{cm_4} We can check that
$$
\nabla F(\bdtheta) = \frac{1}{\sigma^2} \lpp \bdtheta + \lpp 1-\frac{2}{1+e^{-\frac{2}{\sigma^2}\langle \bdy, \bdtheta\rangle }} \rpp \bdy \rpp.
$$
Then, it is not hard to see that $\|\nabla F(\bdtheta)\|$ will go to infinity as $\|\bdtheta\| \rightarrow \infty$. Therefore, for every sufficiently small constant $r_1$, if we let
$$
b = \sup\limits_{\bdtheta:\| \bdtheta\| \leq r_1} \| \nabla F(\bdtheta)\|,
$$
there exist a constant $r_2$ such that
$$
\sup\limits_{\bdtheta:\| \bdtheta - \opt\| \leq r_2\ \text{or}\ \| \bdtheta + \opt\| \leq r_2} \| \nabla F(\bdtheta)\| = b,
$$
and
$$
\inf \{ \|\nabla F(\bdtheta)\|: \|\bdtheta\| > r_1, \| \bdtheta - \opt\| > r_2,\| \bdtheta + \opt\| > r_2\} = b.
$$
As $\lambda_0\triangleq \lambda_{\min}(\nabla^2 F(\bm{0}))<0$, we can choose a sufficiently small $r_1$ such that
$$
\lambda_{\min}(\nabla^2 F(\bdtheta)) \leq \frac{1}{2} \lambda_0,\ \|\bdtheta\| \leq r_1.
$$
Then, for any $\bdtheta$, if $\|\nabla F(\bdtheta)\| \leq b$ and $\lambda_{\min}(\nabla^2 F(\bdtheta)) > \frac{1}{2} \lambda_0$, we must have $\| \bdtheta - \opt\| \leq r_2$ or $\| \bdtheta + \opt\| \leq r_2$.

\item \ref{cm_5} We can check that
$$
\nabla^2 f(\bdtheta;\bdz) = \frac{1}{\sigma^2} \Bigg( I_d - \frac{4}{\sigma^2} \frac{\bdy \bdy^T}{\lppp e^{-\frac{1}{\sigma^2}\langle \bdy, \bdtheta\rangle} + e^{\frac{1}{\sigma^2}\langle \bdy, \bdtheta\rangle}\rppp^2} \Bigg).\\
$$
Therefore, we have
$$
\arraycolsep=1.1pt\def\arraystretch{1.5}
\begin{array}{rl}
 \EE \lvvv \nabla^2 f(\bdtheta;\bdz) \rvvv^2 & =  \frac{1}{\sigma^4} \EE \Bigg\| I_d - \frac{4}{\sigma^2} \frac{\bdy \bdy^T}{\lppp e^{-\frac{1}{\sigma^2}\langle \bdy, \bdtheta\rangle} + e^{\frac{1}{\sigma^2}\langle \bdy, \bdtheta\rangle}\rppp^2}\Bigg\|^2\\

 & \leq \frac{2}{\sigma^4} \Bigg(1+ \frac{16}{\sigma^4} \EE \Bigg\|  \frac{\bdy \bdy^T}{\lppp e^{-\frac{1}{\sigma^2}\langle \bdy, \bdtheta\rangle} + e^{\frac{1}{\sigma^2}\langle \bdy, \bdtheta\rangle}\rppp^2}\Bigg\|^2\Bigg)\\

 & \leq \frac{2}{\sigma^4} \lpp 1+ \frac{16}{\sigma^4} \EE \|\bdy\|^4 \rpp,\\
\end{array}
$$
which is a constant independent of $\bdtheta$. Next, for simplicity, we let
$$
\mathscr{U} = [u_{ijk}]_{1\leq i \leq d,1\leq j \leq d,1\leq k \leq d} \triangleq \nabla^3 f(\bdtheta;\bdz),
$$
where
$$
u_{ijk}=\frac{8}{\sigma^6} \frac{e^{-\frac{1}{\sigma^2}\langle \bdy, \bdtheta\rangle} - e^{\frac{1}{\sigma^2}\langle \bdy, \bdtheta\rangle}}{\lpp e^{-\frac{1}{\sigma^2}\langle \bdy, \bdtheta\rangle} + e^{\frac{1}{\sigma^2}\langle \bdy, \bdtheta\rangle} \rpp^3} y_i y_j y_k.
$$
Therefore,
$$
\arraycolsep=1.1pt\def\arraystretch{1.5}
\begin{array}{rl}
\EE \| \mathscr{U}\|^2 & = \EE \lpp\ssum{i=1}{d}\ssum{j=1}{d}\ssum{k=1}{d} u_{ijk}^2\rpp = \EE \Bigg( \Bigg(\frac{8}{\sigma^6} \frac{e^{-\frac{1}{\sigma^2}\langle \bdy, \bdtheta\rangle} - e^{\frac{1}{\sigma^2}\langle \bdy, \bdtheta\rangle}}{\lpp e^{-\frac{1}{\sigma^2}\langle \bdy, \bdtheta\rangle} + e^{\frac{1}{\sigma^2}\langle \bdy, \bdtheta\rangle} \rpp^3} \Bigg)^2 \ssum{i=1}{d}\ssum{j=1}{d}\ssum{k=1}{d} y_i^2 y_j^2 y_k^2 \Bigg)\\

& \leq \frac{64}{\sigma^{12}} \EE \frac{\|\bdy\|^6}{\lpp e^{-\frac{1}{\sigma^2}\langle \bdy, \bdtheta\rangle} + e^{\frac{1}{\sigma^2}\langle \bdy, \bdtheta\rangle} \rpp^4} \leq \frac{4}{\sigma^{12}} \EE \| \bdy\|^6,\\
\end{array}
$$
which is a constant independent of $\bdtheta$. It is not hard to see that we can verify a uniform upper-bound on $\EE \lvvv \nabla^4 f(\bdtheta;\bdz)\rvvv^2$ in a similar way.
\end{itemize}

\subsubsection{Proof of Proposition \ref{gaussianprop}}
\begin{proof}


We notice that
$$
\nabla^2F(\bdtheta) = \frac{1}{\sigma^2}\Bigg( I_d - \frac{4}{\sigma^2}\EE\frac{\bdy\bdy^T}{\lppp e^{-\frac{1}{\sigma^2}\langle \bdy,\bdtheta\rangle} + e^{\frac{1}{\sigma^2}\langle \bdy,\bdtheta\rangle}\rppp^2}\Bigg),
$$
and therefore the largest eigenvalue of the Hessian matrix $\lambda_{\max}(\nabla^2 F(\bdtheta))$ is upper bounded by $\frac{1}{\sigma^2}$ uniformly. Therefore, for some $\tilde{\bdtheta}_n$ between $\bdtheta_{n+1}$ and $\bdtheta_n$, we have
$$
\arraycolsep=1.1pt\def\arraystretch{1.5}
\begin{array}{rcl}
F(\bdtheta_{n+1}) & = & F(\bdtheta_n) + \langle \nabla F(\bdtheta_n),\bdtheta_{n+1}-\bdtheta_n\rangle + \frac{1}{2} (\bdtheta_{n+1}-\bdtheta_n)^T \nabla^2 F(\tilde{\bdtheta}_n) (\bdtheta_{n+1}-\bdtheta_n)\\

& \leq & F(\bdtheta_n) + \langle \nabla F(\bdtheta_n),\bdtheta_{n+1}-\bdtheta_n\rangle + \frac{1}{2\sigma^2}\| \bdtheta_{n+1}-\bdtheta_n\|^2\\

& = & F(\bdtheta_n) + \langle \nabla F(\bdtheta_n), -\gamma_{n+1} \nabla f(\bdtheta_n;\bdz_{n+1}) \rangle +  \frac{1}{2\sigma^2}\gamma_{n+1}^2\| \nabla f(\bdtheta_n;\bdz_{n+1})\|^2.\\
\end{array}
$$
By taking conditional expectation on both sides, we have
\begin{equation}
\EE_n F(\bdtheta_{n+1}) \leq F(\bdtheta_n) - \gamma_{n+1}\|\nabla F(\bdtheta_n)\|^2 + \frac{1}{2\sigma^2} \gamma_{n+1}^2 \EE_n \| \nabla f(\bdtheta_n;\bdz_{n+1}) \|^2.
\label{gmmsupp1}
\end{equation}
To bound the last term, we have
$$
\arraycolsep=1.1pt\def\arraystretch{1.5}
\begin{array}{rl}
& \EE\| \nabla f(\bdtheta;\bdz)\|^2\\

= &\EE\| \frac{1}{\sigma^2}(\bdtheta+(1-2\phi(\bdy,\bdtheta))\bdy)\|^2 + d\sigma^2_{\xi}\\

\leq & \frac{2}{\sigma^4}  (\|\bdtheta\|^2 + \EE\| (1-2\phi(\bdy,\bdtheta))\bdy\|^2) + d\sigma^2_{\xi}\\

\leq & \frac{2}{\sigma^4}  (\|\bdtheta\|^2 + \EE\| \bdy\|^2) + d\sigma^2_{\xi}\\

= & \frac{2}{\sigma^4}  (\|\bdtheta\|^2 + \EE\| \opt+\bdv\|^2) + d\sigma^2_{\xi}\\

= & \frac{2}{\sigma^4}  (\|\bdtheta\|^2  +\|\opt\|^2 +  \EE\|\bdv\|^2) + d\sigma^2_{\xi}\\

= & \frac{2}{\sigma^4}  (\|\bdtheta\|^2 + \|\opt\|^2 + d\sigma^2) + d\sigma^2_{\xi}\\

\triangleq & \frac{2}{\sigma^4}\|\bdtheta\|^2 + b_1,
\end{array}
$$
where $\bdv$ is a random vector following $N\lppp\bm{0},\frac{1}{\sigma^2} I_d\rppp$. Therefore, we have
\begin{equation}
\EE_n \| \nabla f(\bdtheta_n;\bdz_{n+1})\|^2 \leq \frac{2}{\sigma^4}\|\bdtheta_n\|^2 + b_1.
\label{gmmsupp2}
\end{equation}
Based on (\ref{gmmsupp0}), we know that 
$$
\|\bdtheta\|^2 \leq 4\sigma^2 F(\bdtheta)+16 \lppp \|\opt\|+\frac{1}{\sigma}\rppp^2.
$$
Plugging it back to (\ref{gmmsupp2}), we have
\begin{equation}
\EE_n\| \nabla f(\bdtheta_n;\bdz_{n+1})\|^2 \leq \frac{8}{\sigma^4} F(\bdtheta_n)+ \frac{32}{\sigma^4}  \lppp \|\opt\|+\frac{1}{\sigma}\rppp^2  + b_1.
\label{gmmsupp3}
\end{equation}
Based on (\ref{gmmsupp1}) and (\ref{gmmsupp3}), if we let $b_2 = \frac{32}{\sigma^4}  \lppp \|\opt\|+\frac{1}{\sigma}\rppp^2  + b_1$, we have
$$
\arraycolsep=1.1pt\def\arraystretch{1.5}
\begin{array}{rcl}
\EE_n F(\bdtheta_{n+1}) & \leq &  F(\bdtheta_n) - \gamma_{n+1}\|\nabla F(\bdtheta_n)\|^2 + \frac{4}{\sigma^6} \gamma_{n+1}^2F(\bdtheta_n) + \frac{1}{2\sigma^2}\gamma_{n+1}^2b_2\\

 & = & \left(1+ \frac{4}{\sigma^6}\right) F(\bdtheta_n) + \frac{1}{2\sigma^2}\gamma_{n+1}^2b_2 - \gamma_{n+1}\|\nabla F(\bdtheta_n)\|^2.\\
\end{array}
$$
Applying Lemma \ref{rs71}, we can conclude that $\{F(\bdtheta_n)\}_{n\ge1}$ and $\sum\gamma_{n+1}\|\nabla F(\bdtheta_n)\|^2$ converge almost-surely. Hence, $\|\nabla F(\bdtheta_n)\|\rightarrow 0$ almost-surely, which further implies that $\bdtheta_n\rightarrow -\opt,\opt$ or $\textbf{0}$ almost-surely.

To show that $\{\bdtheta_n\}_{n\ge 1}$ is never trapped at $\textbf{0}$, we just need to show $\nabla f(\bdtheta;\bdz)$ is stochastically rich in every direction. For any $\bdnu\in \RR^d$ with $\|\bdnu\|=1$, if we denote 
$$
B_1\triangleq \left\{\biggl< \left(1-\frac{2}{1+e^{-\frac{2}{\sigma^2}\langle \bdy,\bdtheta\rangle}}\right)\bdy - \EE\left[ (1-\frac{2}{1+e^{-\frac{2}{\sigma^2}\langle \bdy,\bdtheta\rangle}})\bdy\right],\bdnu\biggr> >0\right\},
$$
$$
B_2 \triangleq \{\langle \bdxi,\bdnu\rangle \}>0\},
$$
we have
$$
\arraycolsep=1.1pt\def\arraystretch{1.5}
\begin{array}{rl}
& \inf\limits_{\bdtheta} \EE|\langle \nabla f(\bdtheta;\bdz) - \nabla F(\bdtheta),\bdnu\rangle|\\

= &\inf\limits_{\bdtheta} \frac{1}{\sigma^2}  \EE\left|\biggl< \left(1-\frac{2}{1+e^{-\frac{2}{\sigma^2}\langle \bdy,\bdtheta\rangle}}\right)\bdy - \EE\left[ (1-\frac{2}{1+e^{-\frac{2}{\sigma^2}\langle \bdy,\bdtheta\rangle}})\bdy\right] +\bdxi,\bdnu\biggr> \right|\\

\ge & \inf\limits_{\bdtheta}\left\{ \frac{1}{\sigma^2}  \EE\left[\biggl< \left(1-\frac{2}{1+e^{-\frac{2}{\sigma^2}\langle \bdy,\bdtheta\rangle}}\right)\bdy - \EE\left[ (1-\frac{2}{1+e^{-\frac{2}{\sigma^2}\langle \bdy,\bdtheta\rangle}})\bdy\right] +\bdxi,\bdnu\biggr> \daone_{B_1}\daone_{B_2}\right] \right.\\

& \left. - \frac{1}{\sigma^2}  \EE\left[\biggl< \left(1-\frac{2}{1+e^{-\frac{2}{\sigma^2}\langle \bdy,\bdtheta\rangle}}\right)\bdy - \EE\left[ (1-\frac{2}{1+e^{-\frac{2}{\sigma^2}\langle \bdy,\bdtheta\rangle}})\bdy\right] +\bdxi,\bdnu\biggr> \daone_{B_1^c}\daone_{B_2^c}\right] \right\}\\

\ge  & \inf\limits_{\bdtheta}\left\{\frac{1}{\sigma^2}  \EE\left[ \langle \bdxi,\bdnu\rangle\daone_{B_2}  \right]\PP(B_1) - \frac{1}{\sigma^2}\EE\left[ \langle \bdxi,\bdnu\rangle\daone_{B_2^c}  \right]\PP(B_1^c)\right\}\\

\ge &  \inf\limits_{\bdtheta}\left\{\frac{1}{2\sigma^2}  \EE\left[ \langle \bdxi,\bdnu\rangle\daone_{B_2}  \right]\daone\left\{\PP(B_1)\ge \frac{1}{2}\right\} - \frac{1}{2\sigma^2}\EE\left[ \langle \bdxi,\bdnu\rangle\daone_{B_2^c}  \right]\daone\left\{\PP(B_1^c)\ge \frac{1}{2}\right\} \right\}\\

= & \frac{1}{4\sigma^2}  \EE\left[\left| \langle \bdxi,\bdnu\rangle\right|  \right] \\

\ge & C_{\sigma^2,\sigma^2_{\xi}},
\end{array}
$$
where $C_{\sigma^2,\sigma^2_{\xi}}$ is a constant dependent on $(\sigma^2,\sigma^2_{\xi})$. Therefore, according to Theorem 1 in \cite{brandiere1996algorithmes}, with probability 1, $\{\bdtheta_n\}_{n\ge1}$ only converges to local minimizers of $F$. Finally, we show that $\textbf{0}$ is an unstable saddle point at which the Hessian matrix has a negative eigenvalue. We have
$$
\nabla^2F(\textbf{0}) = \frac{1}{\sigma^2}\left(I_d - \frac{1}{\sigma^2}\EE\left[\bdy\bdy^T\right]\right) = \frac{1}{\sigma^2}\left(I_d - \frac{1}{\sigma^2}(\opt{\opt}^T + \sigma^2 I_d)\right) = -\frac{1}{\sigma^4} \opt{\opt}^T.
$$
Therefore, $\nabla^2F(\textbf{0})$ is not semi-positive definite and $\{\bdtheta_n\}_{n\ge1}$ never converges to $\textbf{0}$.
\end{proof}


\subsection{Example 2: Logistic Regression with Concave Regularization}
\label{suppexample1}

\subsubsection{Discussions on Conditions}
\begin{itemize}
\item \ref{LC1} \& \ref{LC2} Let $i_u$ be a random index uniformly selected from $\{1,2,\ldots,M\}$. Then, for any $\tau>0$, it is easy to know that $\EE\|\nabla f_M(\opt;i_u)\|^{2+\tau}<\infty$ because $i_u$ is from a distribution of a finite support. Next, for any $\delta>0$ and any $\|\bdtheta^*-\opt\|<\delta$, we have
$$
\resizebox{.94\hsize}{!}{$
\arraycolsep=1.1pt\def\arraystretch{1.5}
\begin{array}{rl}
& \EE\|\nabla f_M(\opt;i_u) - \nabla f_M(\bdtheta^*;i_u)\|^{2+\tau}\\

=& \frac{1}{M}\sum\limits_{i=1}\limits^{M} \|\nabla f_M(\opt;i) - \nabla f_M(\bdtheta^*;i)\|^{2+\tau}\\

\leq & C_1\lpp\frac{1}{M}\sum\limits_{i=1}\limits^{M} \Big| \lww 1+e^{({\opt}^Tx_i)\bdy_i}\rww^{-1} - \lww 1+e^{({\bdtheta^*}^Tx_i)\bdy_i}\rww^{-1}\Big|^{2+\tau} \|x_i\|^{2+\tau} + (2\lambda)^{2+\tau} \lpp \sum\limits_{j=1}\limits^d \left| \frac{\opt(j)}{(1+(\opt(j))^2)^2} - \frac{\bdtheta^*(j)}{(1+(\bdtheta^*(j))^2)^2} \right|^2\rpp^{\frac{2+\tau}{2}} \rpp\\

\leq& C_1\lpp\frac{1}{M}\sum\limits_{i=1}\limits^{M} \left| C_2 \|\opt-\bdtheta^*\|\right|^{2+\tau} \|x_i\|^{2+\tau} + (2\lambda)^{2+\tau} \lpp\sum\limits_{j=1}\limits^d \left| C_3 |\opt(j)-\bdtheta^*(j)| \right|^2\rpp^{\frac{2+\tau}{2}} \rpp\\

\leq & C_4  \|\opt-\bdtheta^*\|^{2+\tau},
\end{array}
$}
$$
where the 3rd step is based on the smoothness of functions $h_i(t) = \left\{ 1+e^{({t}^Tx_i)\bdy_i}\right\}^{-1},1\leq i \leq M$ and $h(t) = \frac{t}{(1+t^2)^2}$. $C_1,C_2,C_3$ and $C_4$ are some constants possibly dependent on $\tau$, $d$, $M$, $\opt$, $\delta$, $\{x_i\}_{1\leq i\leq M}$ and $\{\bdy_i\}_{1\leq i\leq M}$.

\item \ref{cm_2} We can observe that for each $i=1,2,\ldots,M$, the random gradient estimator $\lvvv \nabla f_M(\bdtheta;i)\rvvv$ is uniformly bounded, which implies that
$$
\max\limits_{i=1,2,\ldots,M} \sup\limits_{\bdtheta \in \RR^d} \lvvv \nabla f_M(\bdtheta;i)\rvvv
$$
is bounded by some constant. Therefore, the first part of condition \ref{cm_2} is verified.

As we have mentioned in subsection \ref{subsec:logistic}, under mild assumptions on the design matrix, if $\lambda$ is reasonably small, all stationary points will be located in a compact region. In this case, based on Lemmas \ref{lemma:run} and \ref{lemma:baddecrease}, we actually do not need a condition as strong a the second part of condition \ref{cm_2}. To ensure the validity of our theoretical conclusions, we just need to guarantee that for any compact set $S_C$, the following quantity is greater than 0:
$$
\inf\limits_{\bdnu: \| \bdnu\|=1} \inf\limits_{\bdtheta\in S_C} \EE \langle \nabla f_M(\bdtheta;i_u) - \nabla F_M(\bdtheta) , \bdnu\rangle^2.
$$
As $S_C$ is compact, we essentially just need to show that for any $\bdtheta$, $\bdnu$,
$$
\EE \langle \nabla f_M(\bdtheta;i_u) - \nabla F_M(\bdtheta) , \bdnu\rangle^2 > 0.
$$
Let us suppose to the contrary that for some $\bdtheta',\bdnu'$,
$$
\EE \langle \nabla f_M(\bdtheta';i_u) - \nabla F_M(\bdtheta') , \bdnu'\rangle^2 = 0.
$$
It implies that for any $i=1,2,\ldots, M$,
$$
(M-1)\Big\langle \lpp 1+  e^{( (\bdtheta')^T \bdx_i) y_i} \rpp^{-1} y_i\bdx_i, \bdnu' \Big\rangle = \sum\limits_{j \ne i} \Big\langle \lpp 1+  e^{( (\bdtheta')^T \bdx_j) y_j} \rpp^{-1} y_j\bdx_j, \bdnu' \Big\rangle.
$$
Consequently, we should have
$$
\Big\langle \lpp 1+  e^{( (\bdtheta')^T \bdx_i) y_i} \rpp^{-1} y_i\bdx_i, \bdnu' \Big\rangle = \Big\langle \lpp 1+  e^{( (\bdtheta')^T \bdx_j) y_j} \rpp^{-1} y_j\bdx_j, \bdnu' \Big\rangle, i\ne j,
$$
which is not likely to happen when $M$ is moderately larger than $d$ since we only have $2d$ degrees of freedom from $\bdtheta'$ and $\bdnu$.

\item \ref{cm_3} As previously mentioned, under mild conditions, $\Theta^{opt}$ is contained in a compact region. When $\Theta^{opt}$ is bounded, to demonstrate condition \ref{cm_3}, essentially we only need to focus on the behavior of $F_M(\bdtheta)$ when $\|\bdtheta\|$ is large. When $\lambda=0$, due to convexity, we can know that there exists some positive constant $\beta$ such that $\|\bdtheta\| = O\lppp F_M(\bdtheta)^{\beta}\rppp$. As the regularization term can only increase the value of objective function, we have similar conclusion valid for $\lambda >0$.

\end{itemize}

\subsubsection{More Analysis on the Landscape}

Though convergence is not the main concern of our work, to guarantee the applicability of our inferential method, we provide some results on the landscape of $F_M(\bdtheta)$. Let's suppose that data $\{(\bdx_i,y_i)\}_{1\leq i\leq M}$ is i.i.d. generated from a probabilistic model where $\{\bdx_i\}_{1\leq i\leq M}$ are from a sub-Gaussian distribution $\mathcal{P}_X$ and $y_i$ is generated conditional on $\bdx_i$ from $\mathcal{P}_{Y|X}$ with $\PP(y_i\pm1|\bdx_i)=\frac{e^{\pm{\opt}^T\bdx_i}}{1+e^{\pm{\opt}^T\bdx_i}}$ for $1\leq i\leq M$. Then, we can show that for a fixed radius $r>0$, when $M$ is sufficiently large and $\lambda$ is reasonably small, the landscape of $F_M(\bdtheta)$ is highly likely to be same as the landscape of $L(\bdtheta)\triangleq \EE_{X\sim \mathcal{P}_X,Y|X\sim \mathcal{P}_{Y|X}}\left[ \log \left(1+e^{-(\bdtheta^T X)Y}\right)\right]$ on $B_d(\textbf{0},r)$. For more results on nonconvex empirical landscape, see \cite{mei2018landscape}.

\begin{proposition}
Suppose that $\{\bdx_i\}_{1\leq i\leq M}$ are i.i.d. realization of random vector $X\in \RR^d$ which is of mean 0, and $\tau^2$-sub-Gaussian. In addition, we assume that $X$ can always span the whole space $\RR^d$ and satisfy that $\EE\left[XX^T\right] \succeq \varpi \tau^2 I_d$ for some positive constant $\varpi$. Conditional on $\bdx_i$, $y_i$ is generated independently for $i=1,2,\ldots,M$ with $\PP(y_i\pm1|\bdx_i)=\frac{e^{\pm{\opt}^T\bdx_i}}{1+e^{\pm{\opt}^T\bdx_i}}$. For a fixed radius $r>0$  such that $\|\opt\|_2\leq \frac{r}{2}$ and any given probability error $\delta>0$, there exists a constant $\bar{C}$ depends on $(\tau^2,\varpi,r,\delta)$ such that when $M\ge \bar{C}d\log M$, the following events occur simultaneously with probability at least $1-\delta$:
\begin{enumerate}
\item $F_M(\bdtheta)$ has a same landscape as $L(\bdtheta)$ within $B_d(\textbf{0},r)$. That is, $F_M(\bdtheta)$ has a unique minimizer within $B_d(\textbf{0},r)$ denoted by $\bdtheta_M^{opt}$.
\item Consider a projected variant of the SGD procedure. If we let
$$
\bdtheta_{n+1}^p =\Pi_{B_d(\textbf{0},r)}\left( \bdtheta^p_n-\gamma_{n+1}\nabla f_M(\bdtheta_n^p;i_{n+1})\right),
$$
where $i_{n+1}$ is uniformly sampled from $\{1,2,\ldots,M\}$ and $\Pi_{B_d(\textbf{0},r)}$ is the $L_2$-projection to $B_d(\textbf{0},r)$. Then, conditional on $\{(\bdx_i,y_i)\}_{1\leq i\leq M}$, $\bdtheta^p_k \rightarrow \bdtheta_M^{opt}$ almost-surely given that $\sum\limits_{n=1}\limits^{\infty}\gamma_n = \infty$ and $\sum\limits_{n=1}\limits^{\infty}\gamma_n^2 < \infty$.
\end{enumerate}
\label{prop:example1landscape}
\end{proposition}

\begin{proof}
The landscape part's proof leverages similar techniques introduced in \cite{mei2018landscape} and the algorithm part's proof is based on a classical result provided in \cite{robbins1971convergence}. \\

\noindent{\bf $\bullet$ Landscape of $L(\bdtheta)$}

Recall that the population level loss function is $L(\bdtheta)=\EE\left[ \log \left(1+e^{-(\bdtheta^TX)Y}\right)\right]$. At the first step, we characterize the landscape of $L(\bdtheta)$ within $B_d(\textbf{0},r)$. The gradient of $L$ is
\begin{equation}
\nabla L(\bdtheta) = \EE\left[\left( 1+e^{(\bdtheta^TX)Y}\right)^{-1}(-YX)\right],
\label{lseq1}
\end{equation}
and the Hessian is
\begin{equation}
\nabla^2 L(\bdtheta) = \EE\left[ \frac{e^{(\bdtheta^T X)Y}}{1+e^{(\bdtheta^T X)Y}} XX^T\right] \succ 0.
\label{lseq2}
\end{equation}
We can verify that
$$
\arraycolsep=1.1pt\def\arraystretch{1.5}
\begin{array}{rcl}
\nabla L(\opt)&  = &\EE\left[ \EE\left[\left( 1+e^{({\opt}^TX)Y}\right)^{-1}(-YX)\right] |X\right]\\

& = & \EE\left[\left( 1+e^{({\opt}^TX)}\right)^{-1} \frac{e^{({\opt}^TX)}}{1+e^{({\opt}^TX)}}(-X)  + \left( 1+e^{({\opt}^TX)}\right)^{-1} \frac{1}{1+e^{({\opt}^TX)}} X \right]\\

& = & 0.
\end{array}
$$
Therefore, $\opt$ is the unique stationary point (minimizer) of $L$ within $B_d(\textbf{0},r)$.

Next, we show that there exist constant $\varepsilon_0\leq \frac{r}{3},\underline{L_0}>0,\overline{L_0}>0,\underline{\kappa_0}>0,\overline{\kappa_0}>0$ dependent on $(\tau^2,\varpi,r)$ such that 

\begin{subequations}
\begin{eqnarray}
\label{La}
\|\nabla L(\bdtheta)\|_2 \ge \underline{L_0},  \forall \bdtheta \in B_d(\textbf{0},r)\setminus B_d(\opt,\varepsilon_0), \\ \label{Lb}
\|\nabla L(\bdtheta)\|_2 \leq \overline{L_0} , \forall \bdtheta \in B_d(\textbf{0},r)\\\label{Lc}
\lambda_{min}(\nabla^2 L(\bdtheta)) \ge \underline{\kappa_0},  \forall \bdtheta \in  B_d(\opt,\varepsilon_0),\\\label{Ld}
\|\nabla^2 L(\bdtheta)\|_{op} \leq \overline{\kappa_0},  \forall \bdtheta \in B_d(\textbf{0},r).
\end{eqnarray}
\end{subequations}

For any $\bdtheta\in \RR^d$, we have
$$
\|\nabla^2 L(\bdtheta)\|_{op} = \sup\limits_{u\in \RR^d,\|u\|_2=1} \EE\left[ \frac{e^{(\bdtheta^T X)Y}}{1+e^{(\bdtheta^T X)Y}} \langle u,X\rangle^2\right] \leq \sup\limits_{u\in \RR^d,\|u\|_2=1} \EE\left[\langle u,X\rangle^2\right] \leq 4\tau^2.
$$
Therefore, we can see that (\ref{Ld}) holds with $\overline{\kappa_0}=4\tau^2$.

Let's define $A_s = \left\{ |\langle \opt,X\rangle|\leq s\right\}$ for $s>0$. We have
\begin{equation}
\resizebox{.92\hsize}{!}{$
\PP(A_s^c) = \PP(  |\langle \opt,X\rangle| \ge s) = \PP(  |\langle \opt/\|\opt\|_2,X\rangle| \ge s/\|\opt\|_2) \leq \PP(  |\langle \opt/\|\opt\|_2,X\rangle| \ge 2s/r) \leq 2e^{-\frac{2s^2}{r^2\tau^2}}.
\label{lseq4}
$}
\end{equation}
For any $u\in \RR^d$ with $\|u\|_2=1$, we have
$$
\arraycolsep=1.1pt\def\arraystretch{1.5}
\begin{array}{rl}
& u^T \nabla^2 L(\opt) u = \EE\left[ \frac{e^{({\opt}^T X)Y}}{1+e^{({\opt}^T X)Y}} \langle u,X\rangle^2\right] \ge \EE\left[ \frac{e^{({\opt}^T X)Y}}{1+e^{({\opt}^T X)Y}} \langle u,X\rangle^2\daone_{A_s}\right]\\

\ge &   \varphi(s) \EE\left[ \langle u,X\rangle^2\daone_{A_s^c}\right] = \varphi(s) \left( \EE\left[ \langle u,X\rangle^2\right] - \EE\left[ \langle u,X\rangle^2\daone_{A_s^c}\right]\right)\\

\ge & \varphi(s) \left( \varpi \tau^2 - \left( \EE\left[ \langle u,X\rangle^4\right] \PP(A_s^c)\right)^{\frac{1}{2}}\right)\\

\ge & \varphi(s) \left( \varpi \tau^2  - \sqrt{2C_4} \tau^2 e^{-\frac{s^2}{r^2\tau^2}}\right),\\
\end{array}
$$
where $\varphi(s)=(1+e^s)^{-1}$, $C_4$ is a global constant, the 5th step is based on the Cauchy-Schwarz inequality and the last step is based on (\ref{lseq4}). If we pick a sufficiently large constant $C_s$ depending on $\varpi$ such that $\sqrt{2C_4} e^{-C_s^2}\leq \frac{\varpi}{2}$ and let $s = C_sr\tau$, we have
\begin{equation}
u^T \nabla^2 L(\opt) u \ge \frac{1}{2}\varphi (C_sr\tau)\varpi \tau^2.
\label{lseq5}
\end{equation}
Since $u$ is arbitrary, we have $\lambda_{min}(\nabla^2 L(\opt))\ge  \frac{1}{2}\varphi (C_sr\tau)\varpi \tau^2$. For $\bdtheta \in B_d(\opt,\varepsilon_0)$, we have
$$
\arraycolsep=1.1pt\def\arraystretch{1.5}
\begin{array}{rl}
& \left| u^T (\nabla^2 L(\bdtheta)-\nabla^2 L(\opt)) u\right| = \left| \EE\left[ \left( \frac{e^{(\bdtheta^T X)Y}}{1+e^{(\bdtheta^T X)Y}} - \frac{e^{({\opt}^T X)Y}}{1+e^{({\opt}^T X)Y}}\right)\langle u,X\rangle^2\right]\right|\\

\leq & \EE\left[ \left| \frac{e^{(\bdtheta^T X)Y}}{1+e^{(\bdtheta^T X)Y}} - \frac{e^{({\opt}^T X)Y}}{1+e^{({\opt}^T X)Y}} \right| \langle u,X\rangle^2\right] \leq \frac{1}{2} \EE\left[ |(\bdtheta^T X)Y - ({\opt}^T X)Y| \langle u,X\rangle^2\right]\\

= & \frac{1}{2} \EE \left| \langle \bdtheta-\opt,X\rangle \langle u,X\rangle^2 \right| \leq \frac{1}{2}\sqrt{C_4}\tau^2 \left( \EE\left[ \langle \bdtheta-\opt,X\rangle^2\right]\right)^{\frac{1}{2}}\\

\leq & \sqrt{C_4}\tau^3 \|\bdtheta-\opt\|_2 \leq \sqrt{C_4}\tau^3 \varepsilon_0,
\end{array}
$$
where the 5th step is based on the Cauchy-Schwarz inequality. We let
\begin{equation}
\varepsilon_0 = \frac{\varphi(C_s r\tau)\varpi}{4\sqrt{C_4}\tau} \wedge \frac{r}{3}.
\label{eps}
\end{equation}
Then, for any $\bdtheta \in B_d(\opt,\varepsilon_0)$,
$$
\lambda_{min}(\nabla^2 L(\bdtheta)) \ge \lambda_{min}(\nabla^2 L(\opt)) - \| \nabla^2 L(\bdtheta)-\nabla^2 L(\opt) \|_{op} \ge \frac{1}{4} \varphi(C_sr\tau)\varpi \tau^2,
$$
which indicates the validity of (\ref{Lc}).

To upper bound the gradient, for any $\bdtheta\in B_d(\textbf{0},r)$, we have
$$
\|\nabla L(\bdtheta)\|_2 = \sup\limits_{\|u\|_2=1,u\in \RR^d} \left| \EE\left[ \left(1+e^{(\bdtheta^TX)Y}\right)^{-1}(-Y)\langle u,X\rangle\right]\right|\leq  \sup\limits_{\|u\|_2=1,u\in \RR^d} \EE | \langle u,X\rangle| \leq C_1\tau,
$$
where $C_1$ is a universal constant. Therefore, (\ref{Lb}) is valid with $\overline{L_0} = C_1\tau$.

Next, for any $\bdtheta\in B_d(\textbf{0},r)$, we have
\begin{equation}
\arraycolsep=1.1pt\def\arraystretch{1.5}
\begin{array}{rl}
& \langle \bdtheta-\opt,\nabla L(\bdtheta)\rangle\\

 =& \EE\left[ \left(1+e^{(\bdtheta^TX)Y}\right)^{-1}(-Y)\langle \bdtheta-\opt,X\rangle\right]\\

= & \EE\left[ \frac{-\langle \bdtheta-\opt,X\rangle}{1+e^{\bdtheta^TX}} \frac{e^{{\opt}^TX}}{1+e^{{\opt}^TX}} + \frac{\langle \bdtheta-\opt,X\rangle}{1+e^{-\bdtheta^TX}} \frac{1}{1+e^{{\opt}^TX}}\right]\\

= & \EE\left[ \frac{1}{1+e^{\bdtheta^TX}} \frac{1}{1+e^{{\opt}^TX}} (e^{\bdtheta^TX}-e^{{\opt}^TX}) \langle \bdtheta-\opt,X\rangle\right]\\
 \ge& 0.
\end{array}
\label{lseq6}
\end{equation}
If we define $A'_s = \left\{\max\{|\langle \bdtheta,X\rangle|, |\langle \opt,X\rangle|, |\langle \bdtheta-\opt,X\rangle|\} \leq s\right\}$, based on (\ref{lseq6}), we have
\begin{equation}
\arraycolsep=1.1pt\def\arraystretch{1.5}
\begin{array}{rl}
& \langle \bdtheta-\opt,\nabla L(\bdtheta)\rangle\\

\ge & \varphi^2(s) \EE\left[ (e^{\bdtheta^TX}-e^{{\opt}^TX}) \langle \bdtheta-\opt,X\rangle \daone_{A'_s}\right] \ge \varphi^2(s) e^{-s} \EE\left[ \langle \bdtheta-\opt,X\rangle^2 \daone_{A'_s}\right]\\

= & \varphi^2(s)e^{-s} \left( \EE\left[ \langle \bdtheta-\opt,X\rangle^2 \right] - \EE\left[ \langle \bdtheta-\opt,X\rangle^2 \daone_{(A'_s)^c}\right] \right)\\

\ge &  \varphi^2(s)e^{-s} \left( \varpi \tau^2 \|\bdtheta-\opt\|_2^2 - \left( \EE\left[ \langle \bdtheta-\opt,X\rangle^4 \right] \PP((A'_s)^c)\right)^{\frac{1}{2}}\right)\\

\ge &  \varphi^2(s)e^{-s} \left( \varpi \tau^2 \|\bdtheta-\opt\|_2^2 - \sqrt{C_4} \tau^2 \|\bdtheta-\opt\|_2^2 \PP^{\frac{1}{2}}((A'_s)^c)\right).\\
\end{array}
\label{lseq7}
\end{equation}
With the sub-Gaussian assumption, we can bound $\PP^{\frac{1}{2}}((A'_s)^c)$ as
\begin{equation}
\arraycolsep=1.1pt\def\arraystretch{1.5}
\begin{array}{rcl}
\PP^{\frac{1}{2}}((A'_s)^c) & \leq & \PP(|\langle \bdtheta,X\rangle|\ge s) + \PP(|\langle \opt,X\rangle|\ge s) + \PP(|\langle \bdtheta-\opt,X\rangle|\ge s)\\

& \leq & \PP(|\langle \frac{\bdtheta}{\|\bdtheta\|_2},X\rangle|\ge \frac{s}{r}) + \PP(|\langle \frac{\opt}{\|\opt\|_2},X\rangle|\ge \frac{s}{r}) + \PP(|\langle \frac{\bdtheta-\opt}{\|\bdtheta-\opt\|_2},X\rangle|\ge \frac{s}{2r})\\

& \leq & 3\sup\limits_{\|u\|_2=1,u\in \RR^d} \PP(|\langle u,X\rangle |\ge \frac{s}{2r}) \\

& \leq & 6e^{-\frac{s^2}{8\tau^2r^2}}.\\
\end{array}
\label{lseq8}
\end{equation}
Putting (\ref{lseq8}) back to (\ref{lseq7}), we have
$$
\langle \bdtheta-\opt,\nabla L(\bdtheta)\rangle \ge \varphi^2(s)e^{-s} \tau^2\|\bdtheta-\opt\|_2^2 \left(\varpi-\sqrt{6C_4}e^{-\frac{s^2}{16\tau^2r^2}}\right).
$$
We let $C'_s$ be a constant only depending on $\varpi$ such that $\sqrt{6C_4} e^{-\frac{(C'_s)^2}{16}}\leq \frac{\varpi}{2}$. Then, we have
\begin{equation}
\langle \bdtheta-\opt,\nabla L(\bdtheta)\rangle \ge \frac{1}{2} \varphi^2(C'_sr\tau)e^{-C'_sr\tau} \varpi \tau^2\|\bdtheta-\opt\|_2^2. 
\label{lseq81}
\end{equation}
Therefore, if $\bdtheta\in B_d(\textbf{0},r)\setminus B_d(\opt,\varepsilon_0)$, we have
$$
\|\nabla L(\bdtheta)\|_2 \ge \frac{1}{2} \varphi^2(C'_sr\tau)e^{-C'_sr\tau} \varpi \tau^2\|\bdtheta-\opt\|_2 \ge \frac{1}{2} \varphi^2(C'_sr\tau)e^{-C'_sr\tau} \varpi \tau^2\varepsilon_0,
$$
which implies that (\ref{La}) is valid with $\underline{L_0}= \frac{1}{2} \varphi^2(C'_sr\tau)e^{-C'_sr\tau} \varpi \tau^2\varepsilon_0$, where $\varepsilon_0$ is given in (\ref{eps}).\\

\noindent {\bf $\bullet$ Verification of Assumptions in Lemma \ref{mei1}}

First of all, since $X$ is $\tau^2$-sub-Gaussian, it is easy to know that there exists universal constants $C_{A_1}$ and $C_{A_2}$ such that the gradient of the loss $\nabla l(\bdtheta;X,Y)=\left(1+e^{(\bdtheta^TX)Y}\right)^{-1}(-YX)$ is $(C_{A_1}\tau^2)$-sub-Gaussian and the Hessian of the loss $\nabla^2 l(\bdtheta;X,Y)= \frac{e^{(\bdtheta^TX)Y}}{1+e^{(\bdtheta^TX)Y}}XX^T$ is $(C_{A_2}\tau^2)$-sub-exponential. 

To verify the third assumption in Lemma \ref{mei1}, we can let $H=4\tau^2$ as we have already showed the validity of (\ref{Ld}) with $\overline{\kappa_0}=4\tau^2$. Further, we have
$$
\arraycolsep=1.1pt\def\arraystretch{1.5}
\begin{array}{rcl}
\EE\left[ J(Z)\right] & = & \EE \left[ \sup\limits_{\bdtheta_1\ne \bdtheta_2\in B_d(\textbf{0},r)}\frac{\|\nabla^2 l(\bdtheta_1;X,Y)-\nabla^2 l(\bdtheta_2;X,Y)\|_{op}}{\|\bdtheta_1-\bdtheta_2\|_2}\right]\\

& = & \EE \left[\sup\limits_{\bdtheta_1\ne \bdtheta_2\in B_d(\textbf{0},r)}\frac{\left| \frac{e^{(\bdtheta_1^T X)Y}}{1+e^{(\bdtheta_1^T X)Y}} - \frac{e^{(\bdtheta_2^T X)Y}}{1+e^{(\bdtheta_2^T X)Y}}\right|}{\|\bdtheta_1-\bdtheta_2\|_2} \|XX^T\|_{op} \right]\\

& \leq & \frac{1}{2} \EE \left[\sup\limits_{\bdtheta_1\ne \bdtheta_2\in B_d(\textbf{0},r)} \frac{|(\bdtheta_1^TX)Y - (\bdtheta_2^TX)Y|}{\|\bdtheta_1-\bdtheta_2\|_2} \|XX^T\|_{op} \right]\\

& \leq & \frac{1}{2} \EE \left[ \|X\|_2 \|XX^T\|_{op} \right]\\

& \leq & \frac{1}{2} \EE \left[ \|X\|_2^3\right]\\

& \leq & \frac{C_3}{2}d^{3/2}\tau^3,
\end{array}
$$
where $C_3$ is a universal constant. Therefore, we can let $J_* = \frac{C_3}{2}d^{3/2}\tau^3$. Further, when $d\ge 2$, we can let $c_h = \left(\frac{3}{2} + \log_2 \frac{C_3}{2}\right)\vee 2$ so that $H\leq \tau^2 d^{c_h}$ and $J_* \leq \tau^3 d^{c_h}$.\\

\noindent \textbf{$\bullet$ Landscape of $F_M(\bdtheta)$}

We have
$$
\arraycolsep=1.1pt\def\arraystretch{1.5}
\begin{array}{rcl}
\sup\limits_{\bdtheta\in B_d(\textbf{0},r)} \|\nabla F_M(\bdtheta) - \nabla L(\bdtheta)\|_2 & \leq &\sup\limits_{\bdtheta\in B_d(\textbf{0},r)} \|\nabla \hat{L}_M(\bdtheta) - \nabla L(\bdtheta)\|_2 + \sup\limits_{\bdtheta\in B_d(\textbf{0},r)} \|\nabla R(\bdtheta)\|_2\\

& \leq & \sup\limits_{\bdtheta\in B_d(\textbf{0},r)} \|\nabla \hat{L}_M(\bdtheta) - \nabla L(\bdtheta)\|_2 + C_{\bdr_1}\lambda\sqrt{d},\\
\end{array}
$$
and
$$
\arraycolsep=1.1pt\def\arraystretch{1.5}
\begin{array}{rcl}
\sup\limits_{\bdtheta\in B_d(\textbf{0},r)} \|\nabla^2 F_M(\bdtheta) - \nabla^2 L(\bdtheta)\|_{op}& \leq &\sup\limits_{\bdtheta\in B_d(\textbf{0},r)} \|\nabla^2 \hat{L}_M(\bdtheta) - \nabla^2 L(\bdtheta)\|_{op} + \sup\limits_{\bdtheta\in B_d(\textbf{0},r)} \|\nabla^2 R(\bdtheta)\|_{op}\\

& \leq & \sup\limits_{\bdtheta\in B_d(\textbf{0},r)} \|\nabla^2 \hat{L}_M(\bdtheta) - \nabla^2 L(\bdtheta)\|_{op} + C_{r_2}\lambda,\\
\end{array}
$$
where $C_{\bdr_1}$ and $C_{r_2}$ are universal constants. We define events
$$
\Omega_1 \triangleq \left\{\sup\limits_{\bdtheta\in B_d(\textbf{0},r)} \|\nabla \hat{L}_M(\bdtheta) - \nabla L(\bdtheta)\|_2 \leq \frac{\underline{L_0}}{4},  \sup\limits_{\bdtheta\in B_d(\textbf{0},r)} \|\nabla^2 \hat{L}_M(\bdtheta) - \nabla^2 L(\bdtheta)\|_{op} \leq \frac{\underline{\kappa_0}}{4}\right\},
$$
and
$$
\Omega_2 \triangleq \left\{\sup\limits_{\bdtheta\in B_d(\textbf{0},r)} \|\nabla F_M(\bdtheta) - \nabla L(\bdtheta)\|_2 \leq \frac{\underline{L_0}}{2},  \sup\limits_{\bdtheta\in B_d(\textbf{0},r)} \|\nabla^2 F_M(\bdtheta) - \nabla^2 L(\bdtheta)\|_{op} \leq \frac{\underline{\kappa_0}}{2}\right\}.
$$
Since we have verified assumptions required by Lemma \ref{mei1}, by applying this lemma, we know that there exists a constant $\tilde{C}_1$ depending on $(\tau^2,\varpi,r,\delta)$ such that when $M\ge \tilde{C}_1 d\log M$, $\Omega_1$ occurs with probability at least $1-\frac{\delta}{2}$. Further, if $\lambda$ is less than $\frac{\underline{L_0}}{4C_{\bdr_1}\sqrt{d}} \wedge \frac{\underline{\kappa_0}}{4C_{r_2}}$, we can guarantee that $\Omega_2$ occurs with probability at least $1-\frac{\delta}{2}$. On $\Omega_2$, we can apply Lemma 5 in \cite{mei2018landscape} to show that $F_M(\bdtheta)$ has a single stationary point $\bdtheta_M^*$ in $B_d(\textbf{0},r)$, which is further also a minimizer and $\bdtheta_M^* \in B_d(\opt,\varepsilon_0)$.\\

\noindent {\textbf{$\bullet$ Lower Bounding Directional Gradient}}

For any $\bdtheta\in B_d(\textbf{0},r)$, we have
\begin{equation}
\arraycolsep=1.1pt\def\arraystretch{1.5}
\begin{array}{rl}
& \langle \bdtheta-\bdtheta_M^*,\nabla F_M(\bdtheta)\rangle \\

 = & \langle \bdtheta-\bdtheta_M^*, \nabla L(\bdtheta)\rangle + \langle \bdtheta-\bdtheta_M^*, \nabla F_{M}(\bdtheta) -  \nabla L(\bdtheta)\rangle \\

 \ge & \langle \bdtheta-\opt,\nabla L(\bdtheta)\rangle - \|\opt - \bdtheta_M^*\|_2 \|\nabla L(\bdtheta)\|_2 - \|\bdtheta - \bdtheta_M^*\|_2 \|\nabla F_M(\bdtheta) - \nabla L(\bdtheta)\|_2.\\
 
  \ge & \langle \bdtheta-\opt,\nabla L(\bdtheta)\rangle - \overline{L_0}\|\opt - \bdtheta_M^*\|_2  -2r \|\nabla F_M(\bdtheta) - \nabla L(\bdtheta)\|_2.\\
\end{array}
\label{pjsgd1}
\end{equation}
Based on (\ref{lseq81}) showed in the first part, we have
\begin{equation}
\arraycolsep=1.1pt\def\arraystretch{1.5}
\begin{array}{rcl}
\langle \bdtheta-\opt,\nabla L(\bdtheta)\rangle &  \ge& \frac{1}{2} \varphi^2(C'_sr\tau)e^{-C'_sr\tau} \varpi \tau^2\|\bdtheta-\opt\|_2^2\\

& \triangleq & T_1 \|\bdtheta-\opt\|_2^2 = T_1 \|\bdtheta - \bdtheta_M^* + \bdtheta_M^* - \opt\|_2^2\\

& \ge &\frac{T_1}{2} \| \bdtheta - \bdtheta_M^*\|_2^2 - T_1\|\bdtheta_M^* - \opt\|_2^2.\\
\end{array}
\label{pjsgd2}
\end{equation}

Putting (\ref{pjsgd2}) back to (\ref{pjsgd1}), we have
\begin{equation}
\langle \bdtheta-\bdtheta_M^*,\nabla F_M(\bdtheta)\rangle \ge \frac{T_1}{2}\|\bdtheta - \bdtheta_M^*\|_2^2 - \overline{L_0}\|\opt - \bdtheta_M^*\|_2 - T_1\|\opt - \bdtheta_M^*\|_2^2 -2r\|\|\nabla F_M(\bdtheta) - \nabla L(\bdtheta)\|_2.
\label{pjsgd3}
\end{equation}

We denote $T_2 \triangleq \frac{T_1\varepsilon_0^2}{12\overline{L_0}+2T_1\varepsilon}$ and let
$$
\resizebox{.98\hsize}{!}{$
\Omega_3 \triangleq \left\{\sup\limits_{\bdtheta\in B_d(\textbf{0},r)} \|\nabla \hat{L}_M(\bdtheta) - \nabla L(\bdtheta)\|_2 \leq \frac{\underline{\kappa_0}T_2}{2}\wedge \frac{\underline{\kappa_0}\varepsilon_0}{96} \wedge \frac{\underline{\kappa_0}T_1\varepsilon_0}{384} , \sup\limits_{\bdtheta\in B_d(\textbf{0},r)} \|\nabla^2 \hat{L}_M(\bdtheta) - \nabla^2 L(\bdtheta)\|_{op}\leq \frac{\underline{\kappa_0}}{4},\|\bdtheta_M^* - \opt\|_2 \leq \varepsilon_0\right\}.
$}
$$

Based on Lemma \ref{mei1}, there exists a constant $\tilde{C}_2$ depending on $(\tau^2,\varpi,r,\delta)$ such that when $M\ge \tilde{C}_2d\log M$, $\Omega_3$ occurs with probability at least $1-\frac{\delta}{2}$.

Now, if we let 
\begin{equation}
\lambda\leq \frac{\underline{\kappa_0}T_2}{2C_{\bdr_1}\sqrt{d}} \wedge \frac{\underline{\kappa_0}\varepsilon_0}{96C_{\bdr_1}\sqrt{d}} \wedge \frac{\underline{\kappa_0} T_1 \varepsilon_0}{384 C_{\bdr_1}\sqrt{d}} \wedge \frac{\underline{\kappa_0}}{4C_{r_2}},
\label{lambdadef}
\end{equation}
we have
\begin{equation}
\| \nabla R(\opt)\|_2\leq C_{\bdr_1} \lambda \sqrt{d} \leq   \frac{\underline{\kappa_0}T_2}{2} \wedge \frac{\underline{\kappa_0}\varepsilon_0}{96} \wedge \frac{\underline{\kappa_0} T_1 \varepsilon_0}{384 },
\label{pjsgd4}
\end{equation}
and
\begin{equation}
\sup\limits_{\bdtheta\in B_d(\textbf{0},r)} \|\nabla^2 R(\bdtheta)\|_{op} \leq C_{r_2}\lambda \leq \frac{\underline{\kappa_0}}{4}.
\label{pjsgd5}
\end{equation}

Since $\bdtheta_M^*$ is the minimizer of $F_M$, applying first-order Taylor expansion, for some $\bdtheta^*$ between $\bdtheta_M^*$ and $\opt$, we have
$$
F_M(\opt) \ge F_M(\opt) = F_M(\opt) +\langle \nabla F_M(\opt),\bdtheta_M^* - \opt\rangle + \frac{1}{2}(\bdtheta_M^* - \opt)^T \nabla^2 F_M(\bdtheta^*)(\bdtheta_M^* - \opt),
$$
which implies
$$
\arraycolsep=1.1pt\def\arraystretch{1.5}
\begin{array}{rcl}
0&\ge& \langle \nabla F_M(\opt),\bdtheta_M^* - \opt\rangle + \frac{1}{2}(\bdtheta_M^* - \opt)^T \nabla^2 F_M(\bdtheta^*)(\bdtheta_M^* - \opt)\\

& \ge& \langle \nabla F_M(\opt),\bdtheta_M^* - \opt\rangle + \frac{\underline{\kappa_0}}{2}\|\bdtheta_M^* - \opt\|_2^2
\end{array}
$$
on $\Omega_3$. Therefore, on $\Omega_3$, with the Cauchy-Schwarz inequality, we have
$$
\frac{\underline{\kappa_0}}{2} \|\bdtheta_M^* - \opt\|_2^2 \leq \|\nabla F_M(\opt)\|_2 \|\bdtheta_M^* - \opt\|_2,
$$ 
which implies
$$
\|\bdtheta_M^* - \opt\|_2 \leq \frac{2\|\nabla F_M(\opt)\|_2}{\underline{\kappa_0}} \leq \frac{2}{\underline{\kappa_0}} (\|\nabla \hat{L}_M(\opt) - \nabla L(\opt)\|_2 + \|\nabla R(\opt)\|_2) \leq T_2
$$
based on the definition of $\Omega_3$ and (\ref{pjsgd4}). Then, with the definition of $T_2$, we have
\begin{equation}
\arraycolsep=1.1pt\def\arraystretch{1.5}
\begin{array}{rl}
& \|\bdtheta_M^* - \opt\|_2 \leq T_2\\

 \Rightarrow& \overline{L_0} \|\bdtheta_M^* - \opt\|_2 \leq \frac{T_1}{12} (\varepsilon_0^2 - 2\varepsilon_0 \|\bdtheta_M^* - \opt\|_2)\\
 
  \Rightarrow& \overline{L_0} \|\bdtheta_M^* - \opt\|_2 \leq \frac{T_1}{12} (\varepsilon_0 - \|\bdtheta_M^* - \opt\|_2)^2\\
  
\Rightarrow& \overline{L_0} \|\bdtheta_M^* - \opt\|_2 \leq \frac{T_1}{12} \inf\limits_{\bdtheta \in B_d(\textbf{0},r)  \setminus B_d(\opt,\varepsilon_0)} \|\bdtheta - \bdtheta_M^*\|_2^2.\\
\end{array}
\label{pjsgd6}
\end{equation}
Likewise, on $\Omega_3$, we can show
$$
\|\bdtheta_M^* - \opt\|_2 \leq \frac{2}{\underline{\kappa_0}} \frac{\underline{\kappa_0}\varepsilon_0}{48} = \frac{\varepsilon_0}{24},
$$
which implies
\begin{equation}
\arraycolsep=1.1pt\def\arraystretch{1.5}
\begin{array}{rl}
& T_1\varepsilon_0 \|\bdtheta_M^* - \opt\|_2 + \frac{T_1}{6}\varepsilon_0 \|\bdtheta_M^* - \opt\|_2 \leq \frac{T_1}{12}\varepsilon^2\\

\Rightarrow& T_1 \|\bdtheta_M^* - \opt\|_2^2 + \frac{T_1}{6}\varepsilon_0 \|\bdtheta_M^* - \opt\|_2 \leq \frac{T_1}{12}\varepsilon^2\\

\Rightarrow& T_1 \|\bdtheta_M^* - \opt\|_2^2 \leq \frac{T_1}{12} (\varepsilon_0 - \|\bdtheta_M^* - \opt\|_2)^2\\

\Rightarrow& T_1 \|\bdtheta_M^* - \opt\|_2^2 \leq \frac{T_1}{12} \inf\limits_{\bdtheta \in B_d(\textbf{0},r)  \setminus B_d(\opt,\varepsilon_0)} \|\bdtheta - \bdtheta_M^*\|_2^2.\\
\end{array}
\label{pjsgd7}
\end{equation}
Similarly, we can also show that on $\Omega_3$,
\begin{equation}
2r\|\nabla F_M(\bdtheta) - \nabla L(\bdtheta)\|_2 \leq \frac{T_1}{12} \inf\limits_{\bdtheta \in B_d(\textbf{0},r)  \setminus B_d(\opt,\varepsilon_0)} \|\bdtheta - \bdtheta_M^*\|_2^2.
\label{pjsgd8}
\end{equation}
Based on (\ref{pjsgd6}), (\ref{pjsgd7}), (\ref{pjsgd8}) and (\ref{pjsgd3}), we know that on $\Omega_3$, for any $\bdtheta \in B_d(\textbf{0},r)\setminus B_d(\opt,\varepsilon)$,
\begin{equation}
\langle \bdtheta - \opt,\nabla F_M(\bdtheta)\rangle \ge \frac{T_1}{4} \| \bdtheta - \opt\|_2^2.
\label{pjsgd9}
\end{equation}
With the aforementioned requirement on $\lambda$, on $\Omega_3$, $F_M(\bdtheta)$ is $\frac{\underline{\kappa_0}}{2}$-strongly convex on $B_d(\opt,\varepsilon_0)$ and therefore for any $\bdtheta \in B_d(\opt,\varepsilon_0)$,
\begin{equation}
\langle \bdtheta - \opt,\nabla F_M(\bdtheta)\rangle \ge \frac{\underline{\kappa_0}}{2} \| \bdtheta - \opt\|_2^2.
\label{pjsgd10}
\end{equation}
Therefore, if we define
\begin{equation}
\Omega_4 \triangleq \left\{ \langle \bdtheta-\bdtheta_M^*,\nabla F_M(\bdtheta)\rangle \ge T_3 \|\bdtheta - \bdtheta_M^*\|_2^2, \forall \bdtheta \in B_d(\textbf{0},r)\right\}
\label{pjsgd0}
\end{equation}
with $T_3 = \frac{T_1}{4} \wedge \frac{\underline{\kappa_0}}{2}$, we have $\Omega_3\subseteq \Omega_4$ and hence $\Omega_4$ occurs with probability at least $1-\frac{\delta}{2}$ when $M\ge \tilde{C}_2 d\log M$ and $\lambda$ satisfies (\ref{lambdadef}).\\

\noindent {\textbf{$\bullet$ Convergence of the Projected SGD Procedure}}

The following results hold conditional on $\{(x_i(\omega),\bdy_i(\omega))\}_{1\leq i\leq M}$ such that $\omega\in\Omega_4$. Based on the definition of the projected SGD procedure, we have
$$
\arraycolsep=1.1pt\def\arraystretch{1.5}
\begin{array}{rcl}
\|\bdtheta_{n+1}^p - \bdtheta_M^*\|_2^2 & \leq & \| \bdtheta_n^p - \gamma_{n+1}\nabla f_M(\bdtheta_n^p;i_{n+1}) - \bdtheta_M^*\|_2^2\\

& = & \|\bdtheta_n^p - \bdtheta_M^*\|_2^2 + \gamma_{n+1}^2 \| \nabla f_M(\bdtheta_n^p;i_{n+1})\|_2^2 - 2\gamma_{n+1} \langle \bdtheta_n^p - \bdtheta_M^*,\nabla f_M(\bdtheta_n^p;i_{n+1})\rangle.\\
\end{array}
$$
Taking conditional expectation on both sides, we have
$$
\resizebox{.98\hsize}{!}{$
\arraycolsep=1.1pt\def\arraystretch{1.5}
\begin{array}{rcl}
\EE_n\left[ \|\bdtheta_{n+1}^p - \bdtheta_M^*\|_2^2 \right] & \leq & \|\bdtheta_n^p - \bdtheta_M^*\|_2^2 + \gamma_{n+1}^2 \EE_n\left[ \| \nabla f_M(\bdtheta_n^p;i_{n+1})\|_2^2\right] - 2\gamma_{n+1} \langle \bdtheta_n^p - \bdtheta_M^*, \nabla F_M(\bdtheta_n^p)\rangle\\

& \leq &  \|\bdtheta_n^p - \bdtheta_M^*\|_2^2 + \gamma_{n+1}^2 \EE_n\left[ \| \nabla f_M(\bdtheta_n^p;i_{n+1})\|_2^2\right] - 2T_3 \gamma_{n+1} \| \bdtheta_n^p - \bdtheta_M^*\|_2^2\\

& = & (1-2T_3\gamma_{n+1}) \| \bdtheta_n^p - \bdtheta_M^*\|_2^2 + \gamma_{n+1}^2 \EE_n\left[ \| \nabla f_M(\bdtheta_n^p;i_{n+1})\|_2^2\right],\\
\end{array}
$}
$$
where $ \EE_n\left[ \| \nabla f_M(\bdtheta_n^p;i_{n+1})\|_2^2\right] = \frac{1}{M} \sum\limits_{i=1}\limits^M \|\nabla f_M(\bdtheta_n^p;i)\|_2^2$ is bounded by a constant as $\bdtheta_n^p$ is restricted in a compact set. Directly applying Lemma \ref{rs71} originated from \cite{robbins1971convergence} can lead to our desired conclusion that $\bdtheta_n^p\rightarrow \bdtheta_M^*$ almost-surely.

\end{proof}

\clearpage

\newpage
\nocite{*}
\bibliographystyle{plainnat}
\bibliography{nonconvex_sgd}

\end{document}